\documentclass[3p,times]{elsarticle}

\usepackage{ecrc}
\graphicspath{{./figures/} }
\DeclareGraphicsExtensions{.pdf,.jpeg,.png}

\newtheorem{myTheo}{Theorem}
\newtheorem{myLemma}{Lemma}
\newtheorem{myDef}{Definition}

\newtheorem{myCor}{Corollary}

\newtheorem{proof}{Proof}

\newtheorem{lem}{Lemma}

\volume{00}

\firstpage{1}

\journalname{Procedia Computer Science}

\runauth{}

\jid{procs}

\jnltitlelogo{Procedia Computer Science}

\CopyrightLine{2011}{Published by Elsevier Ltd.}

\usepackage{amssymb}
\usepackage{amsfonts}
\usepackage{amsmath}
\usepackage{subfig}
\usepackage{url}
\usepackage{color}
\usepackage{newtxtext}
\usepackage{newtxmath}
\usepackage{setspace}

\usepackage[figuresright]{rotating}

\begin{document}

\begin{frontmatter}



\dochead{}

\title{Robust Manifold Nonnegative Tucker Factorization for Tensor Data Representation}


\author[1]{Jianyu Wang}
\author[1]{Linruize Tang}
\author[1]{Jie Chen}
\author[1]{Jingdong Chen}

\address[1]{Center for Intelligent Acoustics and Immersive Communications, School of Marine Science and Technology, Northwestern Polytechnical University, Xi'an, China.}

\begin{abstract}
Nonnegative Tucker Factorization (NTF) minimizes the euclidean distance or Kullback-Leibler divergence between the original data and its low-rank approximation which often suffers from grossly corruptions or outliers and the neglect of manifold structures of data.
In particular, NTF suffers from rotational ambiguity, whose solutions with and without rotation transformations are equally in the sense of yielding the maximum likelihood.
In this paper, we propose three Robust Manifold NTF algorithms to handle outliers by incorporating structural knowledge about the outliers.
They first applies a half-quadratic optimization algorithm to transform the problem into a general weighted NTF where the weights are influenced by the outliers.
Then, we introduce the correntropy induced metric, Huber function and Cauchy function for weights respectively, to handle the outliers.
Finally, we introduce a manifold regularization to overcome the rotational ambiguity of NTF.
We have compared the proposed method with a number of representative references covering major branches of NTF on a variety of real-world image databases.
Experimental results illustrate the effectiveness of the proposed method under two evaluation metrics (accuracy and nmi).
\end{abstract}


\end{frontmatter}


\section{Introduction}
\label{sec:introduction}

{N}{on-negative} tucker factorization (NTF) also known as nonnegative multilinear singular value decomposition (SVD), is a multiway extension of nonnegative matrix factorization \cite{kim2007nonnegative}. It explores the nonnegative property of data which enhances the ability of part-based representation and has received considerable attention in many fields, {e.g.} text mining \cite{schein2016bayesian}, hyperspectral imaging \cite{karami2012compression}, blind source separation \cite{cichocki2009nonnegative} and data clustering \cite{sun2015heterogeneous}.

Finding and exploiting the low-rank approximation and low-dimensional manifold from high-dimensional data is a fundamental problem in machine learning.
In particular, the extracted components obtained by principle component analysis (PCA) \cite{wold1987principal}, vector quantization (VQ) \cite{gray1984vector} may lose their physical meaning if the nonnegativity is not preserved for high-dimensional real-world data.
Hence, nonnegative matrix factorization (NMF) \cite{lee1999learning,wang2021deep} has been used to explore low-rank representation of given data.
It is intractable for NMF to deal with high-order tensors.
The order of a tensor is the number of dimensions of the array, and a mode is one of its dimensions \cite{kolda2009tensor}.
For example, an RGB image can be represented by a third-order tensor with the dimensions of height $\times$ width $\times$ channel.
When applying NMF to tensorial data, the first step is to reshape tensors into matrices, which often leads to a loss of the meaningful tensor structures, and large scale parameters leads to higher memory demands \cite{zhou2019probabilistic}.
Tensorial data can naturally characterize data from multiple aspects which preserve the structure information in each mode.
However, they are typically high dimensional and difficult to be handled in their original space.
To address this problem, nonnegative tensor decomposition (NTD) methods have been proposed to directly exploit multidimensional structures of tensors.
NTD can be viewed as a special case of NMF which not only inherits the advantages of NMF but also provides physically meaningful representation of multiway structure representation.

For nonnegative tensor data analysis, many nonnegative tensor decomposition methods are based on CANDECOMP/PARAFAC (CP) decompositions \cite{shi2017tensor}, Tucker \cite{shi2018feature} models and low-tubal-rank models \cite{zhou2019bayesian}, respectively.
In this paper, we focus on the tucker model, since it enables more flexible and interpretable decomposition by utilizing the interaction of latent factors.
Existing NTF decomposition methods usually have the following major problems.
First, many NTF methods decompose a high-dimensional tensor into a product of low-rank nonnegative projection matrices and a low-dimensional nonnegative core tensor by minimizing the Euclidean distance between their product and the original tensor data.
They are optimal when the data are condemned by additive Gaussian noise.
However, they may fail on grossly corrupted data, since the corruptions or outliers seriously violate the noise assumption of gaussian distribution.
Second, the tucker-based model suffers from the rotational ambiguity \cite{tipping1999probabilistic}, i.e, solutions with and without rotation transformations are equally good in the sense of yielding the maximum likelihood \cite{zhou2019probabilistic}. This implies that NTF can only find arbitrary bases of the latent subspace.

\subsection{Contributions}
\label{Contributions}
In this paper, we aim to explore robust manifold NTF methods to address the above challenges.
Our contributions are threefold.

\begin{itemize}
  \item Three robust manifold NTF algorithms are proposed. They first apply a half-quadratic optimization algorithm to transform the intractable problem into a weighted NTF model where the weights can handle the outliers. They can be implemented by incorporating different prior distribution for the measurement between the original data and reconstructed data. Unlike the Gaussian distribution of loss functions, other flexible distribution of metric can handle outliers more efficiently.
      The weights are adjusted adaptively with respect to the error.
      To our knowledge, this is the first work of NTF for robust weighted learning to handle outliers.
  \item Three RMNTF algorithms reach the state-of-the-art performance are proposed. The three algorithms fall into the one major subclasses of NTF technologies, denoted as weighted NTF.
      Different distribution of loss function between original data and reconstructed data may influence the ability to handle outliers.
      Specifically, RMNTF with correntropy induced metric (RMNTF-CIM), RMNTF with Huber function (RMNTF-Huber), and RMNTF with Cauchy function (RMNTF-Cauchy) have been studied in this paper.
  \item We demonstrate the convergence, robustness and invariance of RMNTF. 
\end{itemize}

In this work, we first introduce some preliminaries and related work in the following two subsections, then present three RMNTF algorithms in Section \ref{Algorithm}. Section \ref{Exp} presents the experimental results. Finally, Section \ref{Conclusion} concludes our findings.

\subsubsection{Notations}

We denote vectors, matrices and tensors by bold lowercase $\mathbf{x}$, bold uppercase $\mathbf{X}$ and calligraphic letters $\mathcal{X}$ respectively.
$\mathbb{R}_{\geq 0}$ denotes the fields of nonnegative real numbers.
$\langle \cdot \rangle$ denotes the expectation of a certain random variable.
$\mathrm{vec}(\cdot)$ is the vectorization operator that turns a tensor into a column vector.
The transpose of a vector or matrix is denoted by $(\cdot)^T$.
Symbols $\circ$, $\otimes$, $\circledast$ and $\odot$ denote the outer, Kronecker, Hadamard and Khatri-Rao products respectively.
$\mathrm{Diag}^N(\mathbf{x})$ denotes the $N$th order diagonal tensor formed by $\mathbf{x}$.
$\mathbf{x}_n$, $\mathbf{X}_n$ and $\mathcal{X}_n$ denote the $n$th vector, $n$th matrix and $n$th tensor respectively.
$\mathbf{X}_{n(m)}$ denotes the mode-$m$ unfolding of tensor $\mathcal{X}_n$.
$\mathbf{X}^{(m)}$ denotes the mode-$m$ factor matrix.
$\times_n$ denotes the mode-$n$ tensor product.

\begin{myDef}{(Mode-$n$ Product \cite{tucker1966some}):}
A mode-$n$ product of a tensor $\mathcal{S}\in \mathbb{R}^{r_1\times r_2 \times \dots \times r_N}$ with a matrix $\mathbf{A}^{(n)} \in \mathbb{R}^{I_n\times r_n}$ is denoted by $\mathcal{X} = \mathcal{S} \times_n \mathbf{A}^{(n)} \in \mathbb{R}^{r_1\times \dots \times r_{n-1} \times I_n \times r_{n+1} \times \dots \times r_N}$. Each elements can be represented as $\mathcal{X}_{r_1\dots r_{n-1}i_nr_{n+1}\dots r_N} = \sum_{j_n} \mathcal{S}_{r_1\dots r_{n-1}j_nr_{n+1}\dots r_N}\mathbf{A}^{(n)}_{j_n i_n}$.
\end{myDef}

\begin{myDef}{(Mode-$n$ Unfolding \cite{de2000multilinear}):}
It is also known as matricization or flattening, which is the process of reordering the elements of an $N$-way array into a matrix along each mode. A mode-$n$ unfolding matrix of a tensor $\mathcal{X}\in \mathbb{R}^{I_1\times\dots\times I_N}$ is denoted as $\mathbf{X}_{(n)}$ and arranges the mode-$n$ fibers to be the columns of the resulting matrix.
\end{myDef}

\begin{myDef}{(Folding Operator \cite{kolda2009tensor}):}
Given a tensor $\mathcal{X} \in \mathbb{R}^{I_1\times\dots\times I_N}$, the mode-n folding operator of a matrix $\mathbf{M} = \mathcal{X}_{(n)}$ is demoted as $\mathrm{fold}_n(\mathbf{M})$, which is the inverse operator of the unfolding operator.
\end{myDef}

\subsection{Related Work}
\label{Background}

\textbf{\textit{NMF}}:
Given a non-negative data matrix $\mathbf{X} = [x_1,\dots,x_N]\in \mathbb{R}_{\geq 0}^{N\times M}$, where $N$ refers to the number of data points and $M$ indicates the dimension of the feature.
The objective of NMF is to find two nonnegative and low rank factor matrices $\mathbf{U} \in \mathbb{R}_{\geq 0}^{N\times r}$ and $\mathbf{V} \in \mathbb{R}_{\geq 0}^{M\times r}$ and the product of the two matrices approximates the original data matrix $\mathbf{X}$ by $\mathbf{X} \approx \mathbf{U}\mathbf{V}^T$, generally $r \ll \min\{M,N\}$.
NMF incorporates the nonnegativity constraint and obtains the part-based representation as well as enhancing the interpretability \cite{lee1999learning} and these methods have a close relation with K-means \cite{ding2008convex}.
However, the factorization of matrices is generally nonunique and many regularizors \cite{cai2010graph} or constraints \cite{liu2011constrained} have been developed to alleviate nonuniqueness of decomposition.
In particularly, NMF methods are proposed to utilize the priors to obtain a better representation \cite{pan2019generalized} and achieve a robust clustering \cite{guan2017truncated,haeffele2019structured}.

However, the multiway structure information of high-dimensional data such as RGB images or videos cannot be represented by NMF. Vectorization of the high-dimensional data may lead to excess numbers of parameters and structure interruptions. Nonnegative tensor decomposition based methods have been proposed for the high-dimensional data.

\begin{figure}[htb]
\begin{minipage}[b]{1.0\linewidth}
  \centering
  \centerline{\includegraphics[width=9cm]{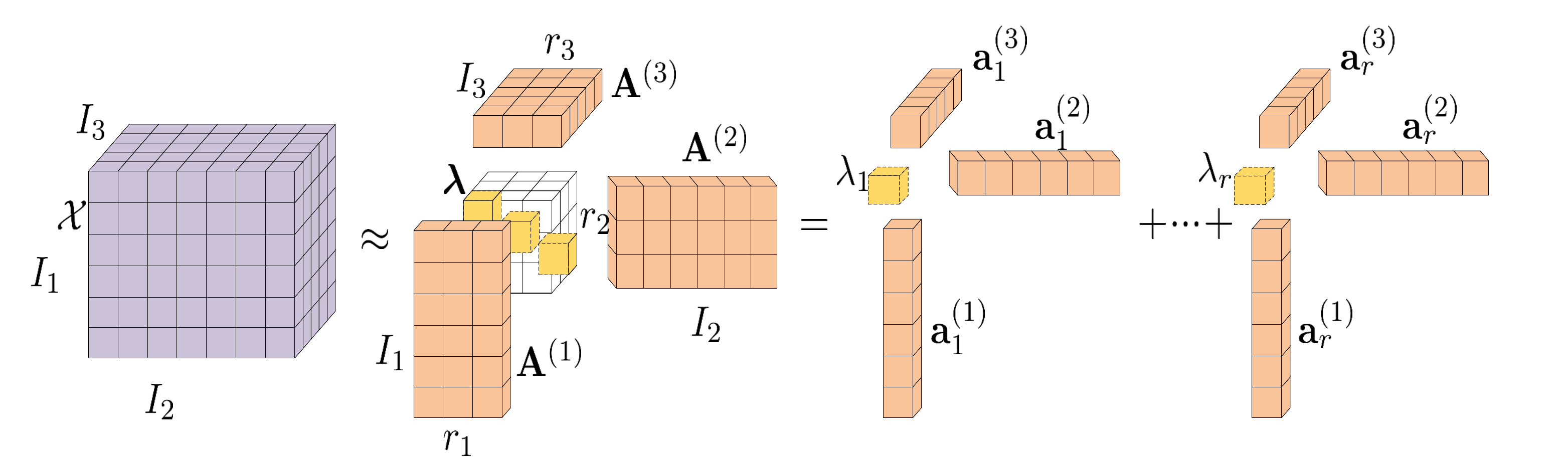}}
\end{minipage}
\caption{Canonical Polyadic Decomposition.}
\label{fig:CPD}
\end{figure}

\textbf{\textit{Nonnegative CANDECOMP/PARAFAC (NCP) methods}}:
The CP-based tensor model (Fig. \ref{fig:CPD}) \cite{harshman1970foundations} decomposes $\mathcal{X}$ into a linear combination of $R$ rank-one tensors as follows:
\begin{equation}\label{CPD}
\begin{split}
 \mathcal{X} & = \sum_{r=1}^R \lambda_r \mathbf{a}^{(1)}_r \circ \mathbf{a}^{(2)}_r \circ \dots \circ \mathbf{a}^{(N)}_r \\
 & = \mathrm{Diag}^N(\boldsymbol{\lambda})\times_{{n}={1}}^N\mathbf{A}^{(n)T},
\end{split}
\end{equation}
where $\mathrm{Diag}^N(\boldsymbol{\lambda})$ is the $N$th-order diagonal tensor.
The CP rank of $\mathcal{X}$ is given by $\mathrm{Rank}_{cp}(\mathcal{X}) = R$ denotes the smallest number of the rank-one tensor decomposition \cite{zhou2019probabilistic}.
The CP-based tensor model assumes each tensor element can be calculated by a summation of $R$ products and it is restrictive since it only considers $R$ possible interactions between latent factors.

Existing CP-based tensor models have a flexible subspace representation which may not consider the intrinsic manifold information of high-dimensional data and the performance of downstream tasks will be limited.
A few regularization and prior knowledge strategies have been studied in these models \cite{zhou2019probabilistic}.
For example, Zhao \emph{et al.} \cite{zhao2015bayesianpami} formulated CP factorization using a hierarchical probabilistic model and employed a fully Bayesian treatment by incorporating a sparsity-inducing prior over multiple latent factors and the appropriate hyperpriors over all hyperparameters to determine the rank of models. Zhao \emph{et al.} \cite{zhao2015bayesian} proposed a probabilistic model to recover the underlying low-rank tensor which modeled by multiplicative interactions among multiple groups of latent factors, and the additive sparse tensor modeling outliers. Zhou \emph{et al.} \cite{zhou2019probabilistic} introduced concurrent regularizations which regularized the entire subspace in a concurrent and coherent way to avoid the strong scale restrictions of $L_2$ regularization. Chen \emph{et al.} \cite{han2018generalized} proposed a generalized weighted low rank tensor factorization which represented the sparse component as a mixture of Gaussian, and unified the Tucker and CP factorization in a joint framework to handle complex noise and outliers.
However, these methods neglect nonnegative constrains and may not learn the part-based and physical meaning representations.

\begin{figure}[htb]
\begin{minipage}[b]{1.0\linewidth}
  \centering
  \centerline{\includegraphics[width=9cm]{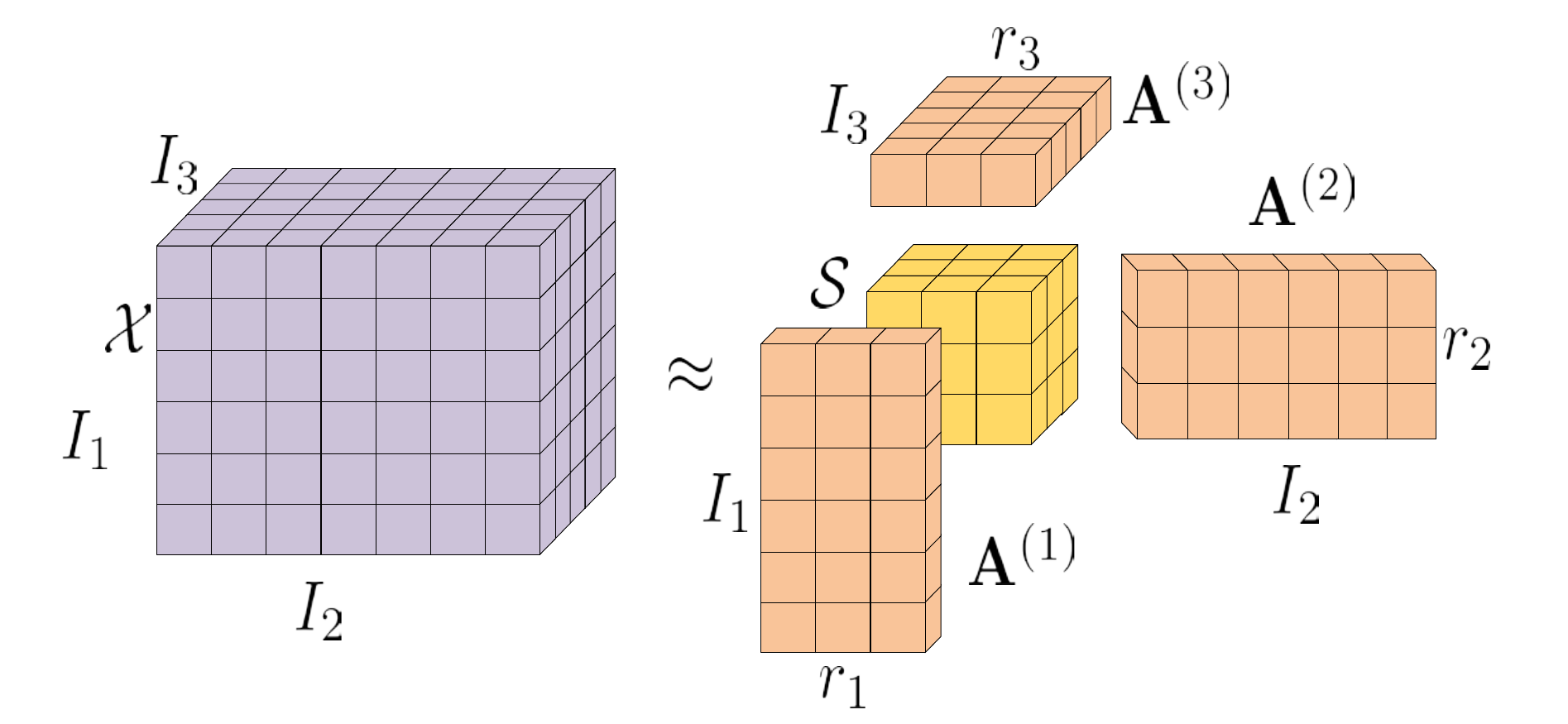}}
\end{minipage}
\caption{Tucker Decomposition.}
\label{fig:TF}
\end{figure}

\textbf{\textit{NTF methods}}:
The Tucker model (Fig. \ref{fig:TF}) assumes that an original tensor $\mathcal{X}$ can be well approximated as
\begin{equation}\label{NTF}
\begin{split}
 \mathcal{X} = \mathcal{S} \times_{{1}} \mathbf{A}^{({1})} \times_{{2}} \dots \times_{N} \mathbf{A}^{(N)},
\end{split}
\end{equation}
where the tucker rank of $\mathcal{X}$ is denoted as $\mathrm{Rank}_{\mbox{tc}}(\mathcal{X}) = (R_1,\dots,R_n,\dots,R_N)$ with $R_n = \mathrm{Rank}(\mathbf{X}_{(n)})$.

We note that the CP decomposition is a special case of the Tucker decomposition.
Although the tucker decomposition is invariant to rotations in the factor matrices, it shares parameters across latent factor matrices by a core tensor.
In contrast, CP decomposition methods force each factor vector to capture potentially redundant information \cite{fang2021bayesian}.
CP decomposition methods are more prone to overfitting than Tucker decompositions.
The main interest in Tucker model is to find subspaces for tensor approximation.

Li \emph{et al.} \cite{li2016mr} introduced a manifold regularization into the core tensors of NTF which preserved the tensor geometry information. But the representation space of the core tensor will increase exponentially as tensor order increases, which results in high computational complexity. Jiang \emph{et al.} \cite{jiang2018image} added a graph Laplacian regularization on a low-dimensional factor matrix to improve the robustness of tensor decomposition. Sun \emph{et al.} \cite{sun2015heterogeneous} proposed a heterogeneous tensor decomposition for clustering by performing dimensionality reduction on the first $N-1$ order of the tensor, and incorporating some useful constraints on the last-mode factor matrix for clustering.
However, these methods neglect nonnegative constraints and may loss physical meaning of low-rank representations. Yin \emph{et al.} \cite{yin2019lle} incorporated Laplacian Eigenmaps and Locally Linear Embedding as the manifold regularization terms into the least square form of NTF model. Pan \emph{et al.} \cite{pan2021orthogonal} introduced the orthogonal constraint into the group of factor matrices of NTF, which not only helps to keep the inherent tensor structure but also well performs in data compression. Yin \emph{et al.} \cite{yin2021hyperntf} proposed Hypergraph Regularized NTF which preserved nonnegativity in tensor factorization and uncover the higher-order relationship among the nearest neighborhoods.
However, these methods decompose tensor data by minimizing the Euclidean distance which fails on the not cleaned data.

\textbf{\textit{Manifold learning}}:
Manifold learning is a problem which encodes the geometric information of the data space.
Its goal is to find a representation in which two objects are close to each other after dimension reduction if they are close in the intrinsic geometry of data manifold.
Based on this idea, many types of manifold learning algorithms have been proposed, such as ISOMAP \cite{tenenbaum2000global}, LLE \cite{roweis2000nonlinear}, Laplacian eigenmaps \cite{belkin2001laplacian} and locality preserving projections \cite{he2004locality}.
The advantage of introducing manifold structures is that it can preserve the intrinsic geometric information of data points, and it has been shown to be useful in a wide-range of applications, such as face recognition \cite{he2005face}, text mining \cite{cai2010graph}, and multimedia interaction \cite{yang2008harmonizing}.

Recently, the idea of manifold learning has been employed to matrix and tensor analysis.
For instance, Cai \textit{et al.} \cite{cai2010graph} proposed graph regularized nonnegative matrix factorization (GNMF) which utilized the intrinsic geometric information.
However, GNMF may not have optimal solutions due to the noises or outliers of data, consequently other extension methods based on the GNMF have been proposed.
Moreover, NMF based on manifold learning methods may break the structure of mutliway data points.
Hence, tensor based manifold learning methods have been proposed such as Graph regularized Nonnegative Tucker Decomposition \cite{qiu2020generalized}, LLE based nonnegative tensor decomposition \cite{yin2019lle}.
These manifold nonnegative tensor decomposition methods assume that the distribution of noise is Gaussian and may fail on grossly corrupted datasets.

\begin{figure}[htb]
\begin{minipage}[b]{1.0\linewidth}
  \centering
  \centerline{\includegraphics[width=7cm]{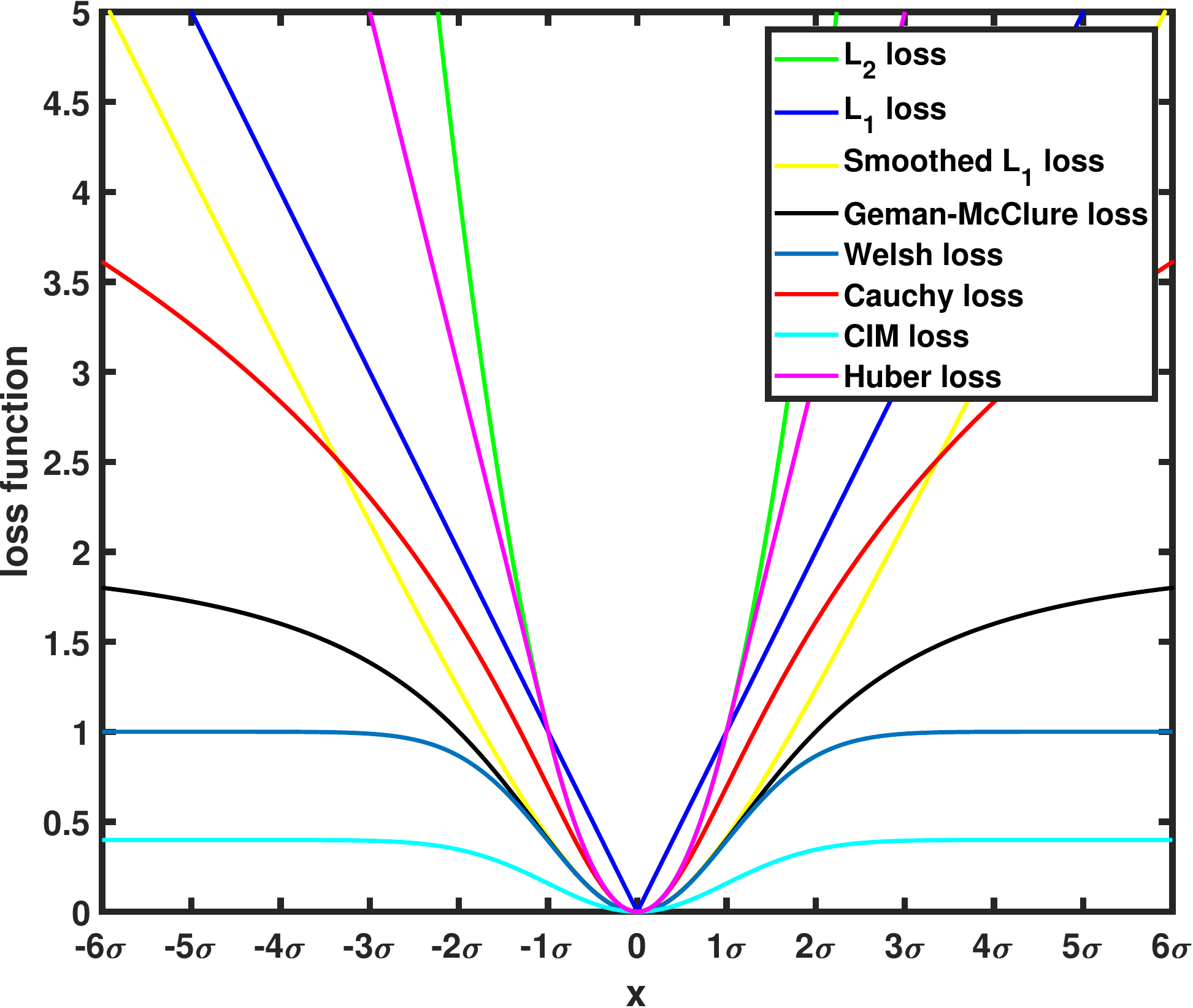}}
\end{minipage}
\caption{The comparison of loss function.}
\label{fig:Los}
\end{figure}

\section{Robust Manifold NTF methods}
\label{Algorithm}
Because our methods rely on the half-quadratic theory, we first introduce the Half-Quadratic \cite{boyd2004convex} minimization technique for the generic robust NTF framework.

\subsection{Half-Quadratic Programming for Nonnegative Tucker Factorization}
Half-Quadratic minimization was pioneered by Geman and Reynolds \cite{geman1992constrained}, and it was used to alleviate the computational task in the context of image reconstruction with nonconvex regularization.
Let $\mathcal{X}$ denote an original tensor and $\hat{\mathcal{X}}$ denote a reconstruction tensor of $\mathcal{X}$.
Replacing the squared residual of data-fidelity terms in \cite{nikolova2007equivalence} on each entry with a generic function:
\begin{equation}\label{Half_Q}
\begin{split}
 \mathcal{J}(\mathcal{X},\hat{\mathcal{X}}) = g(\mathcal{E}) + \lambda\Phi(\{\mathbf{A}^{(n)}\}_{n=1}^N) \\
 \mathrm{s.t.} \quad \hat{\mathcal{X}} = \mathcal{S} \times_{n=1}^{N} \mathbf{A}^{(n)}, \quad \mathcal{E} = \mathcal{X} - \hat{\mathcal{X}} \\
 \mathcal{S} \geq 0, \quad \mathbf{A}^{(n)} \geq 0,
\end{split}
\end{equation}
where $\mathcal{E}$ represents the residual error between original tensor $\mathcal{X}$ and reconstructed tensor $\hat{\mathcal{X}}$, $g(\cdot)$ is chosen to be robust to outliers or gross errors, and $\Phi(\cdot)$ denotes the regularization terms with respect to $\{\mathbf{A}^{(n)}\}_{n=1}^N$.
The minimizer $\hat{\mathcal{X}}$ of cost function $\mathcal{J}(\mathcal{X},\hat{\mathcal{X}})$ involving the reconstruction error which is nonlinear with respect to $\mathcal{X}$ and the regularization term $\Phi(\{\mathbf{A}^{(n)}\}_{n=1}^N)$.
When factor matrices $\mathbf{A}^{(n)}, n = 1,\dots,N-1$ and core tensor $\mathcal{S}$ have many nonzero entries or ill-conditioned, the computation of factorization is costly.
specifically, the loss function is possibly non-quadratic and non-convex, and it is difficult to optimize directly.
Fortunately, the half-quadratic minimization \cite{nikolova2005analysis} has been developed to solve the intractable optimization.
According to the conjugate function and half quadratic theory \cite{nikolova2005analysis}, the reconstruction error term $g(\mathcal{E})$ can be performed as
\begin{equation}\label{Half_Qg}
\begin{split}
 g(\mathcal{E}) = \min_{\mathcal{W}\in\mathbb{R}^{I_1\times \dots\times I_N}} Q(\mathcal{E},\mathcal{W}) + \phi(\mathcal{W}),
\end{split}
\end{equation}
where $\phi(\mathcal{W})$ is the conjugate function of $g(\mathcal{E})$, $\mathcal{W}$ is the corresponded additional auxiliary variable, and $Q(\cdot,\cdot)$ is a quadratic term for $\mathcal{E}$ and $\mathcal{W}$.
In this paper, we only consider the quadratic term of the multiplicative form \cite{geman1992constrained}:
\begin{equation}\label{Half_QQ}
\begin{split}
 Q(\mathcal{E},\mathcal{W}) = \frac{1}{2} \mathcal{W} \circledast \mathcal{E}^2.
\end{split}
\end{equation}
Substituting \eqref{Half_Qg} and \eqref{Half_QQ} into \eqref{Half_Q}, we have the augmented cost function $\hat{\mathcal{J}}(\mathcal{X},\hat{\mathcal{X}};\mathcal{W})$:
\begin{equation}\label{Half_Q1}
\begin{split}
 &  \hat{\mathcal{J}}(\mathcal{X},\hat{\mathcal{X}};\mathcal{W}) = \frac{1}{2}\mathcal{W} \circledast \mathcal{E}^2 + \phi(\mathcal{W}) + \lambda\Phi(\mathbf{A}^{(N)})  \\
 & \mathrm{s.t.} \quad \hat{\mathcal{X}} = \mathcal{S} \times_{n=1}^{N} \mathbf{A}^{(n)}, \quad \mathcal{E} = \mathcal{X} - \hat{\mathcal{X}} \\
 & \quad \quad \mathcal{S} \geq 0, \quad \mathbf{A}^{(n)} \geq 0, \quad \forall n \in \{1,\dots,N\}.
\end{split}
\end{equation}

The reconstruction error term involved in $\hat{\mathcal{J}}(\mathcal{X},\hat{\mathcal{X}};\mathcal{W})$ is half-quadratic.
Hence, the minimizer $(\hat{\mathcal{X}},\mathcal{W})$ of $\min_{\hat{\mathcal{X}},\mathcal{W}} \left\{ \hat{\mathcal{J}}(\mathcal{X},\hat{\mathcal{X}};\mathcal{W}) \right\}$ is calculated by alternate minimization.
At iteration $k$ we calculate
\begin{equation}\label{solveHQ}
\begin{split}
 & \mathcal{W}^{k}: \quad \hat{\mathcal{J}}(\mathcal{X},\hat{\mathcal{X}}^{k-1};\mathcal{W}^{k}) \leq \hat{\mathcal{J}}(\mathcal{X},\hat{\mathcal{X}}^{k-1};\mathcal{W}^{k-1}) \\
 & \hat{\mathcal{X}}^{k}: \quad \hat{\mathcal{J}}(\mathcal{X},\hat{\mathcal{X}}^{k};\mathcal{W}^{k}) \leq \hat{\mathcal{J}}(\mathcal{X},\hat{\mathcal{X}}^{k-1};\mathcal{W}^{k}).
\end{split}
\end{equation}

When $\hat{\mathcal{X}}$ is fixed, the minimization of the reconstruction error term $g(\cdot)$ is convex with respect to $\mathcal{W}$. The explicit optimum solution of $\mathcal{W}$ \cite{charbonnier1997deterministic} can be determined as
\begin{eqnarray}
\mathcal{W} =
\left\{\begin{array}{ll}
\frac{\partial^2 g(0)}{\partial \mathcal{E}^2} \quad \mbox{if} \quad \mathcal{E} = 0\\
\frac{\left[\frac{\partial g(\mathcal{E})}{\partial \mathcal{E}}\right]}{\left[\mathcal{E}\right]} \quad \mbox{if} \quad \mathcal{E} \neq 0,
 \end{array}
 \right.\label{solveWeightedtensor}
 \end{eqnarray}
where $\frac{[\cdot]}{[\cdot]}$ denotes element-wise division.

It is shown that the auxiliary variable $\mathcal{W}$ only depends on the loss function $g(\cdot)$.
Since the outliers often cause large fitting errors, $\mathcal{W}$ is important for the objective functions to constrain overfitting.
For the large outliers, the weights $\mathcal{W}$ should be small.
On contrary, for the small errors, the weights $\mathcal{W}$ should be large.
Therefore, the weights variable $\mathcal{W}$ can be seen as an outlier mask.
The frequency used loss functions are shown in Fig. \ref{fig:Los}.

\subsection{Robust NTF with Manifold Regularization}
By using the nonnegative constraints and robust loss function for outliers, robust NTF can learn a part-based representation. Many robust NTF methods perform well in euclidean space. They fail to discover the intrinsic geometrical and discriminating structure of the data space.
Here, we introduce a geometrically based regularization for our robust NTF framework.

First, supposing that the real data points $\mathcal{X}_i$ lies on a low-dimensional manifold $\mathcal{M}$ and $\mathbf{A}^{(N)}_i$ is the representation of $\mathcal{X}_i$ in the subspace. We make an assumption that if $\mathcal{X}_i$ and $\mathcal{X}_j$ are neighbors in data space, then their low-rank representations $\mathbf{A}^{(N)}_i$ and $\mathbf{A}^{(N)}_j$ are close enough to each other in $\mathcal{M}$. We build a regularization term as follows
\begin{equation}\label{RNTFMR1}
\begin{split}
 \mathcal{R}(\mathcal{M}) = \sum_p \| f_p \|_\mathcal{M}^2,
\end{split}
\end{equation}
where $f_p$ is the mapping function which project data point $\mathcal{X}_i$ to the low-rank representation $\mathbf{A}^{(N)}_i = f_p(\mathcal{X}_i)$. $\mathcal{R}(\mathcal{M})$ is a measurement of the smoothness of $f_p$ along the geodesics in the intrinsic geometry of the data.

Based on \cite{cai2010graph}, we use the similarity graph of $\mathcal{X}$. Suppose that $\mathbf{V}\in\mathbb{R}^{I_N\times I_N}$ defines the affinity of data points, then we use the Heat Kernel to describe the similarity between each pair of data points if nodes $i$ and $j$ are connected:
\begin{equation}\label{RNTFMR2}
\begin{split}
 \mathbf{V}_{ij} = \exp \left( - \frac{\| \mathcal{X}_i - \mathcal{X}_j \|^2_2}{\tau} \right),
\end{split}
\end{equation}
where $\tau$ is the width of the kernel used to control the similarity.
Then, we calculate the diagonal matrix $\mathbf{D}\in \mathbb{R}^{I_N\times I_N}$, where $[\mathbf{D}]_{ii} = \sum_j[\mathbf{V}]_{ij}$ and the Laplacian matrix is $\mathbf{L} = \mathbf{D} - \mathbf{V}$. The graph regularization can be estimated as follows:
\begin{equation}\label{RNTFMR3}
\begin{split}
 \mathcal{R}(\mathcal{M}) = & \frac{1}{2} \sum_{p=1}^P \sum_{i=1}^N \sum_{j=1}^N [\mathbf{V}]_{ij}\left[ f_p(\mathcal{X}_i) - f_p(\mathcal{X}_j) \right]^2 \\
 = & \mathrm{Tr}(\mathbf{A}^{(N)T}\mathbf{L}\mathbf{A}^{(N)}).
\end{split}
\end{equation}

\subsection{RMNTF-CIM}

Liu \textit{et al.} \cite{liu2007correntropy} proposed the concept of Cross correntropy which is a generalized similarity measure between two arbitrary scalar random variables $\mathbf{X}$ and $\mathbf{Y}$ defined by
\begin{equation}\label{Cross_cor}
\begin{split}
 V_\sigma(\mathbf{X},\mathbf{Y}) = \langle \kappa_\sigma(\mathbf{X} - \mathbf{Y}) \rangle,
\end{split}
\end{equation}
where $\kappa_\sigma(\cdot)$ is the kernel function.
In practice, the joint probabilistic density function is unknown and only a finite number of data points $\{(x_i,y_i)\}_{i=1}^N$ are available.
The sample estimator of correntropy can be represented as
\begin{equation}\label{Cross_corS}
\begin{split}
 \hat{V}_{N,\sigma}(\mathbf{X},\mathbf{Y}) = \frac{1}{N} \sum_{i=1}^N \kappa_\sigma(x_i - y_i).
\end{split}
\end{equation}
Based on the above definition of correntropy, Liu \textit{et al.} \cite{liu2007correntropy} proposed correntropy induced metric (CIM) in the sample space which denoted as
\begin{equation}\label{CIM}
\begin{split}
 \mathrm{CIM}(\mathbf{X},\mathbf{Y}) = \left[ \kappa(0) - \frac{1}{n} \sum_{i=1}^n \kappa_\sigma(e_i) \right]^{\frac{1}{2}},
\end{split}
\end{equation}
where we use the Gaussian kernel in this paper, i.e., $\kappa_\sigma(e) = \frac{1}{\sqrt{2\pi}\sigma}\exp\left( -e_i^2/2\sigma^2 \right)$, and $e_i$ is denoted as $e_i = x_i - y_i$.

Substituting the error on each entry in Tucker model with the CIM, we obtain the objective function of RMNTF-CIM:
\begin{equation}\label{RMNTF_CIM}
\begin{split}
 \mathcal{J}(\mathcal{X},\hat{\mathcal{X}}) = 1-\frac{1}{I_1 \dots I_N} \sum_{i_1 = 1}^{I_1} \dots \sum_{i_N = 1}^{I_N} \kappa_\sigma\left( \mathcal{X}_{i_1\dots i_N} - \hat{\mathcal{X}}_{i_1\dots i_N} \right),
\end{split}
\end{equation}
which is equivalent to solving the following optimization problem
\begin{equation}\label{objRMNTF_CIM}
\begin{split}
 \min_{\mathbf{A}^{(1)},\dots,\mathbf{A}^{(n)},\mathcal{S}} \mathcal{J}(\mathcal{X},\hat{\mathcal{X}}) + \frac{\lambda}{2} \sum_{i=1}^{I_N} \sum_{j=1}^{I_N} \| \mathbf{a}_i^{(N)} - \mathbf{a}_j^{(N)} \|_2^2 v_{ij} \\
 \mathrm{s.t.} \quad \hat{\mathcal{X}} = \mathcal{S} \times_{n=1}^N \mathbf{A}^{(n)}, \quad \mathcal{S} \geq 0, \quad \mathbf{A}^{(n)} \geq 0.
\end{split}
\end{equation}
Then, we introduce the half-quadratic minimization,
$\mathcal{J}(\mathcal{X},\hat{\mathcal{X}}) = \frac{1}{2} Q(\mathcal{W} \circledast \mathcal{E}^2) + \phi(\mathcal{W})$, where $\phi(\mathcal{W})$ is the conjugate function of $\mathcal{J}(\mathcal{X},\hat{\mathcal{X}})$.

\subsubsection{Optimization of weighted tensor $\mathcal{W}$}
When the factor matrices $\mathbf{A}^{(n)}$ and the core tensor $\mathcal{S}$ are fixed, the optimization problem with respect to $\mathcal{W}$ can be solved separately:
\begin{equation}\label{WeightedT}
\begin{split}
 & \mathcal{W}_{(n)}^\star =  \arg\min f_w(\mathcal{W}_{(n)}) \\
 & \mbox{with} \quad f_w(\mathcal{W}_{(n)}) =   \sum_{i=1}^{N} \sum_{j=1}^{M} \left[ \mathcal{W}_{(n)} \circledast \left( \mathcal{X}_{(n)} - \mathbf{A}^{(n)} \mathbf{B}^{(n)T} \right)^2 \right]_{ij} \\
 & + \phi(\mathcal{W}_{(n)}),
\end{split}
\end{equation}
where $\mathbf{B}^{(n)T} = \mathcal{S}_{(n)}\left( \otimes_{i\neq n}\mathbf{A}^{(i)T} \right)$, $[\mathbf{X}]_{ij}$ denotes the $i$th row and $j$th column element of matrix $\mathbf{X}$.

Let $\mathcal{E}_{(n)} = \mathcal{X}_{(n)}-\mathbf{A}^{(n)}\mathbf{B}^{(n)T}$, then
\begin{equation}\label{updateWeightedT}
\begin{split}
 \mathcal{W}_{(n)}^\star & = \frac{[\nabla_{\mathcal{E}_{(n)}} f_{w}(\mathcal{W}_{(n)})]}{[\mathcal{E}_{(n)}]} = \frac{[\nabla_{\mathcal{E}_{(n)}} \left( 1-\kappa_\sigma(\mathcal{E}_{(n)}) \right)]}{[\mathcal{E}_{(n)}]} \\
 & \varpropto \exp \left( - \frac{\left(\mathcal{X}_{(n)} - \mathbf{A}^{(n)} \mathbf{B}^{(n)T}\right)^2}{2\sigma^2} \right),
\end{split}
\end{equation}
where $\frac{[\cdot]}{[\cdot]}$ denotes the element-wise division operation.

\subsubsection{optimization of factor matrices $\mathbf{A}^{(n)}$}
We use the Lagrange multiplier method and consider the mode-$n$ unfolding form, then:
\begin{equation}\label{FactorMAn}
\begin{split}
 & \mathbf{A}^{(n)\star}  = \arg \min f_A(\mathbf{A}^{(n)}), \\
 & \mbox{with} \quad f_A(\mathbf{A}^{(n)})  = \sum_{i=1}^I \sum_{j=1}^J \left[ \mathcal{W}_{(n)} \circledast \left( \mathcal{X}_{(n)} - \mathbf{A}^{(n)} \mathbf{B}^{(n)T} \right)^2 \right]_{ij} \\
 & + \frac{\lambda}{2}\mathrm{Tr}\left[ \mathbf{A}^{(N)T}\mathbf{L}\mathbf{A}^{(N)} \right] + \mathrm{Tr}\left[\boldsymbol{\Omega}_n\mathbf{A}^{(n)}\right] \\
 & \rm{s.t.} \quad \boldsymbol{\Omega} \geq 0.
\end{split}
\end{equation}

If $n \neq N$, the objective function with respect to $\mathbf{A}^{(n)}$ can be transformed as:
\begin{equation}\label{FactorMAn1}
\begin{split}
 & f_A(\mathbf{A}^{(n)}) = \sum_{i=1}^I \left( [\mathcal{X}_{(n)}]_{i\cdot} - [\mathbf{A}^{(n)}]_{i\cdot} \mathbf{B}^{(n)T} \right) \mathbf{T}_i \\
 & \left( [\mathcal{X}_{(n)}]_{i\cdot} - [\mathbf{A}^{(n)}]_{i\cdot} \mathbf{B}^{(n)T} \right)^T + \mathrm{Tr}(\boldsymbol{\Omega}_n \circledast \mathbf{A}^{(n)}),
\end{split}
\end{equation}
where $\mathbf{T}_i = \mathrm{Diag}\left([\mathcal{W}_{(n)}]_{i\cdot}\right) \in \mathbb{R}^{M\times M}_{\geq 0}$, $M = I_1 \times \dots \times I_{n-1}  \times I_{n+1} \times \dots \times I_N$, $\boldsymbol{\Omega}_n$ is the nonneative Lagrange multiplier for the nonnegative constraint.
The partial derivative of $f_{A}(\mathbf{A}^{(n)})$ with respect to $\mathbf{A}^{(n)}$ is:
\begin{equation}\label{FactorMAn2}
\begin{split}
 \frac{\partial f_A([\mathbf{A}^{(n)}]_{i\cdot})}{\partial {[\mathbf{A}^{(n)}}]_{i\cdot}} = & -2 \left[ [\mathcal{X}_{(n)}]_{i\cdot} \mathbf{T}_i \mathbf{B}^{(n)} \right] \\
  & + 2\left[ [\mathbf{A}^{(n)}]_{i\cdot} \mathbf{B}^{(n)T} \mathbf{T}_i \mathbf{B}^{(n)} \right] + [ \boldsymbol{\Omega}_n ]_{i\cdot}.
\end{split}
\end{equation}

Using the KKT conditions $[\boldsymbol{\Omega}_n]_{ik}[\mathbf{A}^{(n)}]_{ik} = 0$, we obtain the following equations
\begin{equation}\label{FactorMAn3}
\begin{split}
 - \left( [\mathcal{X}_{(n)}]_{i\cdot} \mathbf{T}_i \mathbf{B}^{(n)} \right) + \left( [\mathbf{A}^{(n)}]_{i\cdot} \mathbf{B}^{(n)T} \mathbf{T}_i \mathbf{B}^{(n)} \right) = \mathbf{0}.
\end{split}
\end{equation}
Then, we obtain the update rules of $\mathbf{A}^{(n)}$:
\begin{equation}\label{updateAn}
\begin{split}
 [\mathbf{A}^{(n)}]_{i\cdot} & = [\mathbf{A}^{(n)}]_{i\cdot} \circledast \frac{\left[\left( [\mathcal{X}_{(n)}]_{i\cdot} \mathbf{T}_i \mathbf{B}^{(n)} \right)\right]}{\left[\left( [\mathbf{A}^{(n)}]_{i\cdot} \mathbf{B}^{(n)T} \mathbf{T}_i \mathbf{B}^{(n)} \right)\right]} \\
 \mathbf{A}^{(n)} & = \mathbf{A}^{(n)} \circledast \frac{\left[\left( \mathcal{W}_{(n)} \circledast \mathcal{X}_{(n)} \mathbf{B}^{(n)} \right)\right]}{\left[\left( \mathcal{W}_{(n)} \circledast \left( \mathbf{A}^{(n)} \mathbf{B}^{(n)T} \right) \mathbf{B}^{(n)} \right)\right]}.
\end{split}
\end{equation}

If $n = N$, the objective function with respect to $\mathbf{A}^{(N)}$ can be represented as:
\begin{equation}\label{FactorMAN}
\begin{split}
 & f_A(\mathbf{A}^{(N)}) = \sum_{i=1}^I \left( [\mathcal{X}_{(N)}]_{i\cdot} - [\mathbf{A}^{(N)}]_{i\cdot} \mathbf{B}^{(N)T} \right) \mathbf{T}_i \\
 & \left( [\mathcal{X}_{(N)}]_{i\cdot} - [\mathbf{A}^{(N)}]_{i\cdot} \mathbf{B}^{(N)T} \right)^T + \mathrm{Tr}(\boldsymbol{\Omega}_N \circledast \mathbf{A}^{(N)})\\
 & +  \frac{\lambda}{2}\mathrm{Tr}\left( \mathbf{A}^{(N)T} \mathbf{L}\mathbf{A}^{(N)} \right).
\end{split}
\end{equation}

The partial derivative of $f_{A}(\mathbf{A}^{(N)})$ is:
\begin{equation}\label{FactorMAN1}
\begin{split}
  & \frac{\partial f_A([\mathbf{A}^{(N)}]_{i\cdot})}{\partial [{\mathbf{A}^{(N)}}]_{i\cdot}} = -2 \left( [ \mathcal{X}_{(n)} ]_{i\cdot} \mathbf{T}_i \mathbf{B}^{(n)} \right) \\
  & + 2\left( [\mathbf{A}^{(n)}]_{i\cdot} \mathbf{B}^{(n)T} \mathbf{T}_i \mathbf{B}^{(n)} \right) + \lambda \mathbf{L} \mathbf{A}^{(N)} + \boldsymbol{\Omega}_N.
\end{split}
\end{equation}

Using the KKT conditions $[\boldsymbol{\Omega}_N]_{ik}[\mathbf{A}^{(N)}]_{ik} = 0$, we obtain the following equations:
\begin{equation}\label{FactorMAn3}
\begin{split}
& 2 \left( \mathcal{W}_{(N)} \circledast (\mathbf{A}^{(N)} \mathbf{B}^{(N)T}) \mathbf{B}^{(N)} \right) \\
& - 2 \left( \mathcal{W}_{(N)} \circledast \mathcal{X}_{(N)} \mathbf{B}^{(N)} \right)
  + \lambda \mathbf{L} \mathbf{A}^{(N)} = \mathbf{0}.
\end{split}
\end{equation}
Then, we obtain the update rules of $\mathbf{A}^{(N)}$:
\begin{equation}\label{updateAN}
\begin{split}
 \mathbf{A}^{(N)}
 & = \mathbf{A}^{(N)} \circledast \frac{[\left( \mathcal{W}_{(N)} \circledast \mathcal{X}_{(N)} \mathbf{B}^{(N)} \right)] + \lambda \mathbf{V} \mathbf{A}^{(N)}]}{[\left( \mathcal{W}_{(N)} \circledast \left( \mathbf{A}^{(N)} \mathbf{B}^{(N)T} \right) \mathbf{B}^{(N)} \right) + \lambda \mathbf{D} \mathbf{A}^{(N)}]}.
\end{split}
\end{equation}

\subsubsection{optimization of the core tensor $\mathcal{S}$}
For the subproblem of core tensor $\mathcal{S}$, we consider the vectorization form of \eqref{RMNTF_CIM}:
\begin{equation}\label{CoreTensor}
\begin{split}
 & \mathcal{S}^\star = \arg \min f_{\mathcal{S}}(\mathcal{S}), \\
 & \mbox{with} \quad f_{\mathcal{S}}(\mathcal{S}) = \mathrm{vec}(\mathbf{1})^T \left[ \mathrm{vec}(\mathcal{W}) \circledast \left( \mathrm{vec}(\mathcal{X}) - \mathbf{F}\mathrm{vec}(\mathcal{S}) \right)^2 \right] \\
 & + \mathrm{vec}(\mathcal{S})^T \mathrm{vec}(\boldsymbol{\Omega}_{\mathcal{S}}),
\end{split}
\end{equation}
where $\mathbf{1} \in \mathbb{R}^{I_1\times I_2 \times \dots \times I_N}$ denotes the all-one vector of length $\prod_{i=1}^N I_i$, $\mathbf{F} = \mathbf{A}^{(1)} \otimes \mathbf{A}^{(2)} \otimes \dots \otimes \mathbf{A}^{(N)} \mathbb{R}^{I_1I_2\dots I_N\times r_1r_2\times r_N}$, and $\mathrm{vec}(\boldsymbol{\Omega}_{\mathcal{S}})$ denotes the Lagrange multipliers of $\mathrm{vec}(\mathcal{S})$.

\begin{equation}\label{CoreTensor1}
\begin{split}
 f_\mathcal{S}(\mathcal{S}) = & \Bigg\{ \left[ \mathrm{vec}(\mathcal{X}) - \mathbf{F}\mathrm{vec}(\mathcal{S}) \right]^T \mathbf{T}_{\mathcal{S}} \left[ \mathrm{vec}(\mathcal{X}) - \mathbf{F}\mathrm{vec}(\mathcal{S}) \right] \Bigg\} \\
 & + \mathrm{vec}(\mathcal{S})^T \mathrm{vec}(\boldsymbol{\Omega}_{\mathcal{S}}),
\end{split}
\end{equation}
where $\mathbf{T}_\mathcal{S} = \mathrm{Diag}(\mathrm{vec}(\mathcal{W})) \in \mathbb{R}^{I_1I_2\dots I_N\times I_1I_2\dots I_N}$.

\begin{equation}\label{PartialCoreTensor}
\begin{split}
 \frac{\partial f_\mathcal{S}(\mathcal{S}) }{\partial \mathrm{vec}(\mathcal{S}) } = 2 \mathbf{F}^T \mathbf{T}_\mathcal{S} \mathbf{F} \mathrm{vec}(\mathcal{S}) - 2 \mathbf{F}^T \mathbf{T}_{\mathcal{S}} \mathrm{vec}(\mathcal{X}) + \mathrm{vec}(\boldsymbol{\Omega}_\mathcal{S}).
\end{split}
\end{equation}
Using the KKT conditions $\mathrm{vec}(\mathcal{S}) \circledast \mathrm{vec}(\boldsymbol{\Omega}_{\mathcal{S}}) = \mathbf{0}$, where $\mathbf{0}$ denote the all-zero vector, we obtain the following equations
\begin{equation}\label{PartialCoreTensor1}
\begin{split}
 \left[\mathbf{F}^T \mathbf{T}_\mathcal{S} \mathbf{F} \mathrm{vec}(\mathcal{S}) - \mathbf{F}^T \mathbf{T}_{\mathcal{S}} \mathrm{vec}(\mathcal{X})\right] \circledast \mathrm{vec}(\mathcal{S}) = \mathbf{0}.
\end{split}
\end{equation}
Then, we obtain the update rules of $\mathcal{S}$:
\begin{equation}\label{updateCoreTensor}
\begin{split}
 \mathrm{vec}(\mathcal{S}) & = \mathrm{vec}(\mathcal{S}) \circledast \frac{[\mathbf{F}^T \mathbf{T}_{\mathcal{S}} \mathrm{vec}(\mathcal{X})]}{[\mathbf{F}^T \mathbf{T}_\mathcal{S} \mathbf{F} \mathrm{vec}(\mathcal{S})]} \\
 \mathcal{S} & = \mathcal{S} \circledast \frac{[\mathbf{F}^T \mathbf{T}_{\mathcal{S}} (\mathcal{X})]}{[\mathbf{F}^T \mathbf{T}_\mathcal{S} \mathbf{F} (\mathcal{S})]}.
\end{split}
\end{equation}

\subsection{RMNTF-Huber}

Robust statistics work well on model reconstruction under some observation with noise or outliers. Some popular M-estimators \cite{zhang1997parameter} such as Huber loss function and Cauchy function have been proposed to solve noisy data mining.

In this section, we take the Huber function in reconstruction error term $g(\cdot)$ to measure the quality of approximation by considering the connection between $L_1$ norm and $L_2$ norm:
\begin{equation}
 g_{\rm{huber}}(\mathcal{E}) =
 \left\{\begin{array}{ll}
 \mathcal{E}^2   \quad\quad\quad\quad \mbox{    if} \quad |\mathcal{E}| \leq c\\
 2 c |\mathcal{E}| - c^2 \quad \mbox{if} \quad |\mathcal{E}| \geq c,
 \end{array}
 \right.\label{Huber}
\end{equation}
where $c$ is the cutoff parameter to tradeoff between the $L_1$-norm and $L_2$-norm.

Substituting the Huber function on each entry in \eqref{Half_Q}, we have the RMNTF-Huber by minimizing the following objective function:
\begin{equation}\label{RMNTF_Huber}
\begin{split}
 \min_{\mathbf{A}^{(1)},\dots, \mathbf{A}^{(n)},\mathcal{S}} g_{\rm{huber}}(\mathcal{E}) + \frac{\lambda}{2} \sum_{i=1}^{I_N} \sum_{j=1}^{I_N} \| \mathbf{a}_i^{(N)} - \mathbf{a}_j^{(N)} \|_2^2 v_{ij} \\
 \mathrm{s.t.} \hat{\mathcal{X}} = \mathcal{S} \times_{n=1}^N \mathbf{A}^{(n)},  \mathcal{S} \geq 0,  \mathbf{A}^{(n)} \geq 0, \forall n \in \{1,\dots,N\},
\end{split}
\end{equation}
where $\mathcal{E} = \mathcal{X} - \hat{\mathcal{X}} = \mathcal{X} - \mathcal{S}\times_{n=1}^N \mathbf{A}^{(n)}$.

Following the equation \eqref{solveWeightedtensor}, we obtain the optimization of weighted tensor $\mathcal{W}$:
\begin{equation}
 \mathcal{W}_{(n)}^\star =
 \left\{\begin{array}{ll}
 1   \quad\quad \mbox{if} \quad |\mathcal{E}| \leq c\\
 \frac{c}{|\mathcal{E}|} \quad \mbox{    if} \quad |\mathcal{E}| \geq c.
 \end{array}
 \right.\label{updateWeightedT1}
\end{equation}

The optimization of factor matrices $\mathbf{A}^{(n)}$ and the core tensor $\mathcal{S}$ are the same as RMNTF-CIM.
Here, we note that the cutoff parameter $c$ is set to the median of reconstruction errors, i.e., $c=\mathrm{median}(|\mathcal{E}|)$.

\subsection{RMNTF-Cauchy}
For any tensor $\mathcal{X}$, we define the reconstruction error of RMNTF-Cauchy as follows:
\begin{equation}\label{RMNTF_Cauchy}
\begin{split}
 \min_{\mathbf{A}^{(1)},\dots, \mathbf{A}^{(n)},\mathcal{S}} g_{\rm{Cauchy}}(\mathcal{E}) + \frac{\lambda}{2} \sum_{i=1}^{I_N} \sum_{j=1}^{I_N} \| \mathbf{a}_i^{(N)} - \mathbf{a}_j^{(N)} \|_2^2 v_{ij} \\
 \mathrm{s.t.} \quad \hat{\mathcal{X}} = \mathcal{S} \times_{n=1}^N \mathbf{A}^{(n)}, \quad \mathcal{S} \geq 0, \quad \mathbf{A}^{(n)} \geq 0,
\end{split}
\end{equation}
where $g_{\rm{Cauchy}}(x) = \ln(1+x), x\geq 0$.

As equation \eqref{solveWeightedtensor}, the Optimization of weighted tensor $\mathcal{W}$ can be represented as:
\begin{equation}\label{updateWeightedT2}
\begin{split}
 \mathcal{W}_{(n)}^\star = \frac{1}{1+\left(\frac{\mathcal{E}}{\gamma}\right)^2}.
\end{split}
\end{equation}

The optimization of factor matrices $\mathbf{A}^{(n)}$ and the core tensor $\mathcal{S}$ are the same as RMNTF-CIM.

\subsection{Discussion}
The convergence of the algorithms are guaranteed by the following theorem:
\begin{myTheo}\label{Theorem1}
Updating projection matrices $\{\mathbf{A}^{(n)}\}_{n=1}^N$, core tensor $\mathcal{S}$ and weight tensor $\mathcal{W}$ iteratively as \eqref{updateAn}, \eqref{updateAN}, \eqref{updateCoreTensor} and \eqref{updateWeightedT}. It leads to a nonincreasing of the objective function in \eqref{RMNTF_CIM}, and converges to a stationary point.
\end{myTheo}
\begin{proof}\label{Proof1}
See Appendix A.1.1 for the proof of Theorem \ref{Theorem1}.
\end{proof}

The robustness of RMNTF is guaranteed by the following Theorem:
\begin{myTheo}\label{Theorem3}
Suppose there are the training images $(\mathcal{X}^1,\mathcal{X}^2,\dots,\mathcal{X}^t)$ and test image $\mathcal{X}^{t\prime}$.
If the optimal parameters, core tensor $\mathcal{S}$, factor matrices $\{\mathbf{A}^{(n)}\}_{n=1}^N$ and weight tensor $\mathcal{W}$ are learned from training images.
The low rank representation of test image $\mathbf{a}^\prime_t$ learned by RMNTF will have close form solution:
\begin{equation}\label{RobustRMNTFsolv}
\begin{split}
  \! \mathbf{a}^\prime_t \! = \! \left[ \! \mathbf{B}^\prime\mathbf{B}^{\prime T} \! + \! \left( \! \sum_{j=1}^tv_j \! \right) \! \mathbf{I} \right]^{-1} \! \left( \! {\mathcal{X}^{t\prime}}_{(N)}\mathbf{B}^{\prime T} \! + \! \lambda \! \sum_{j=1}^t[\mathbf{A}_0^{(N)}]_{j\cdot}v_j \! \right) \!,
\end{split}
\end{equation}
where $\mathbf{B}^\prime = \sqrt{\mathcal{W}_{(N)}} \circledast \mathbf{B}^{(N)T}$, $\mathbf{B}^{(N)T} = \mathcal{S}_{(N)} \left(  \otimes_{i=1}^{N-1} \mathbf{A}^{(i)T} \right) \in \mathbb{R}_{\geq0}^{r_N \times I_1 \times \dots \times I_{N-1}}$, $\mathbf{I} \in \mathbb{R}^{r_N \times r_N}$ denotes an identity matrix, and $\mathbf{V}$ the affinity of data points.
If the test image is outlier, then the weight tensor $\mathcal{W}$ is constrained to a small value.
The low rank representation of this outlier $\mathbf{a}_t^\prime$ will be repaired through the manifold structure of the training data.
\end{myTheo}
\begin{proof}\label{Proofsolv}
See Appendix A.1.2 for the proof of Theorem \ref{RobustRMNTFsolv}.
\end{proof}

The uniqueness of RMNTF is guaranteed by the following Theorem:
\begin{myTheo}\label{Theorem2}
Since the solution of $\mathbf{A}^{(N)}$ in equation (A.39) is unique and $\mathbf{A}^{(n)}, n = 1,\dots, N-1$, can be uniquely estimated from $\hat{\mathcal{Z}}$ due to the lemma 2, the mode-N unfolding of RMNTF model $\mathcal{X}_{N} = \mathbf{A}^{(N)}\mathcal{S}_{(N)}\mathbf{Z}^{(N)T}$ has an essentially unique solution.
Because of the lemma 1, the RMNTF model $\mathcal{X} = \mathcal{S}\times_{n=1}^{N} \mathbf{A}^{(n)}$ is uniqueness.
\end{myTheo}
\begin{proof}\label{Proof2}
See Appendix A.1.3 for the proof of Theorem \ref{Theorem2}.
\end{proof}

\begin{figure}[t]
\vspace{-0.5cm} 
\setlength{\abovecaptionskip}{0cm} 
\setlength{\belowcaptionskip}{-0cm} 
\centering
\subfloat[]{\includegraphics[width=1.25in]{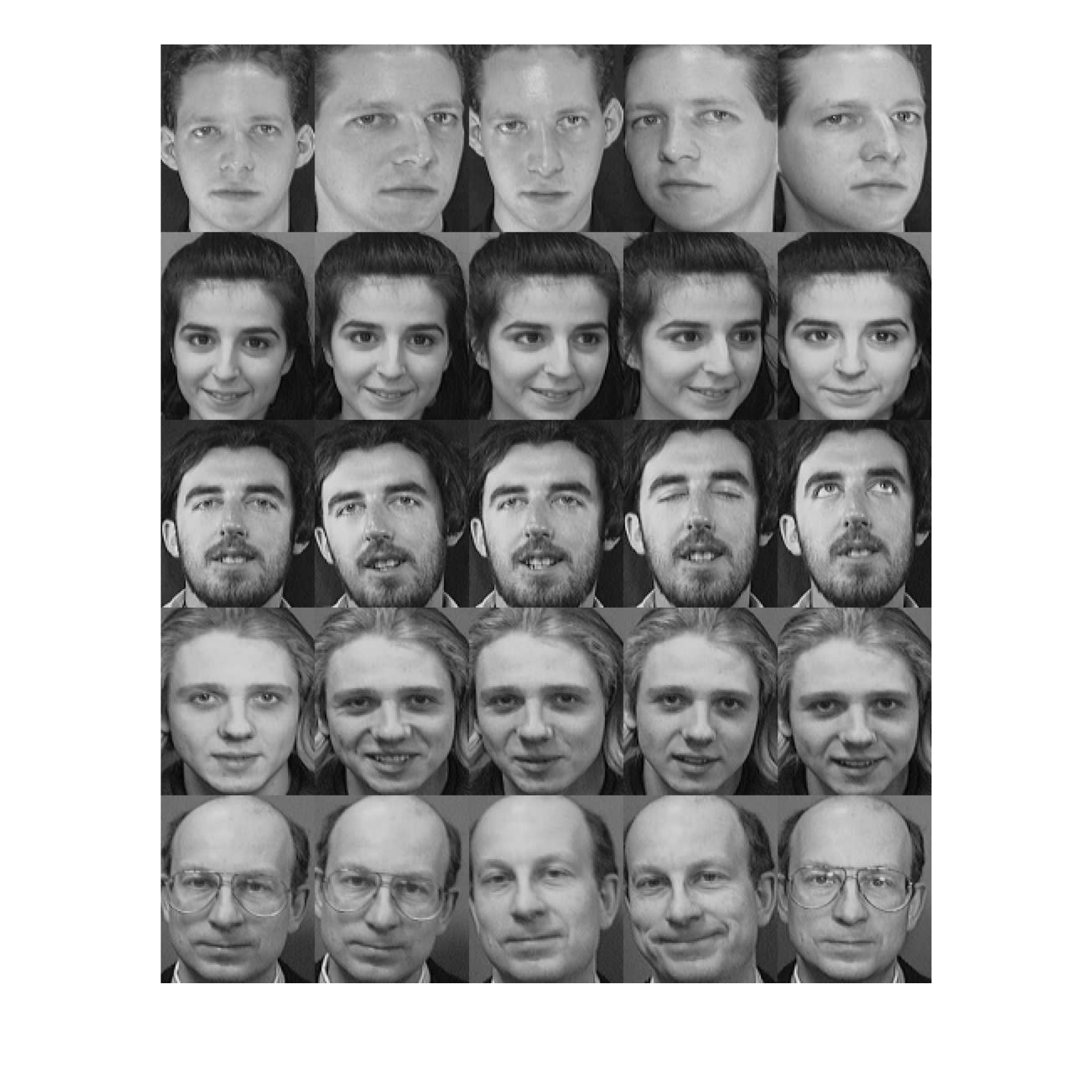}%
\label{Reuters_delta_Coh}}\hspace{-5mm}
\hfil
\subfloat[]{\includegraphics[width=1.25in]{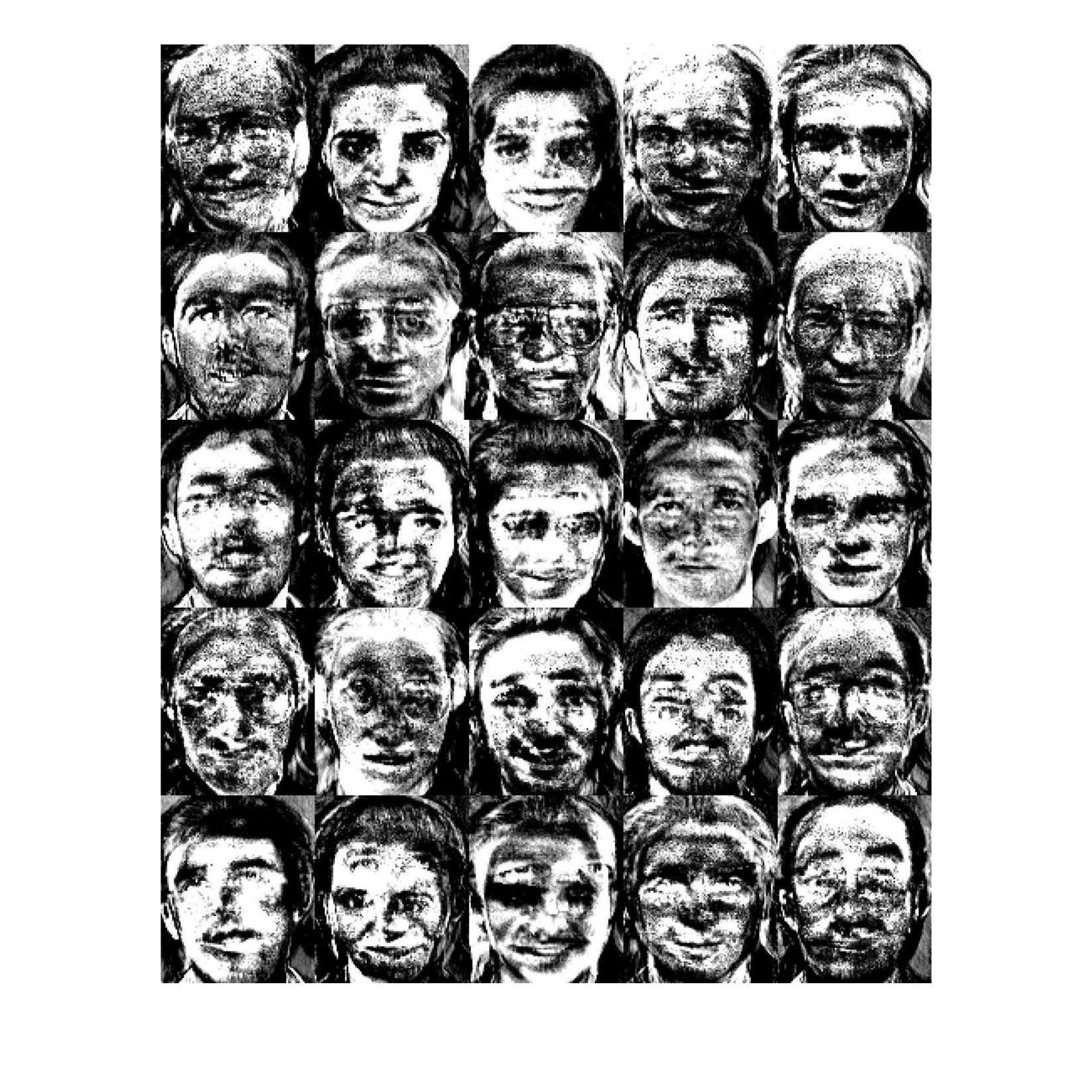}%
\label{Reuters_delta_SC}}\hspace{-5mm}
\hfil
\subfloat[]{\includegraphics[width=1.25in]{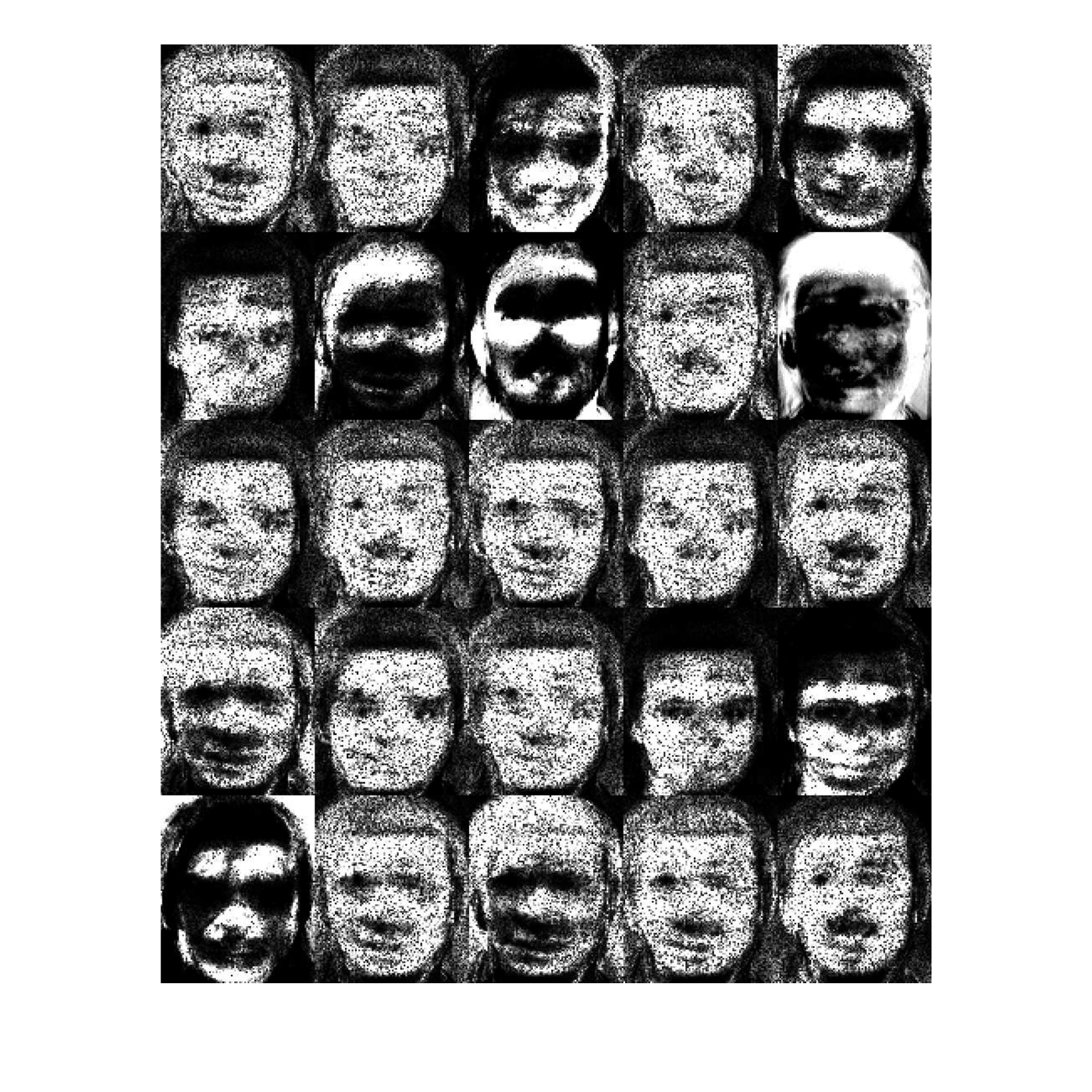}%
\label{Reuters_delta_SC}}\hspace{-5mm}
\hfil
\vspace{-0.3cm}

\subfloat[]{\includegraphics[width=1.25in]{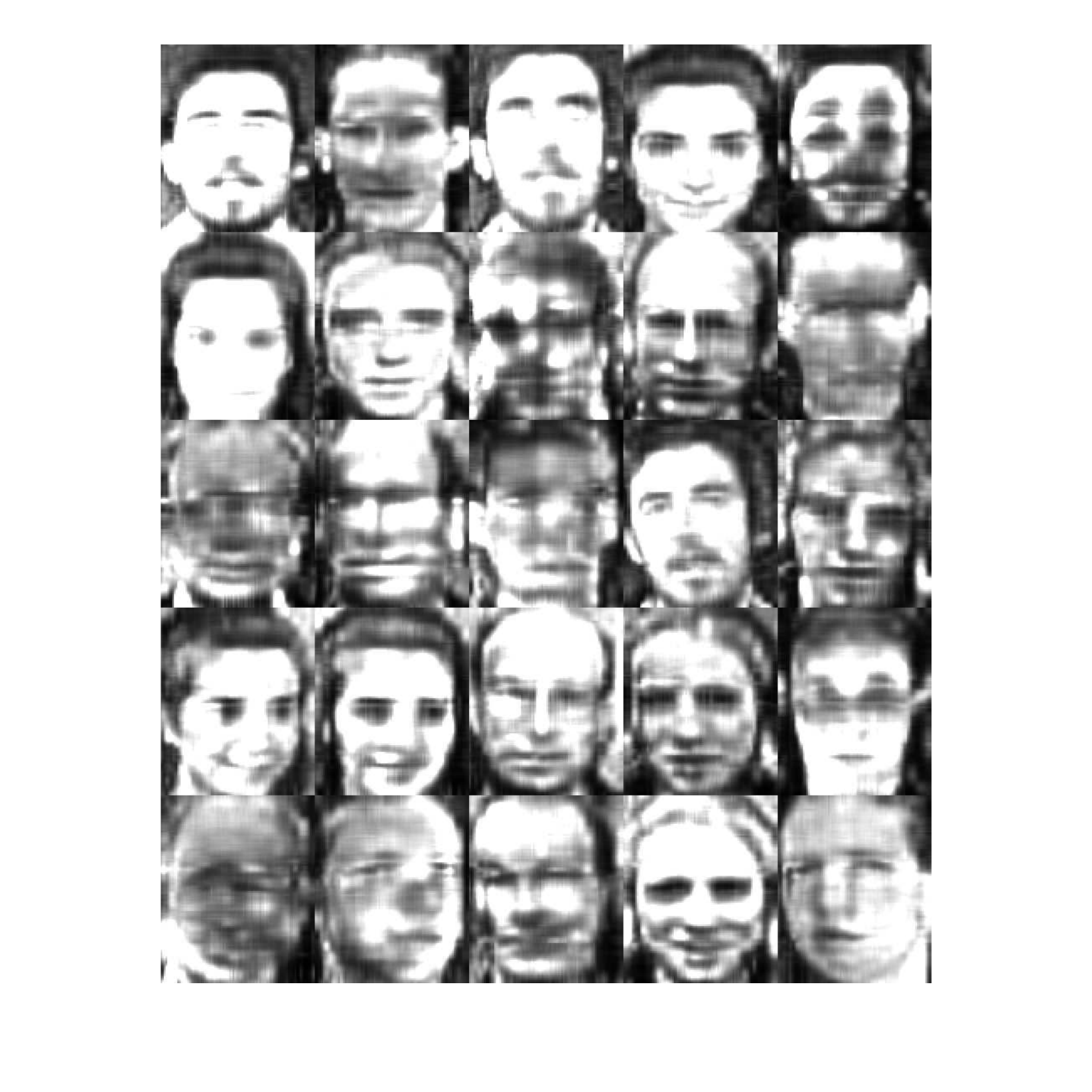}%
\label{Reuters_delta_SC}}\hspace{-5mm}
\hfil
\subfloat[]{\includegraphics[width=1.25in]{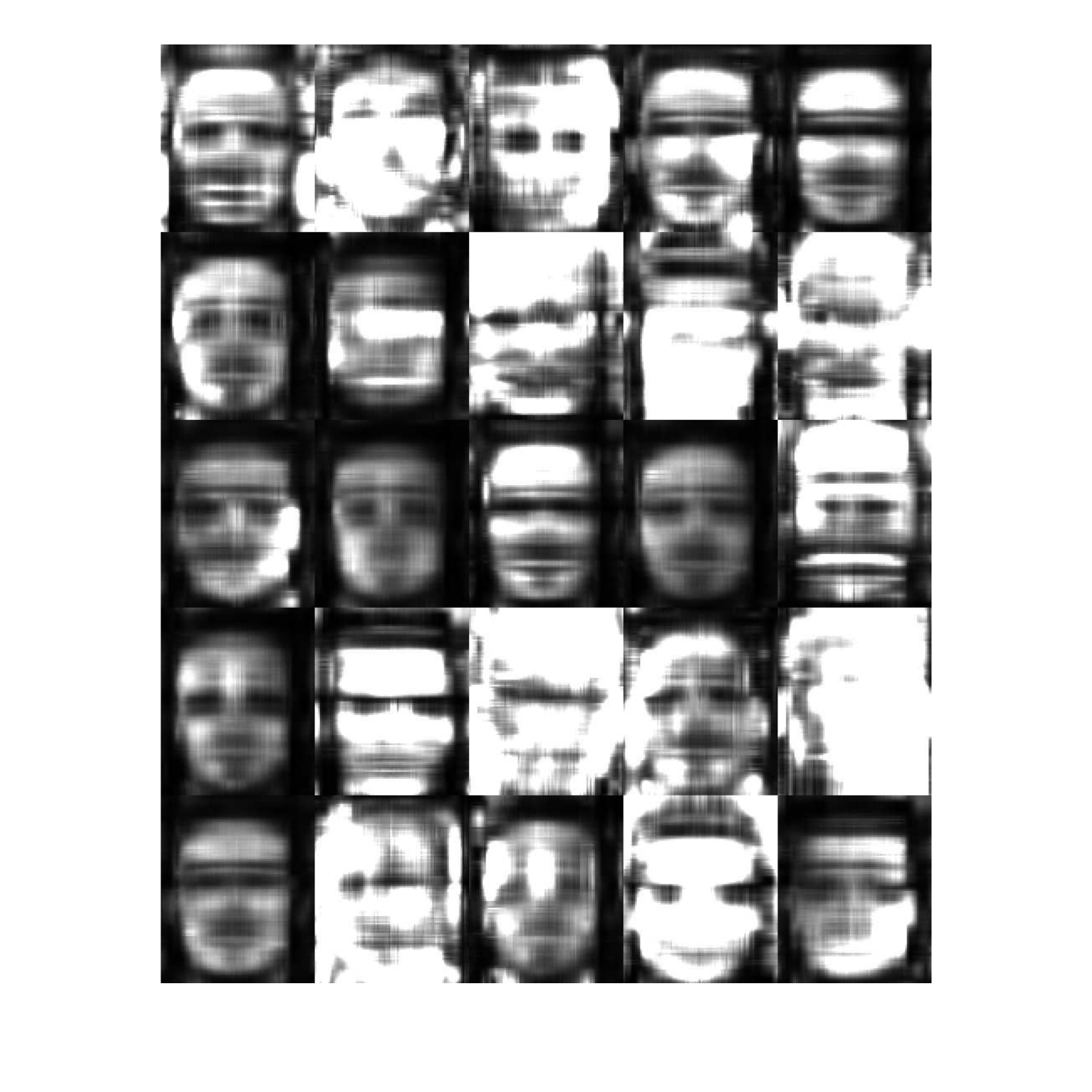}%
\label{Reuters_delta_SC}}\hspace{-5mm}
\hfil
\subfloat[]{\includegraphics[width=1.25in]{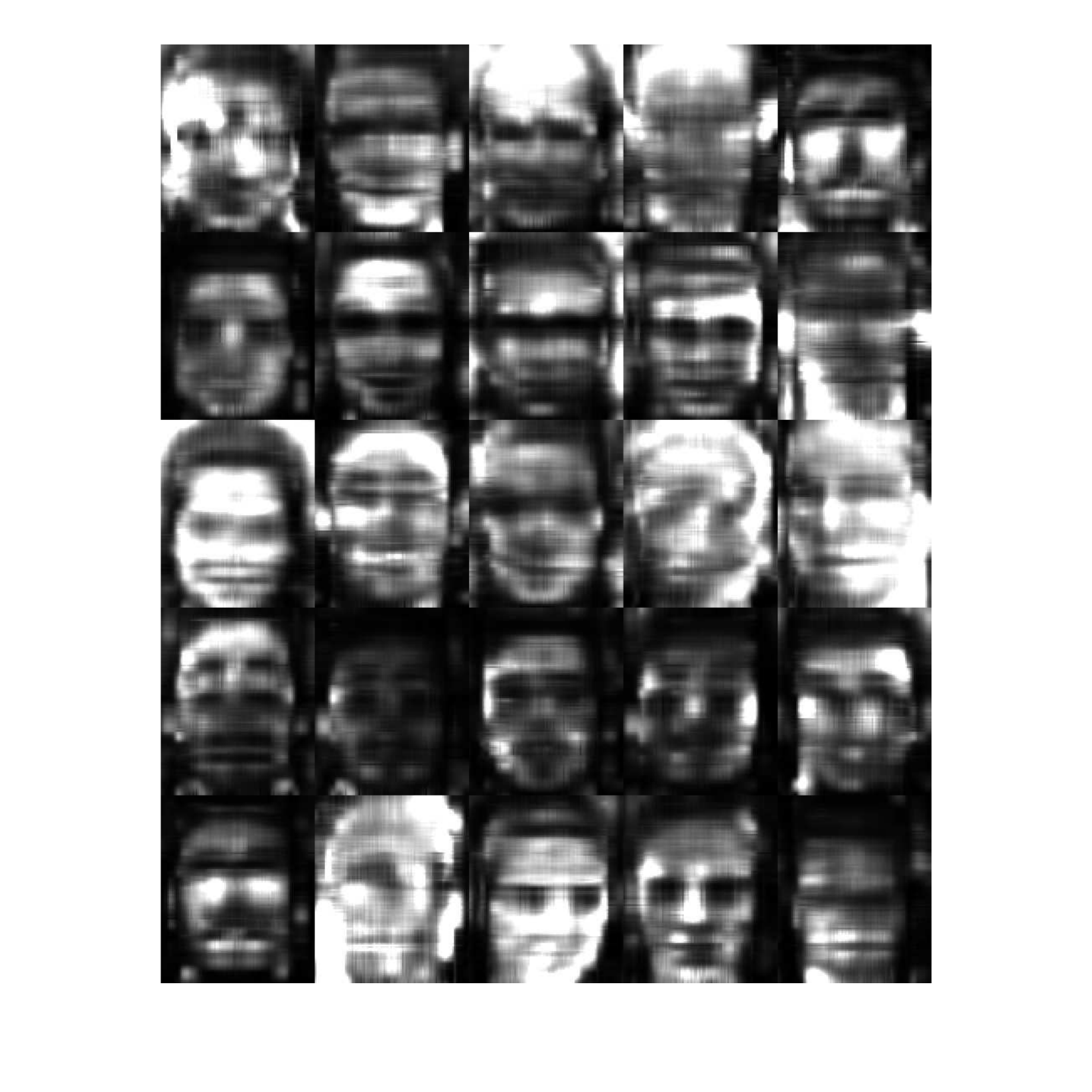}%
\label{Reuters_delta_SC}}\hspace{-5mm}
\hfil
\vspace{2mm}
\caption{Basic images learned by NMF, GNMF, NTD, GNTD, RMNTF from randomly chosen 25 subjects on the AT\&T ORL dataset. (a) Original images. (b) NMF. (c) GNMF. (d) NTD. (e) GNTD. (f) RMNTF$_\text{CIM}$. Each image denotes a basis vector learned by the above-mentioned methods.}
\label{fig:Reuters_delta1}
\end{figure}

\section{Experiments}
\label{Exp}

In this section, we compare the proposed RMNTF-CIM, RMNTF-Huber, RMNTF-Cauchy with nine nonnegative matrix and tensor factorization methods on five image data sets.

\subsection{Datasets}
We conducted experiments on the COIL100, USPS, FEI, ORL and FERET image datasets.

COIL100\footnote{\url{https://www.kaggle.com/jessicali9530/coil100}} is an object categorization image database including 100 classes of objective, each of which contains $72$ images with $128\times 128$ of difference observation angle. For reprocessing, we resize each image into $64 \times 64$ with the RGB representation by nearest neighbor interpolation algorithm. As a result, each object is represented as a $64 \times 64 \times 3$ tensor. In total, we have 7200 tensor objects.

USPS Dataset\footnote{\url{https://www.kaggle.com/bistaumanga/usps-dataset}} consists of $11000$ images of handwritten digits $0\sim 9$ with $16 \times 16$. Each of handwritten digits contains $1100$ images.

FEI Part 1 Dataset\footnote{\url{https://fei.edu.br/~cet/facedatabase.html}}
is the subset of FEI database, which consists of 700 color images of size $480 \times 640$ collected from 50 individuals. Each individual has 14 different images under different observed and facial expressions. In our experiment, the images are resized to $48 \times 64$. These image finally construct a fourth-order $48\times 64\times 3\times 700$ tensor.

AT \& T ORL Dataset\footnote{\url{https://www.cl.cam.ac.uk/research/dtg/attarchive/facedatabase.html}}: is the subset of FEI database, which consists of 700 color images of size $480 \times 640$ collected from 50 individuals. Each individual has 14 different images under different observed and facial expressions. In our experiment, the images are resized to $48 \times 64$ pixels. These image finally construct a fourth-order $48\times 64\times 3\times 700$ tensor.

FERET Dataset\footnote{\url{https://www.nist.gov/itl/products-and-services/color-feret-database}}: This dataset \cite{phillips2000feret} collects $14125$ grayscale face images acquired from $1199$ subjects. It is widely used for evaluating face recognition and clustering problem. In our experiment, we use a subset of $1400$ images with a size of $80 \times 80$ from $200$ subjects where each subject contains $7$ images with varying postures, genders, lighting condition, shooting directions and race. In total, we stack the images into a third-order tensor with size of $80\times 80\times 1400$.

For COIL100, we randomly selected $5,10,20,30,50$ categories from the whole COIL100 for evaluation.
For USPS, we randomly selected $3,5,7$ categories and the whole $10$ categories for the evaluation.
For FEI, ORL and FERET data sets, we randomly selected $5,10,15,20$ categories for the evaluation.
For each comparison, we reported the average results over $50$ Monte-Carlo runs.

\begin{figure*}[t]
\setlength{\abovecaptionskip}{0cm} 
\setlength{\belowcaptionskip}{-0cm} 
\centering
\subfloat[]{\includegraphics[width=1.25in]{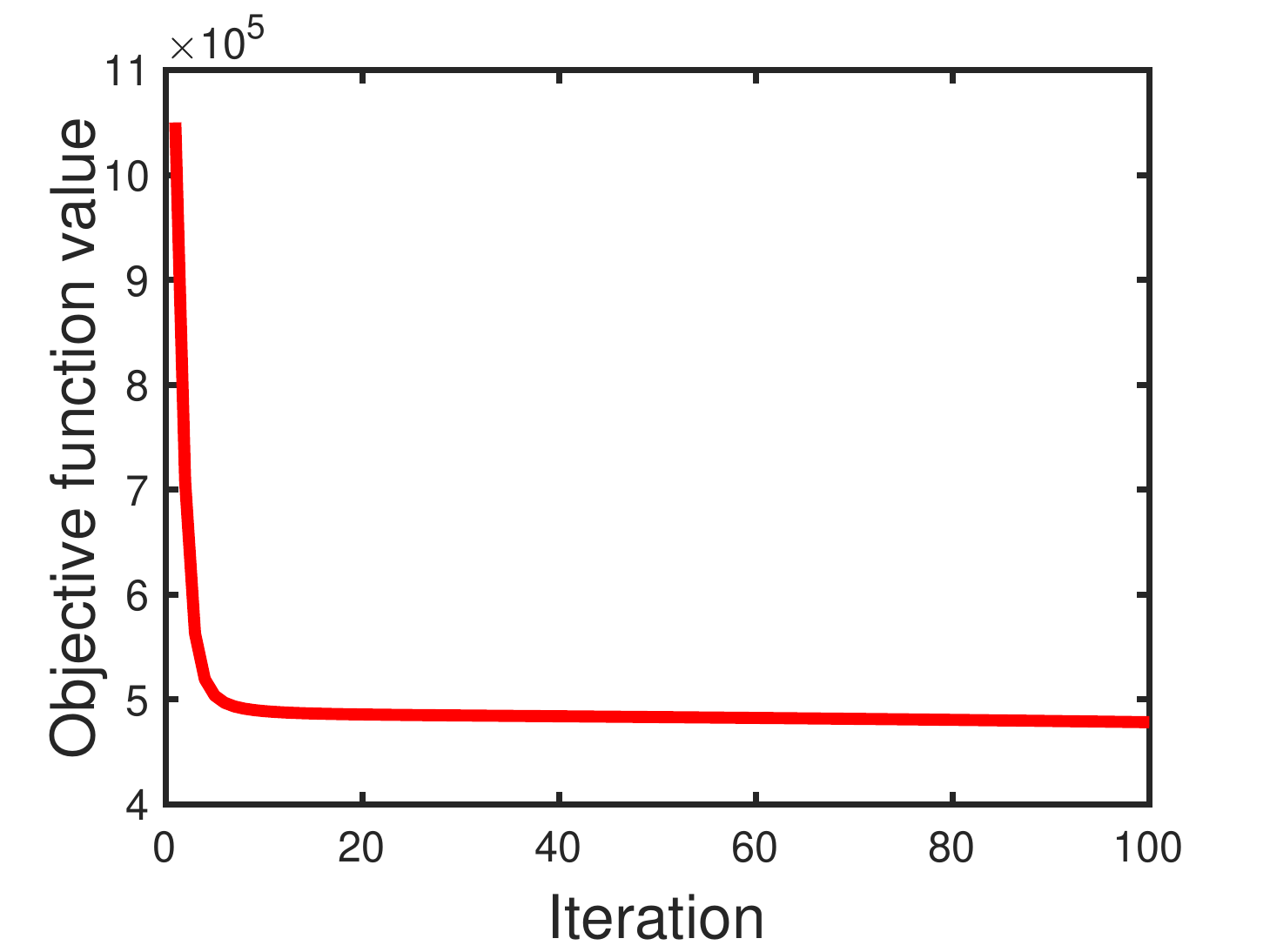}%
\label{Reuters_delta_Coh}}\hspace{-5mm}
\hfil
\subfloat[]{\includegraphics[width=1.25in]{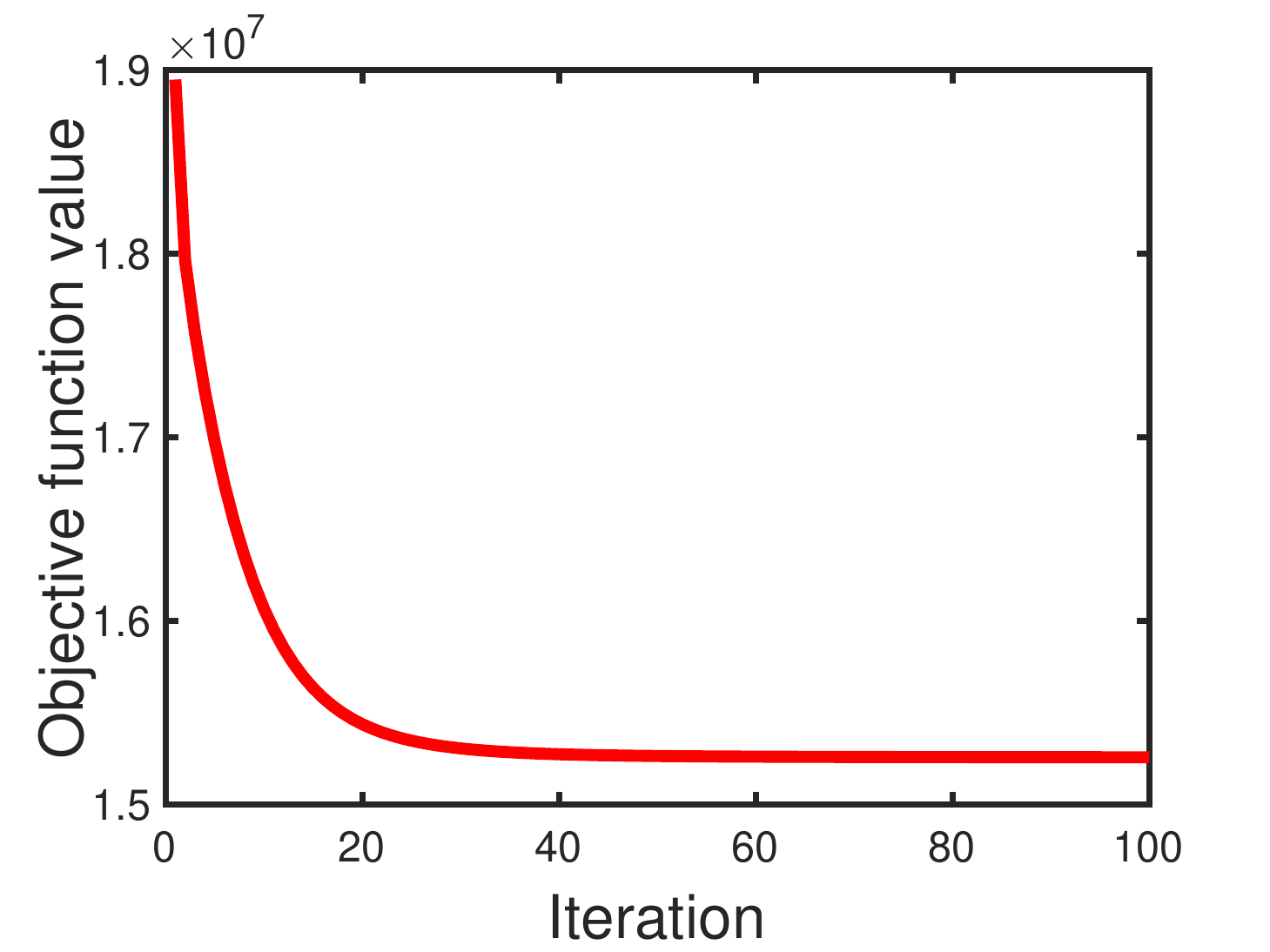}%
\label{Reuters_delta_SC}}\hspace{-5mm}
\hfil
\subfloat[]{\includegraphics[width=1.25in]{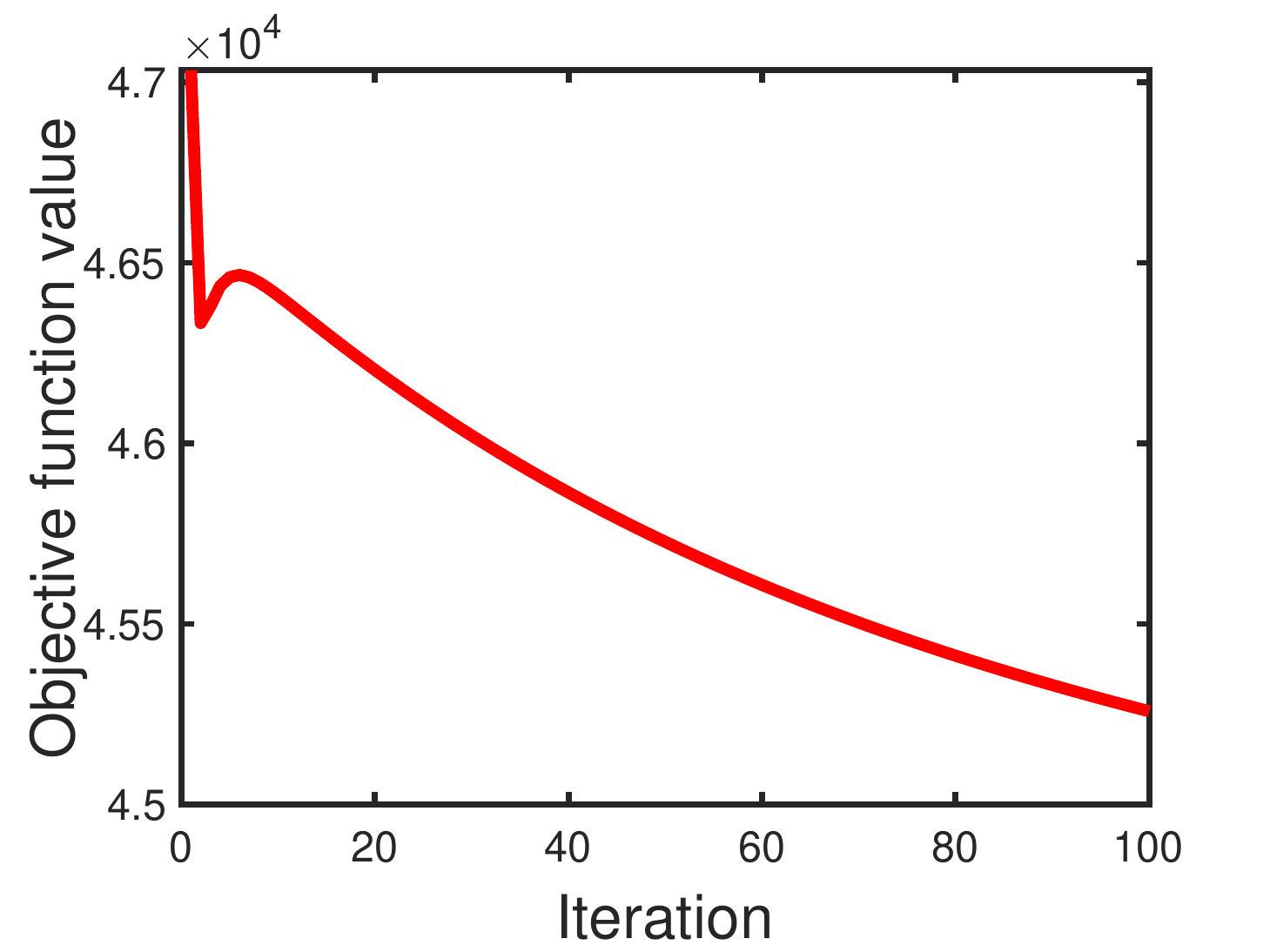}%
\label{Reuters_delta_SC}}\hspace{-5mm}
\hfil
\subfloat[]{\includegraphics[width=1.25in]{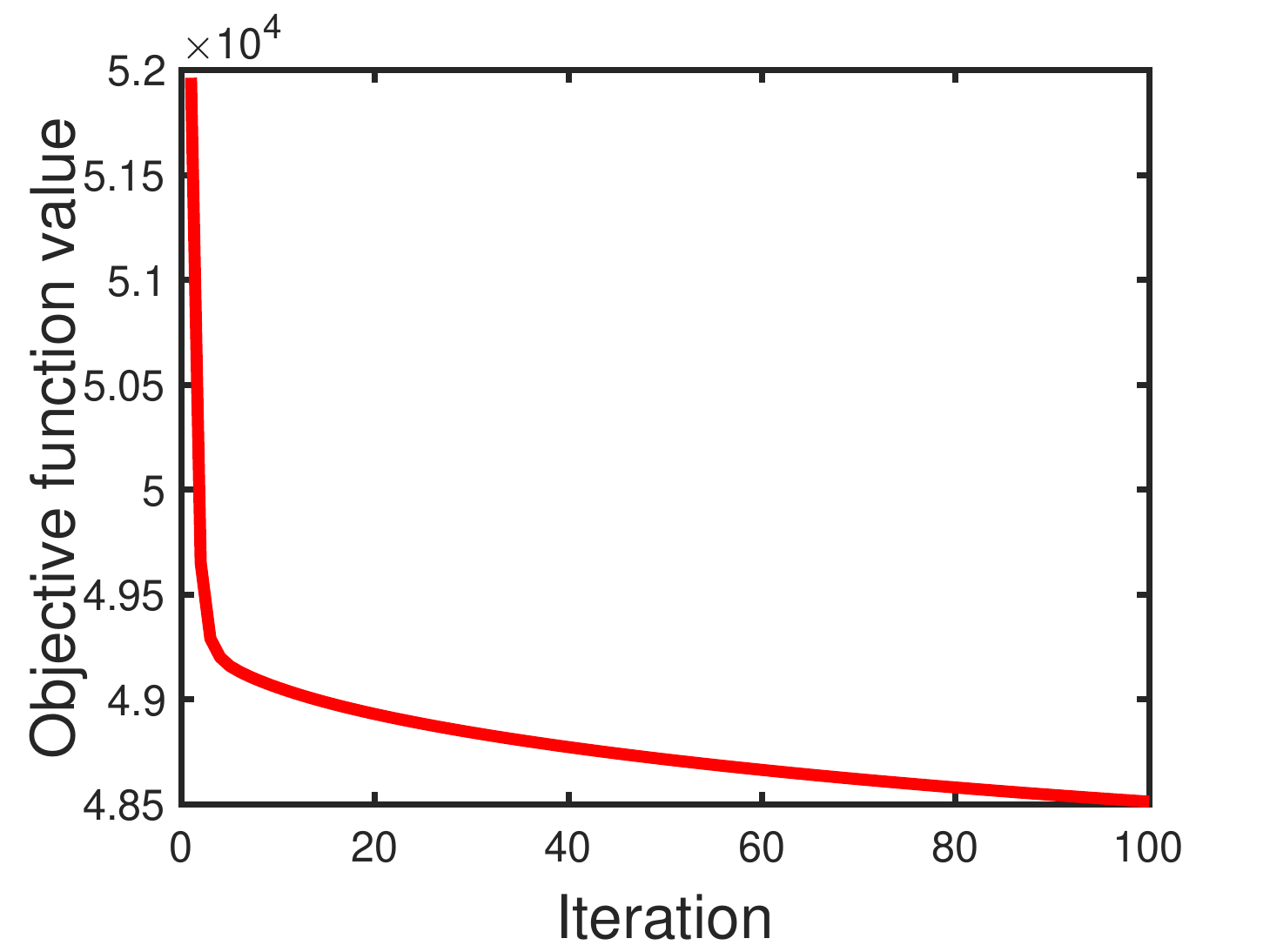}%
\label{Reuters_delta_SC}}\hspace{-5mm}
\hfil
\subfloat[]{\includegraphics[width=1.25in]{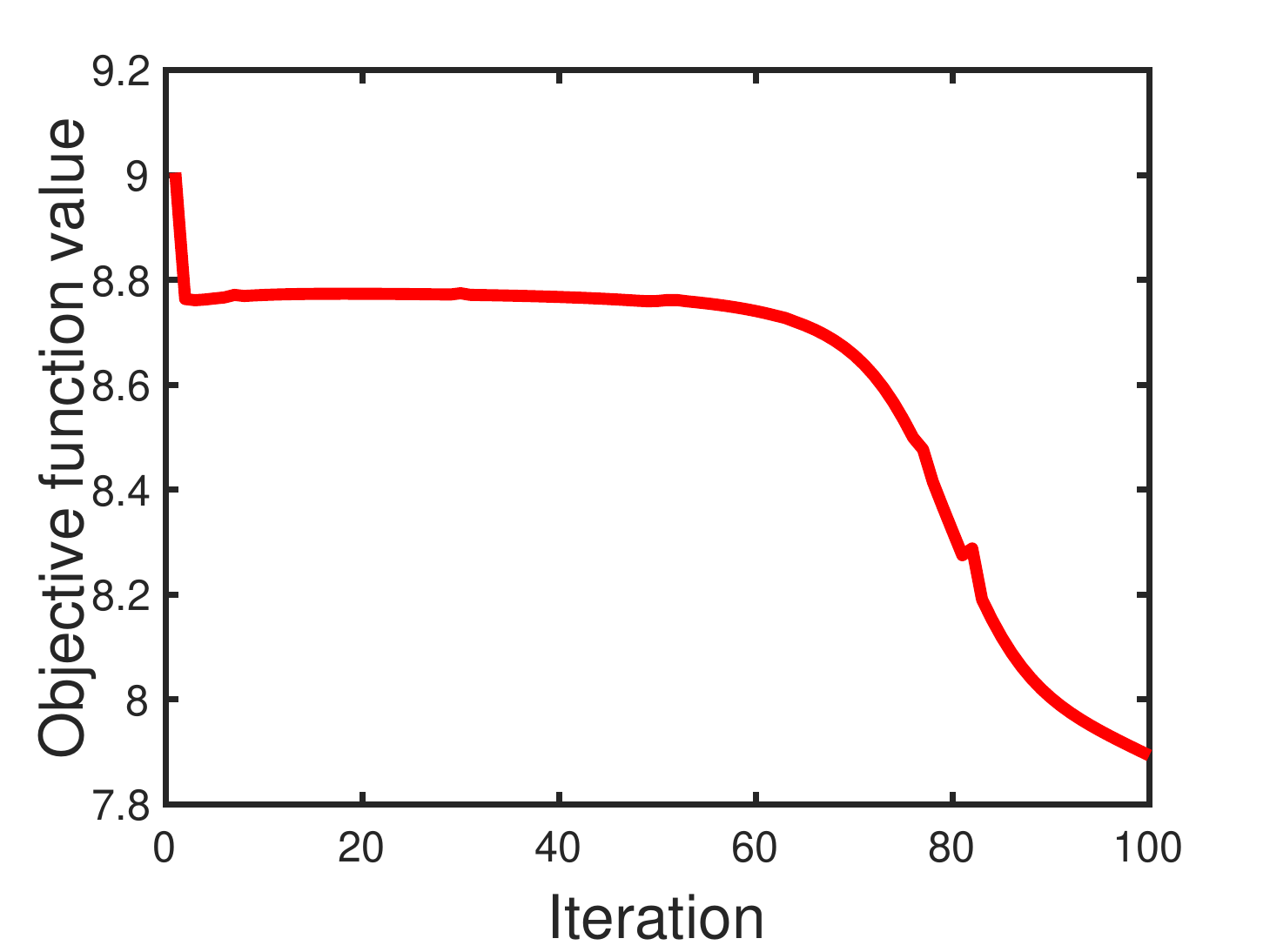}%
\label{Reuters_delta_SC}}\hspace{-5mm}
\hfil

\vspace{-3mm}
\subfloat[]{\includegraphics[width=1.25in]{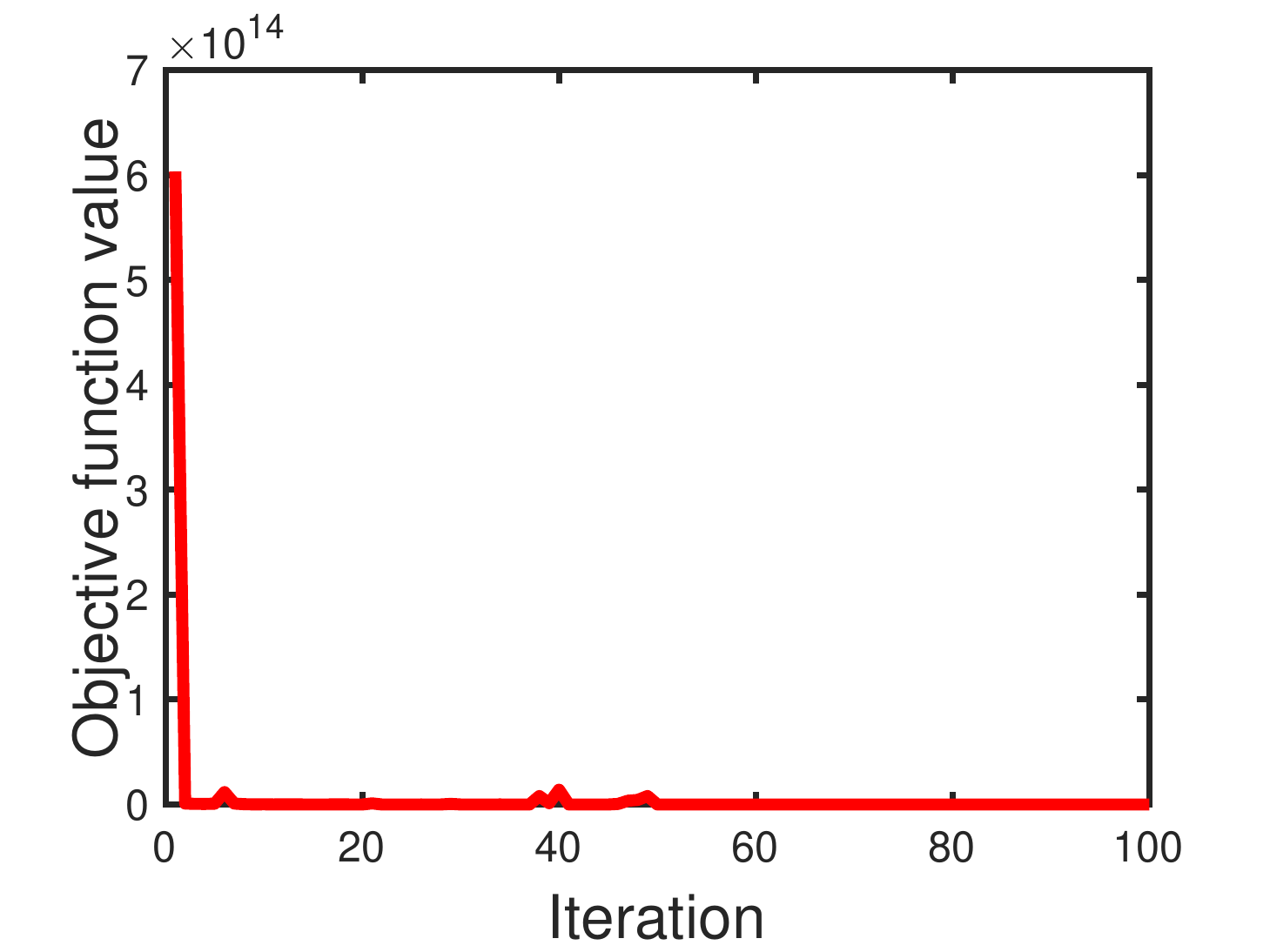}%
\label{Reuters_delta_Coh}}\hspace{-5mm}
\hfil
\subfloat[]{\includegraphics[width=1.25in]{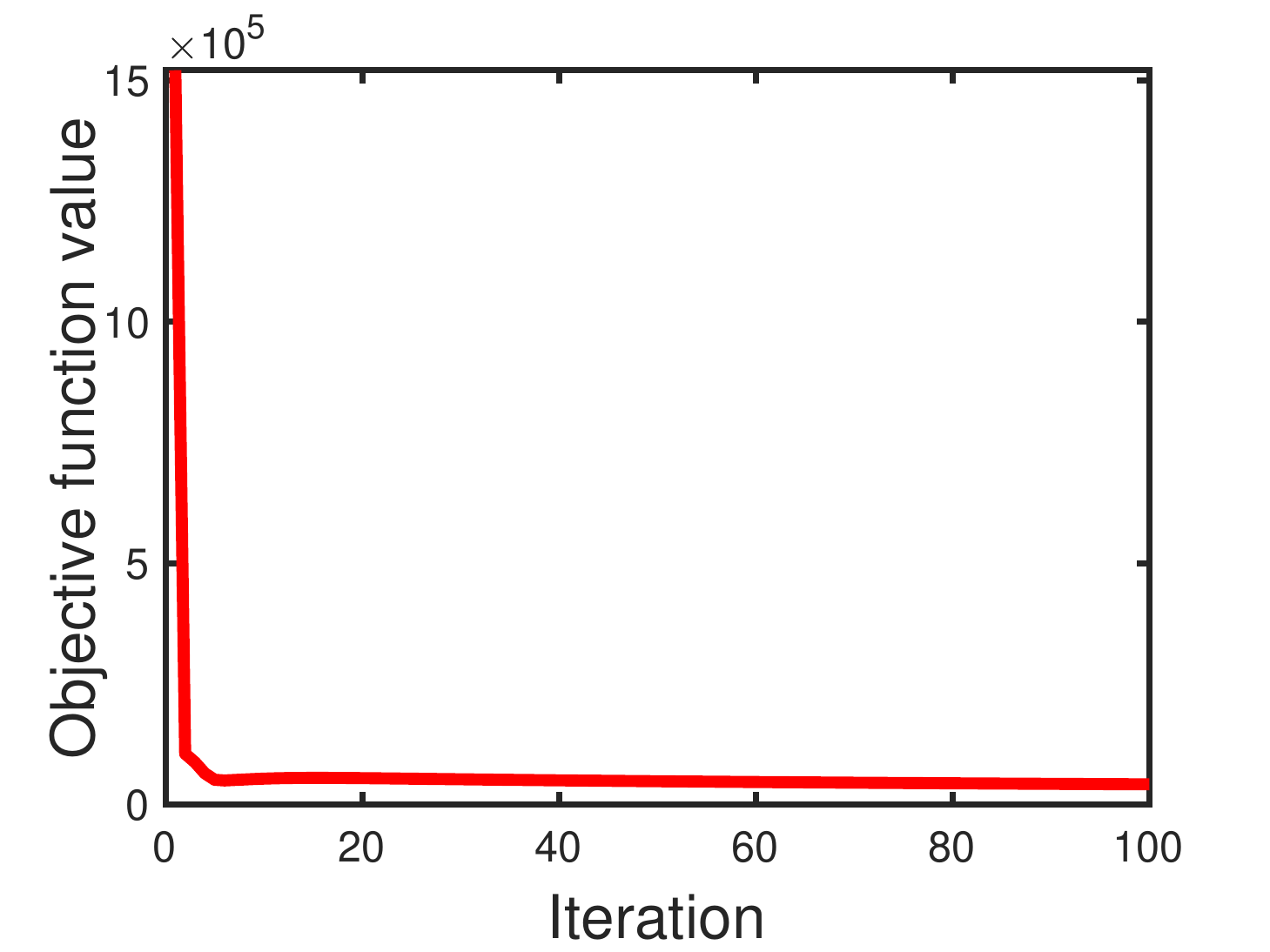}%
\label{Reuters_delta_SC}}\hspace{-5mm}
\hfil
\subfloat[]{\includegraphics[width=1.25in]{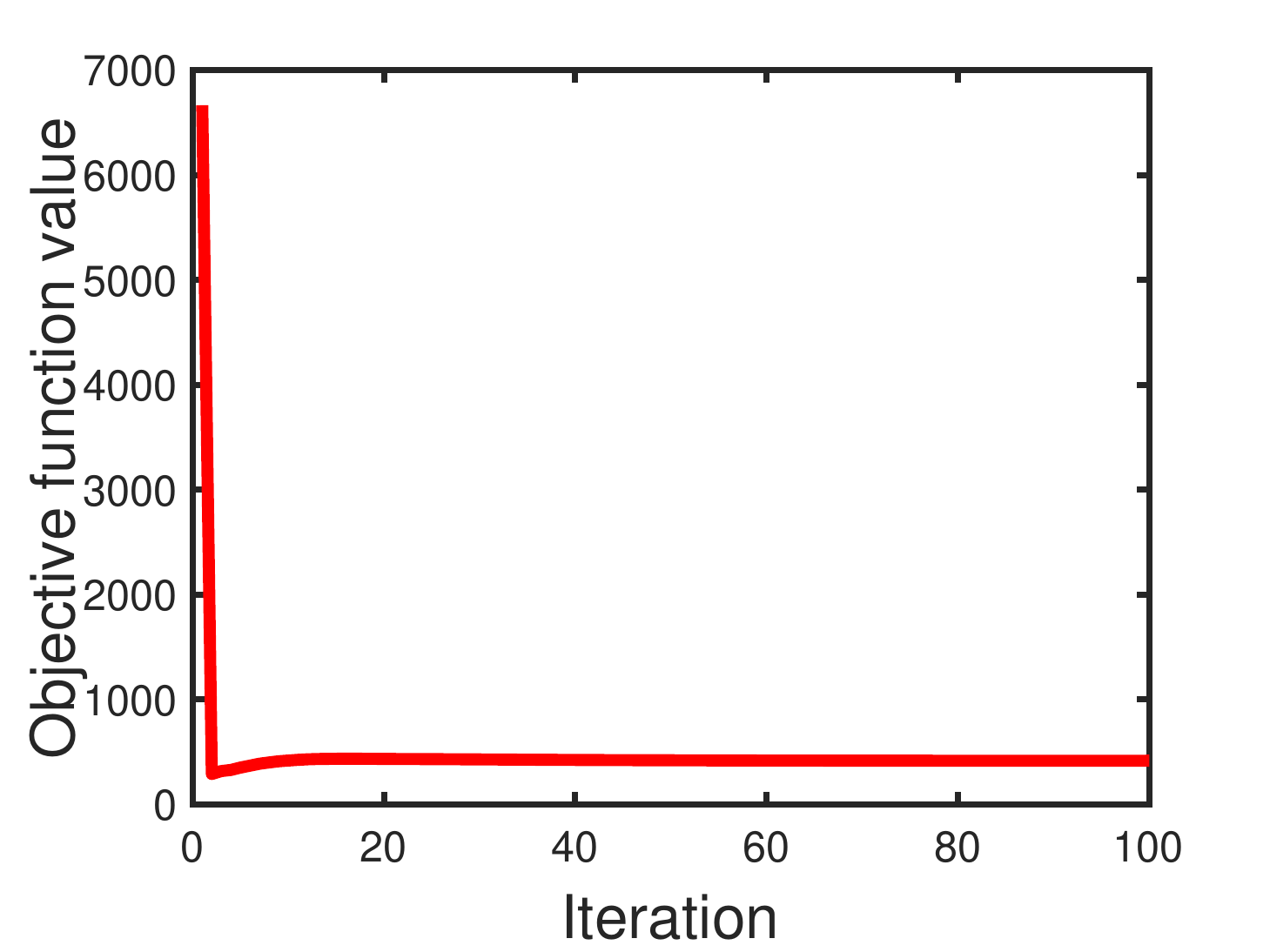}%
\label{Reuters_delta_SC}}\hspace{-5mm}
\hfil
\subfloat[]{\includegraphics[width=1.25in]{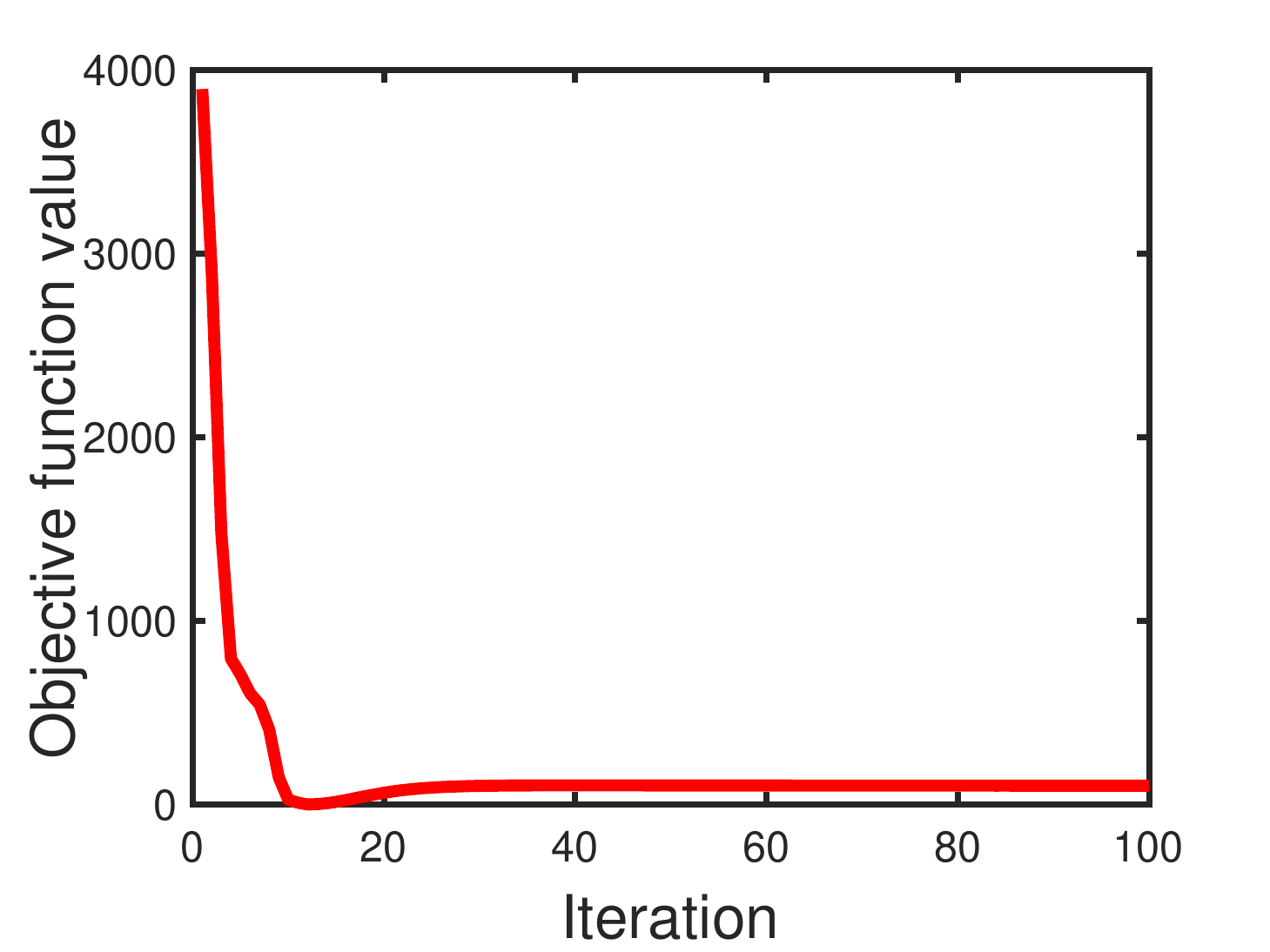}%
\label{Reuters_delta_SC}}\hspace{-5mm}
\hfil
\subfloat[]{\includegraphics[width=1.25in]{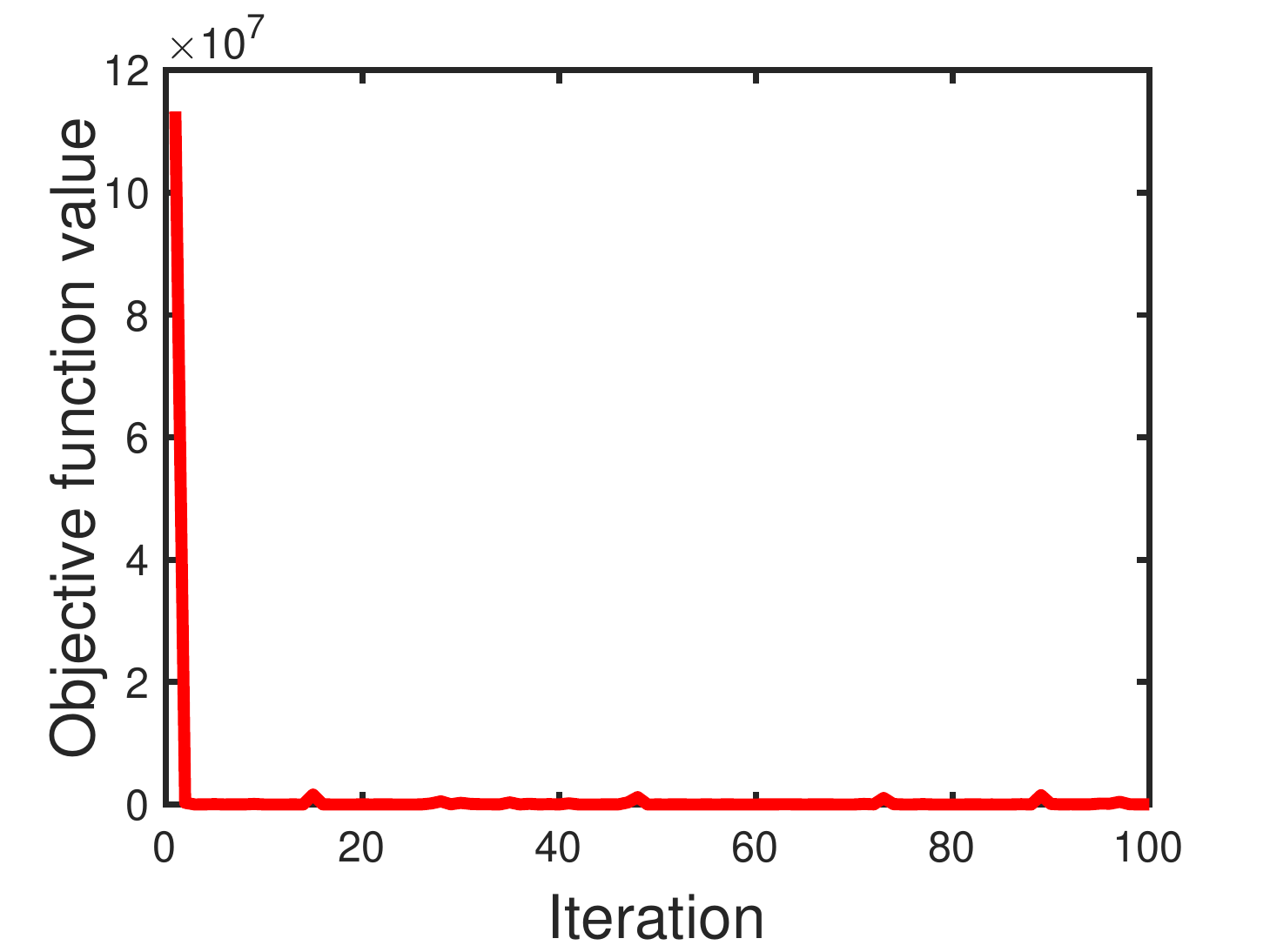}%
\label{Reuters_delta_SC}}\hspace{-5mm}
\hfil

\vspace{-3mm}
\subfloat[]{\includegraphics[width=1.25in]{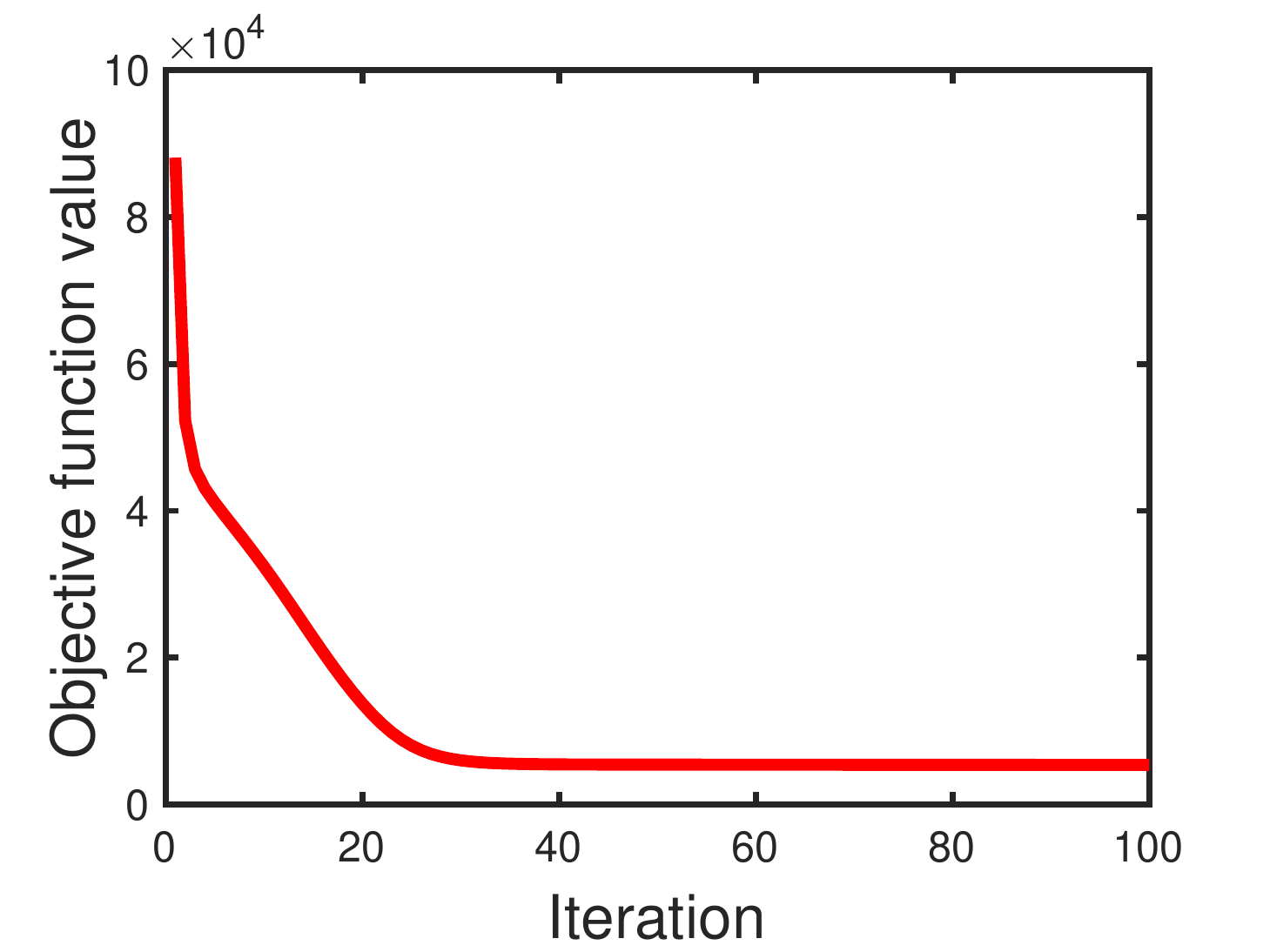}%
\label{Reuters_delta_Coh}}\hspace{-5mm}
\hfil
\subfloat[]{\includegraphics[width=1.25in]{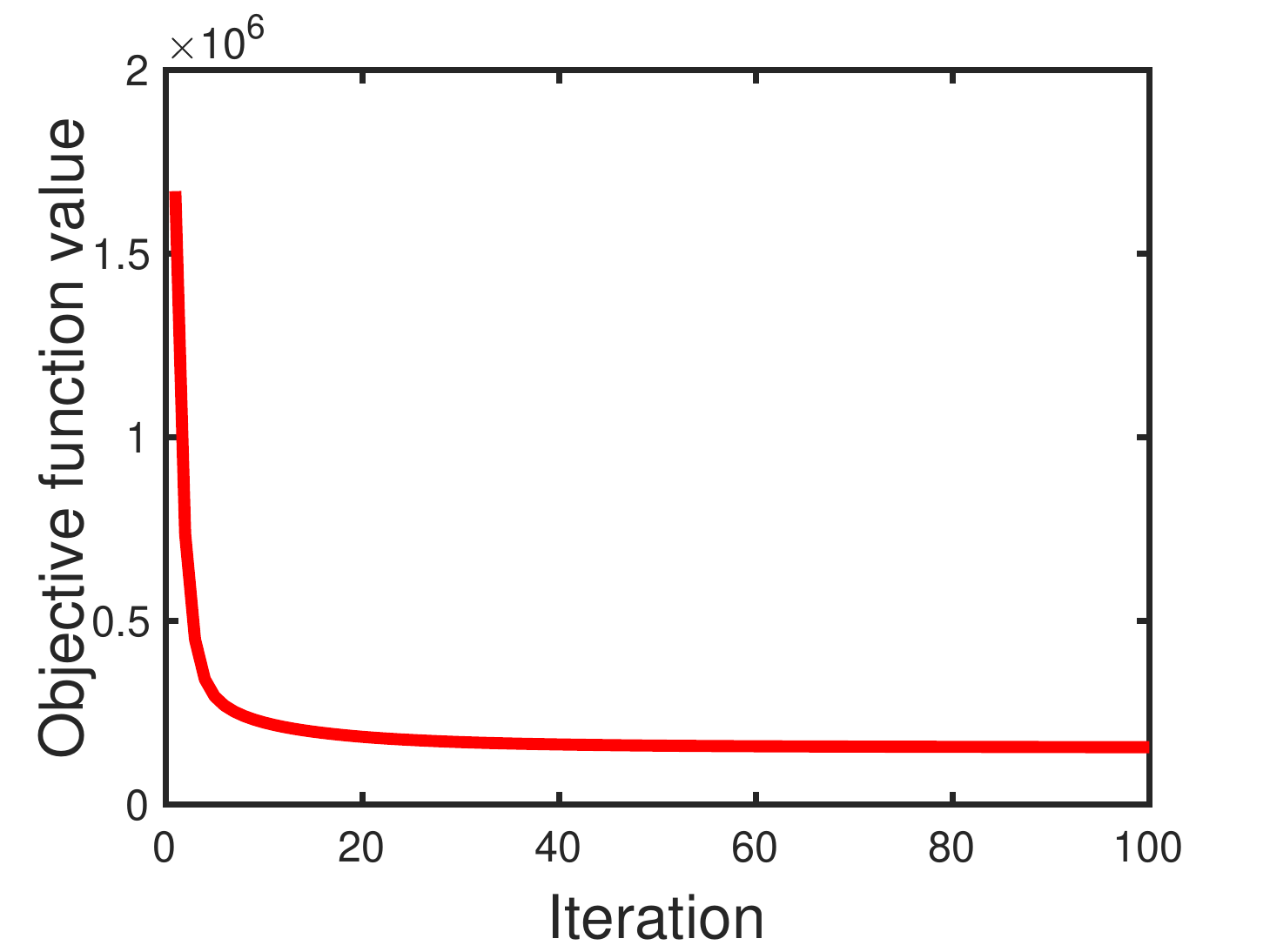}%
\label{Reuters_delta_SC}}\hspace{-5mm}
\hfil
\subfloat[]{\includegraphics[width=1.25in]{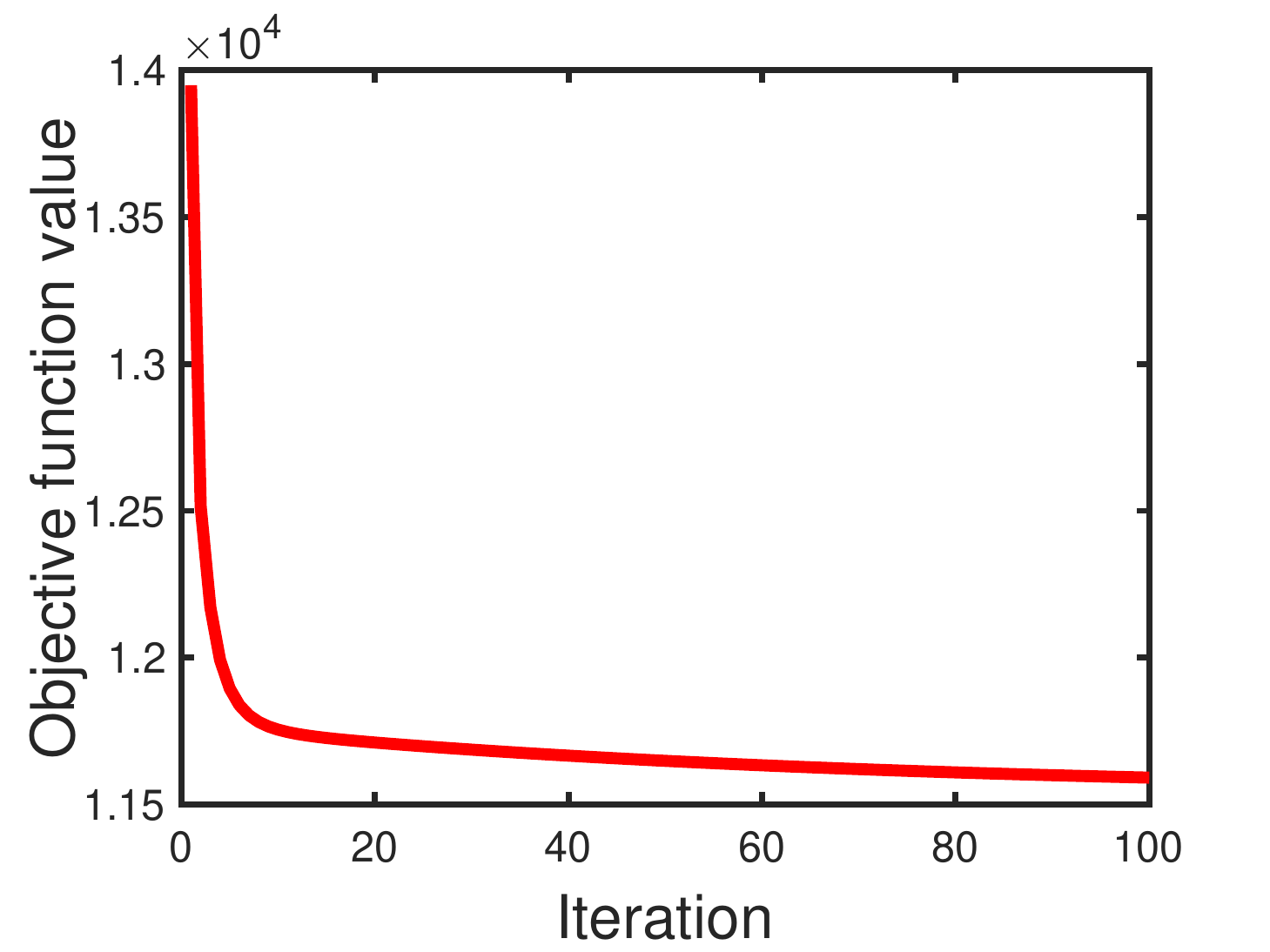}%
\label{Reuters_delta_SC}}\hspace{-5mm}
\hfil
\subfloat[]{\includegraphics[width=1.25in]{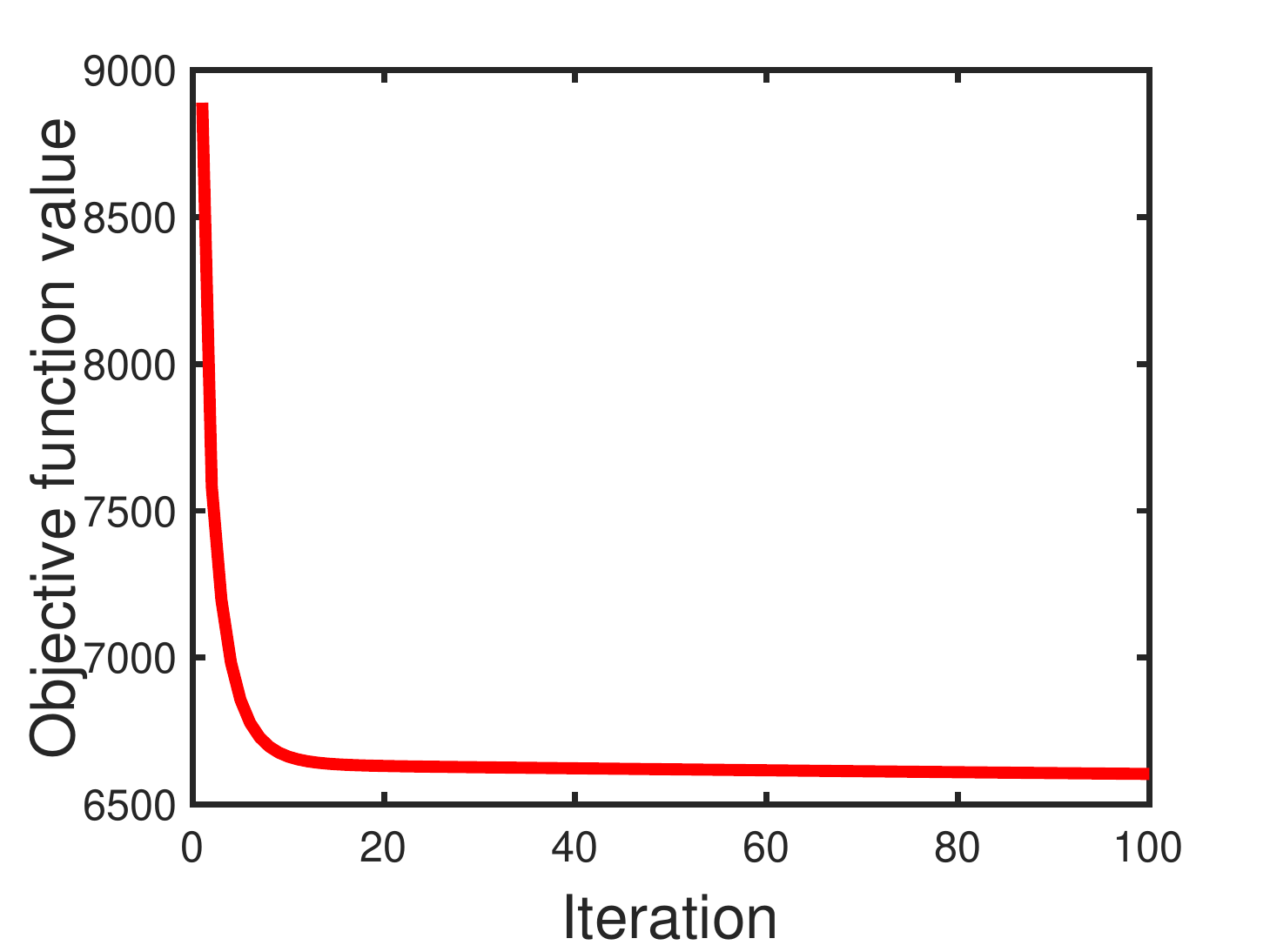}%
\label{Reuters_delta_SC}}\hspace{-5mm}
\hfil
\subfloat[]{\includegraphics[width=1.25in]{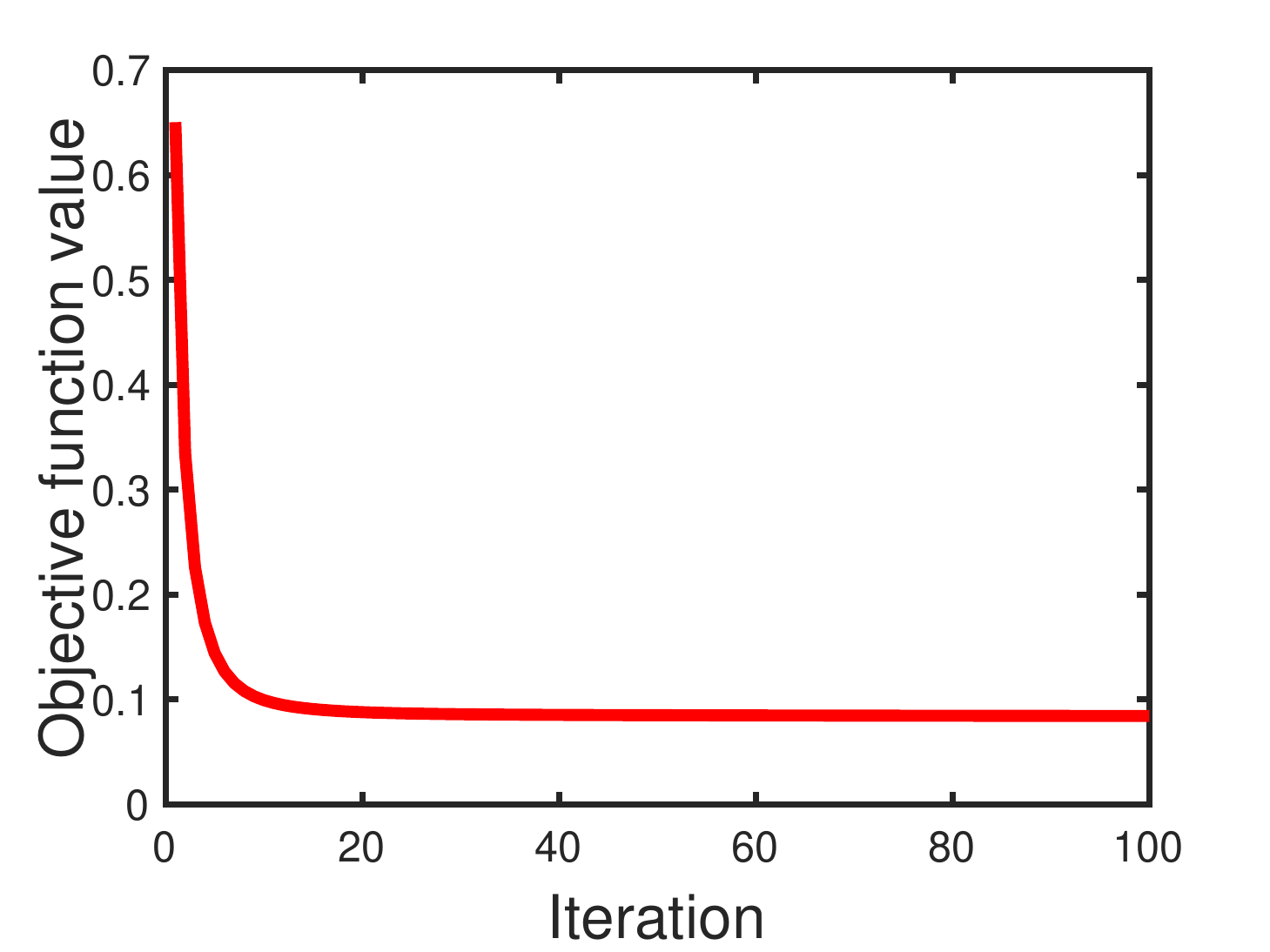}%
\label{Reuters_delta_SC}}\hspace{-5mm}
\hfil

\vspace{2mm}
\caption{The convergence of the proposed RMNTF. (a)$\sim$(e) are the convergence curves of RMNTF-CIM on five datasets: COIL100, FEI, FERET, ORL and USPS, respectively. (f)$\sim$(j) are the convergence curves of RMNTF-Huber on five datasets, respectively. (k)$\sim$(o) are the convergence curves of RMNTF-Cauchy on five datasets, respectively.}
\label{fig:Reuters_delta2}
\end{figure*}

\begin{figure*}[t]
\setlength{\abovecaptionskip}{0cm} 
\setlength{\belowcaptionskip}{-0cm} 
\centering
\subfloat[]{
\begin{minipage}[b]{0.22\textwidth}
\includegraphics[width=1.25in]{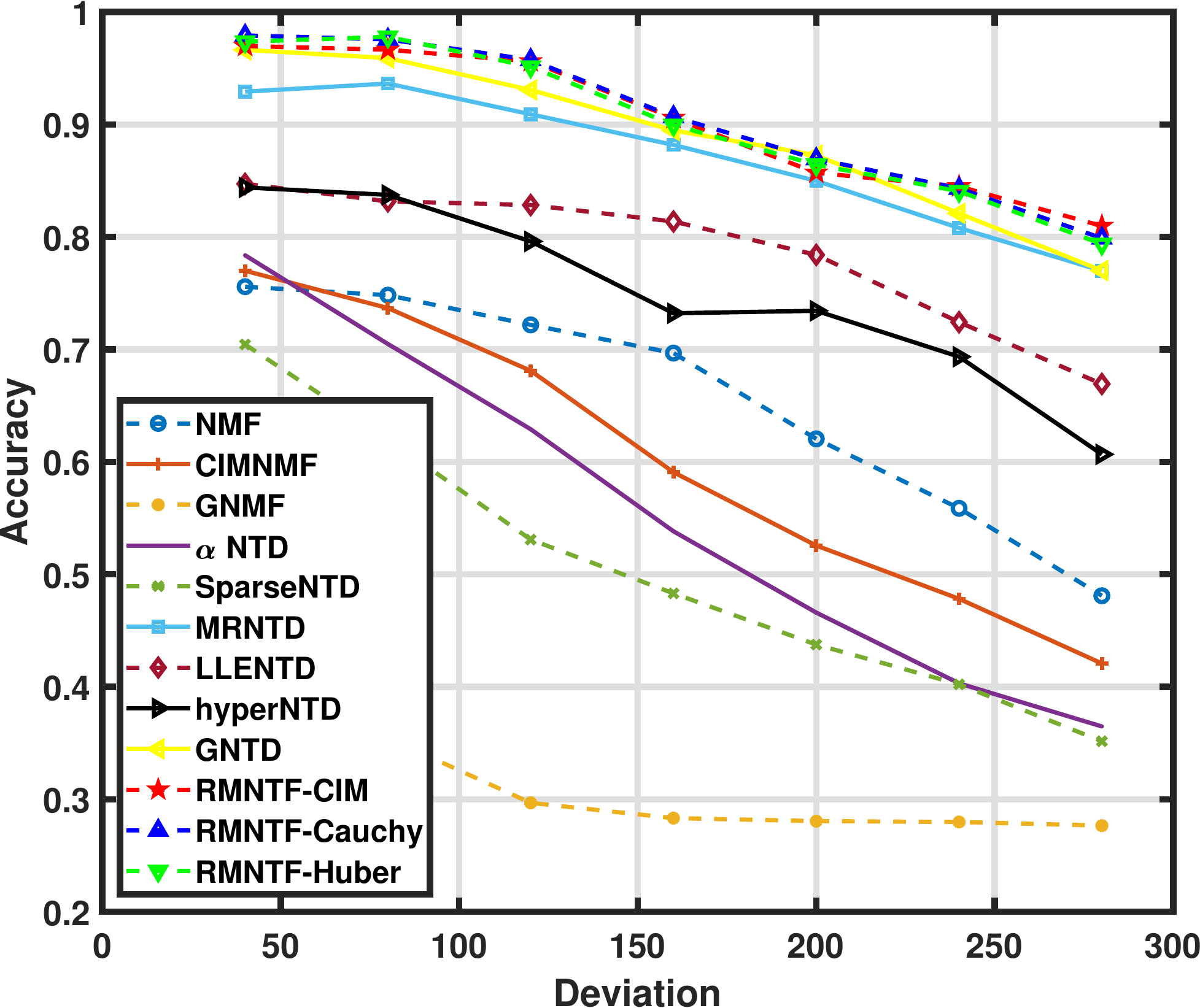}%
\label{Reuters_delta_Coh} \\
\includegraphics[width=1.25in]{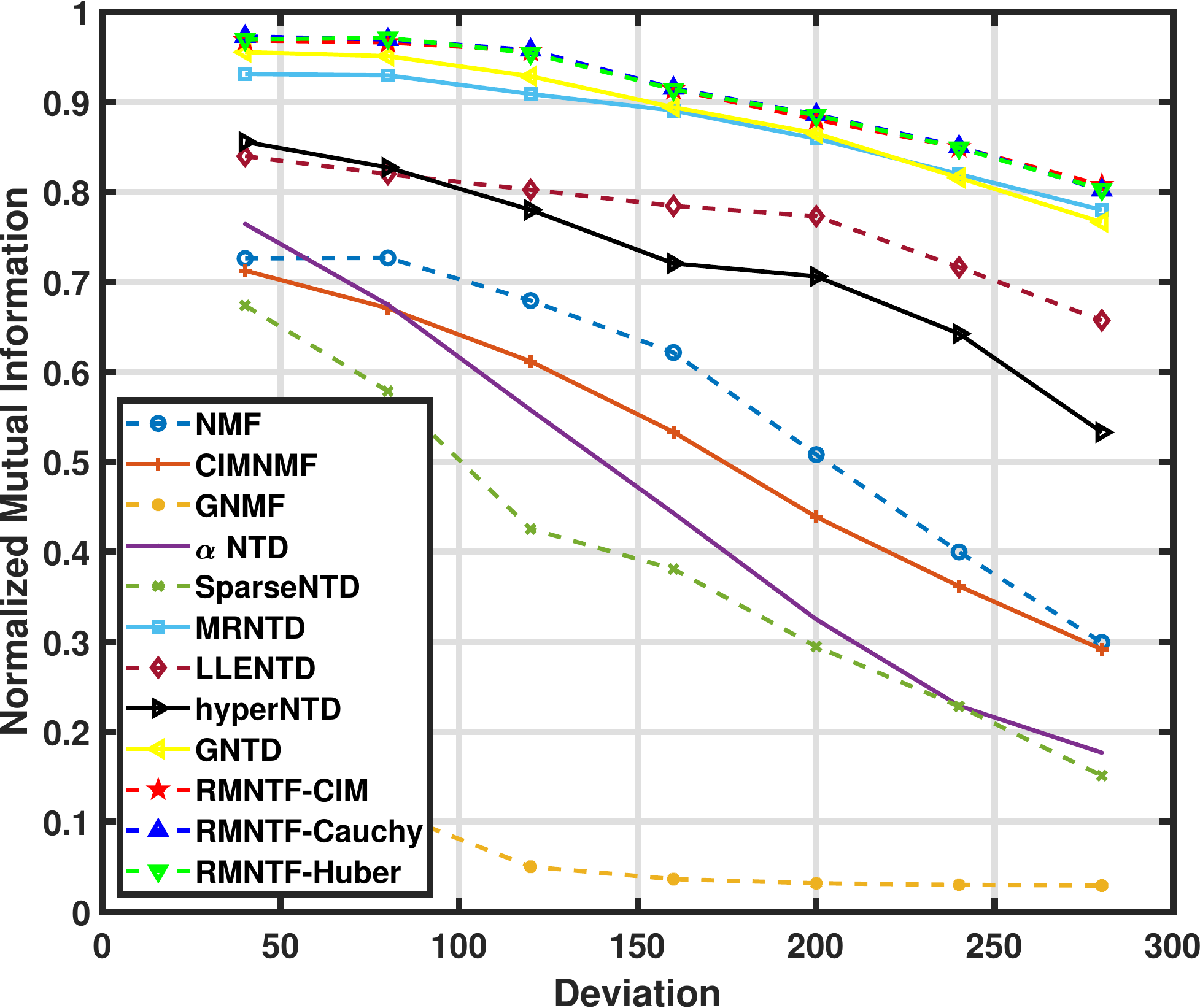}%
\label{Reuters_delta_Coh} \\
\includegraphics[width=1.25in]{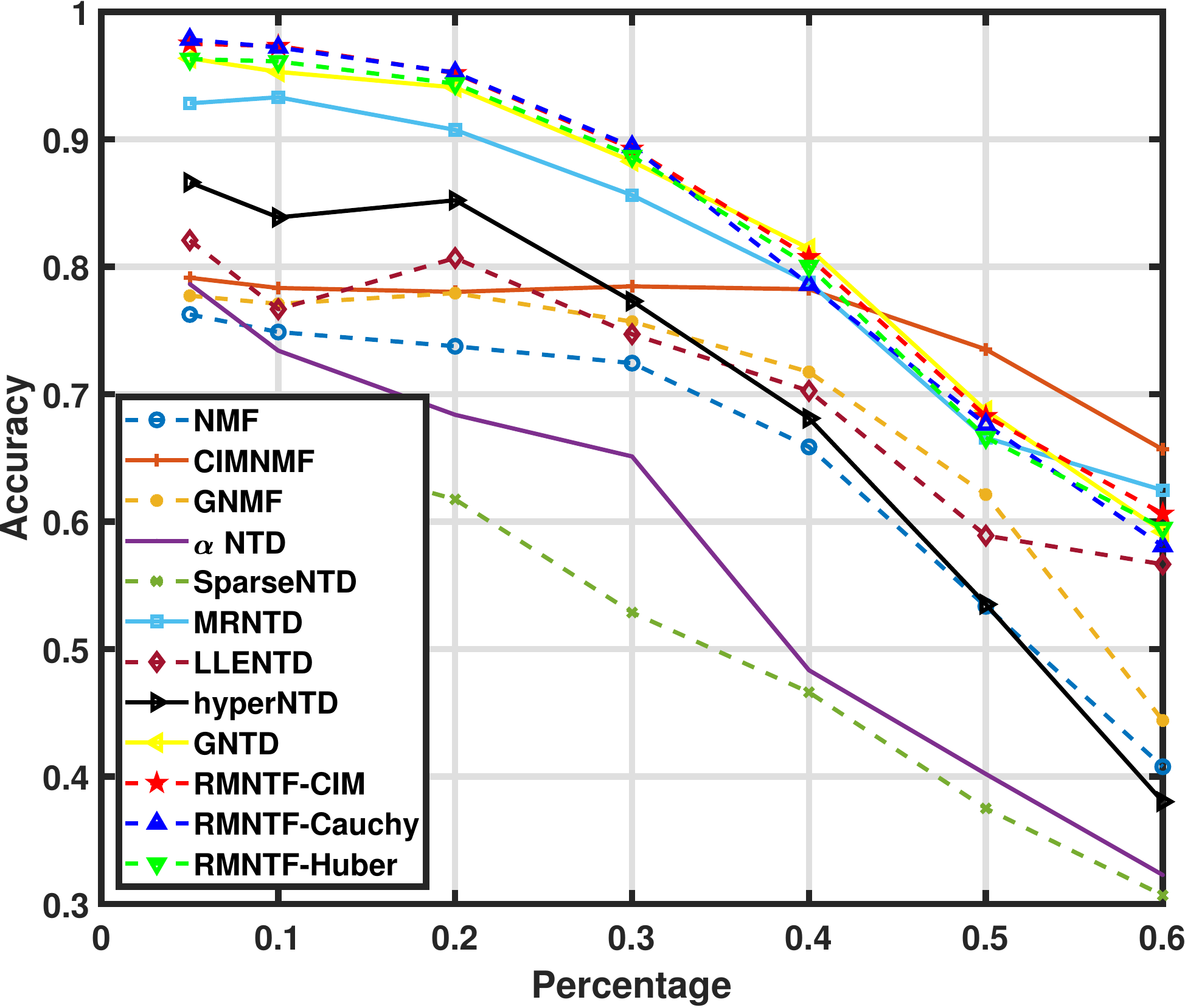}%
\label{Reuters_delta_Coh} \\
\includegraphics[width=1.25in]{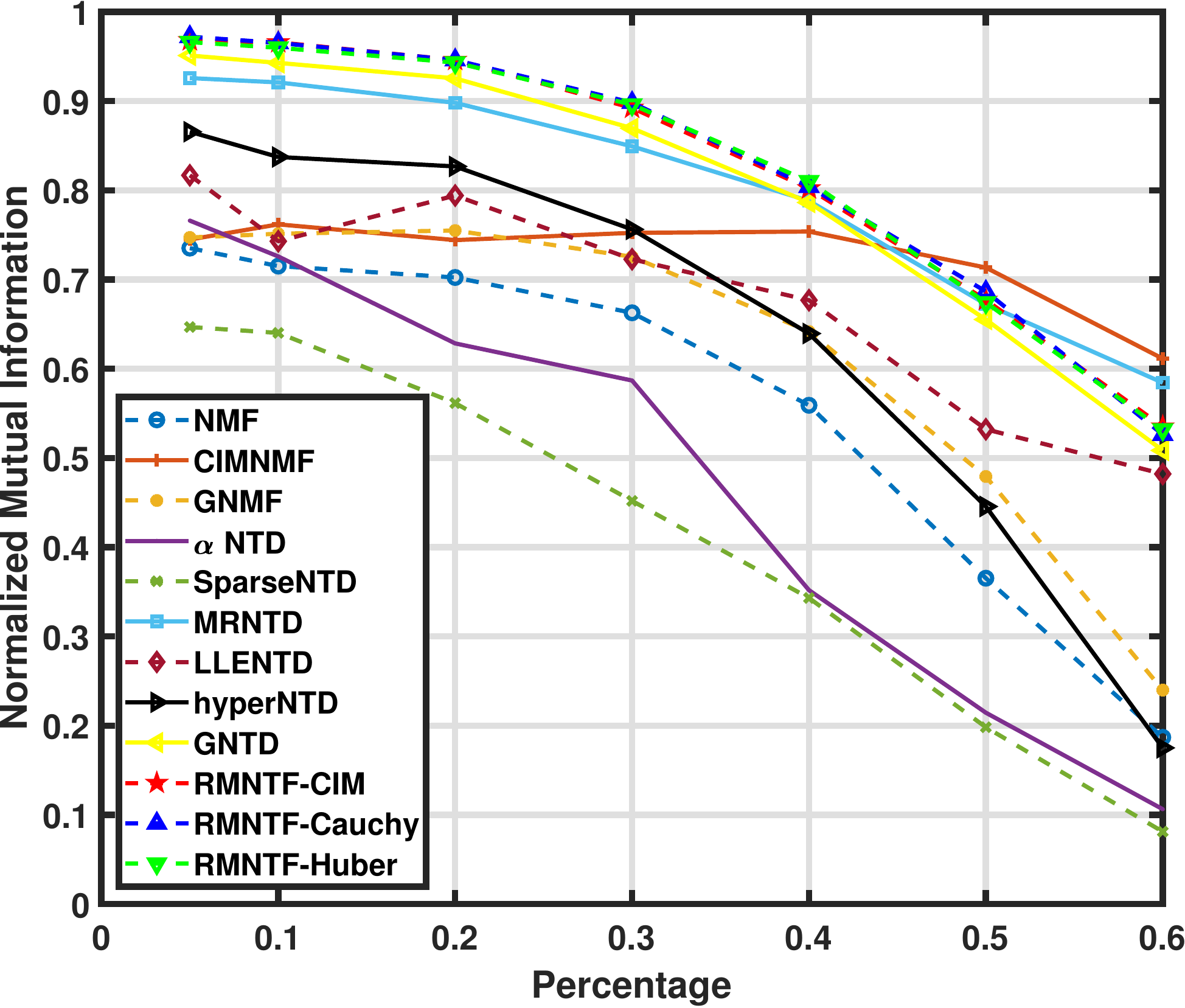}%
\end{minipage}
}\hspace{-8mm}
\hfil
\subfloat[]{
\begin{minipage}[b]{0.2\textwidth}
\includegraphics[width=1.25in]{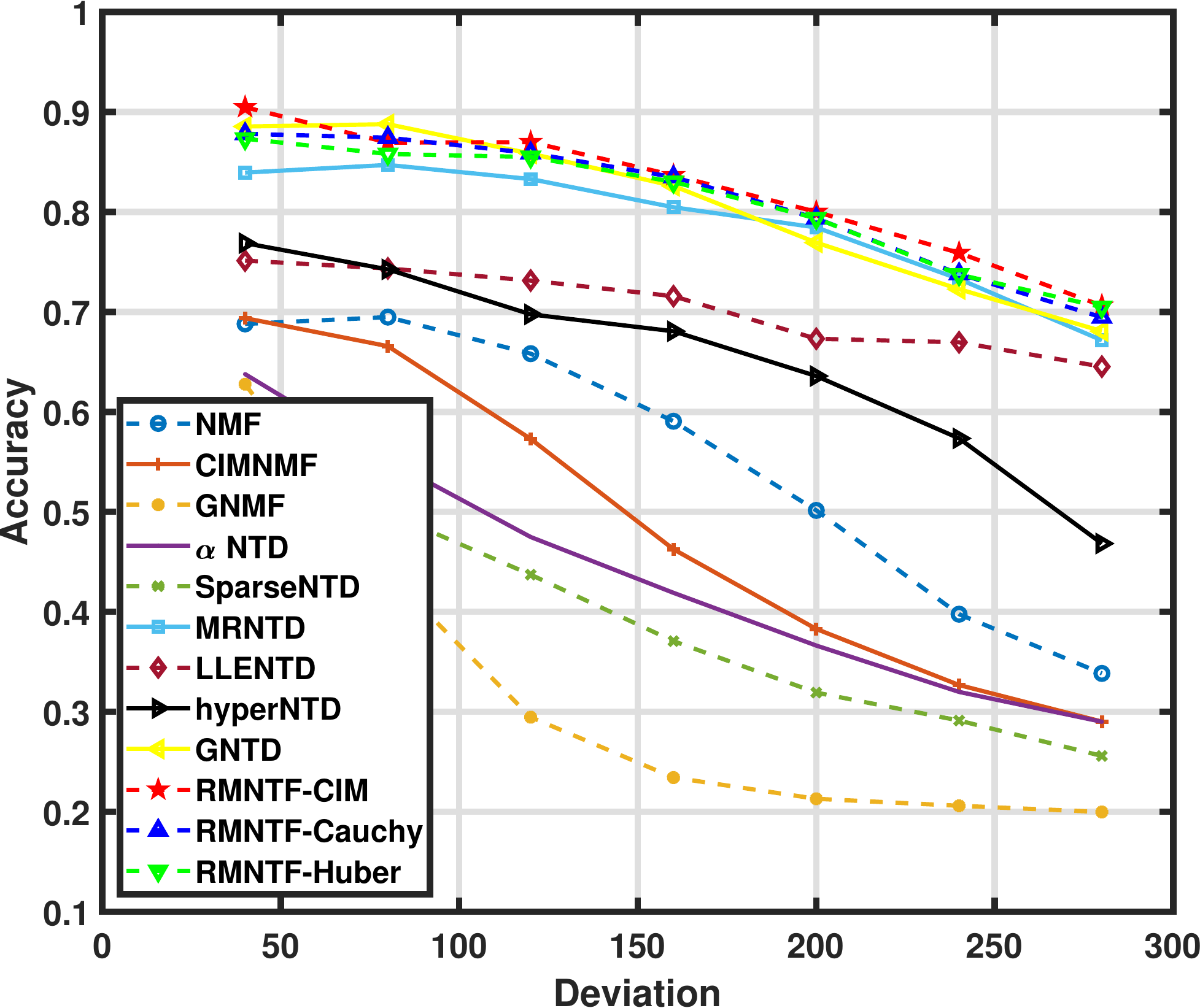}%
\label{Reuters_delta_Coh} \\
\includegraphics[width=1.25in]{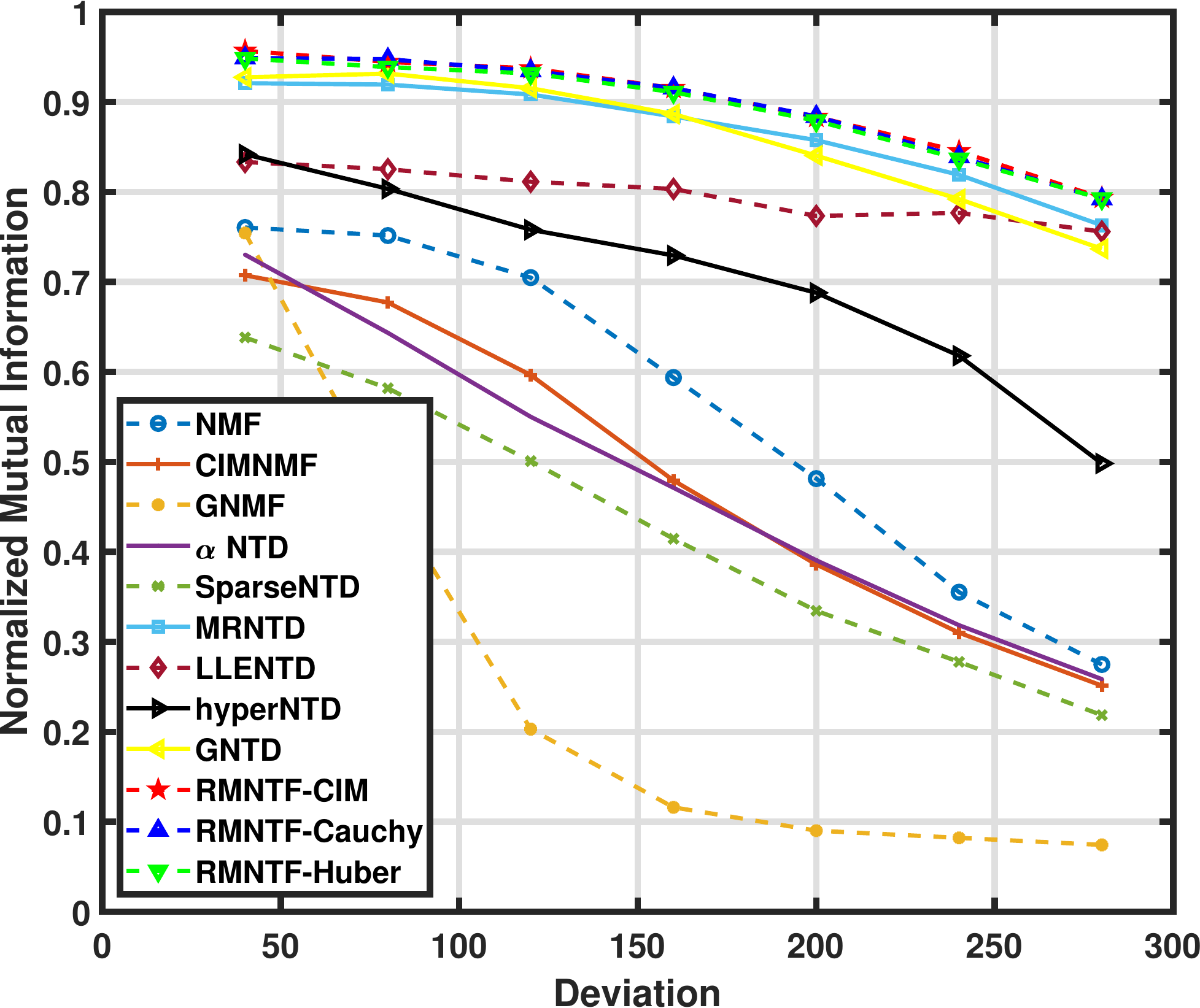}%
\label{Reuters_delta_Coh} \\
\includegraphics[width=1.25in]{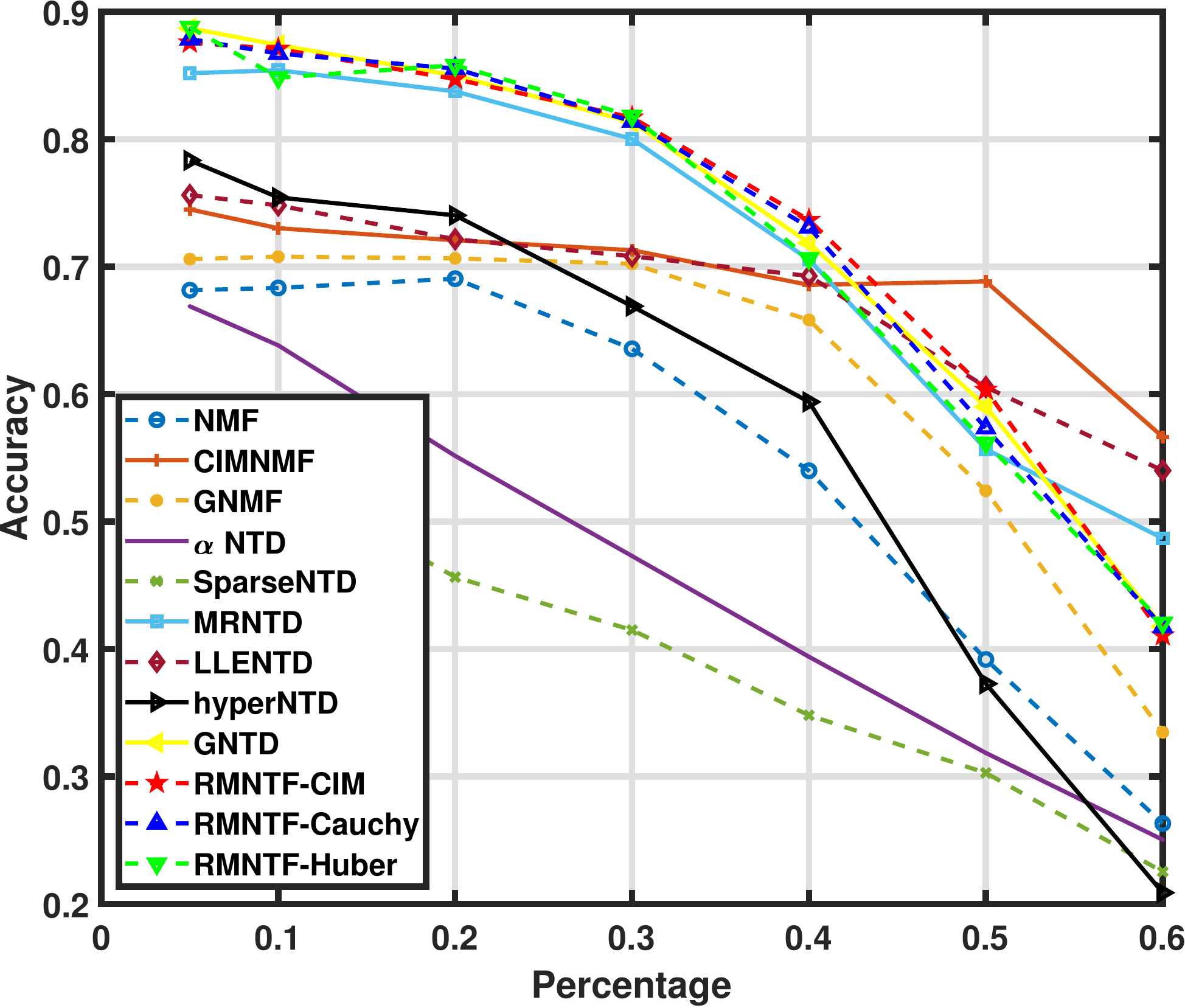}%
\label{Reuters_delta_Coh} \\
\includegraphics[width=1.25in]{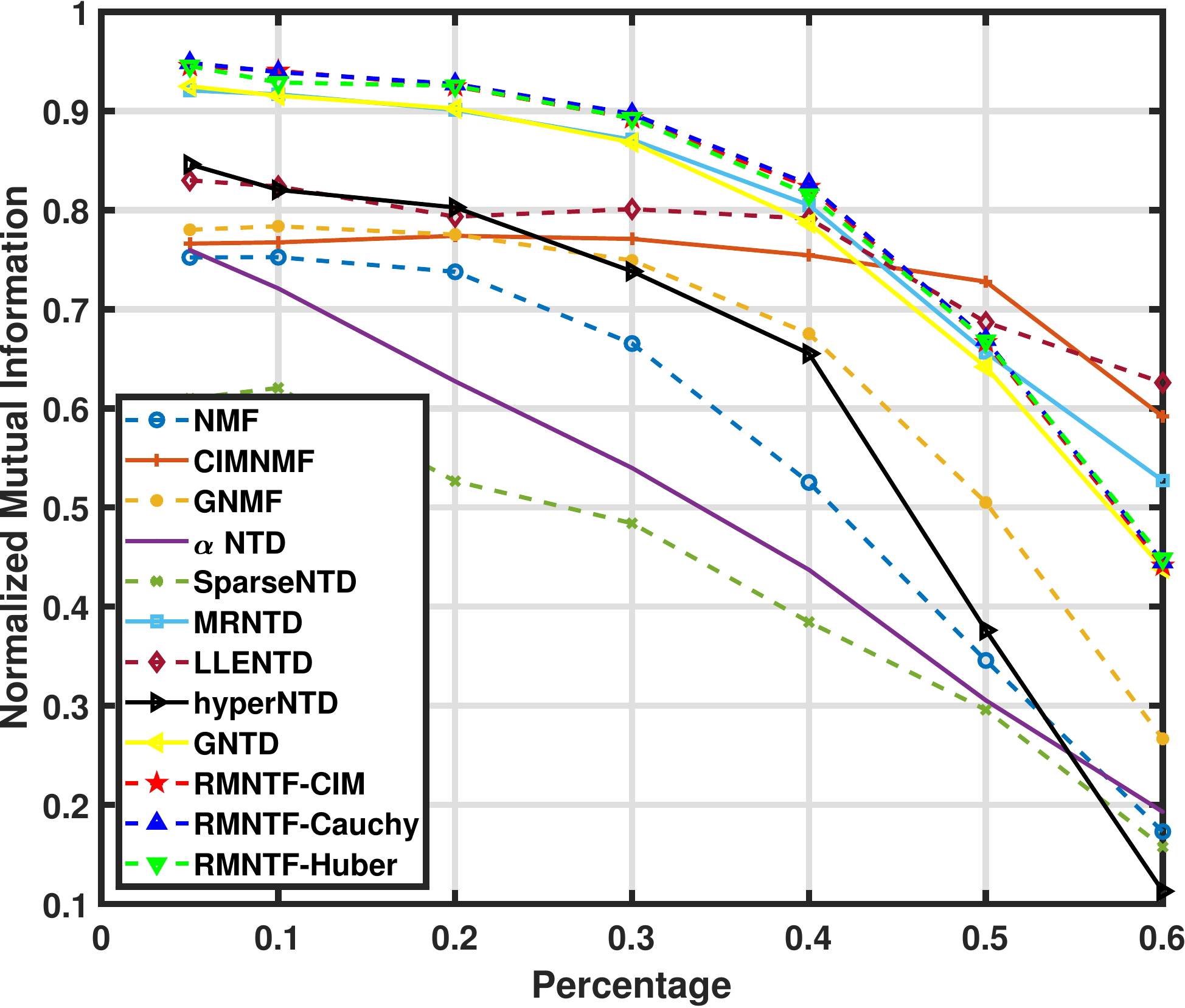}%
\end{minipage}
}\hspace{-3mm}
\subfloat[]{
\begin{minipage}[b]{0.2\textwidth}
\includegraphics[width=1.25in]{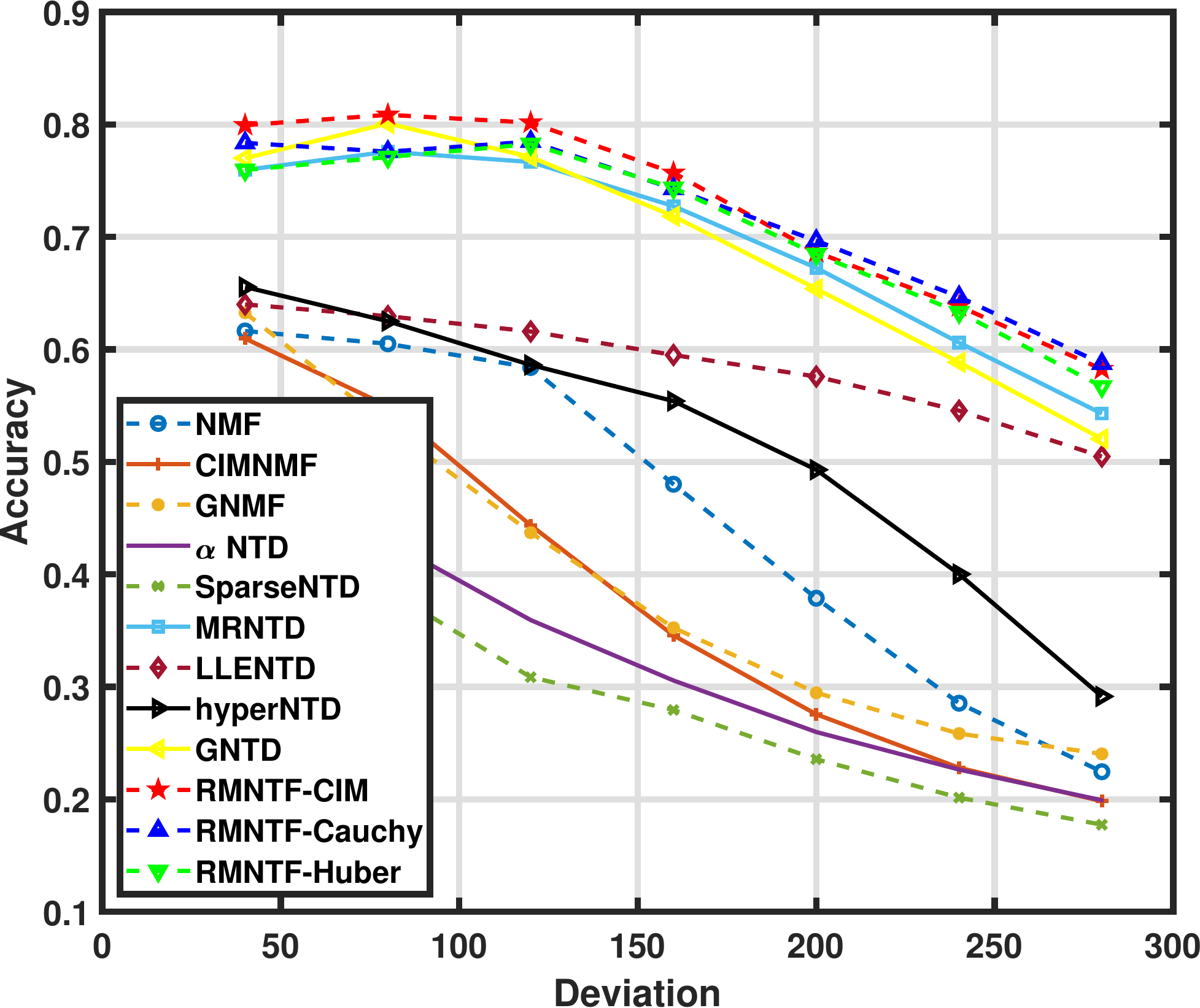}%
\label{Reuters_delta_Coh} \\
\includegraphics[width=1.25in]{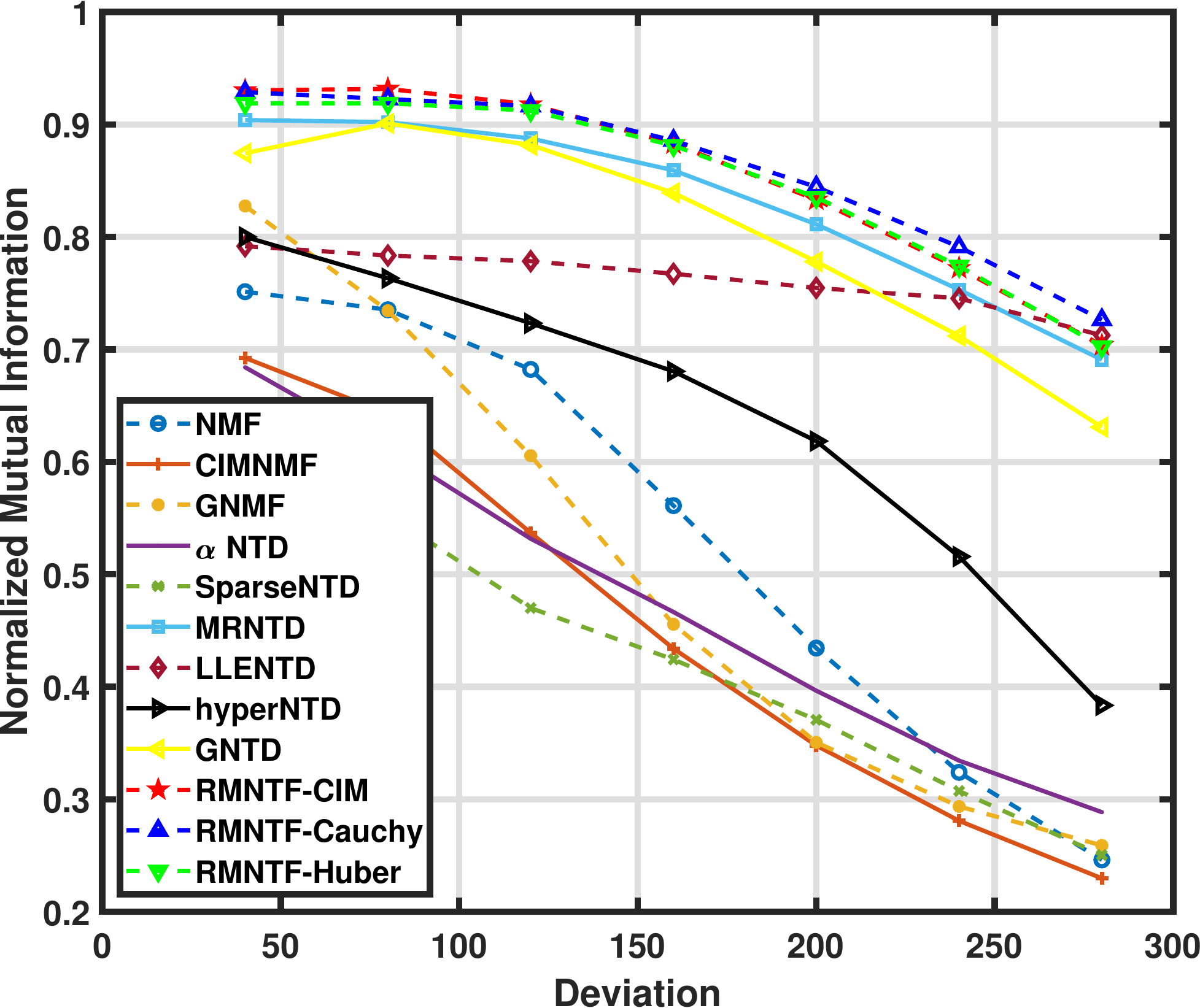}%
\label{Reuters_delta_Coh} \\
\includegraphics[width=1.25in]{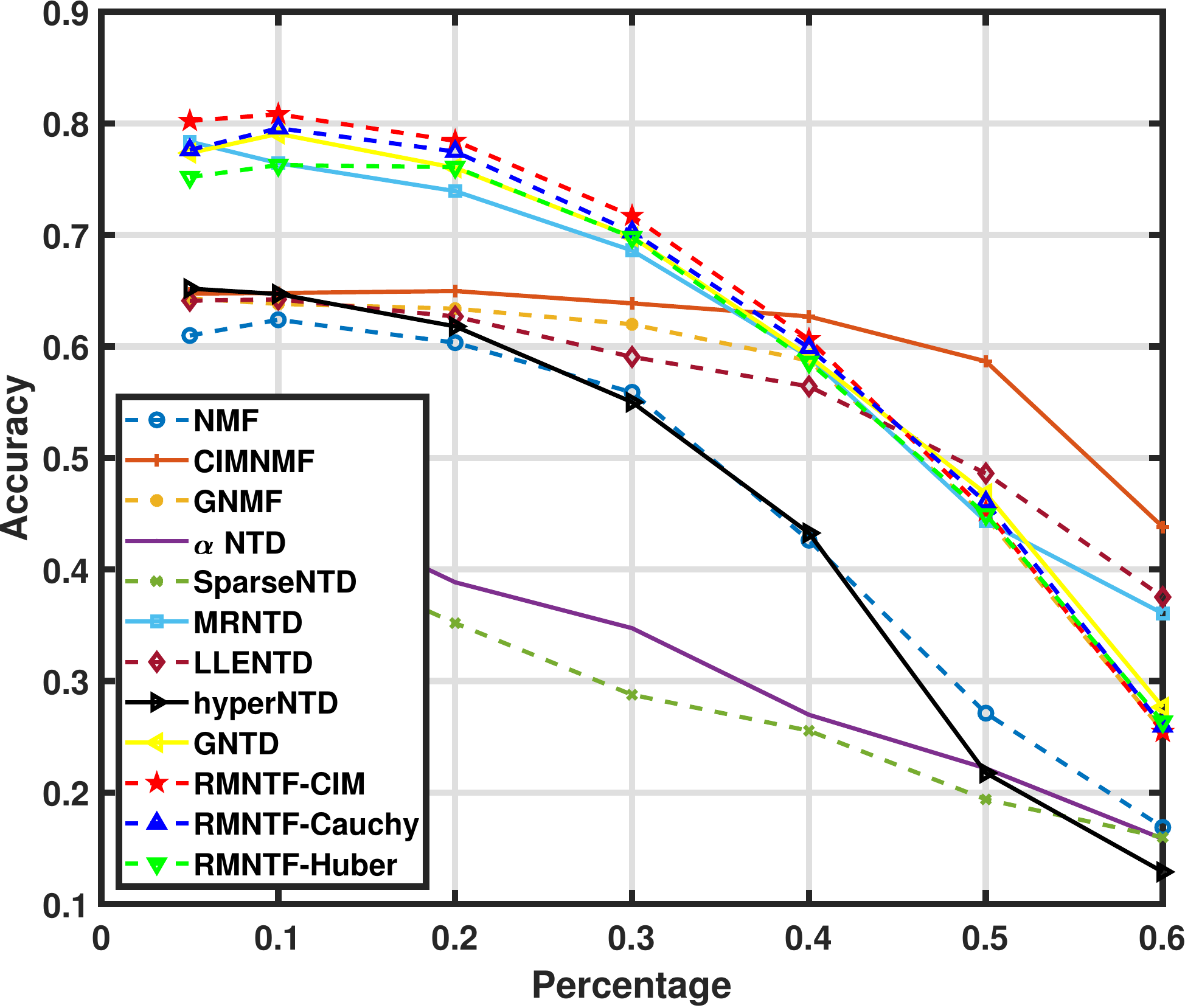}%
\label{Reuters_delta_Coh} \\
\includegraphics[width=1.25in]{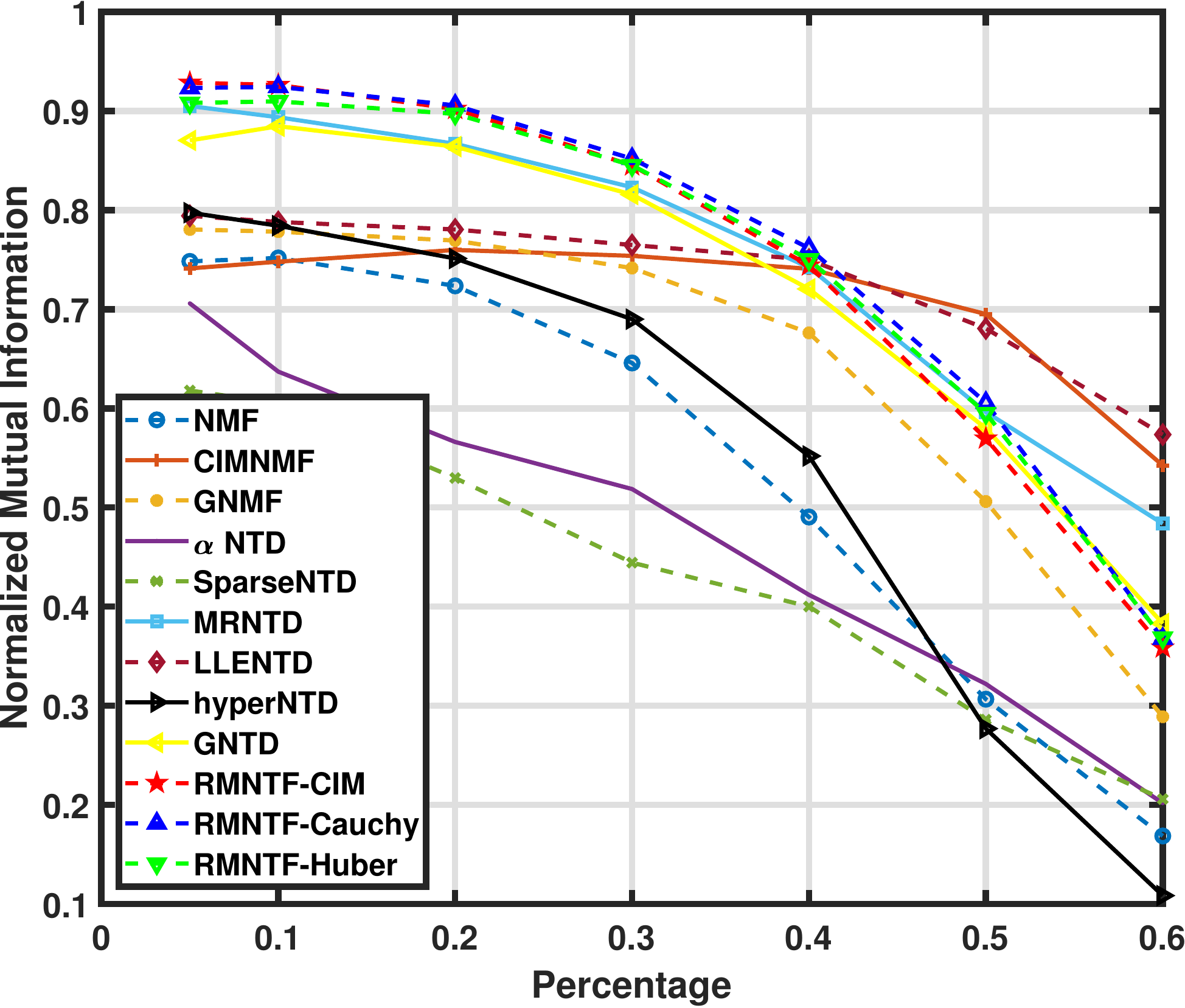}%
\end{minipage}
}\hspace{-3mm}
\subfloat[]{
\begin{minipage}[b]{0.2\textwidth}
\includegraphics[width=1.25in]{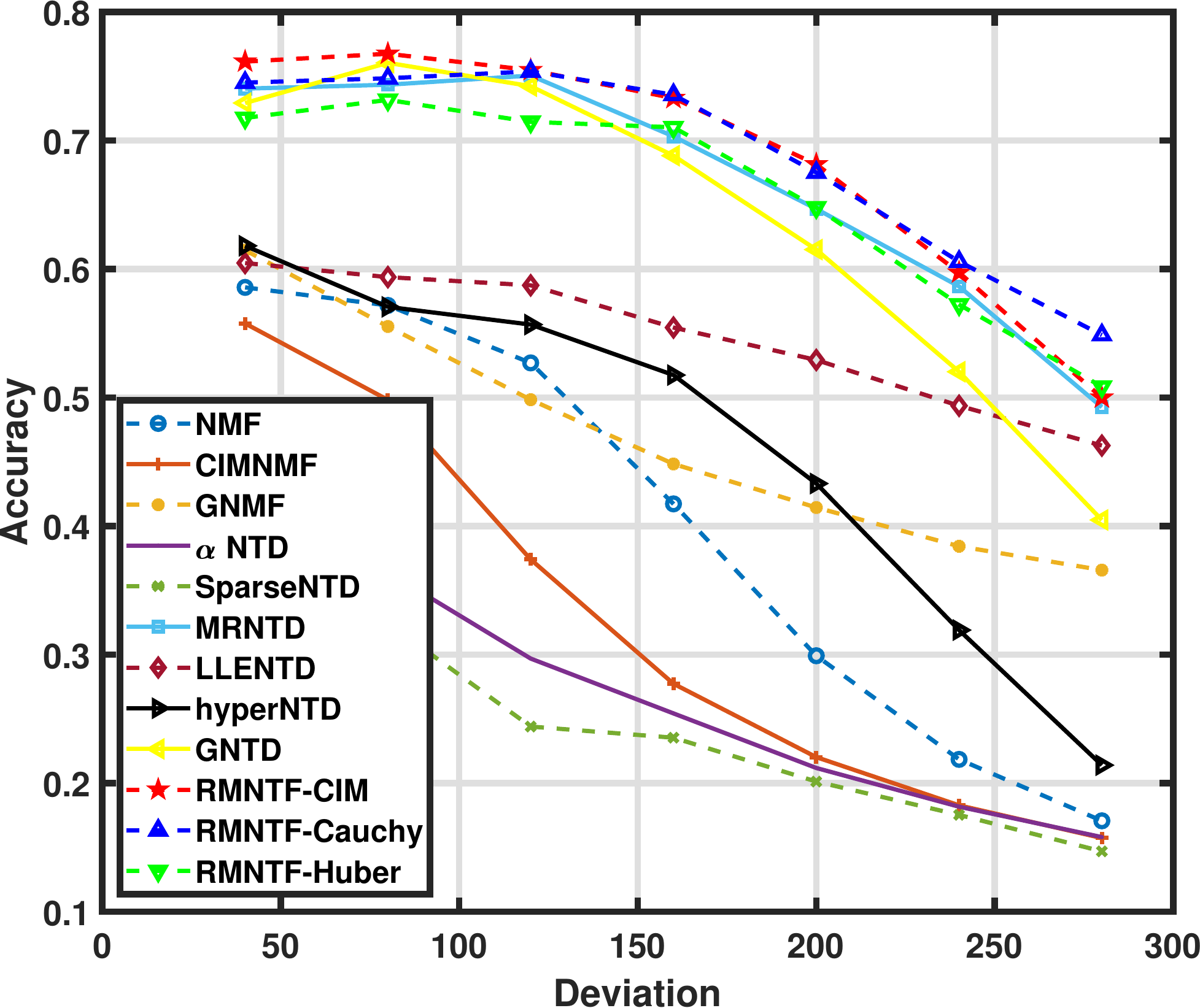}%
\label{Reuters_delta_Coh} \\
\includegraphics[width=1.25in]{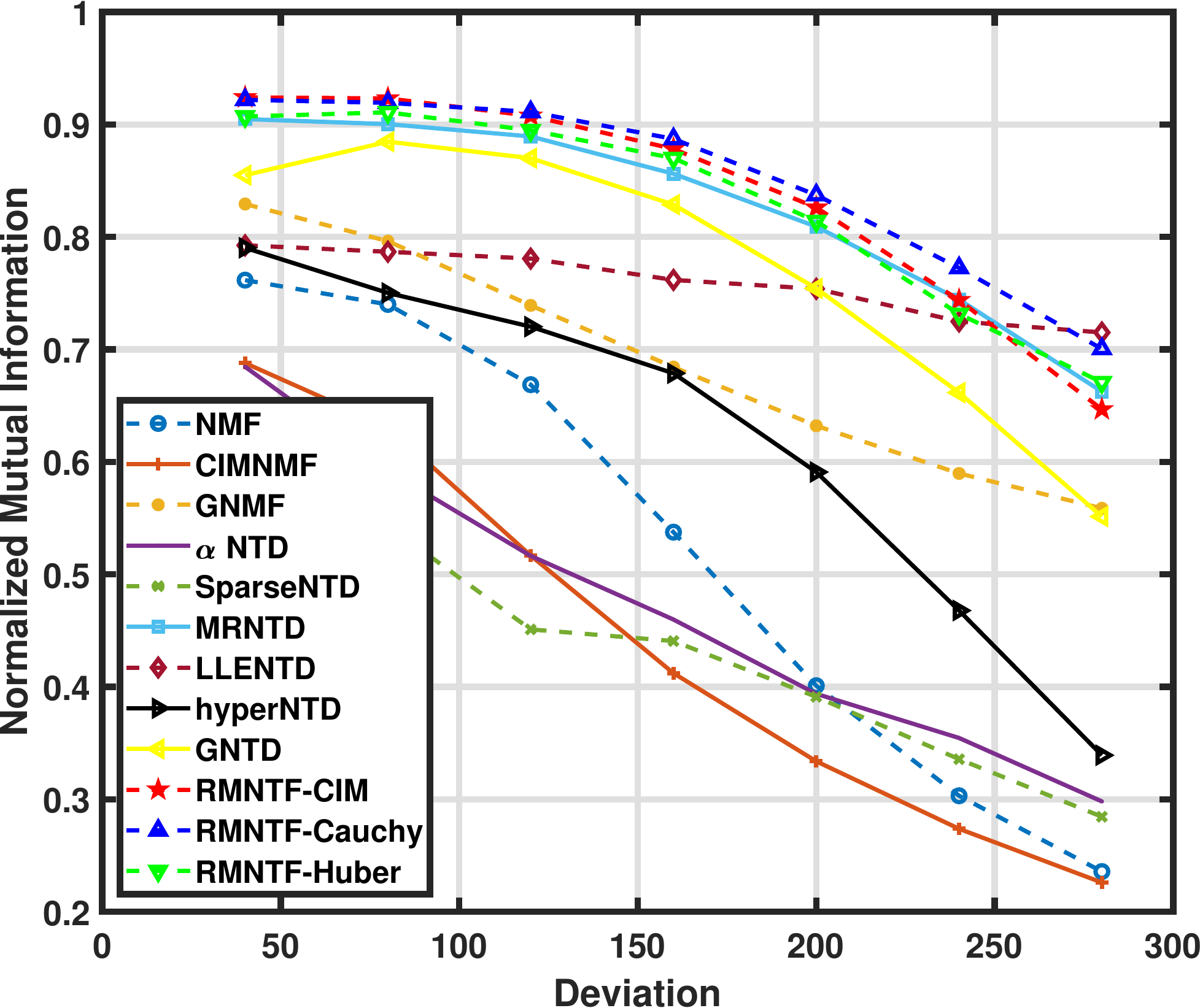}%
\label{Reuters_delta_Coh} \\
\includegraphics[width=1.25in]{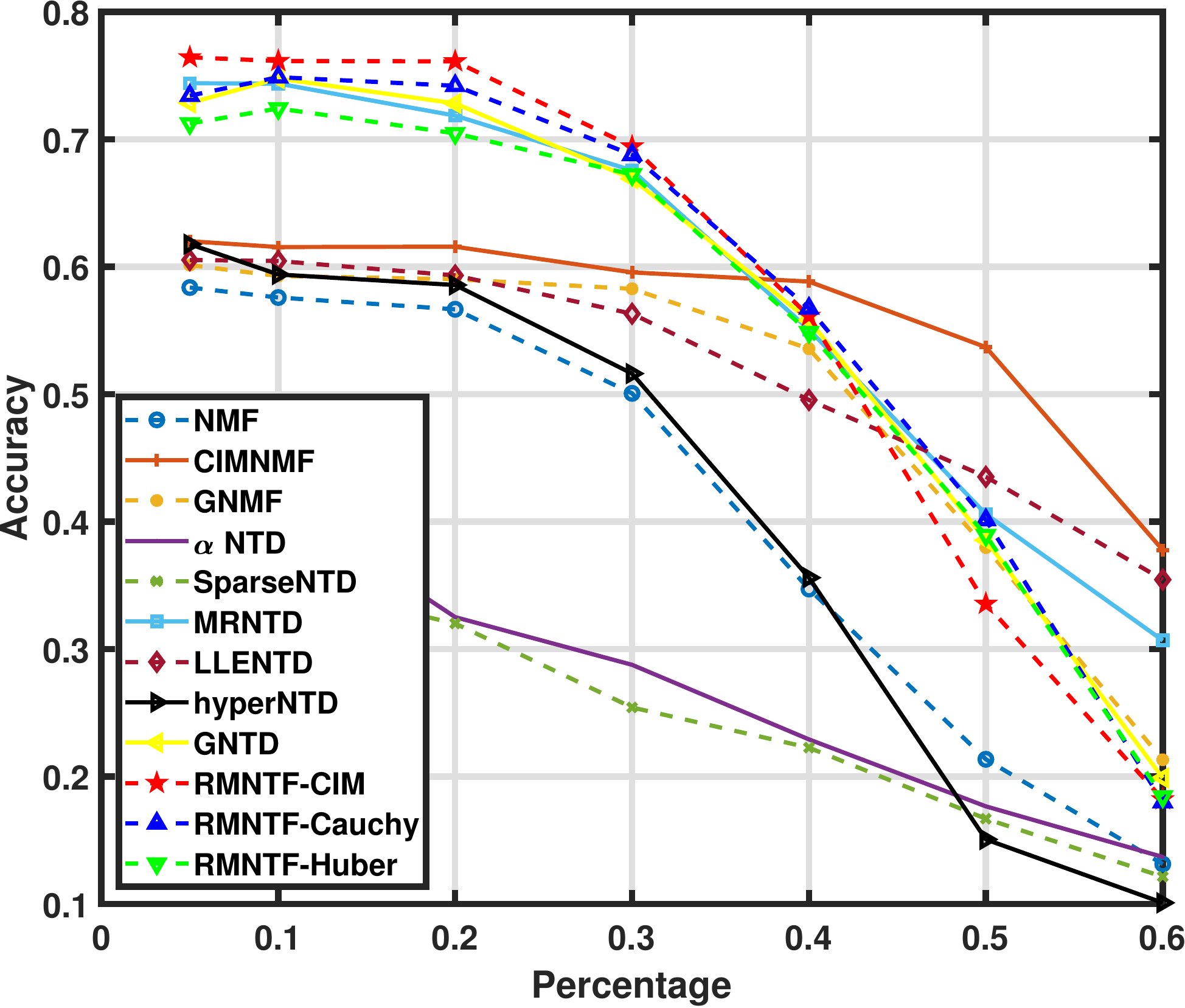}%
\label{Reuters_delta_Coh} \\
\includegraphics[width=1.25in]{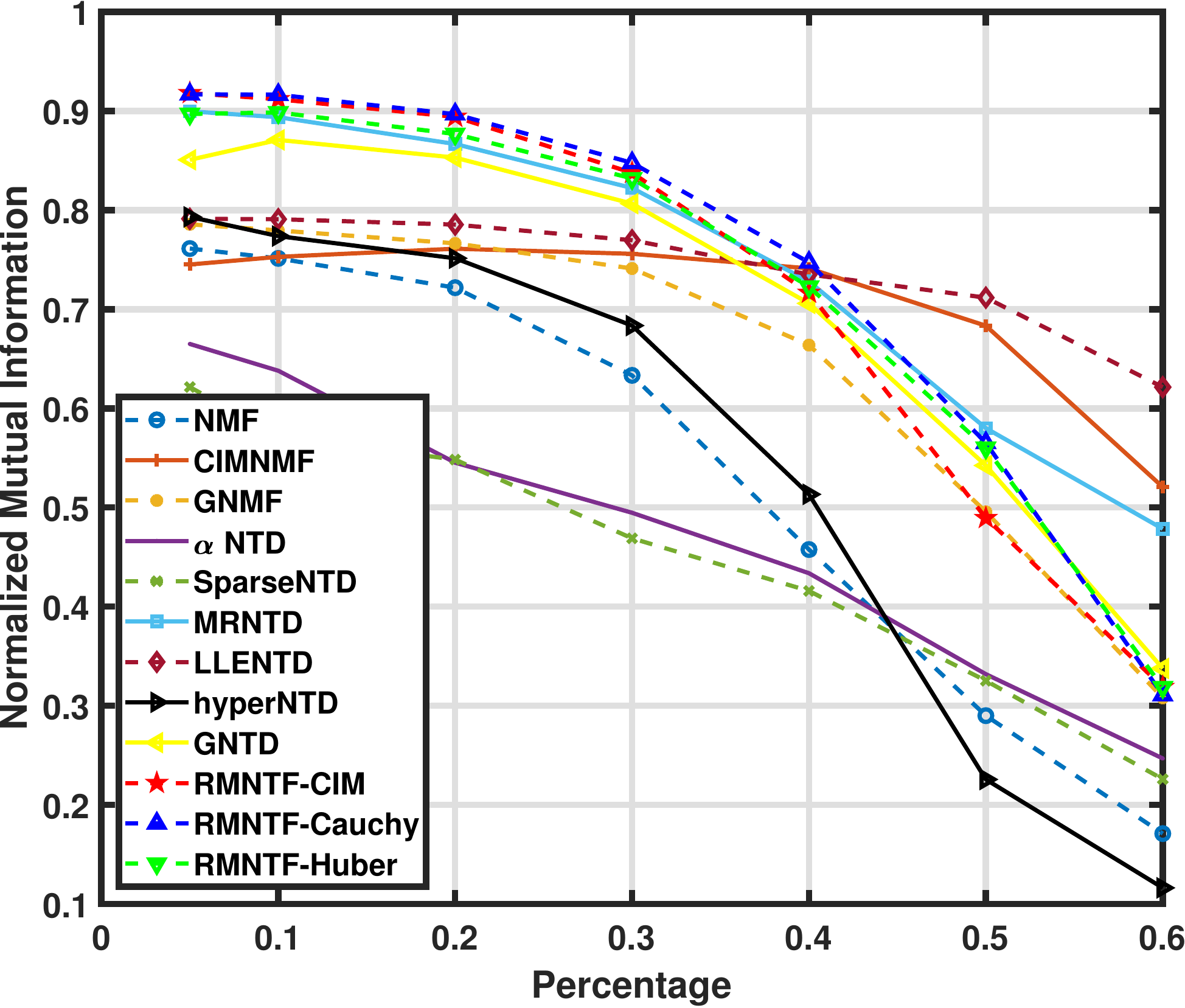}%
\end{minipage}
}\hspace{-3mm}
\subfloat[]{
\begin{minipage}[b]{0.2\textwidth}
\includegraphics[width=1.25in]{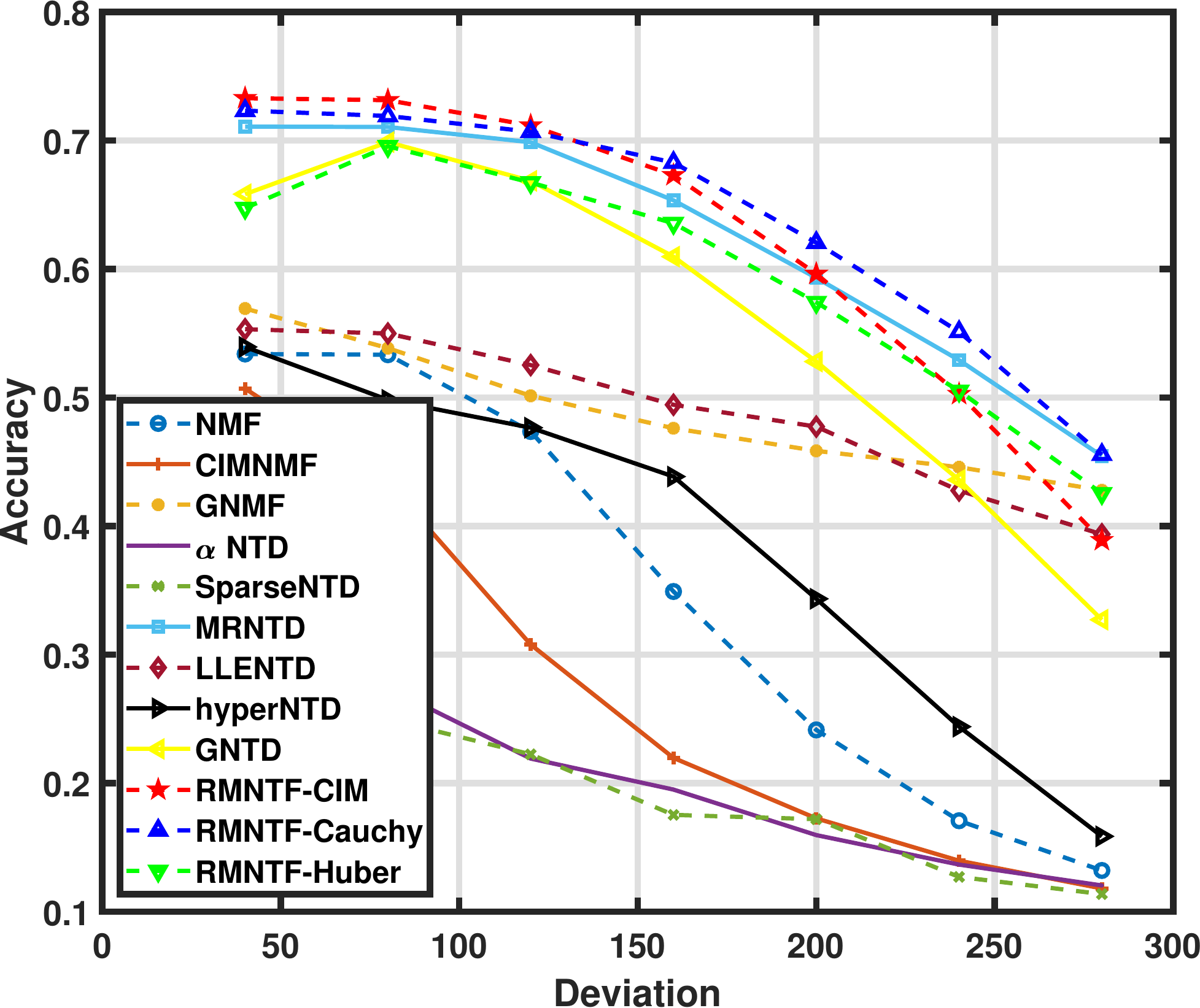}%
\label{Reuters_delta_Coh} \\
\includegraphics[width=1.25in]{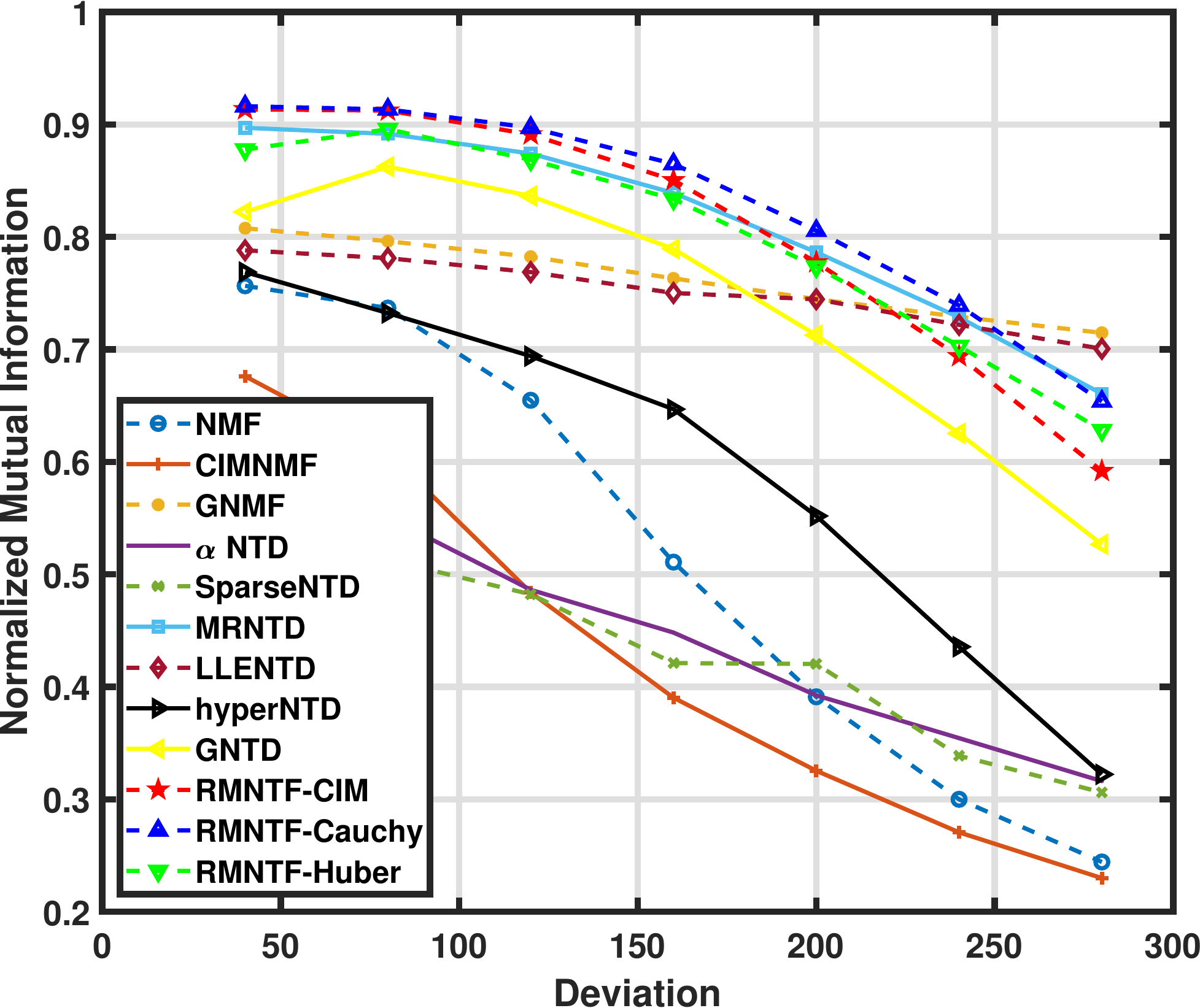}%
\label{Reuters_delta_Coh} \\
\includegraphics[width=1.25in]{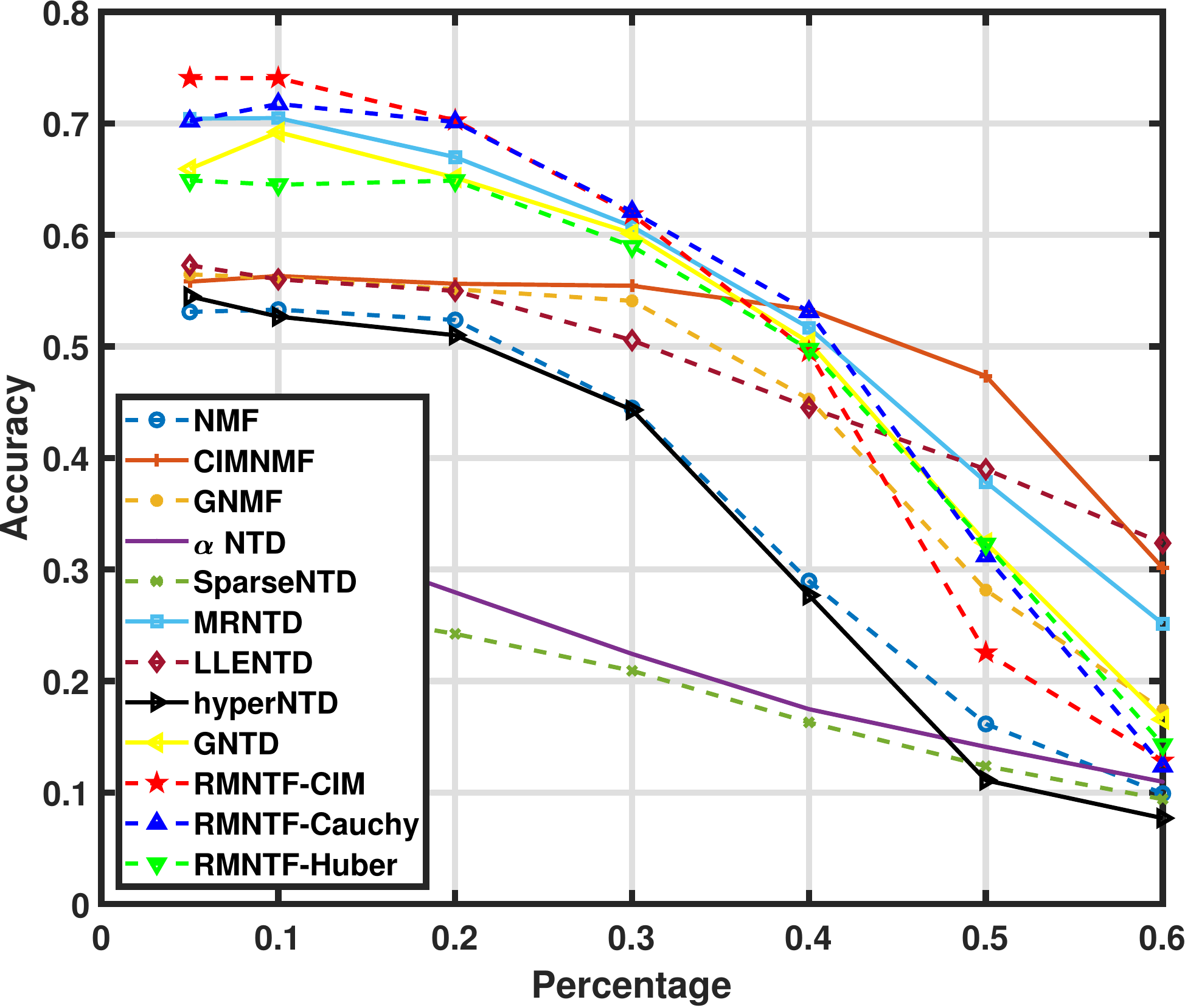}%
\label{Reuters_delta_Coh} \\
\includegraphics[width=1.25in]{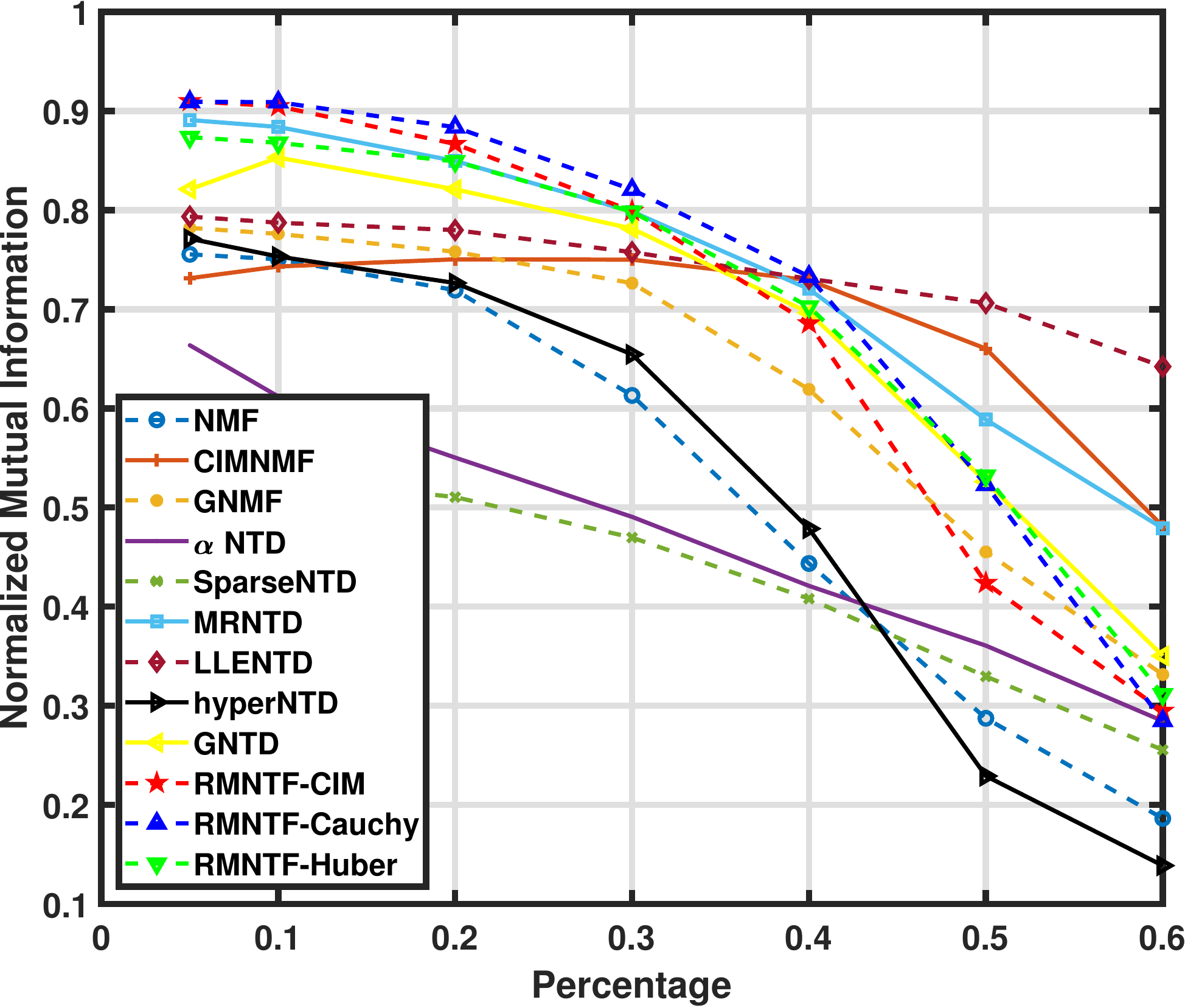}%
\end{minipage}
}\hspace{-5mm}

\vspace{2mm}
\caption{Evaluation of proposed methods on COIL100 database contaminated by Laplace noise and salt \& pepper noise, respectively. (a) Average accuracy and NMI on the subset of COIL100 which contains $5$ categories, the first two figures show the results contaminated by Laplace noise and the remains show the results contaminated by salt \& pepper noise. (b) Average evaluation on the subset of COIL100 which contains $10$ categories. (c) Average evaluation on the subset of COIL100 which contains $20$ categories. (d) Average evaluation on the subset of COIL100 which contains $30$ categories. (e) Average evaluation on the subset of COIL100 which contains $50$ categories.}
\label{fig:Reuters_delta3}
\end{figure*}

\subsection{Compared algorithms}
The hyperparameters of RMNTF-CIM, RMNTF-Huber, and RMNTF-Cauchy in all experiments were set as follows:
$\lambda = 10^5, p = 3$, unless otherwise stated.
We compared RMNTF with three NMF methods, and six NTD methods.
They are listed as follows:
\begin{itemize}
 \itemsep=0.0pt
  \item \textbf{Nonnegative Matrix Factorization (NMF)} \cite{lee1999learning}:
  We unfold tensorial original data into a matrix to apply it to NMF algorithm and obtain a low rank activation matrix $\mathbf{V}$.
  Finally, we apply $\mathbf{V}$ to downstream tasks.
  \item \textbf{Graph Regularization Nonnegative Matrix Decomposition (GNMF)} \cite{cai2010graph}: We unfold tensorial data into a matrix $\mathbf{X}$ and reduce the dimensionality of $\mathbf{X}$ into $K$ by GNMF, where the dictionary matrix $\mathbf{U}$ is used in graph regularization term. The coefficient of regularization $\lambda = 10$, and the hyperparameter $p=3$.
  \item \textbf{Correntropy Induced Metric Nonnegative Matrix Factorization (CIMNMF)} \cite{du2012robust}:Similar to the above, we unfold the tensorial data into a matrix and apply it to CIMNMF. The loss function between the unfolding tensor and reconstruction data is based on the correntropy induced metric.
      It handles outlier rows by incorporating structural knowledge about outliers.
  \item \textbf{Nonnegative Tucker Decomposition (NTD)} \cite{kim2007nonnegative}:
      NTD works directly on the tensorial data.
      It is constructed by the tucker model with nonnegativity constraints and updated by multiplicative rules.
      This method can be considered as a special case of RMNTF when all the entries of weight tensor $\mathcal{W}$ are $1$, and $\lambda = 0$.
  \item \textbf{Nonnegative Tucker Decomposition with alpha-Divergence ($\alpha$NTD)} \cite{kim2008nonnegative}:
      We directly use the tensorial data to $\alpha$NTD.
      It is considered $\alpha$-divergence as a discrepancy measure and derive multiplicative updating rules for NTD.
      The hyperparameter $\alpha = 1$.
  \item \textbf{Sparse Nonnegative Tucker Decomposition (SparseNTD)} \cite{xu2015alternating}: The tensorial data is directly utilized in SparseNTD. It adopts the sparsity and nonnegativity constraints to a core tensor and several factor matrices.
      The hyperparameter $\lambda = 0.5$.
  \item \textbf{Graph Regularized Nonnegative Tucker Decomposition (GNTD)} \cite{qiu2020generalized}: GNTD use tensorial data and it constructs the nearest neighbor graph to maintain the intrinsic manifold structure of tensor and applies this constraint on the $N$th factor matrix.
      The hyperparameters $\lambda = 10^5$, and $p = 3$.
  \item \textbf{Locally Linear Embedding Regularized Nonnegative Tucker Decomposition (LLENTD)} \cite{yin2019lle}: We directly use tensor in LLENTD. It incorporates Laplacian Eigenmaps as manifold regularization terms into the least square form of nonnegative tucker model.
      The hyperparameters $\lambda = 10^5$, and $p = 3$.
  \item \textbf{Manifold Regularization Nonnegative Tucker Decomposition (MRNTD)} \cite{li2016mr}: MRNTD utilizes tensor data directly. It employs the manifold regularization terms for the core tensors constructed in the NTD.
      The hyperparameters $\lambda = 10^5$, and $p = 3$.
  \item \textbf{Hypergraph regularized nonnegative tensor factorization (hyperNTD)} \cite{yin2021hyperntf}:We take tensorial data into hyperNTD directly. It incorporates a higher-order relationship among the nearest neighborhoods and the nonnegative tucker decomposition.
      The hyperparameters $\lambda = 10^5$, and $p = 3$.
\end{itemize}

\begin{figure*}[!t]
\setlength{\abovecaptionskip}{0cm} 
\setlength{\belowcaptionskip}{-0cm} 
\centering
\subfloat[]{
\begin{minipage}[b]{0.18\textwidth}
\includegraphics[width=1.5in]{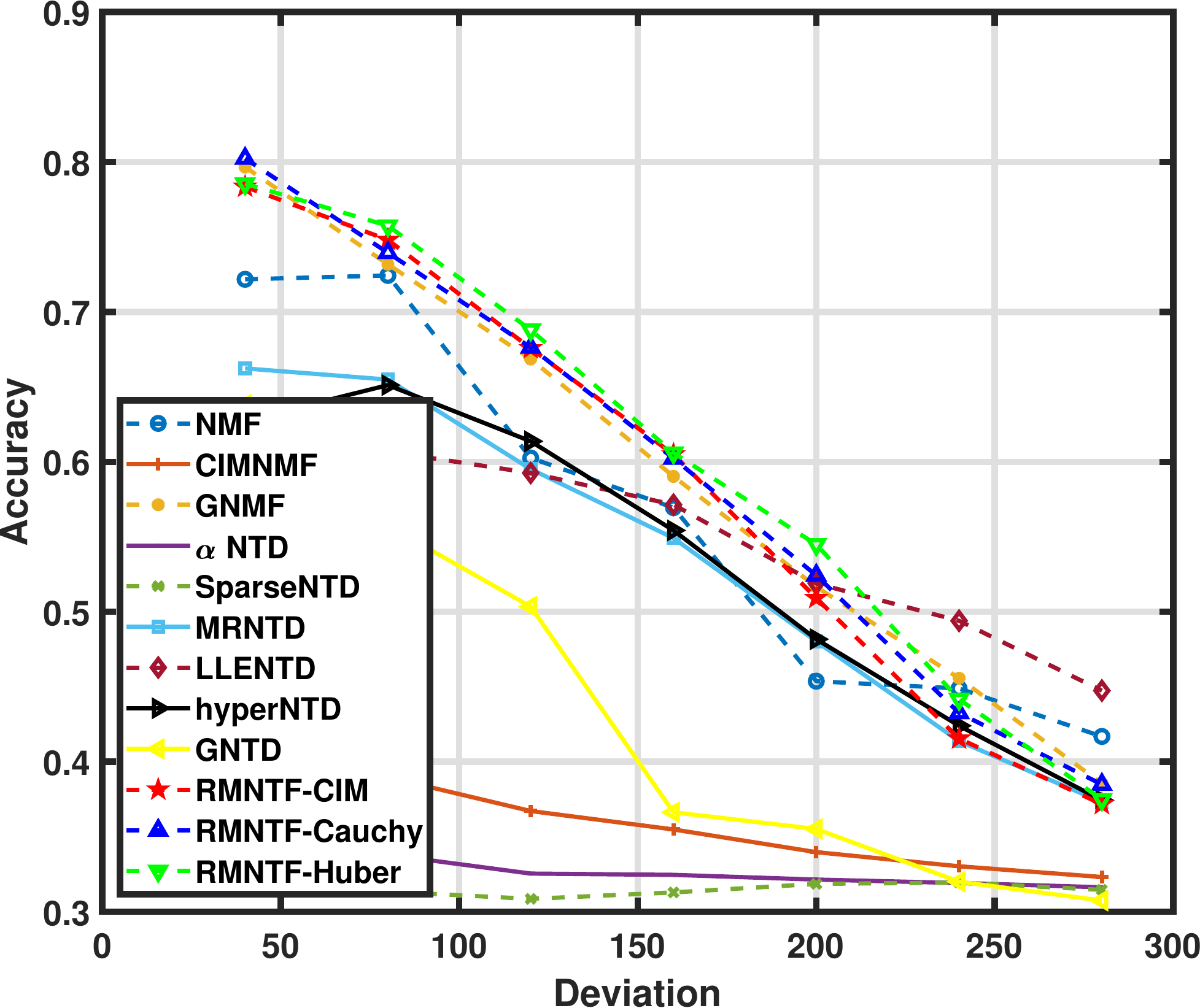}%
\label{Reuters_delta_Coh} \\
\includegraphics[width=1.5in]{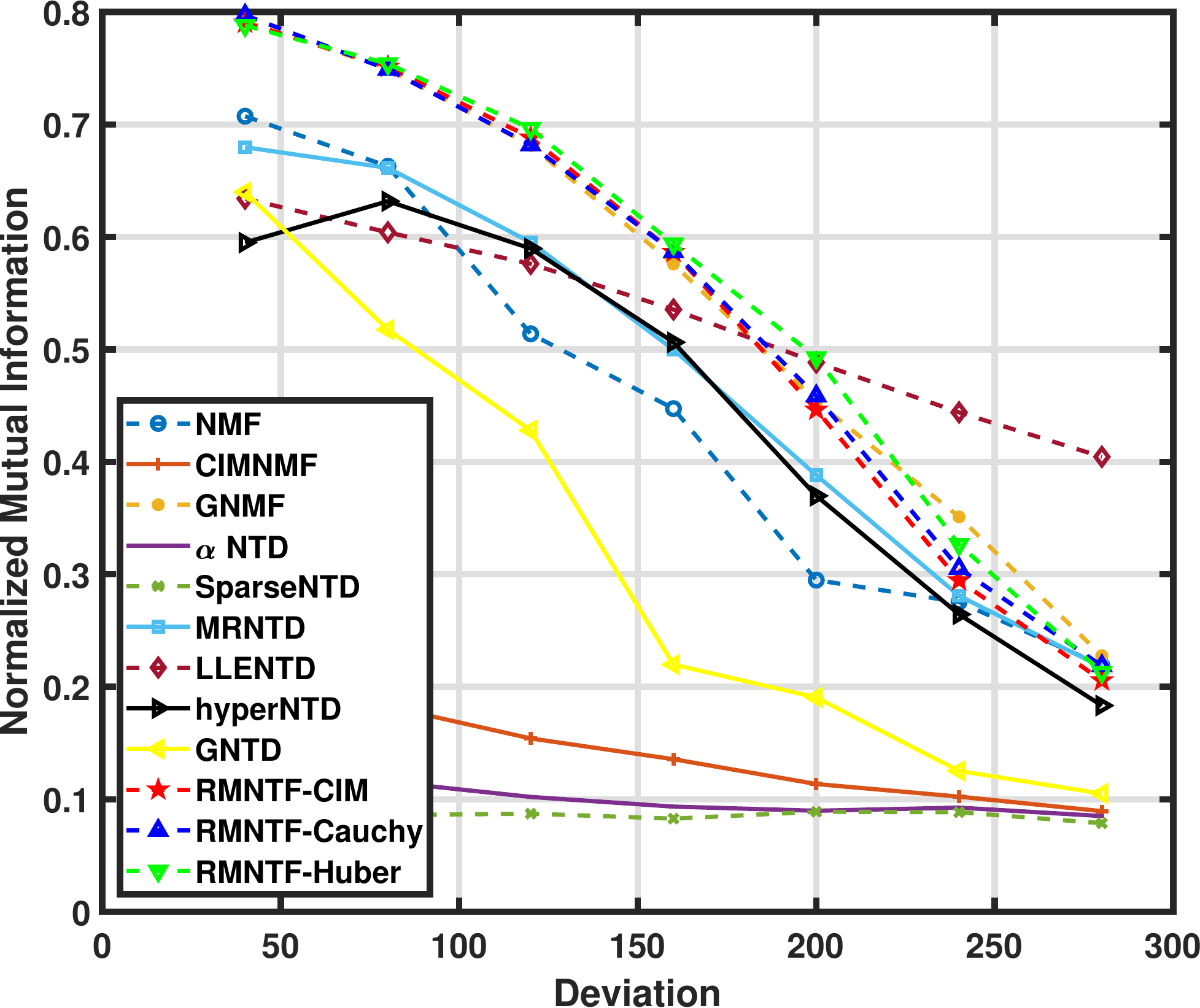}%
\label{Reuters_delta_Coh} \\
\includegraphics[width=1.5in]{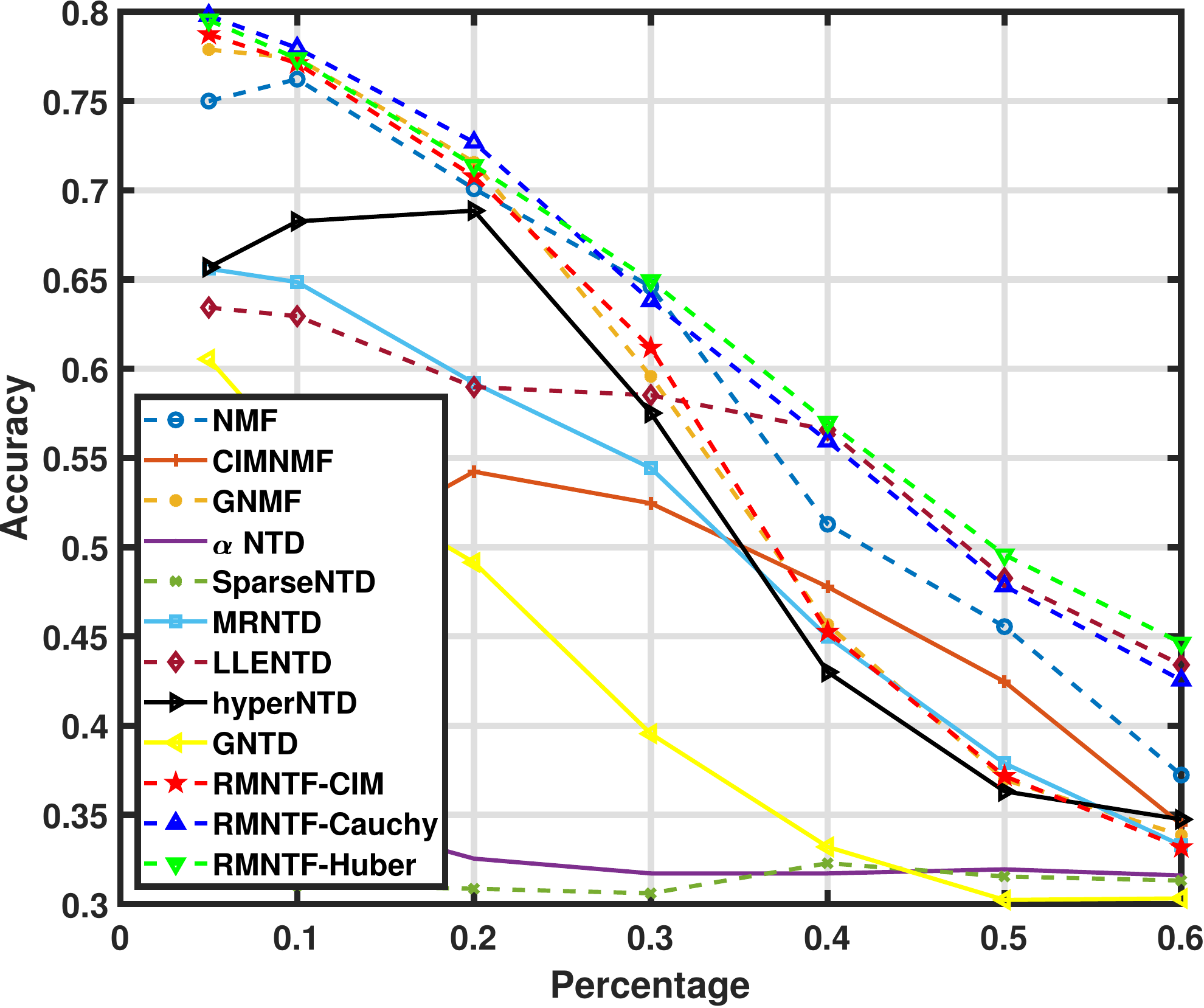}%
\label{Reuters_delta_Coh} \\
\includegraphics[width=1.5in]{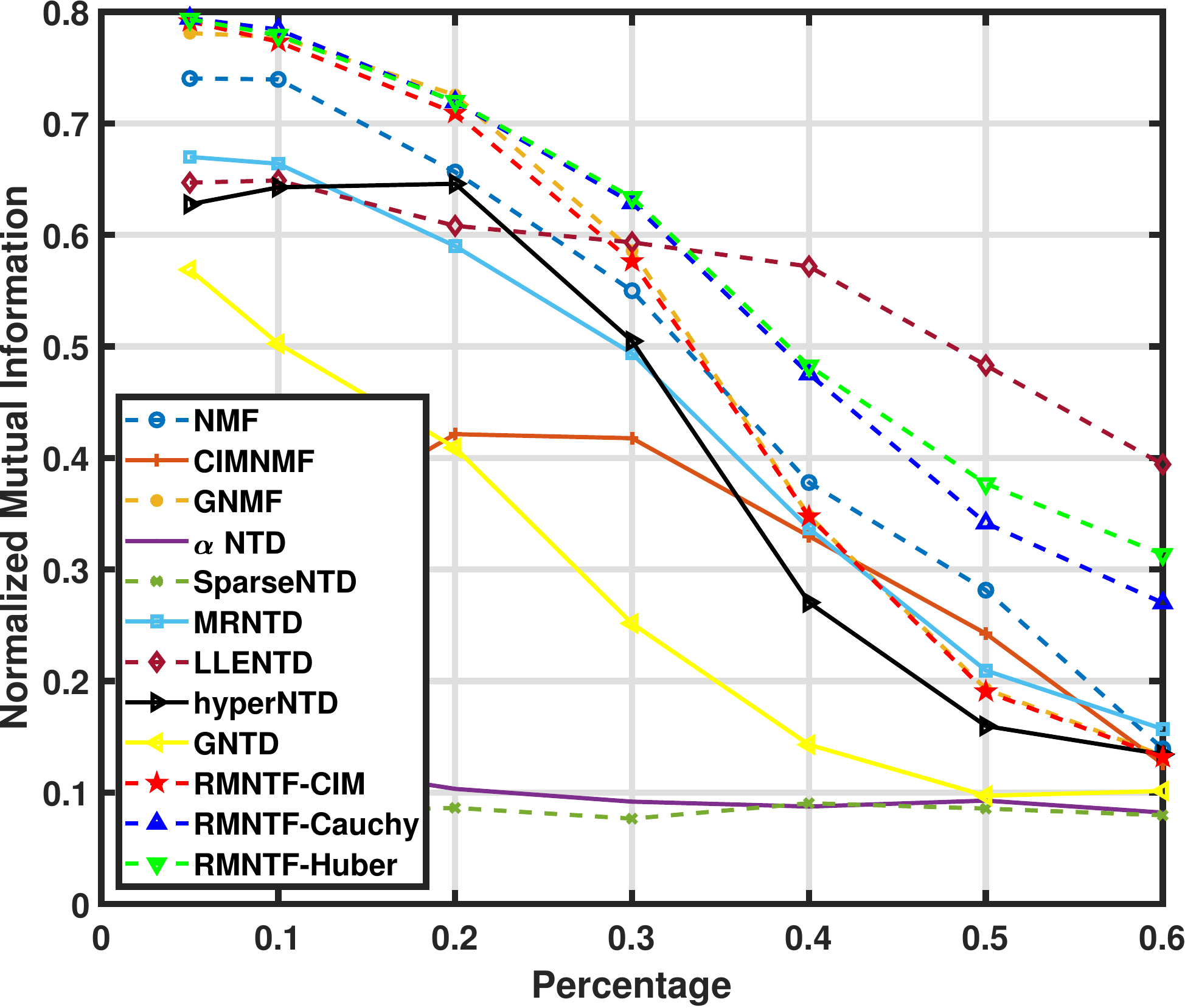}%
\end{minipage}
}
\hfil
\subfloat[]{
\begin{minipage}[b]{0.22\textwidth}
\includegraphics[width=1.5in]{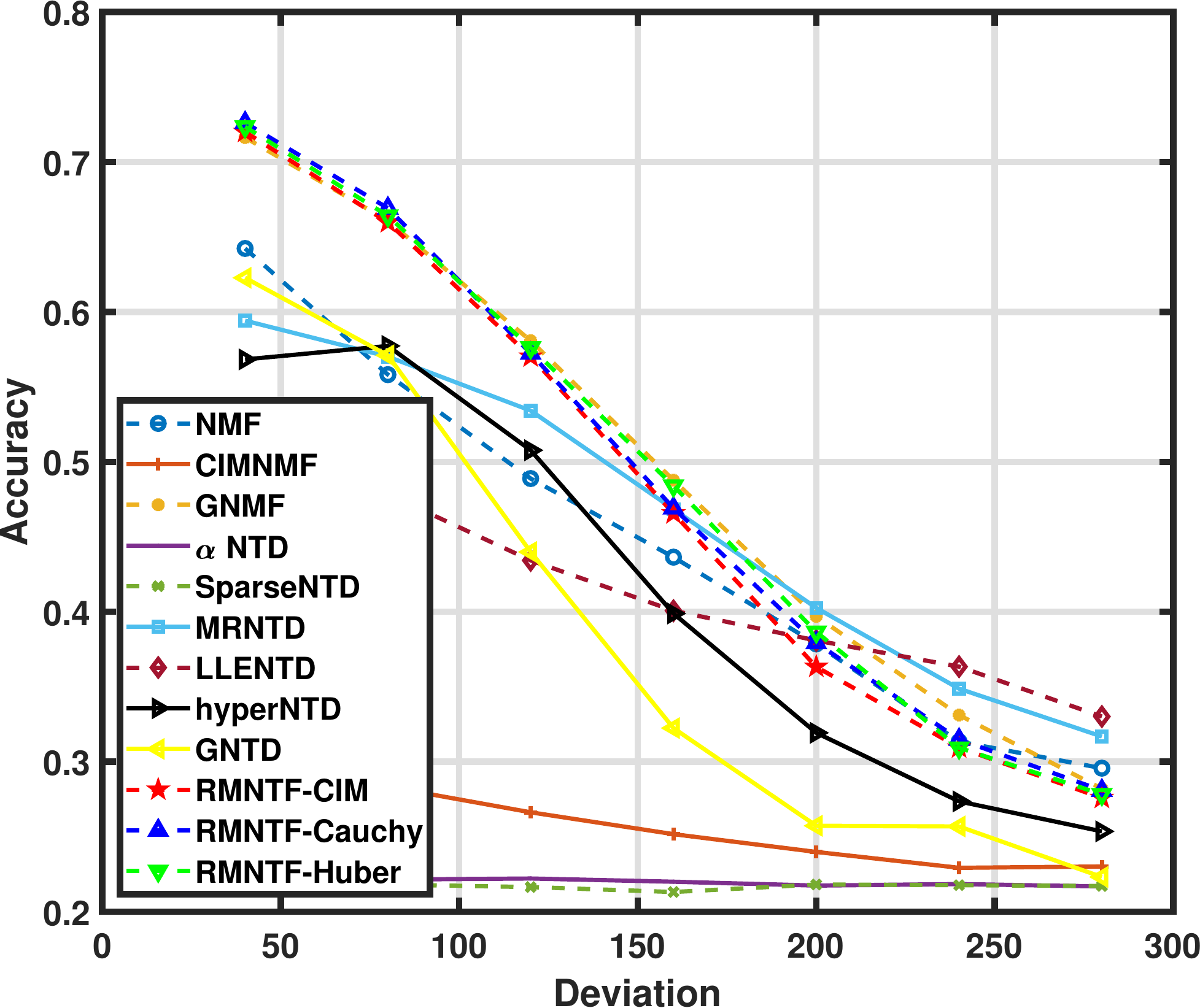}%
\label{Reuters_delta_Coh} \\
\includegraphics[width=1.5in]{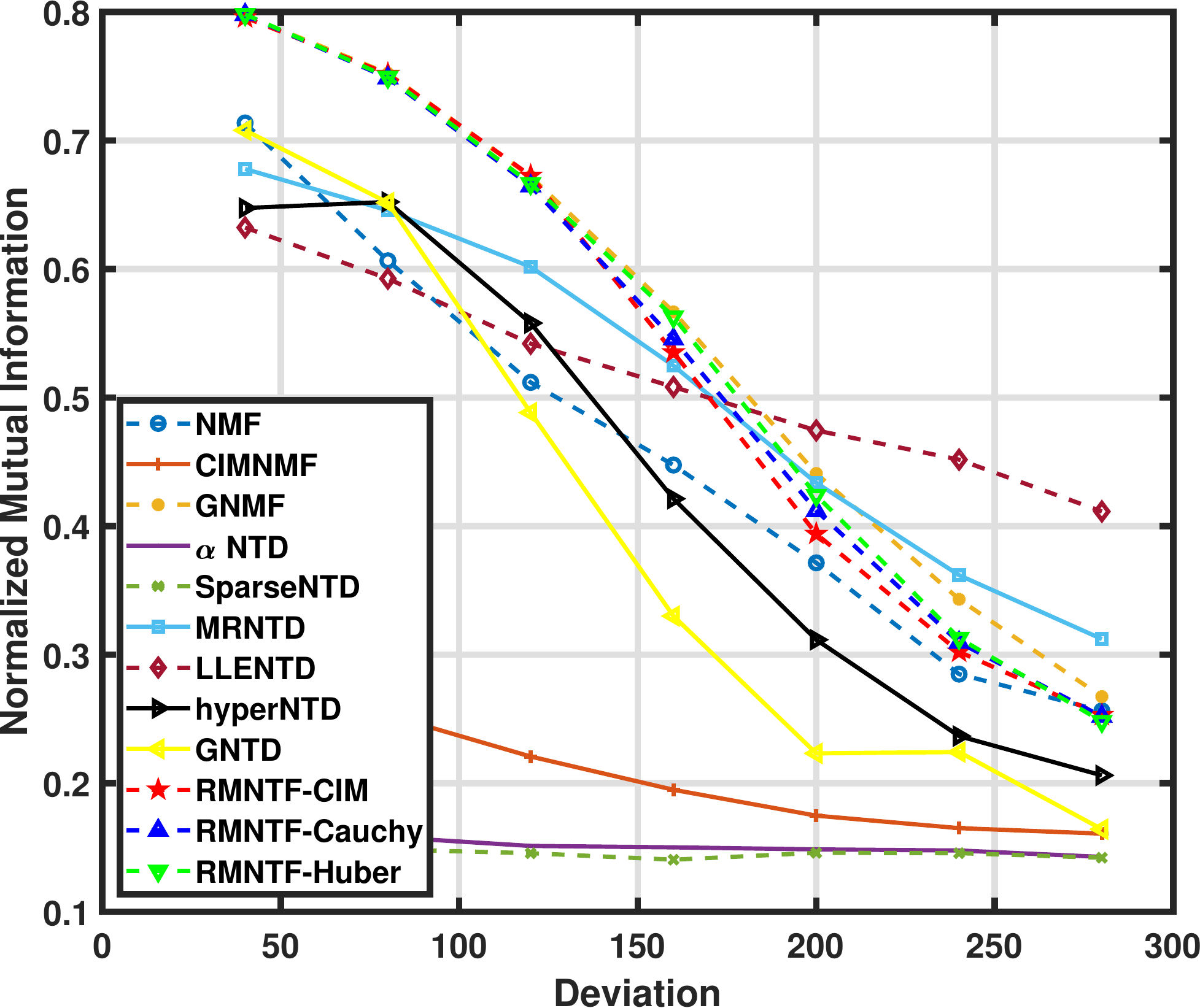}%
\label{Reuters_delta_Coh} \\
\includegraphics[width=1.5in]{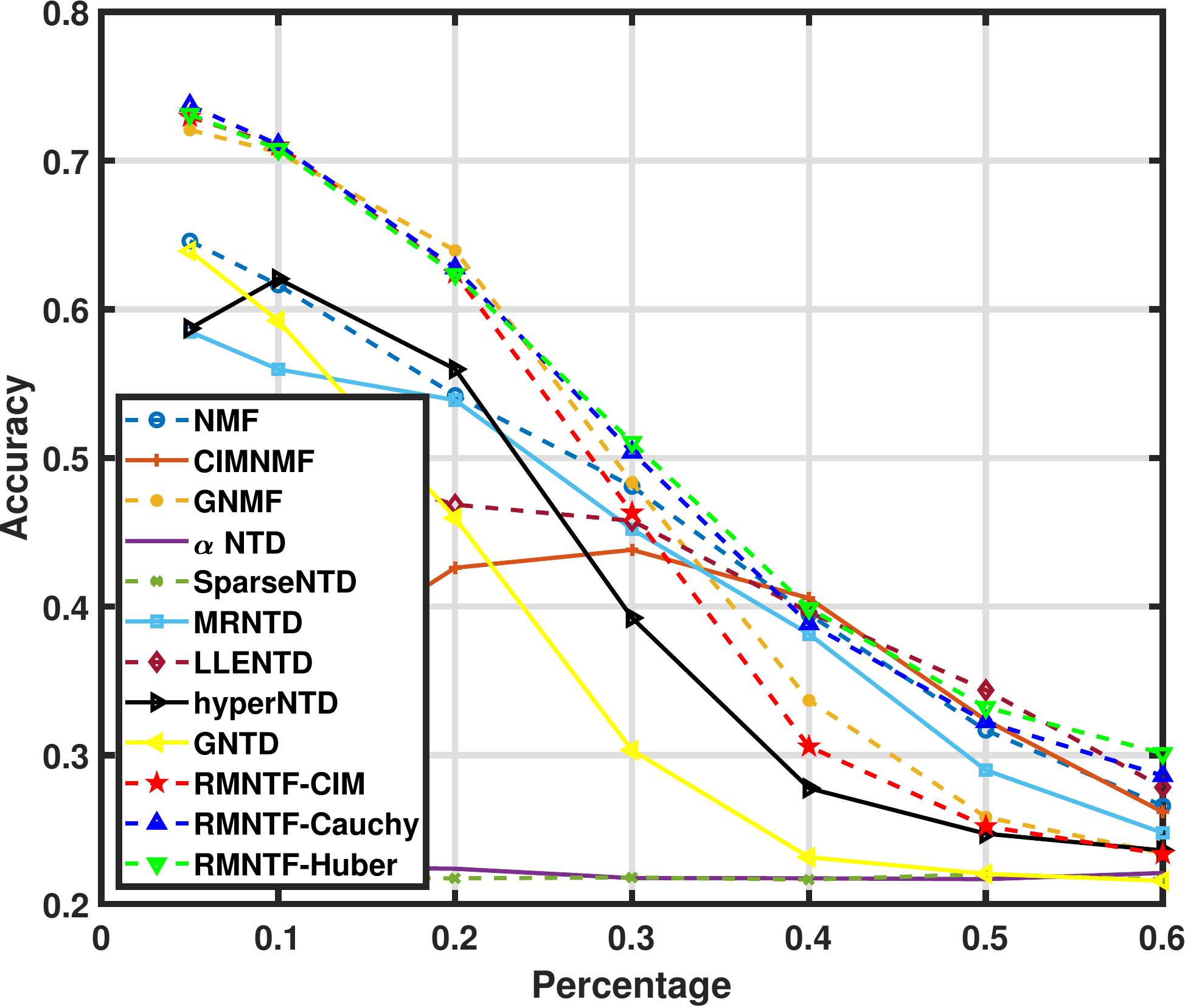}%
\label{Reuters_delta_Coh} \\
\includegraphics[width=1.5in]{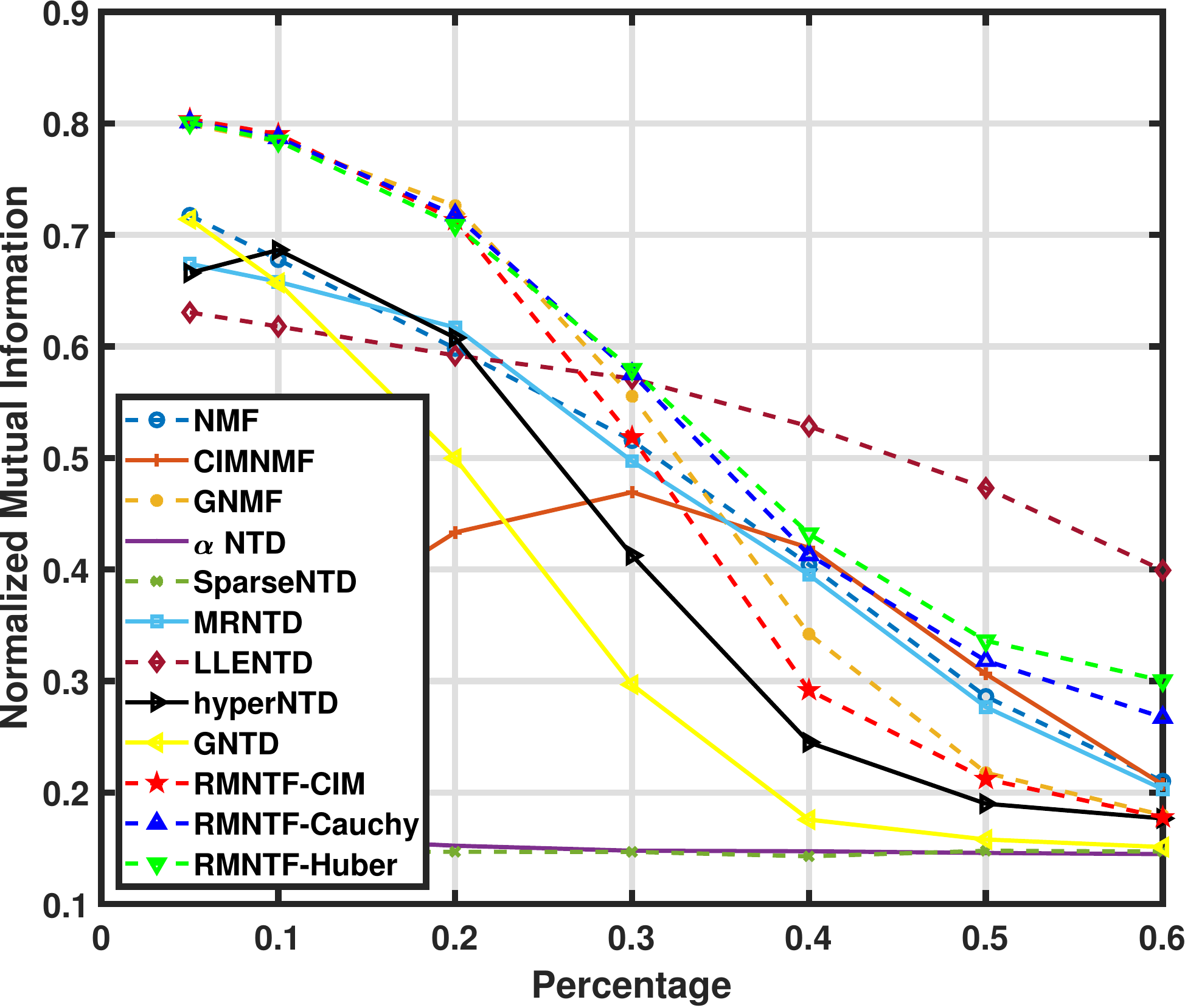}%
\end{minipage}
}
\subfloat[]{
\begin{minipage}[b]{0.22\textwidth}
\includegraphics[width=1.5in]{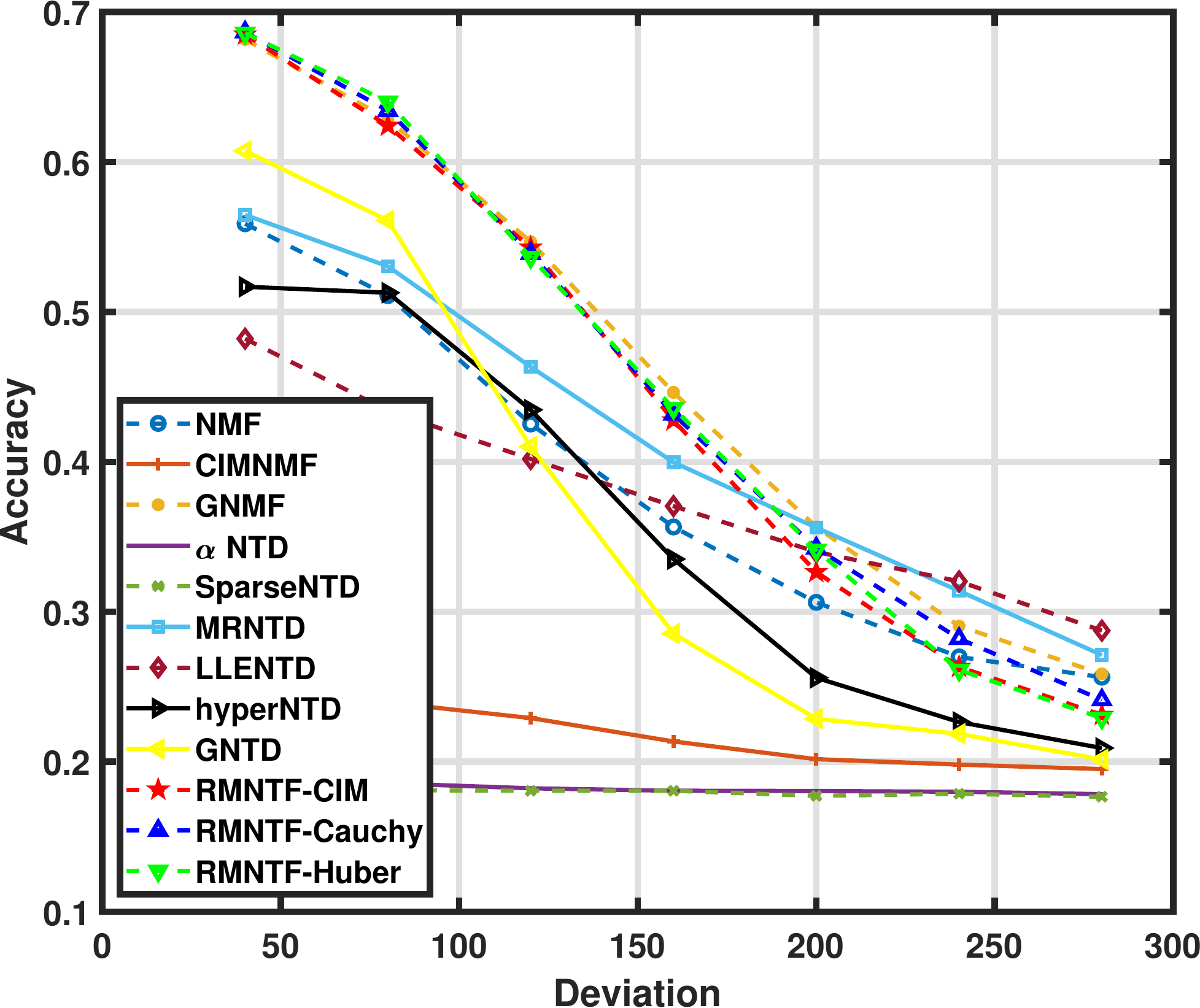}%
\label{Reuters_delta_Coh} \\
\includegraphics[width=1.5in]{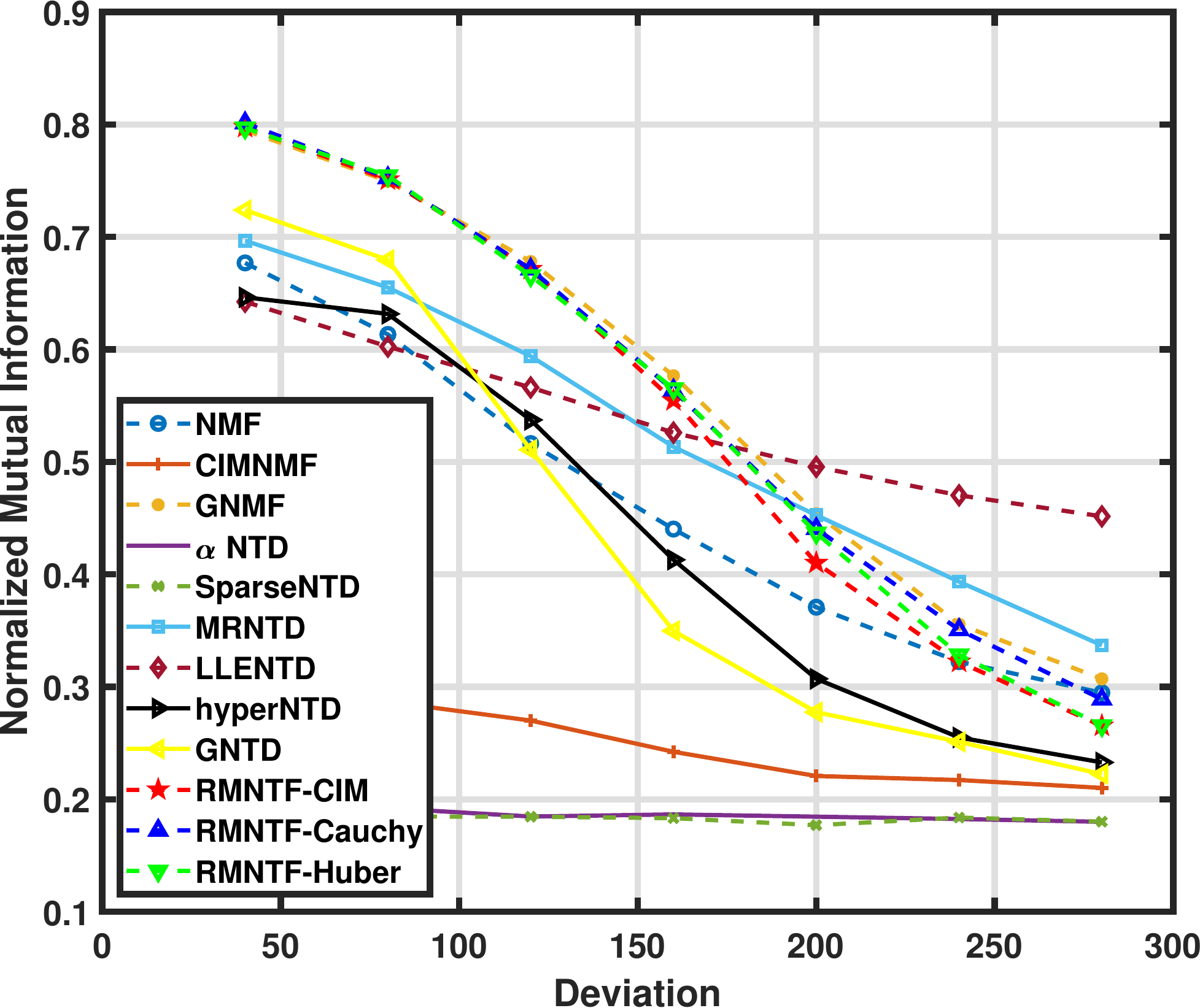}%
\label{Reuters_delta_Coh} \\
\includegraphics[width=1.5in]{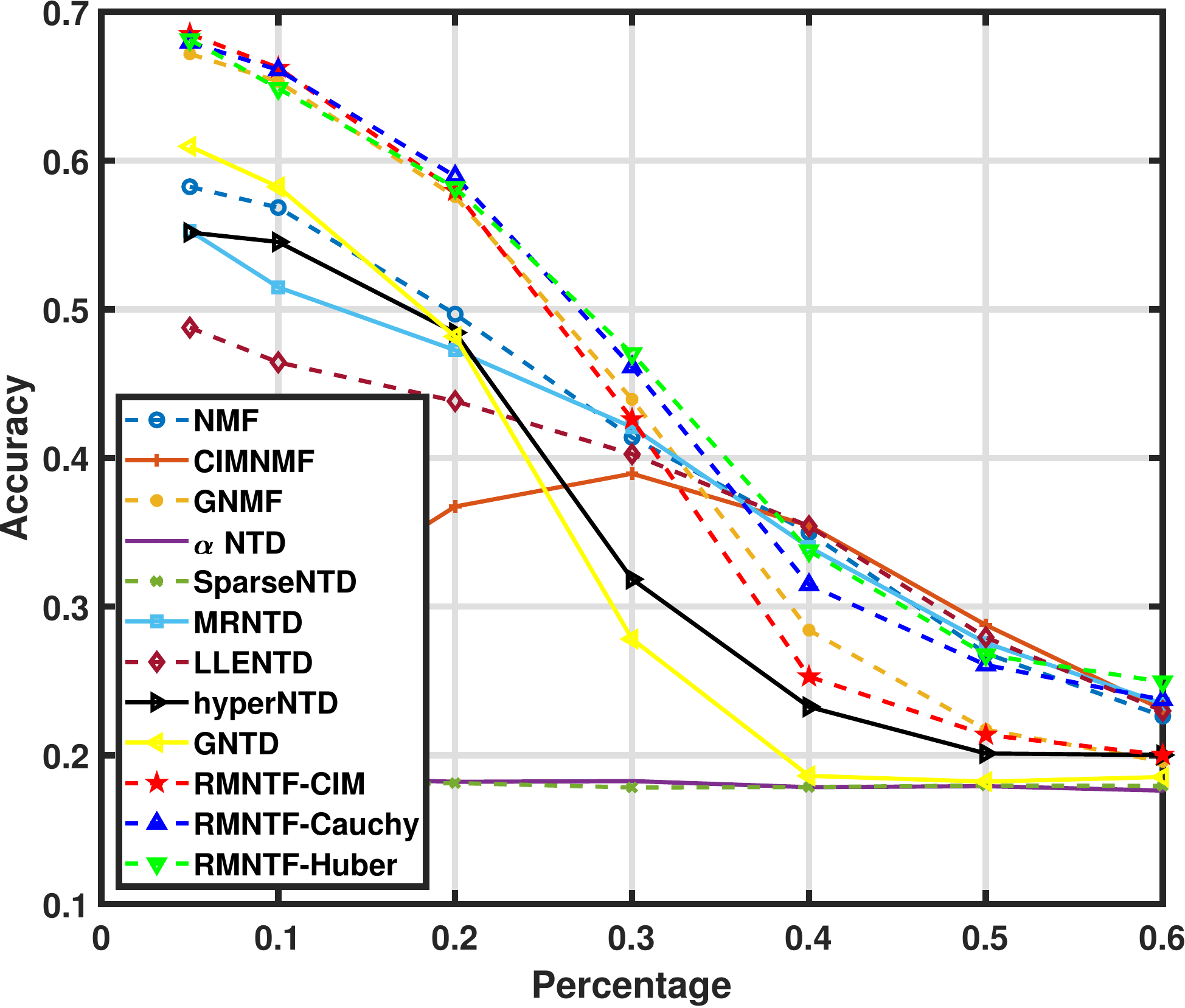}%
\label{Reuters_delta_Coh} \\
\includegraphics[width=1.5in]{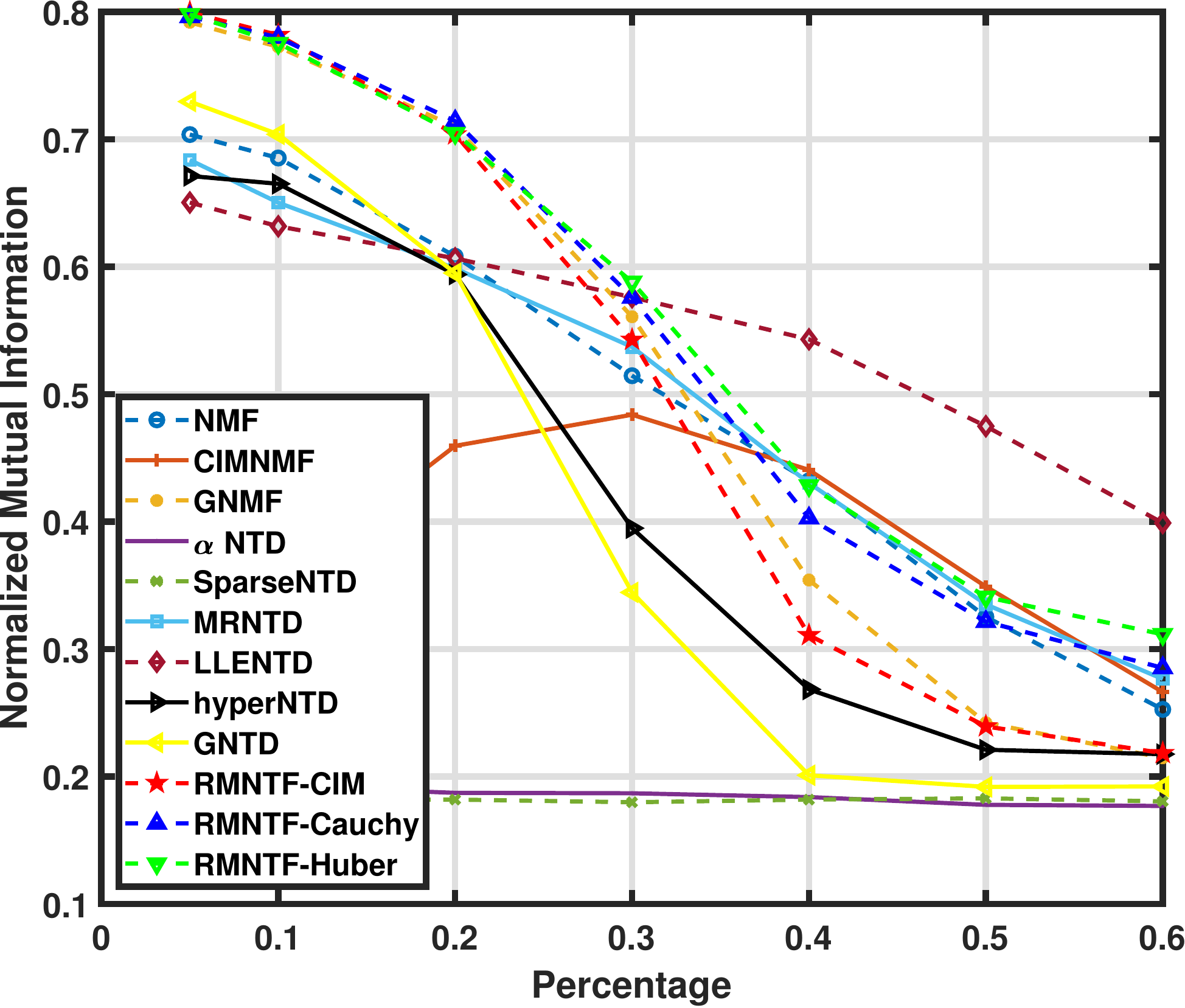}%
\end{minipage}
}
\subfloat[]{
\begin{minipage}[b]{0.22\textwidth}
\includegraphics[width=1.5in]{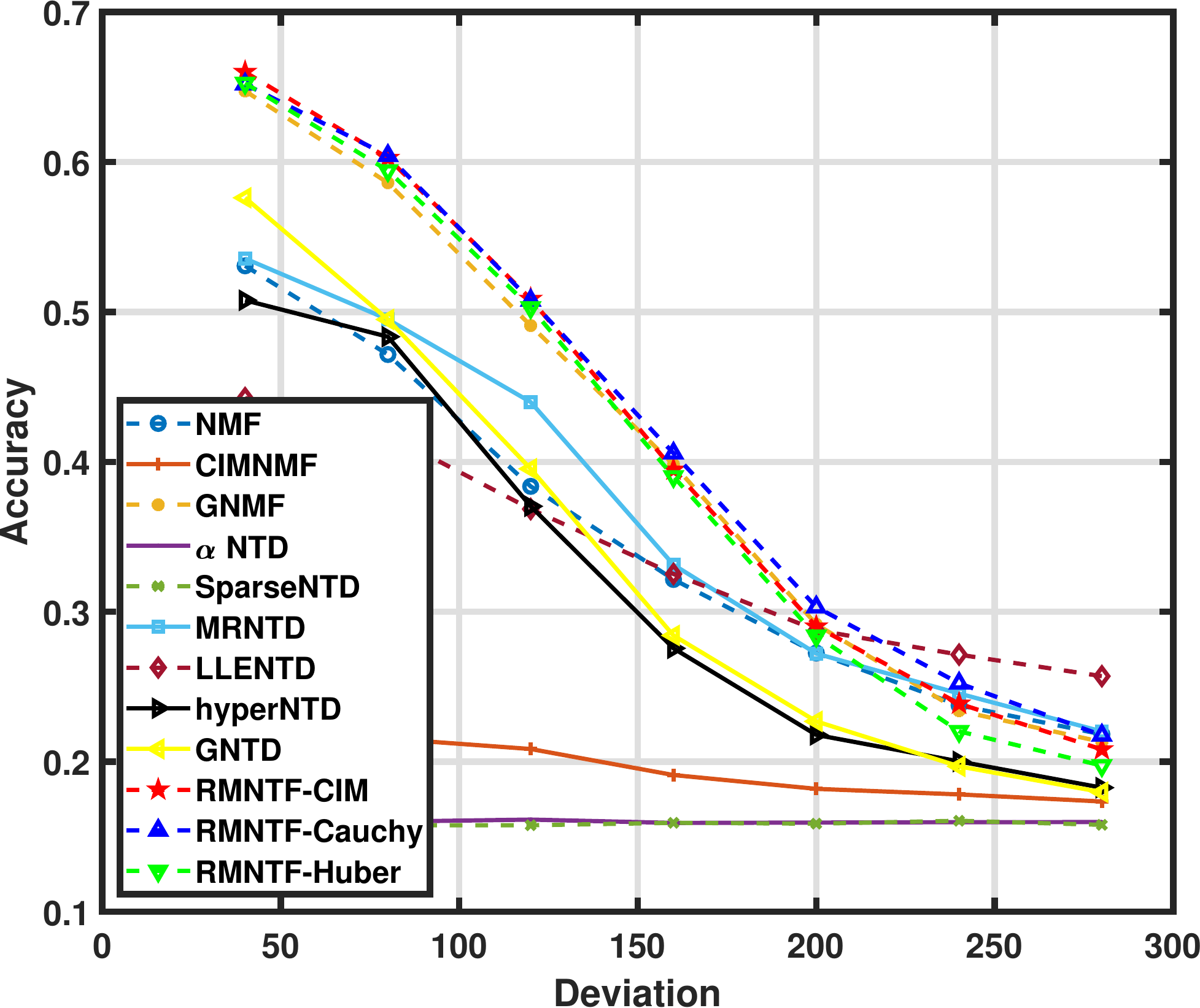}%
\label{Reuters_delta_Coh} \\
\includegraphics[width=1.5in]{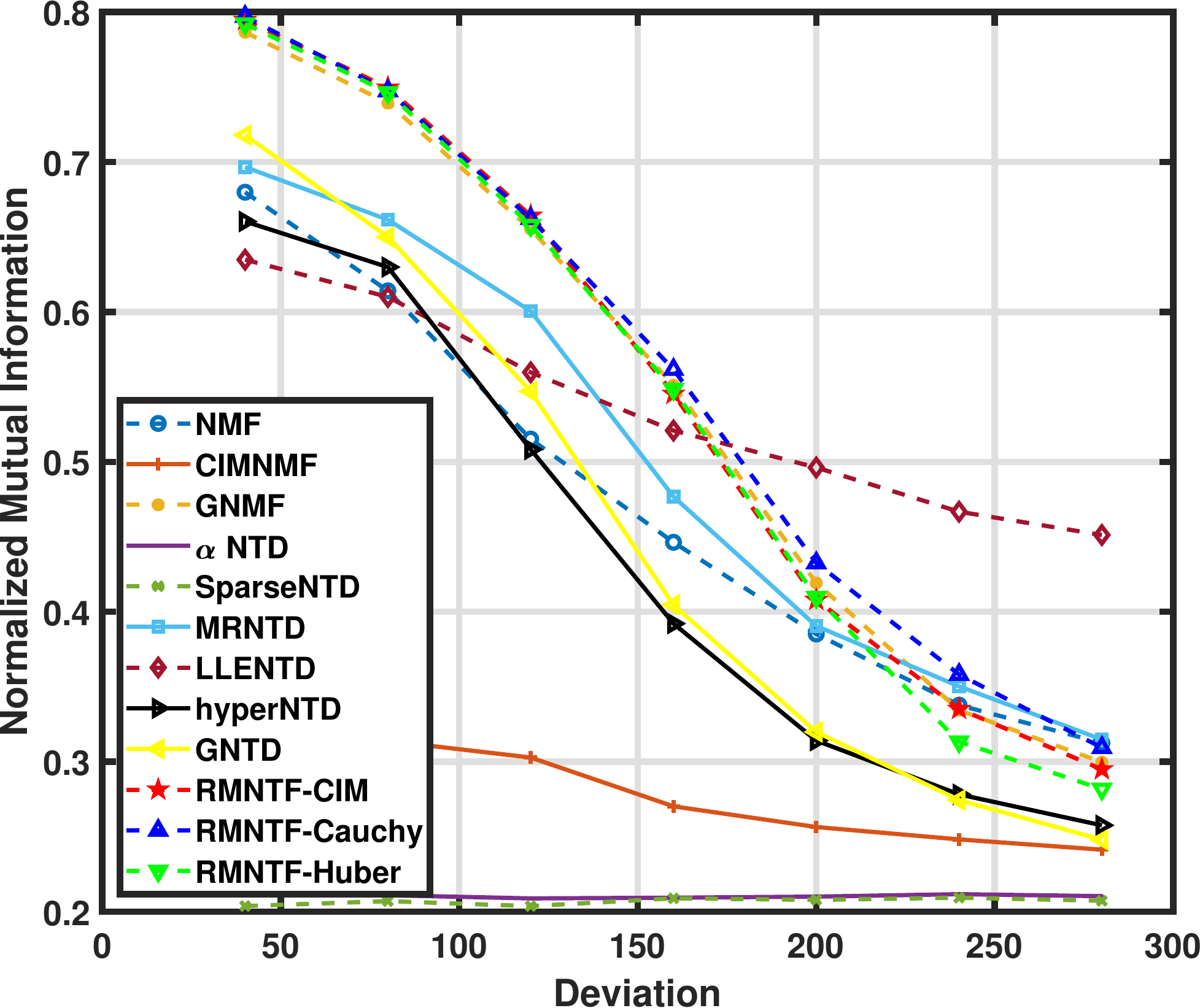}%
\label{Reuters_delta_Coh} \\
\includegraphics[width=1.5in]{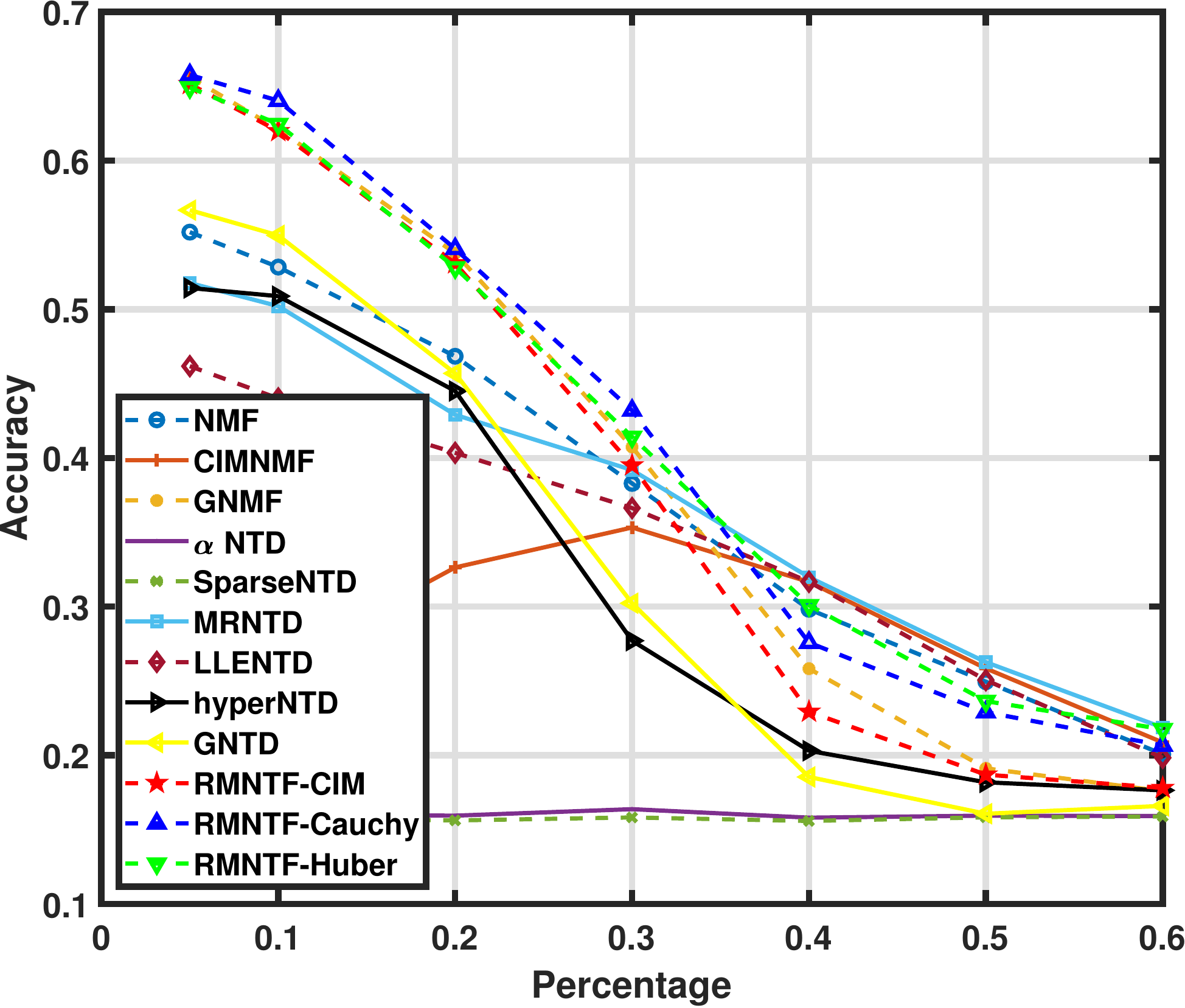}%
\label{Reuters_delta_Coh} \\
\includegraphics[width=1.5in]{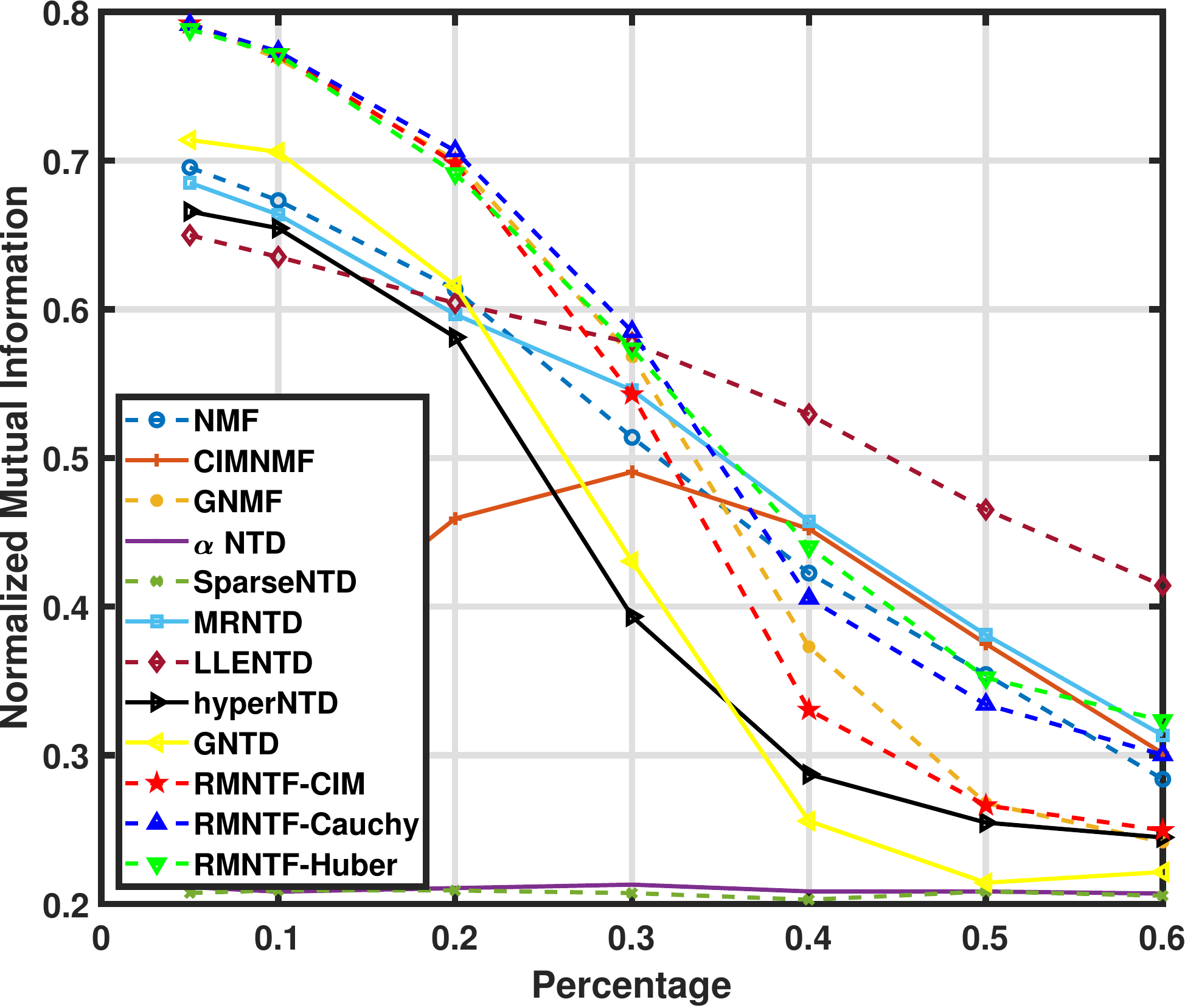}%
\end{minipage}
}

\vspace{2mm}
\caption{Evaluation of proposed methods on FEI database contaminated by Laplace noise and salt \& pepper noise, respectively, the first two rows show the results contaminated by Laplace noise and the last two rows show the results contaminated by salt \& pepper noise. (a) Average accuracy and NMI on the subset of FEI which contains $5$ categories. (b) Average evaluation on the subset of FEI which contains $10$ categories. (c) Average evaluation on the subset of FEI which contains $15$ categories. (d) Average evaluation on the subset of FEI which contains $20$ categories.}
\label{fig:Reuters_delta4}
\end{figure*}

\begin{figure*}[!t]
\vspace{-0.5cm} 
\setlength{\abovecaptionskip}{0cm} 
\setlength{\belowcaptionskip}{-0cm} 
\centering
\subfloat[]{
\begin{minipage}[b]{0.18\textwidth}
\includegraphics[width=1.25in]{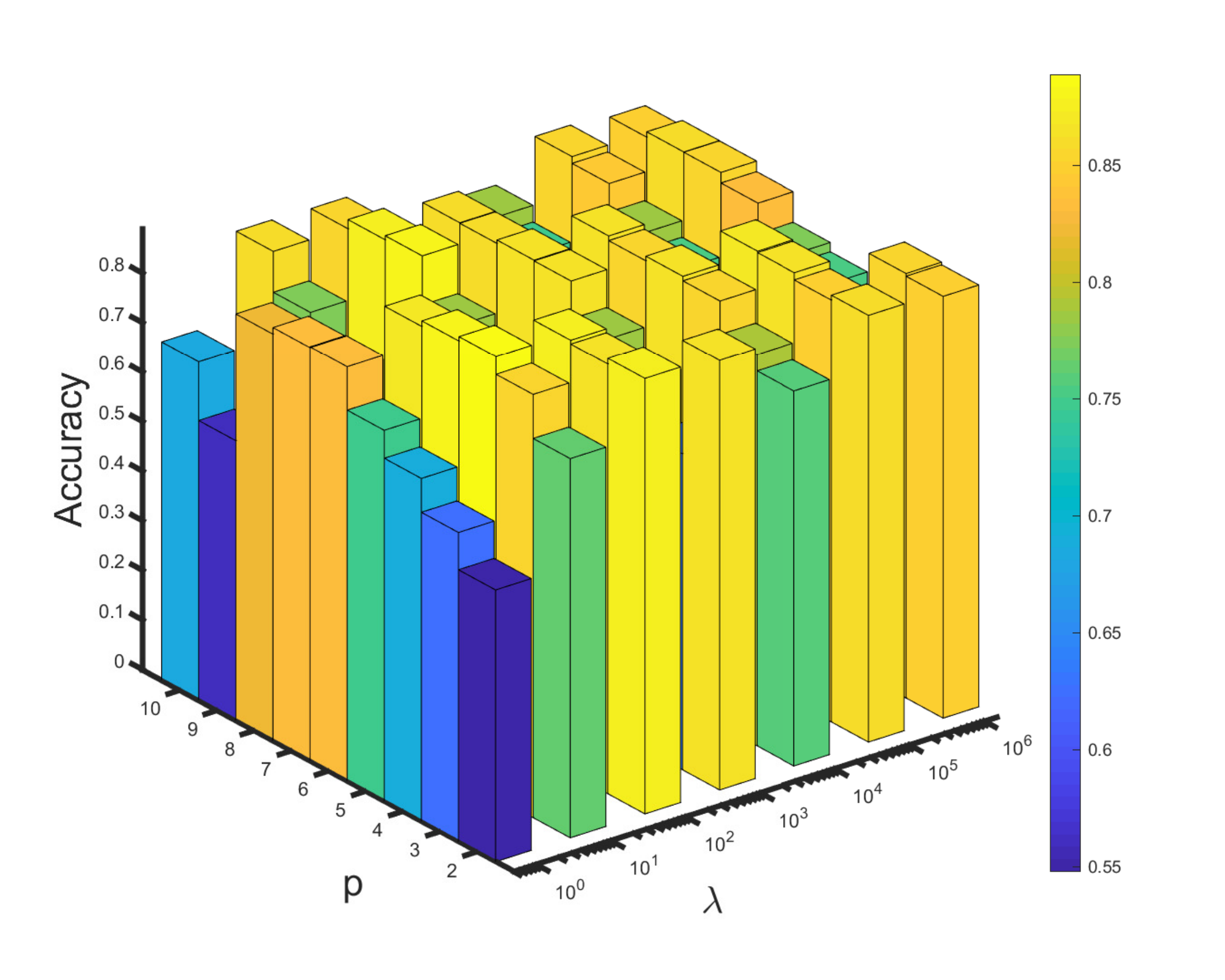}%
\label{Reuters_delta_Coh} \\
\includegraphics[width=1.25in]{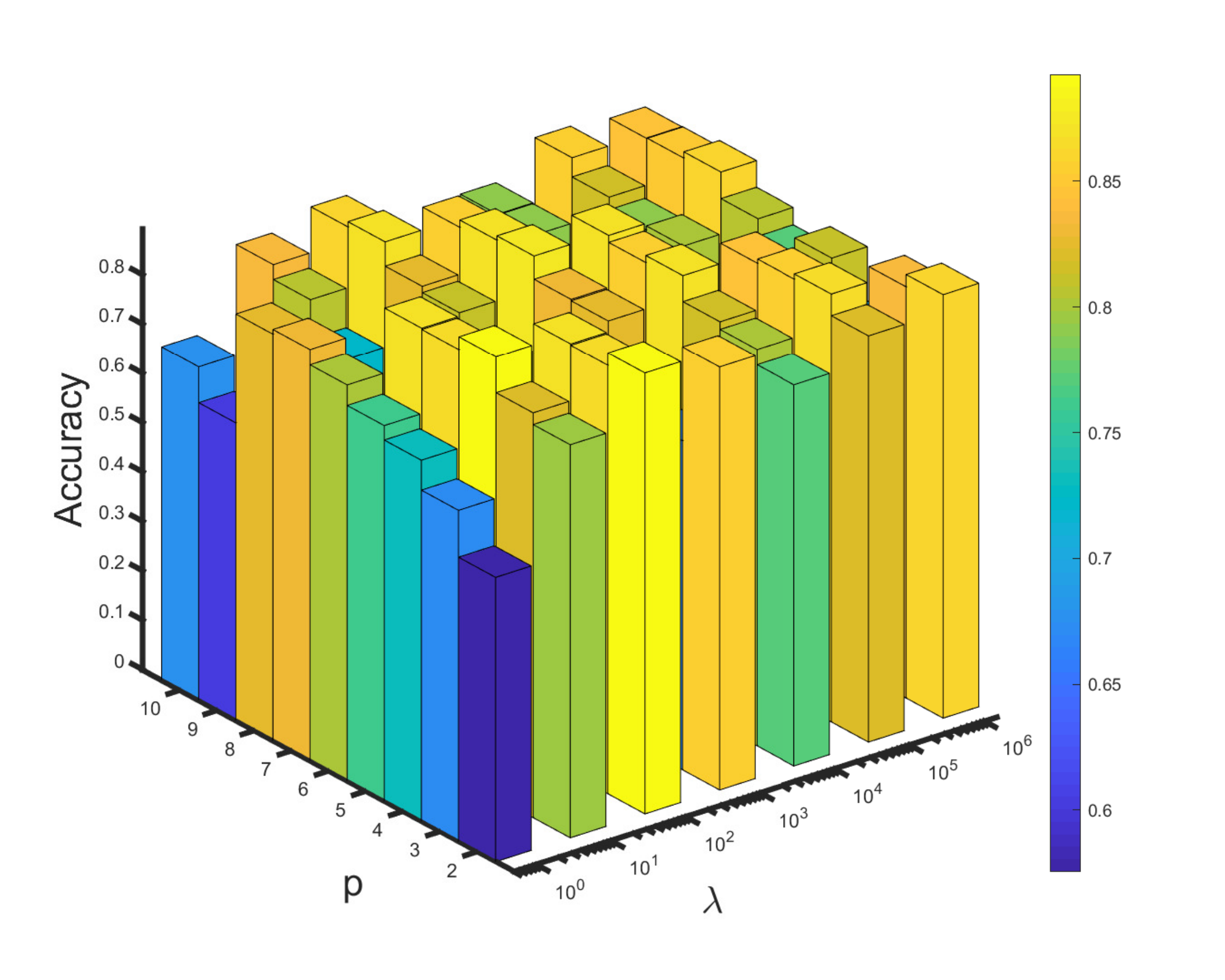}%
\label{Reuters_delta_Coh} \\
\includegraphics[width=1.25in]{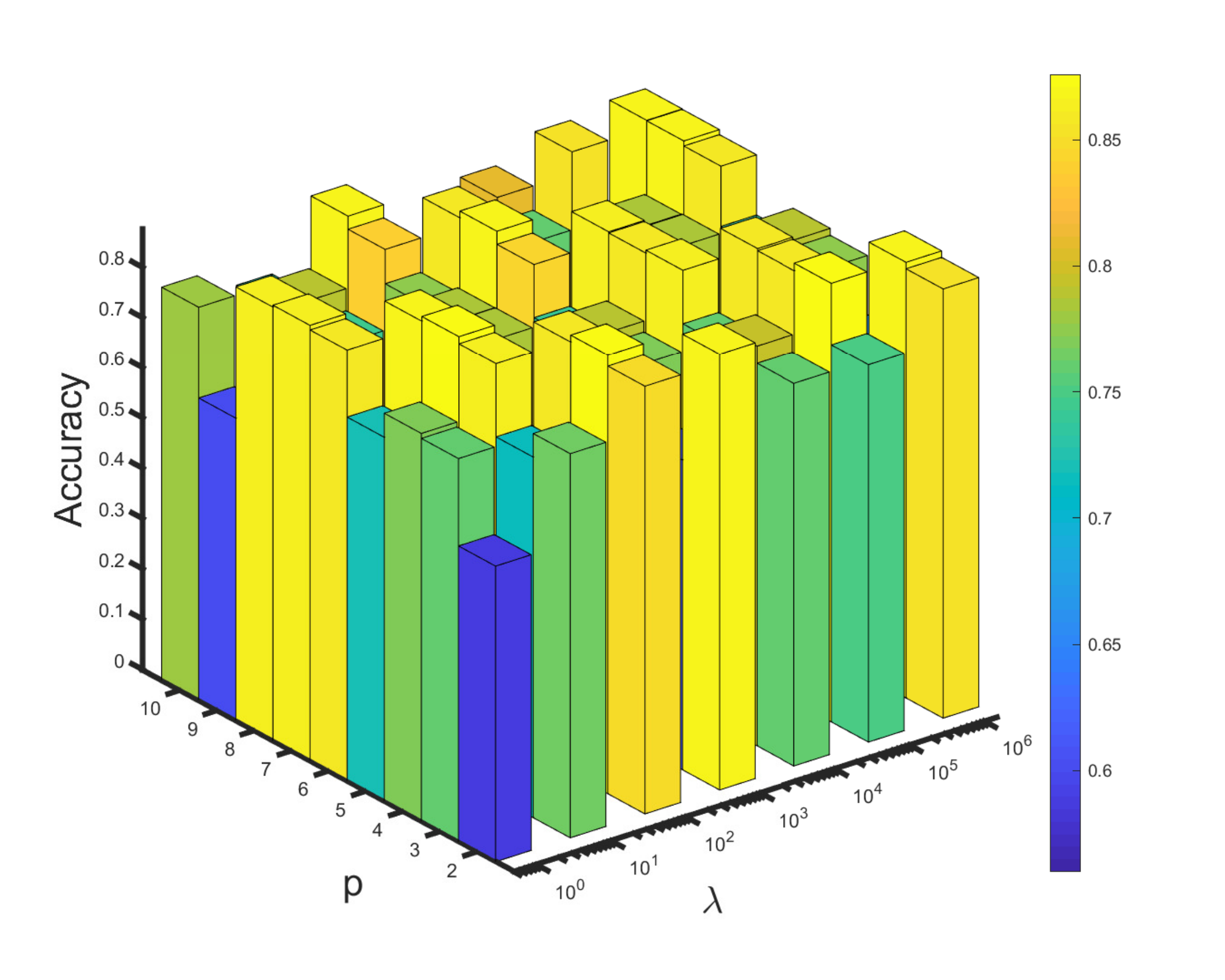}%
\label{Reuters_delta_Coh} \\
\includegraphics[width=1.25in]{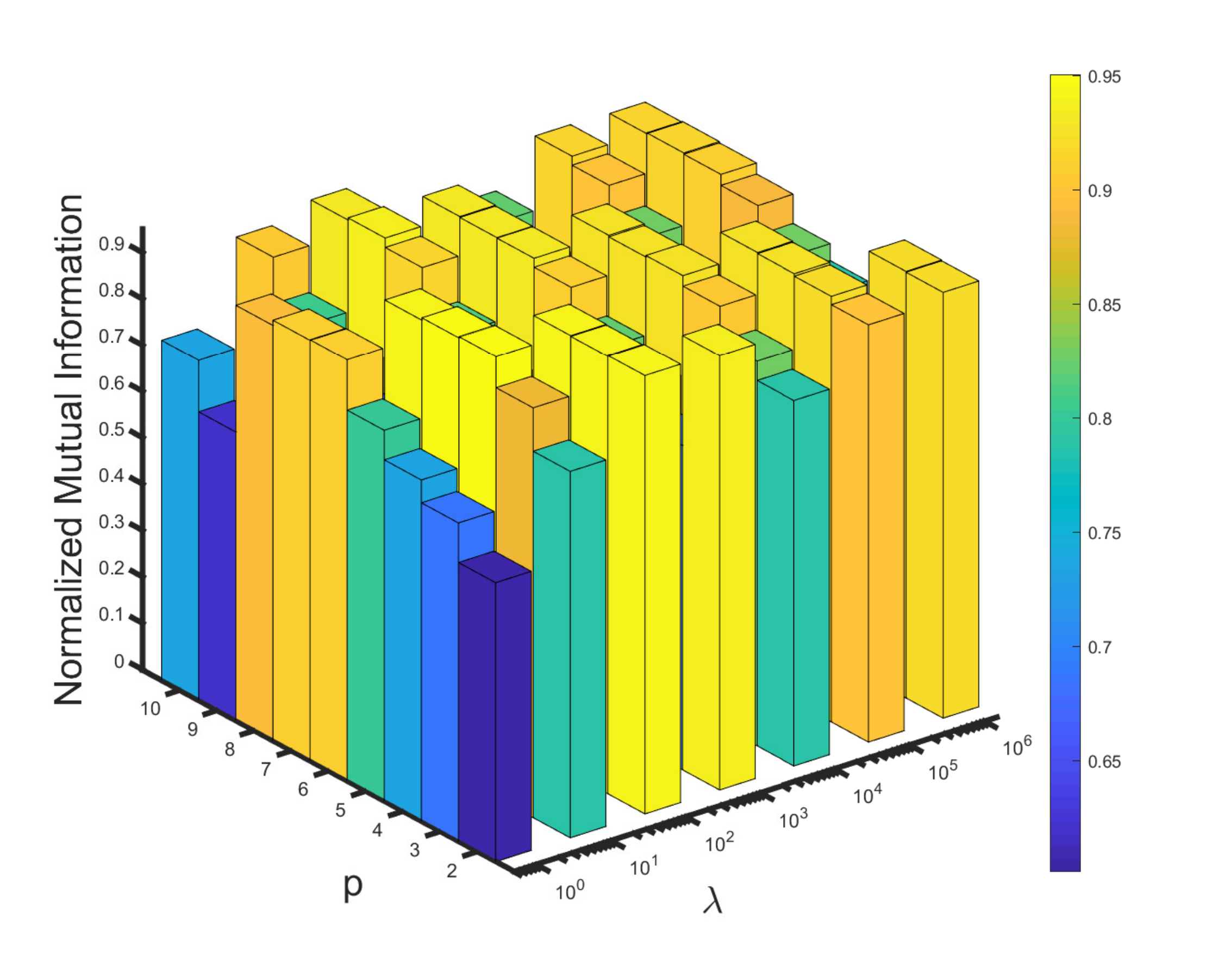}%
\label{Reuters_delta_Coh} \\
\includegraphics[width=1.25in]{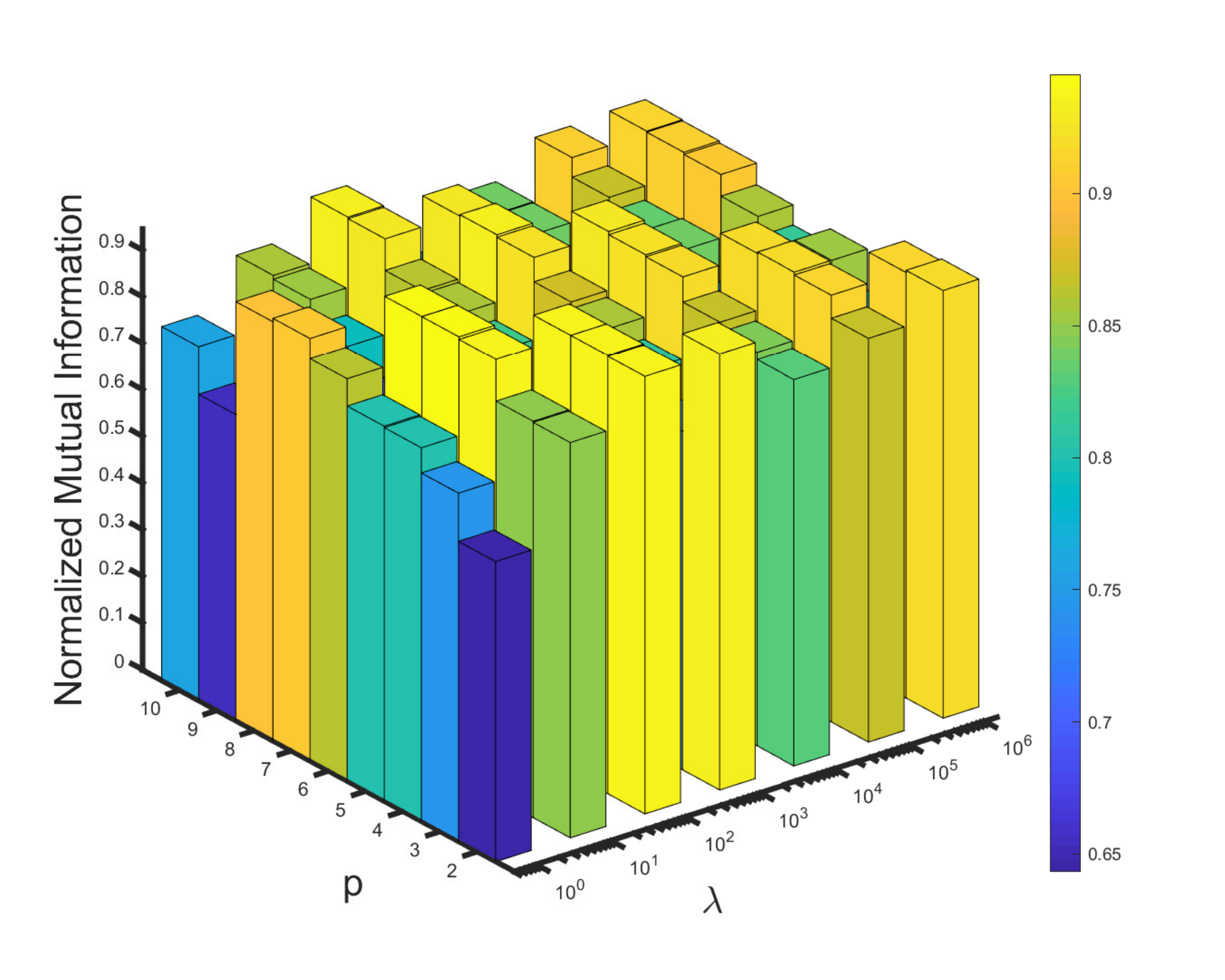}%
\label{Reuters_delta_Coh} \\
\includegraphics[width=1.25in]{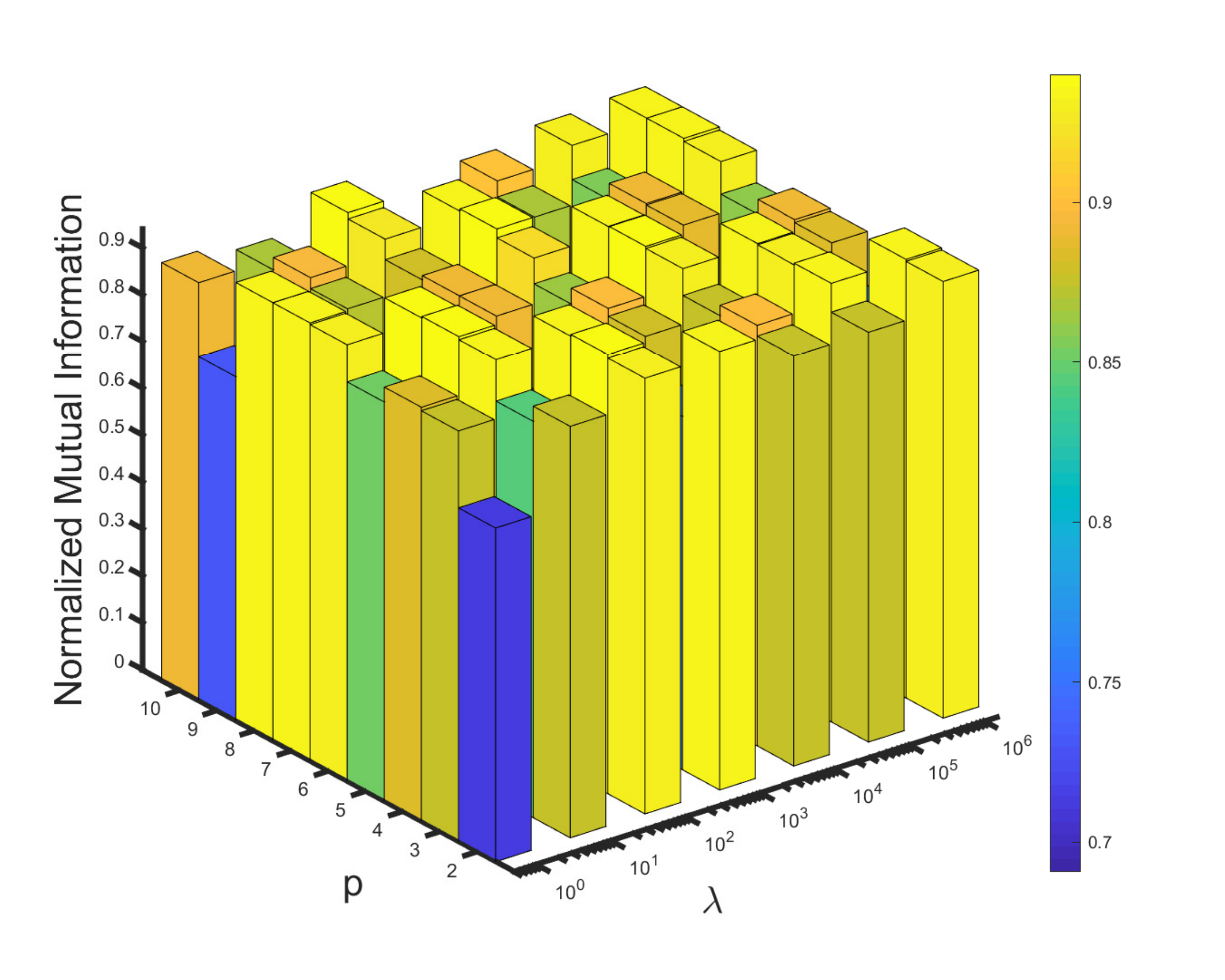}%
\label{Reuters_delta_Coh}
\end{minipage}
}\hspace{-5mm}
\hfil
\subfloat[]{
\begin{minipage}[b]{0.2\textwidth}
\includegraphics[width=1.25in]{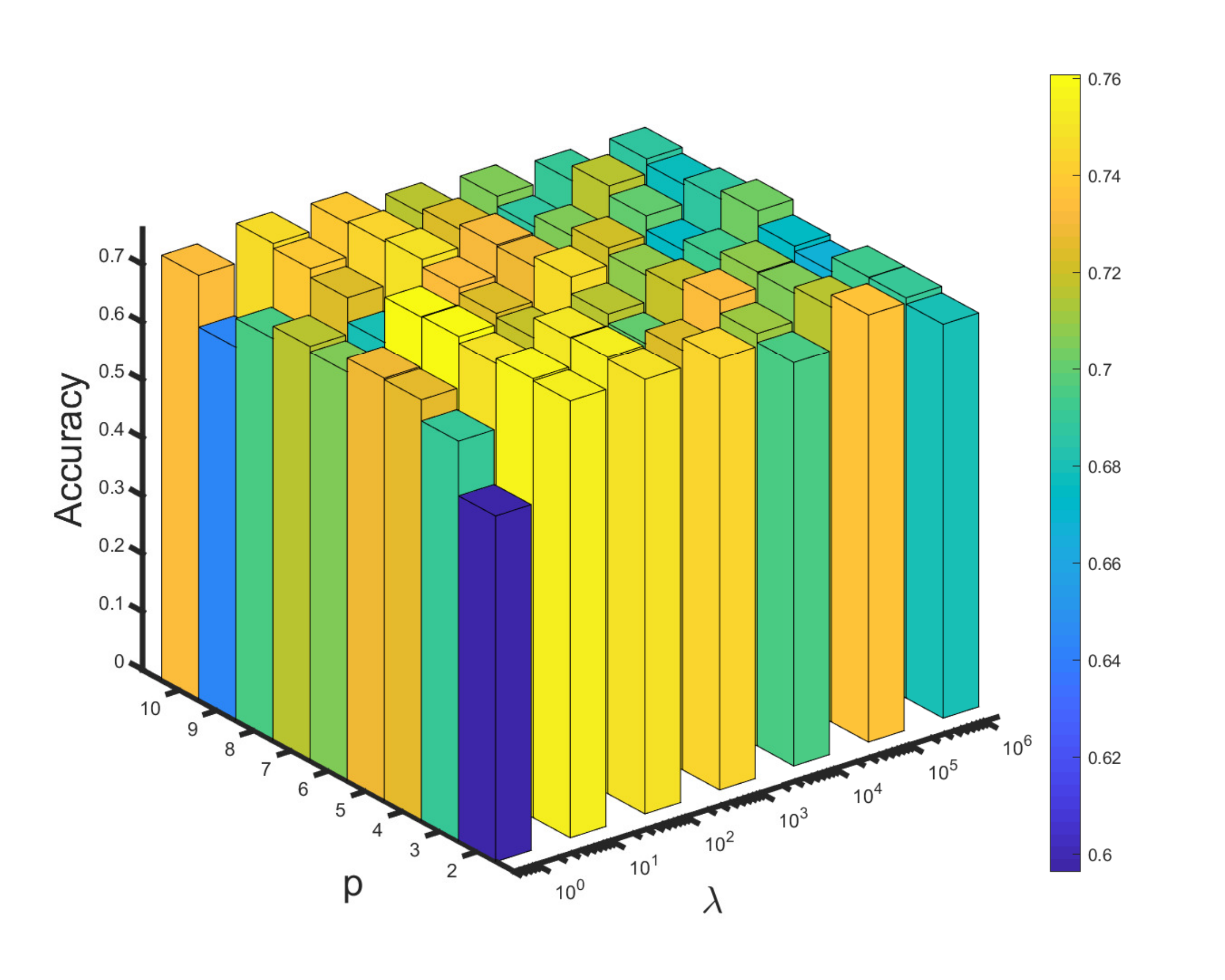}%
\label{Reuters_delta_Coh} \\
\includegraphics[width=1.25in]{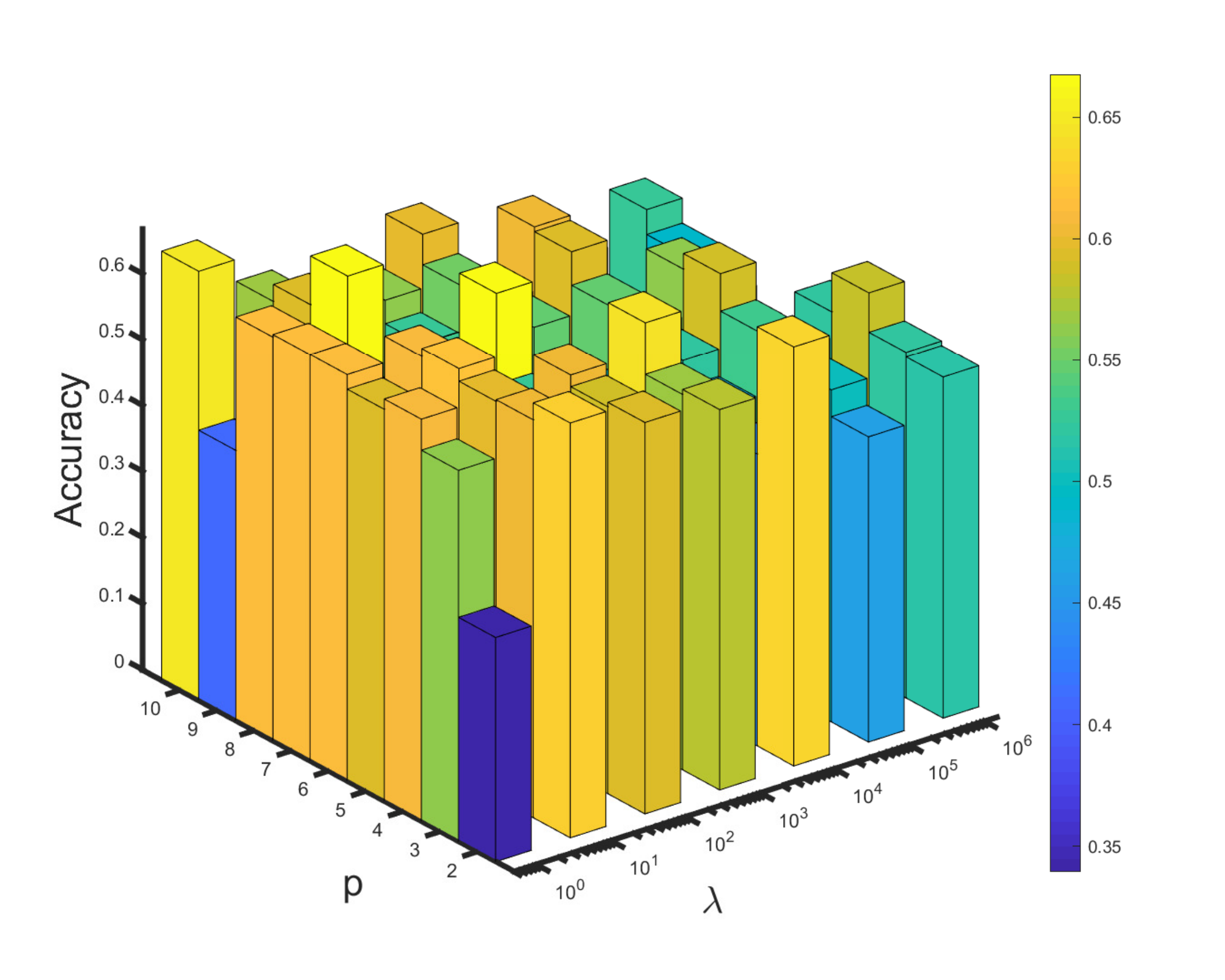}%
\label{Reuters_delta_Coh} \\
\includegraphics[width=1.25in]{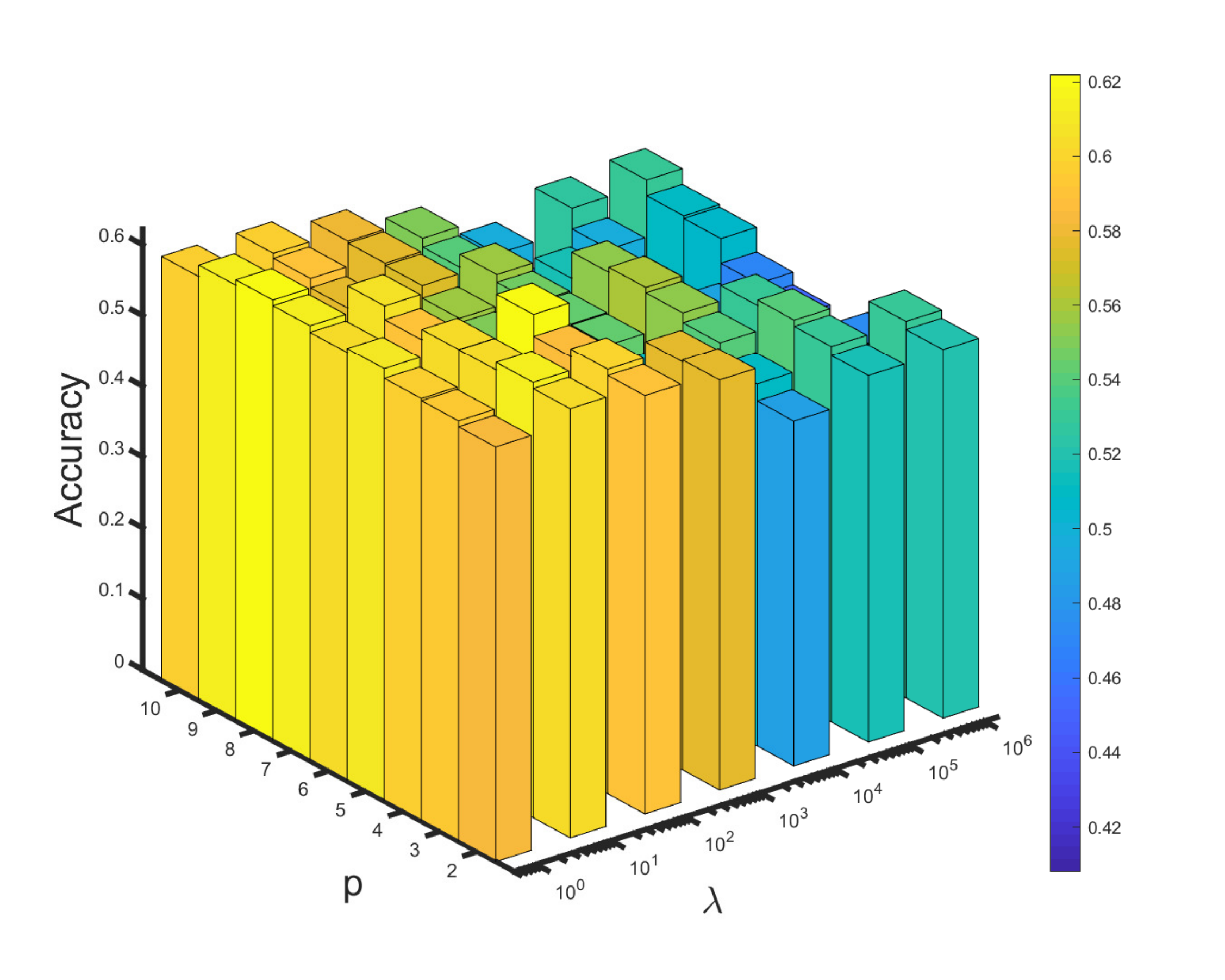}%
\label{Reuters_delta_Coh} \\
\includegraphics[width=1.25in]{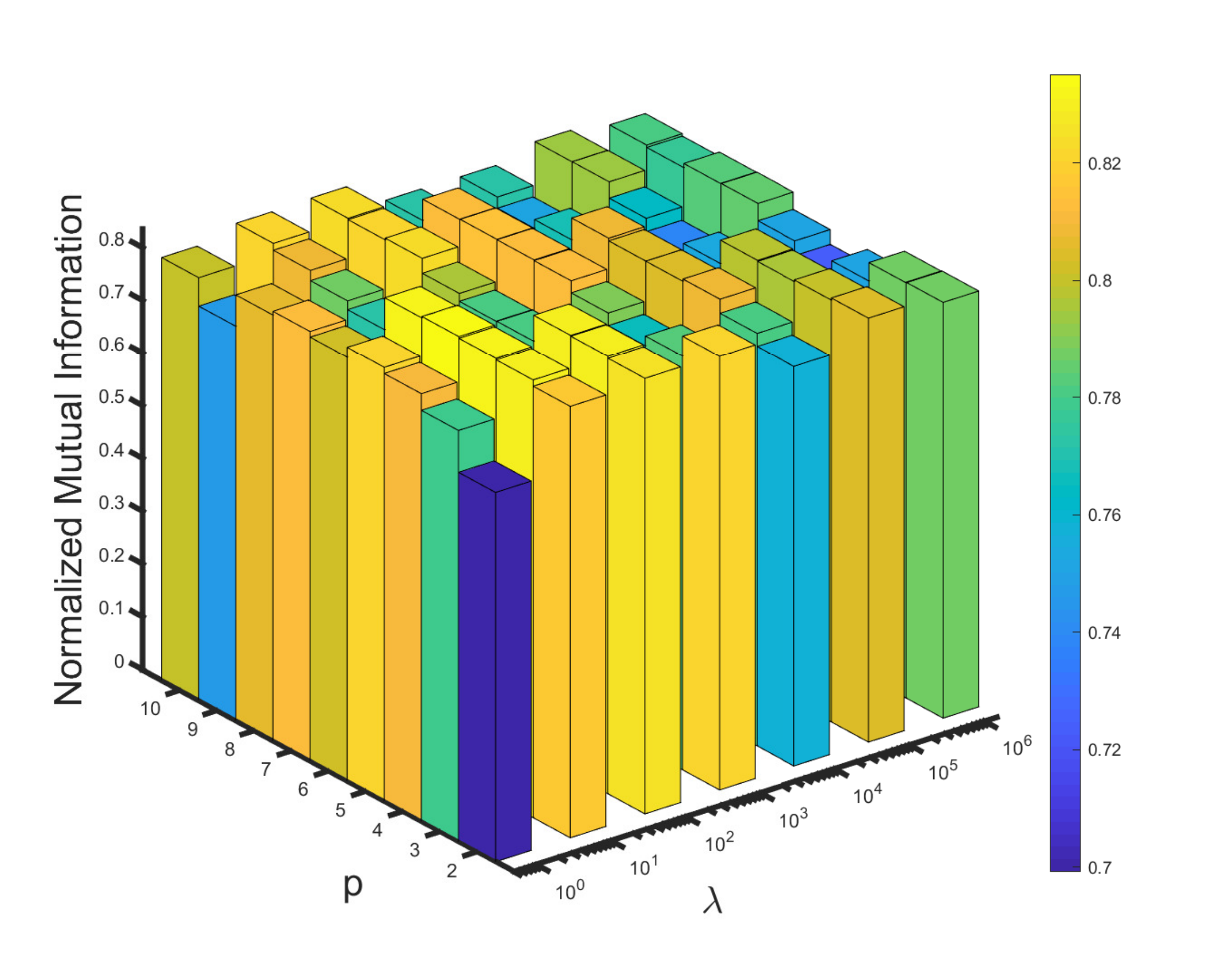}%
\label{Reuters_delta_Coh} \\
\includegraphics[width=1.25in]{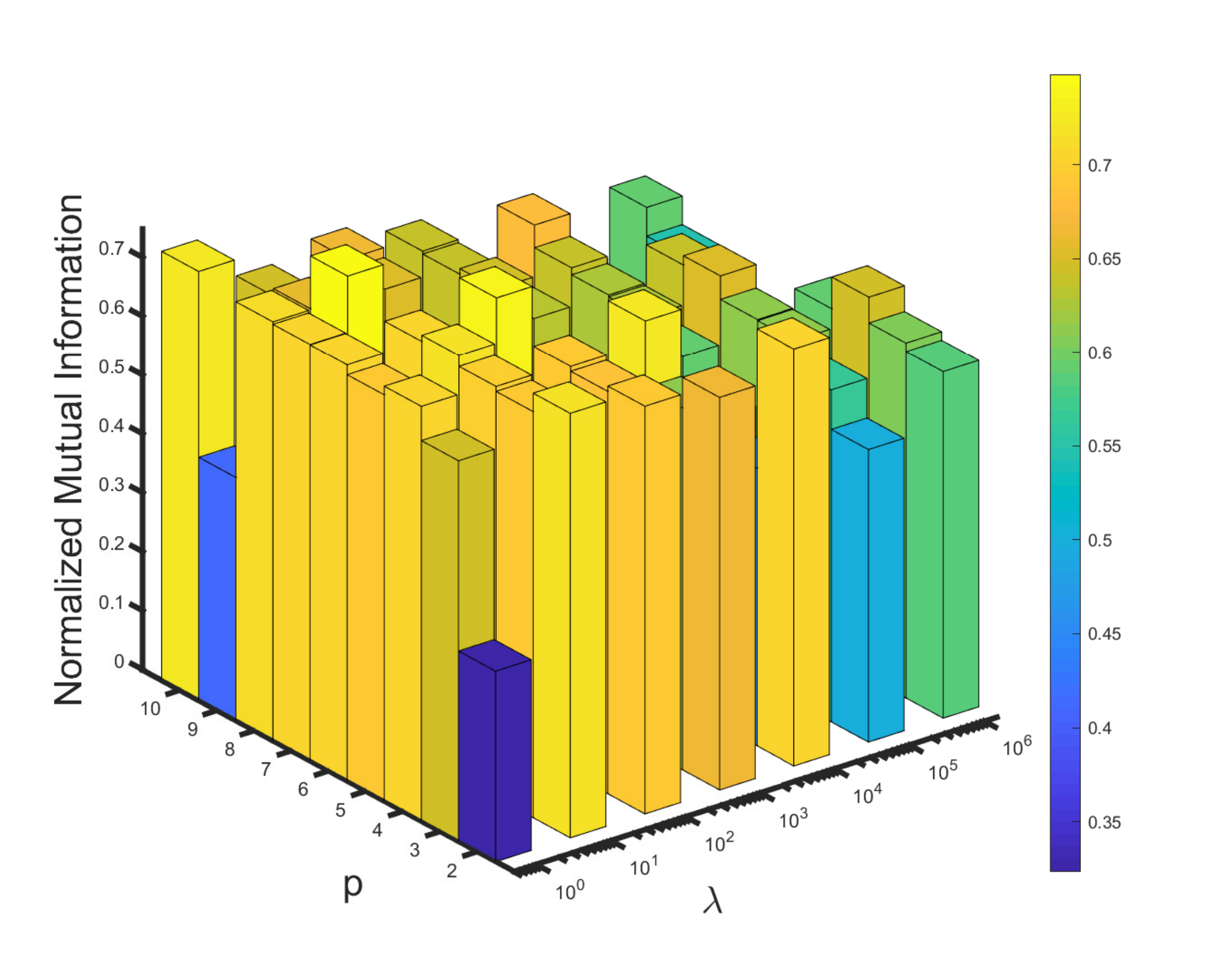}%
\label{Reuters_delta_Coh} \\
\includegraphics[width=1.25in]{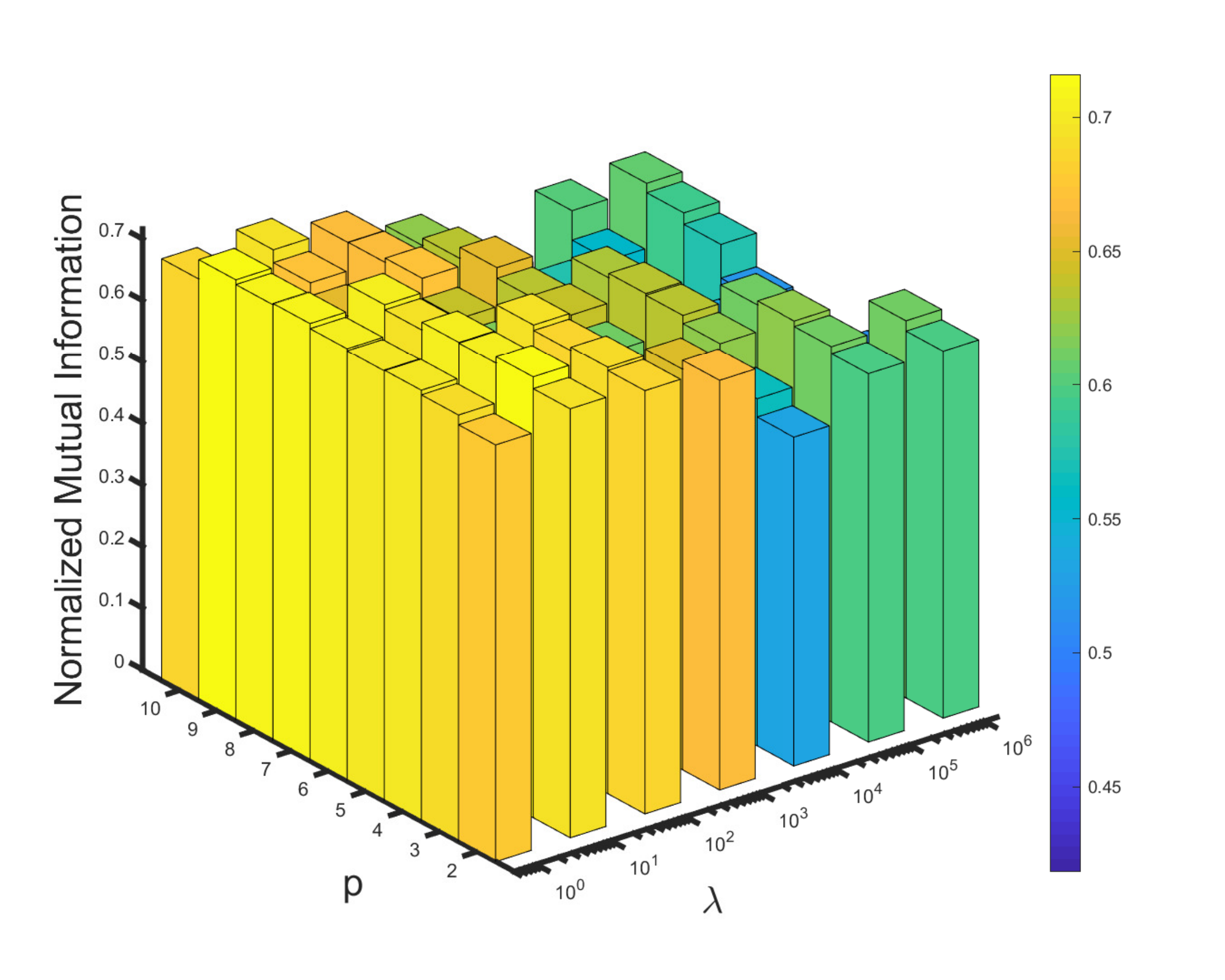}%
\label{Reuters_delta_Coh}
\end{minipage}
}\hspace{-5mm}
\subfloat[]{
\begin{minipage}[b]{0.2\textwidth}
\includegraphics[width=1.25in]{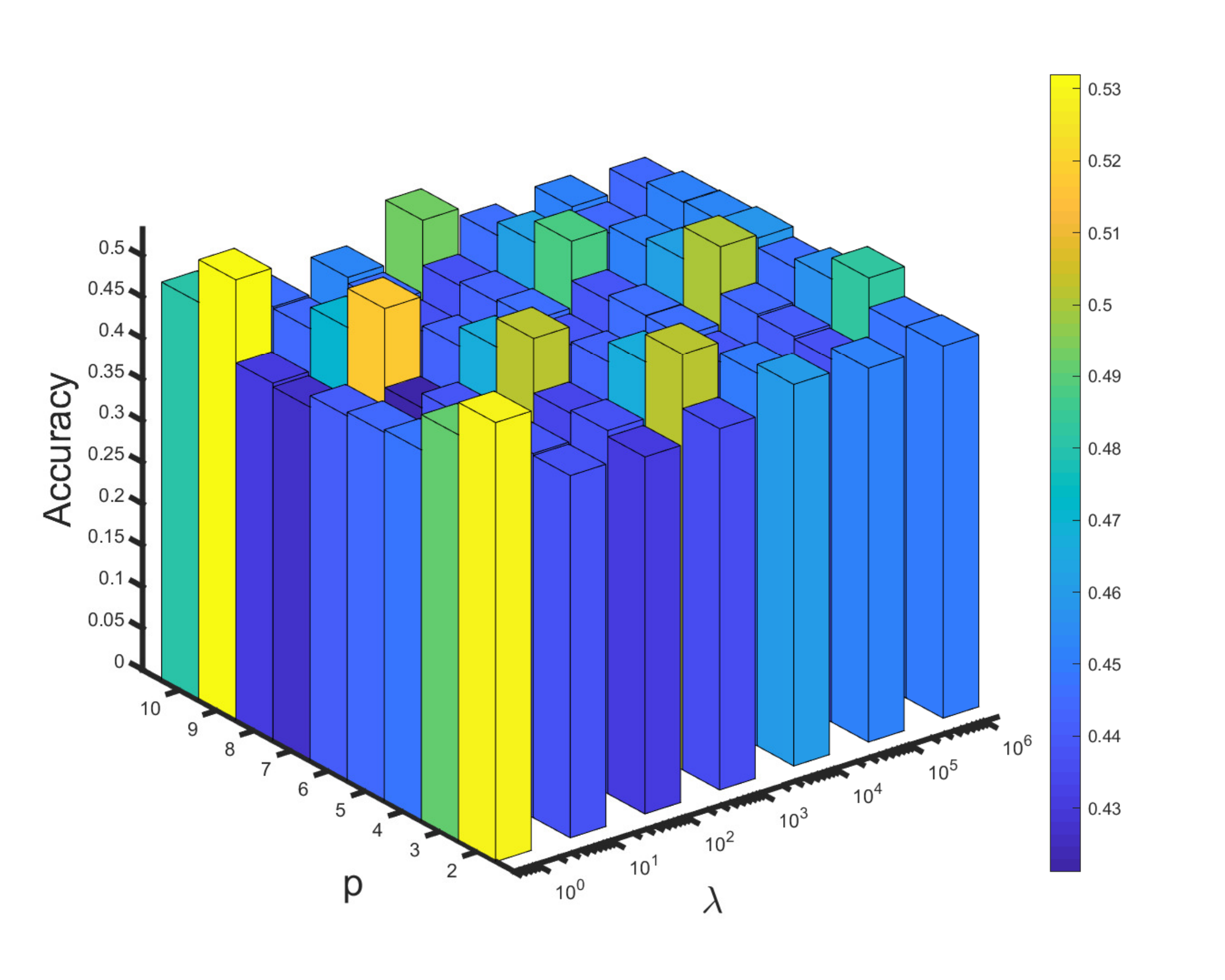}%
\label{Reuters_delta_Coh} \\
\includegraphics[width=1.25in]{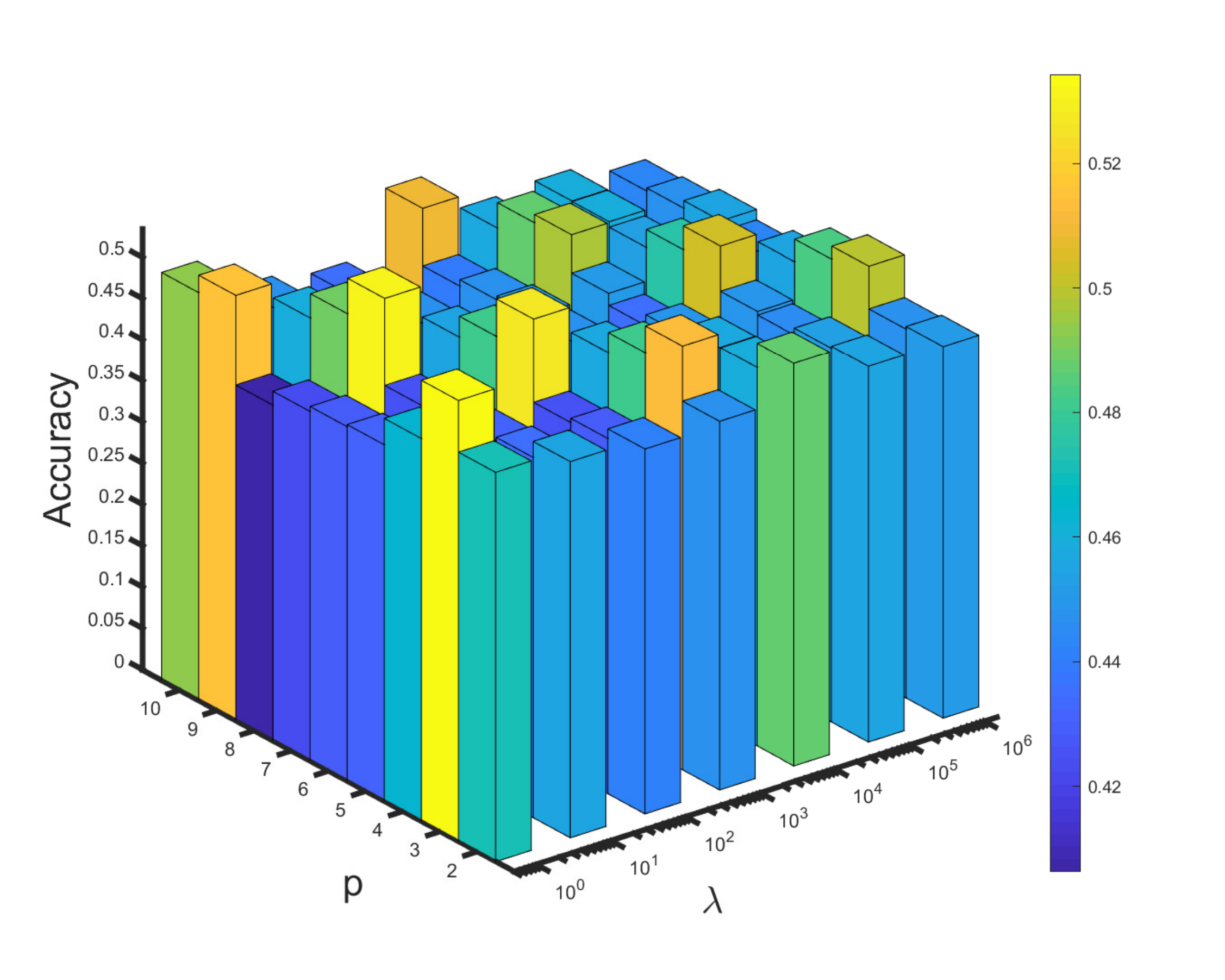}%
\label{Reuters_delta_Coh} \\
\includegraphics[width=1.25in]{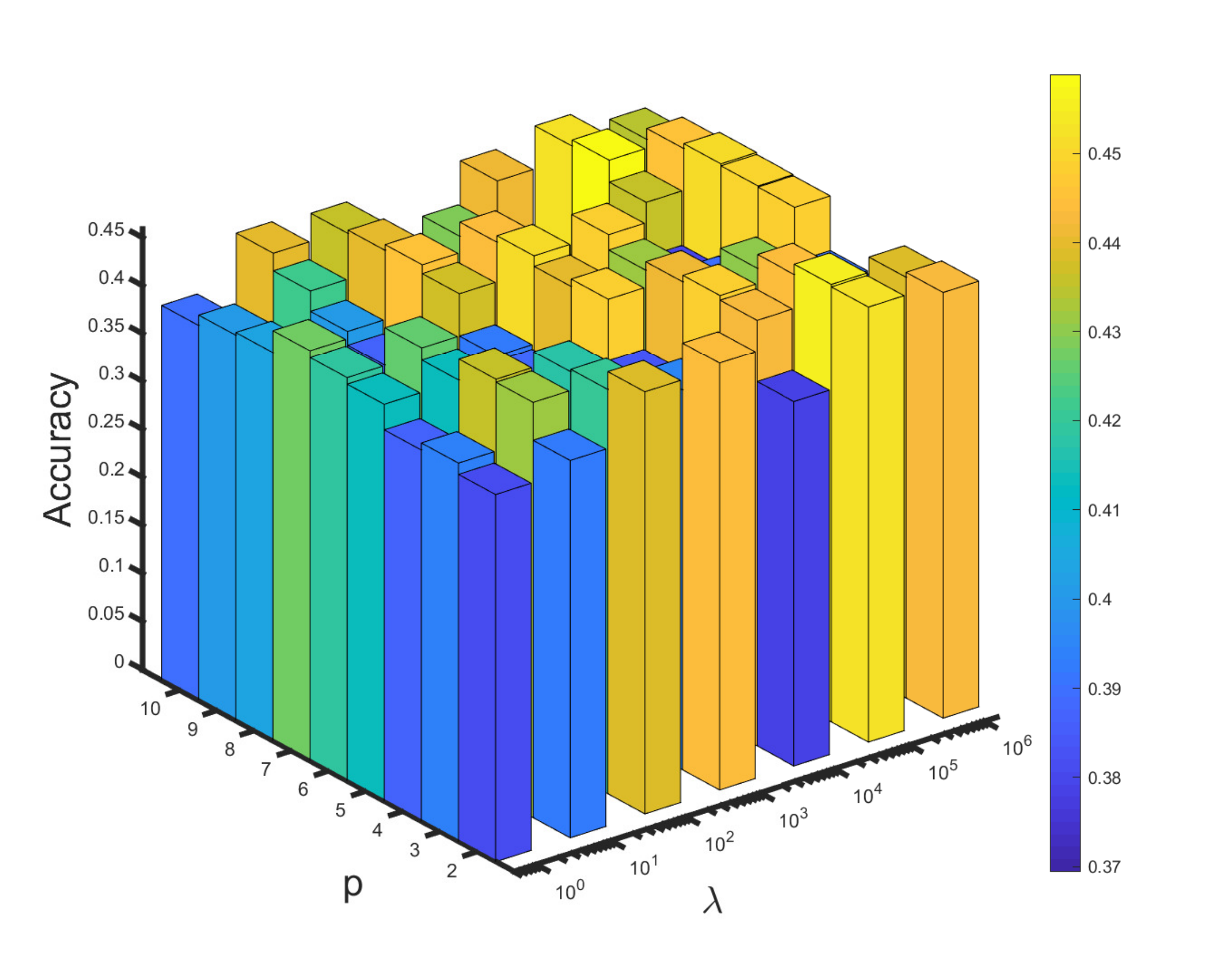}%
\label{Reuters_delta_Coh} \\
\includegraphics[width=1.25in]{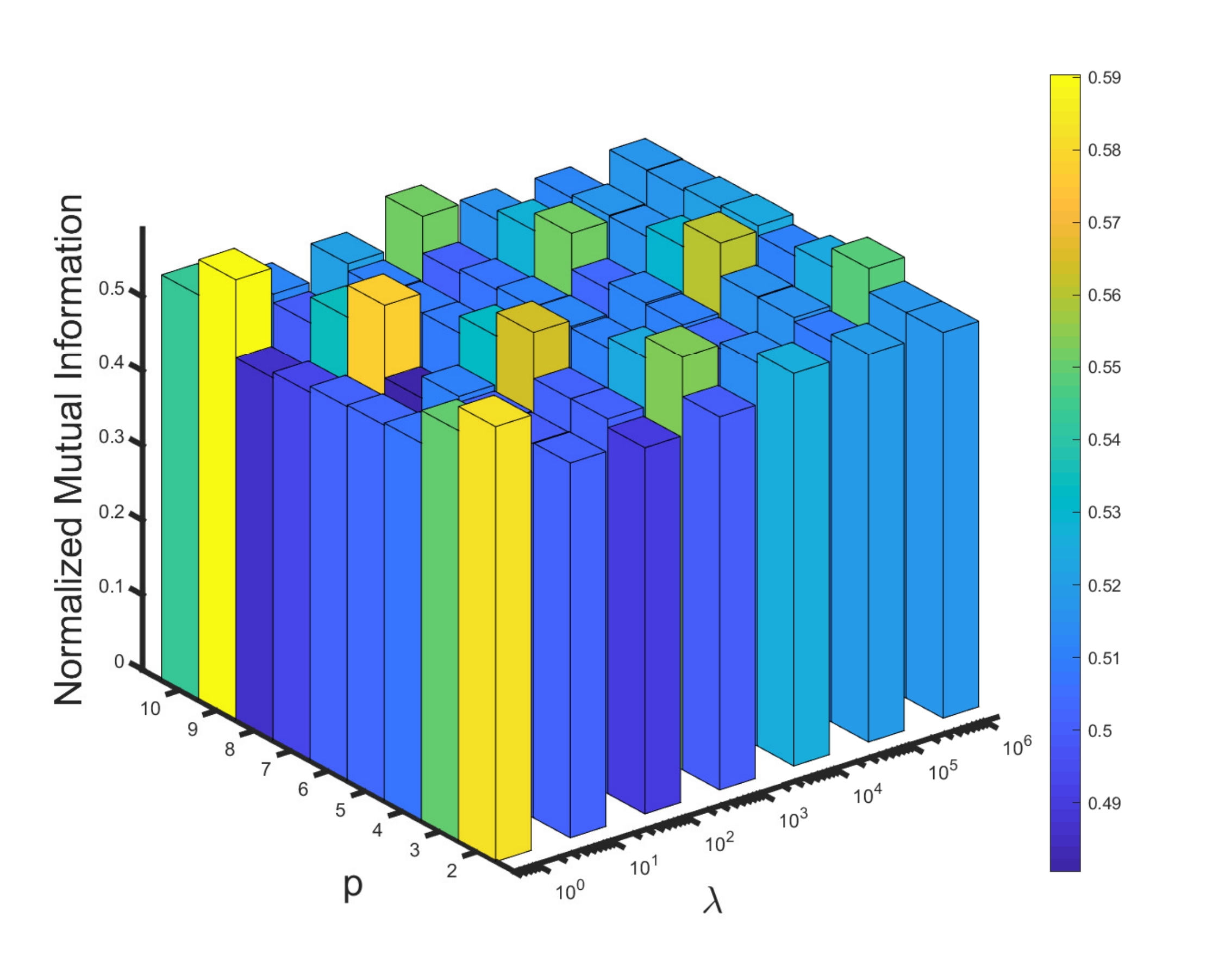}%
\label{Reuters_delta_Coh} \\
\includegraphics[width=1.25in]{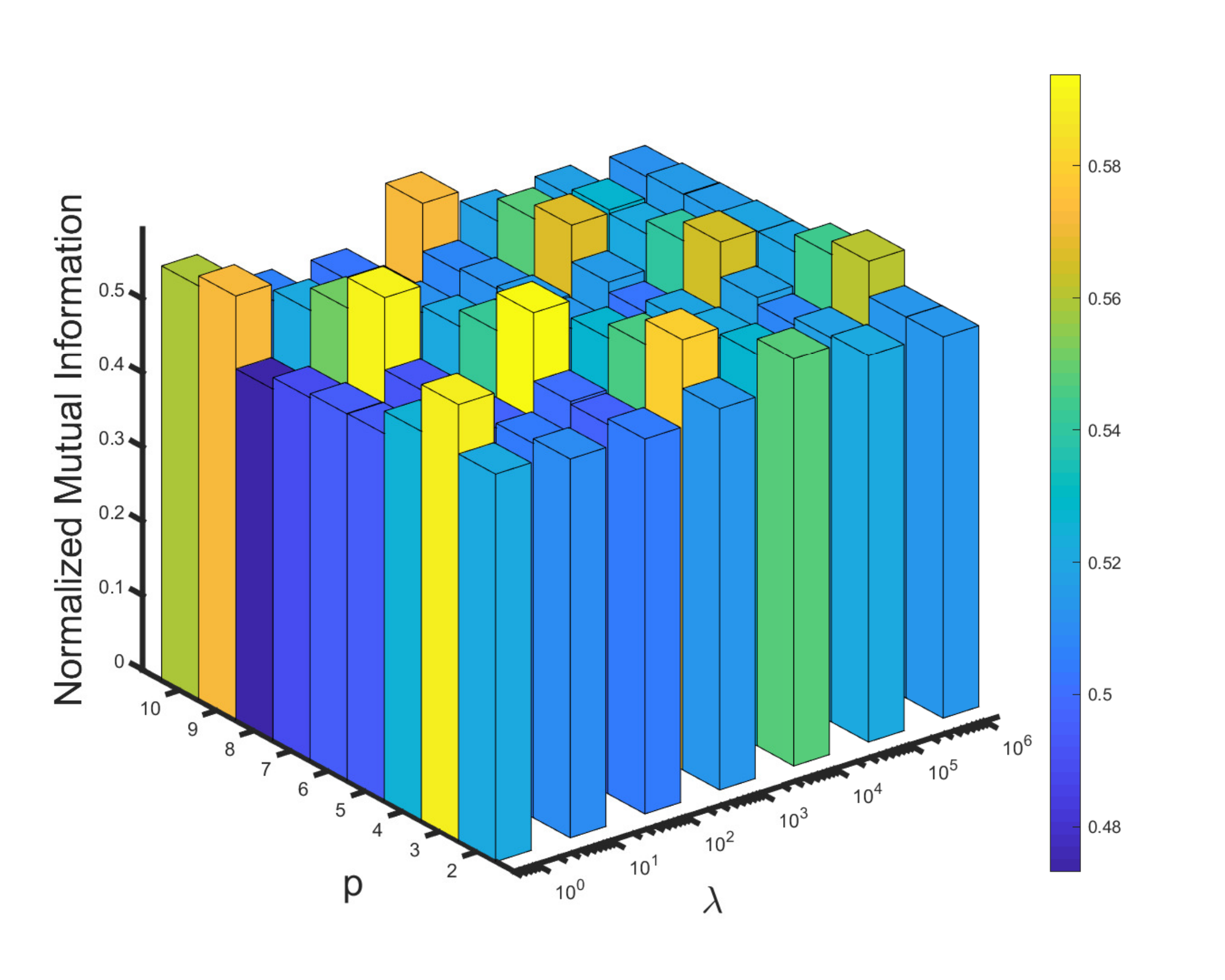}%
\label{Reuters_delta_Coh} \\
\includegraphics[width=1.25in]{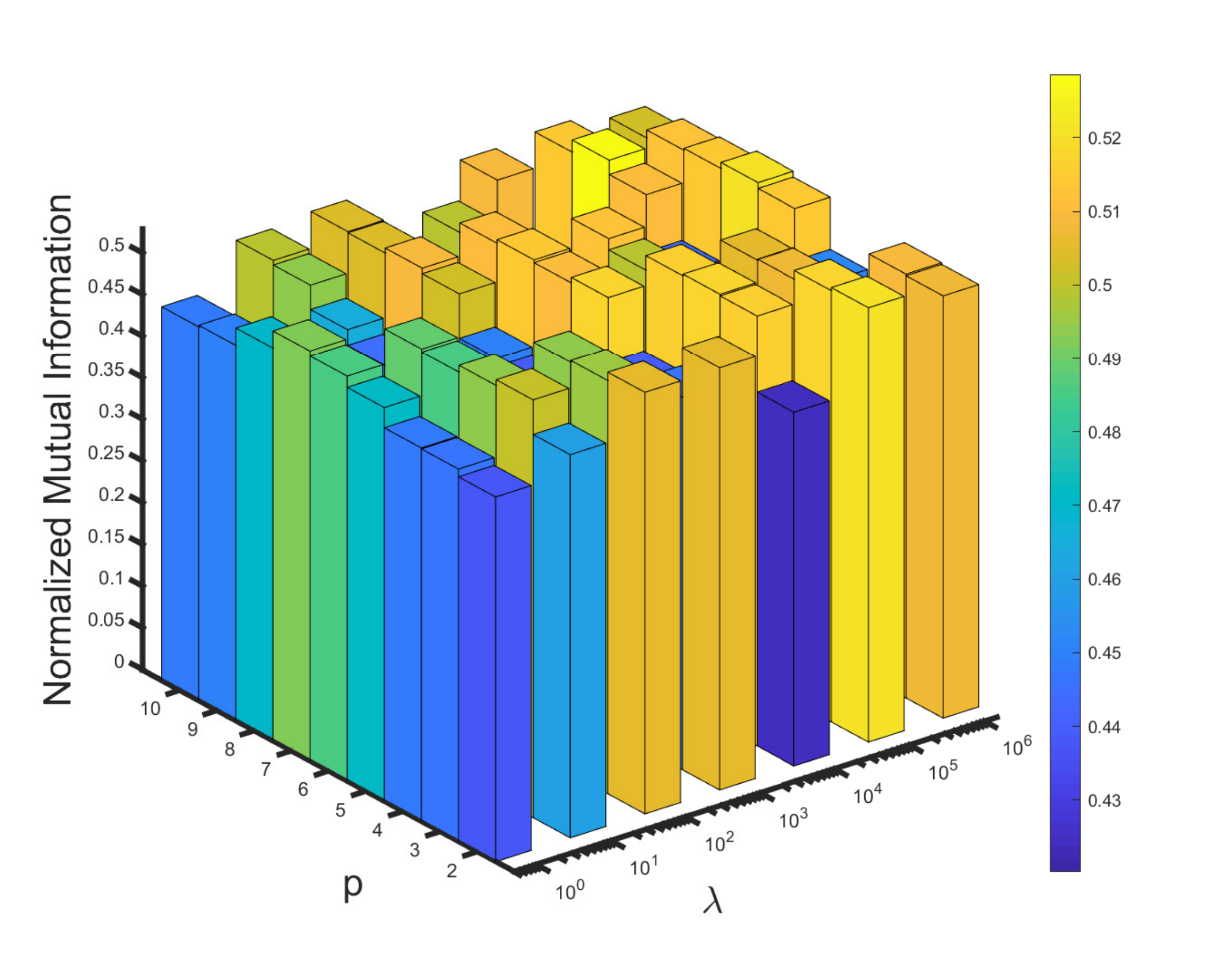}%
\label{Reuters_delta_Coh}
\end{minipage}
}\hspace{-5mm}
\subfloat[]{
\begin{minipage}[b]{0.2\textwidth}
\includegraphics[width=1.25in]{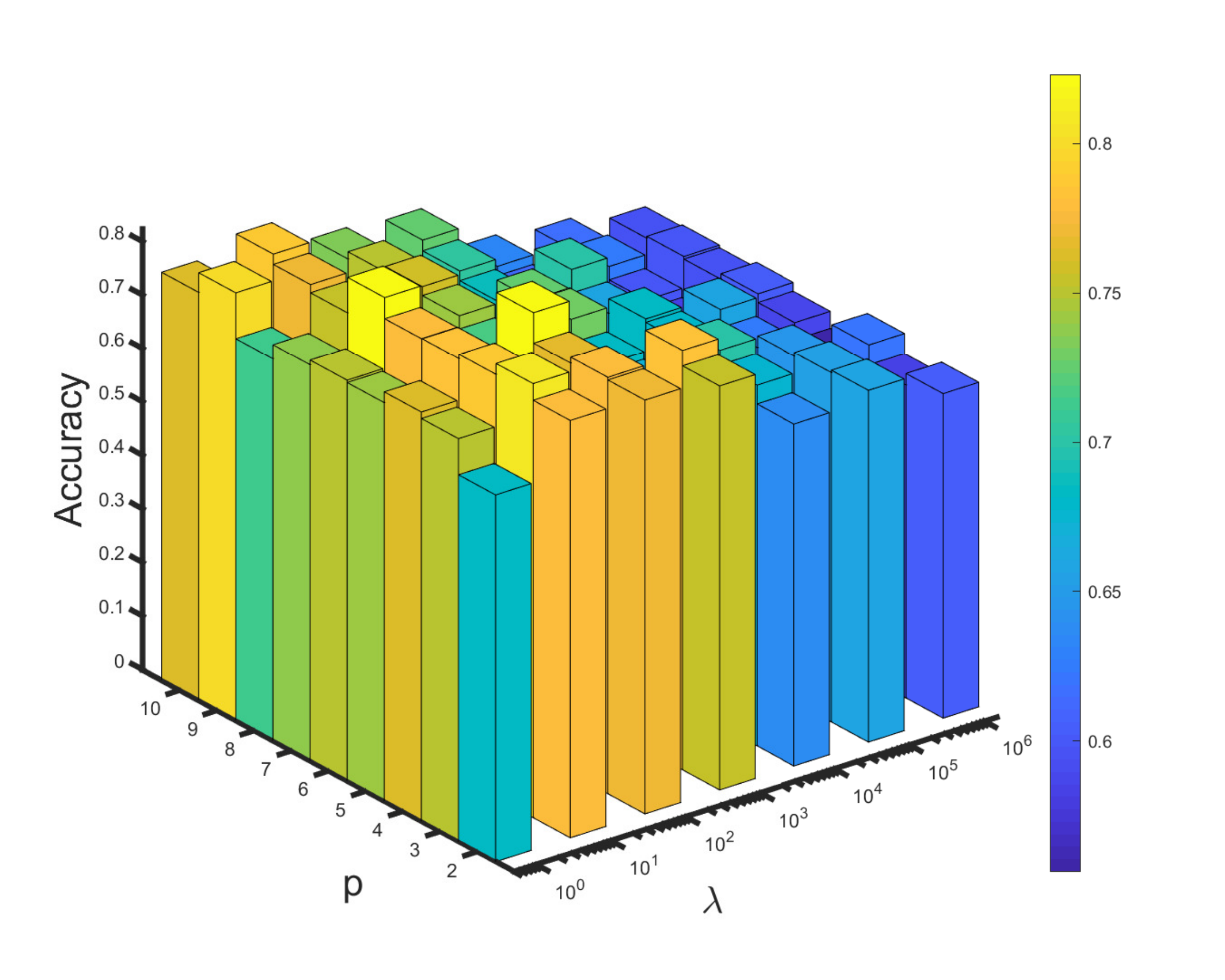}%
\label{Reuters_delta_Coh} \\
\includegraphics[width=1.25in]{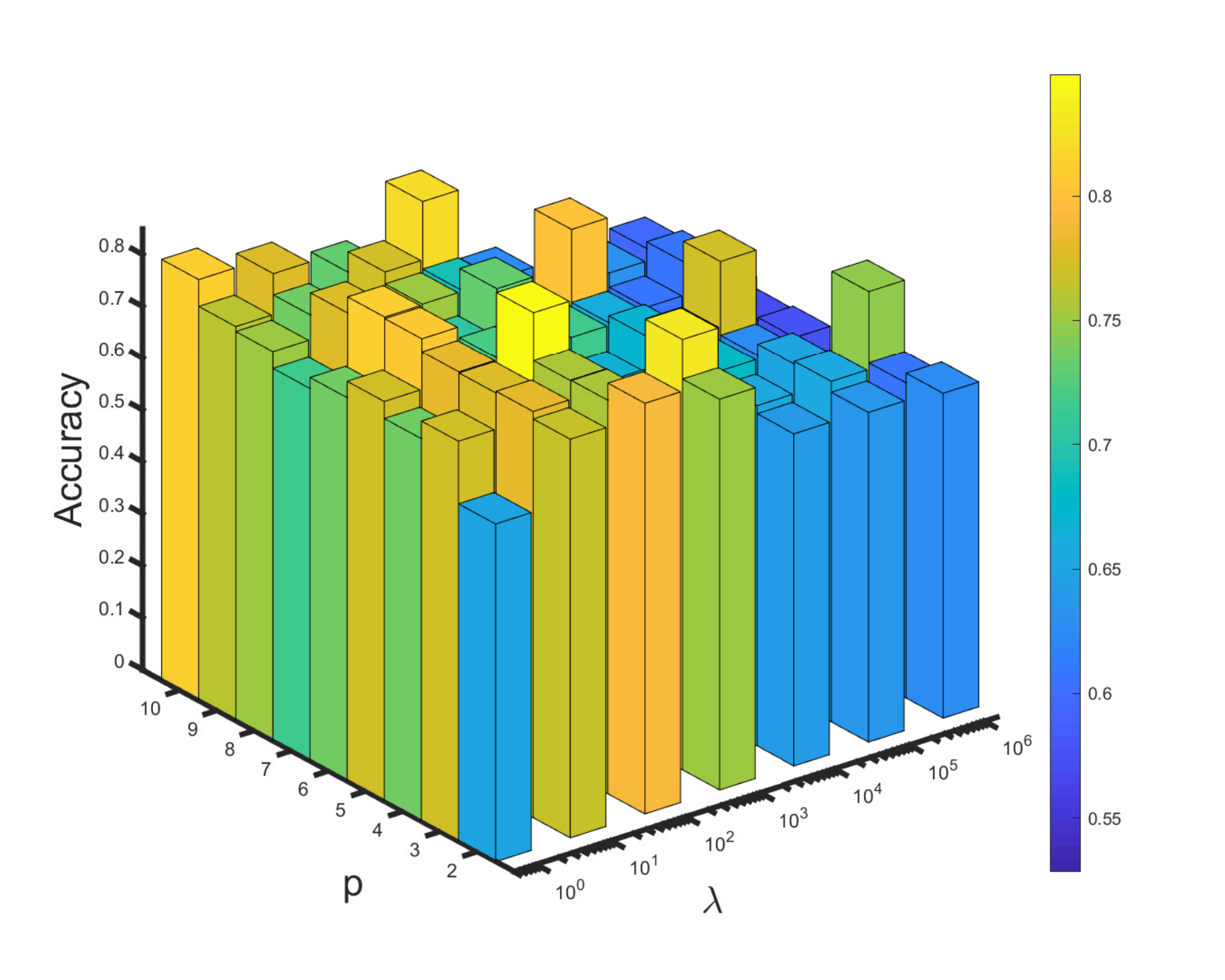}%
\label{Reuters_delta_Coh} \\
\includegraphics[width=1.25in]{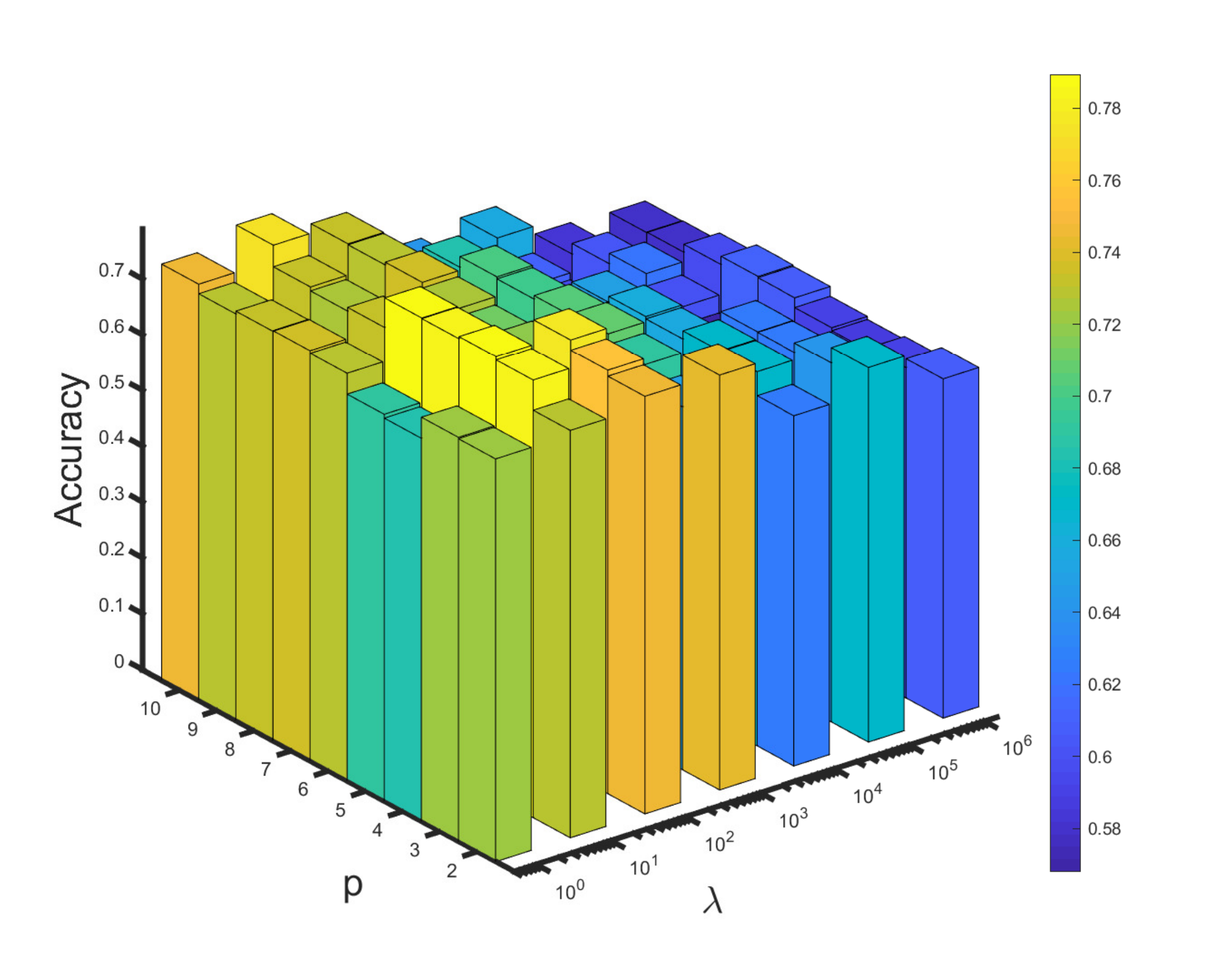}%
\label{Reuters_delta_Coh} \\
\includegraphics[width=1.25in]{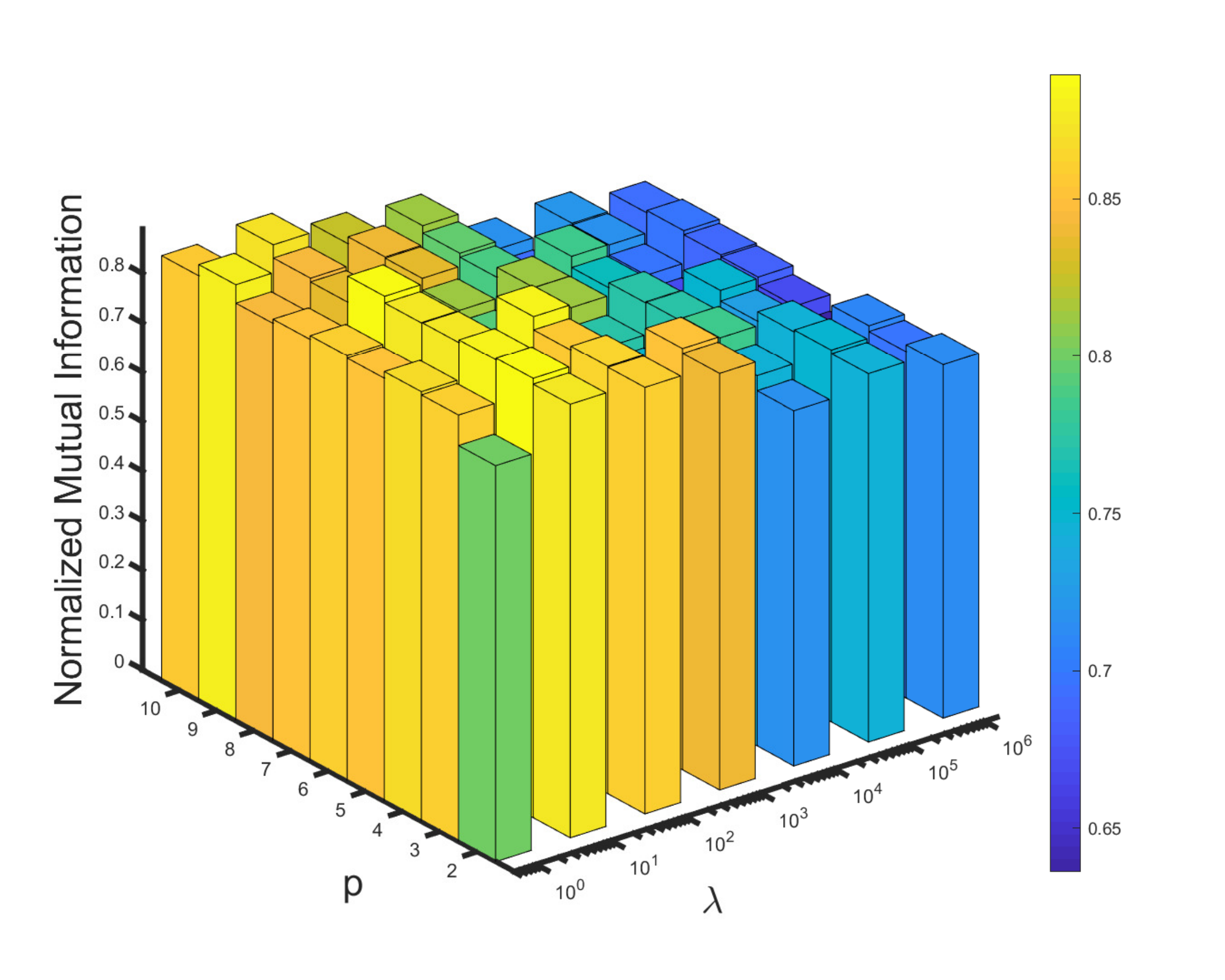}%
\label{Reuters_delta_Coh} \\
\includegraphics[width=1.25in]{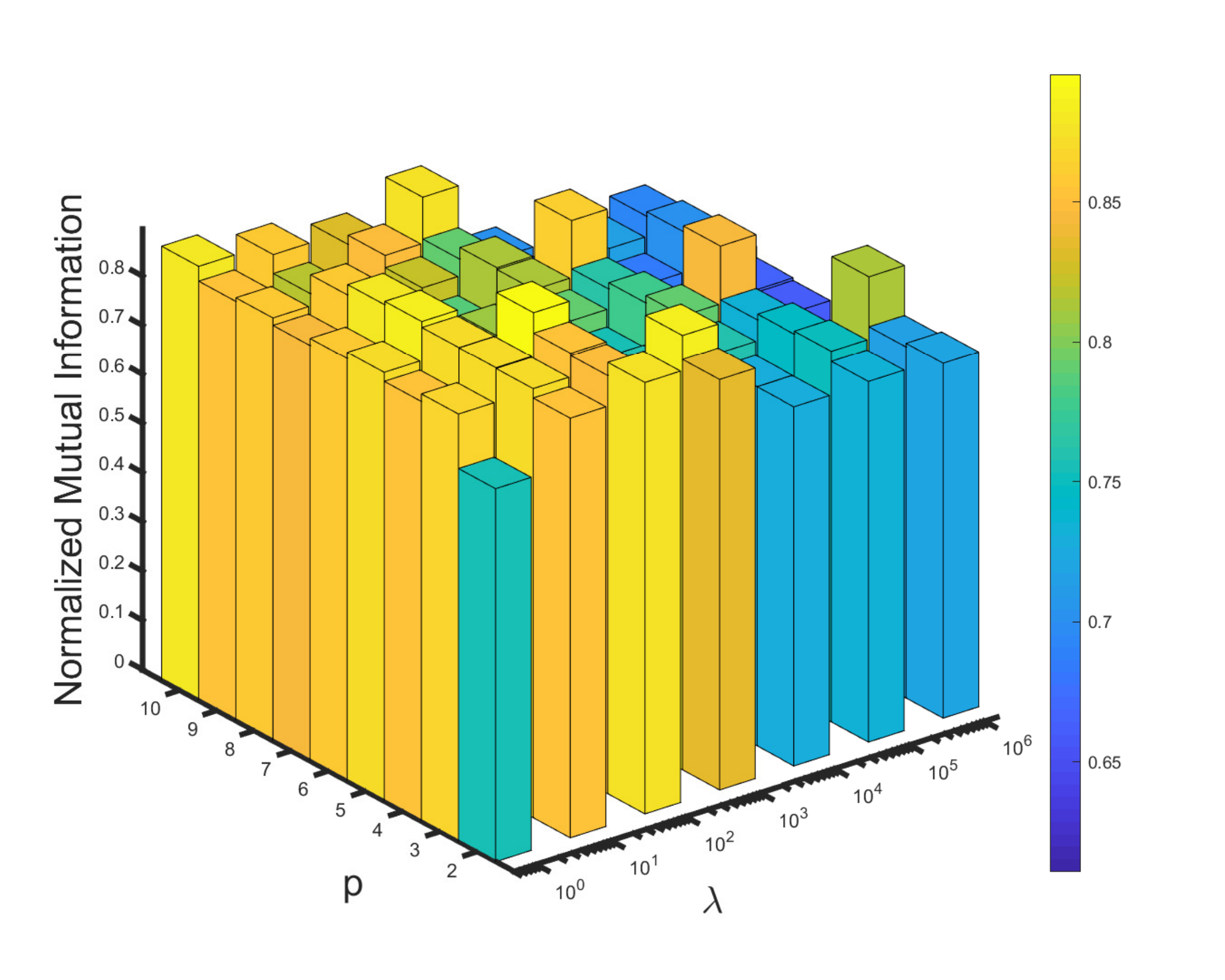}%
\label{Reuters_delta_Coh} \\
\includegraphics[width=1.25in]{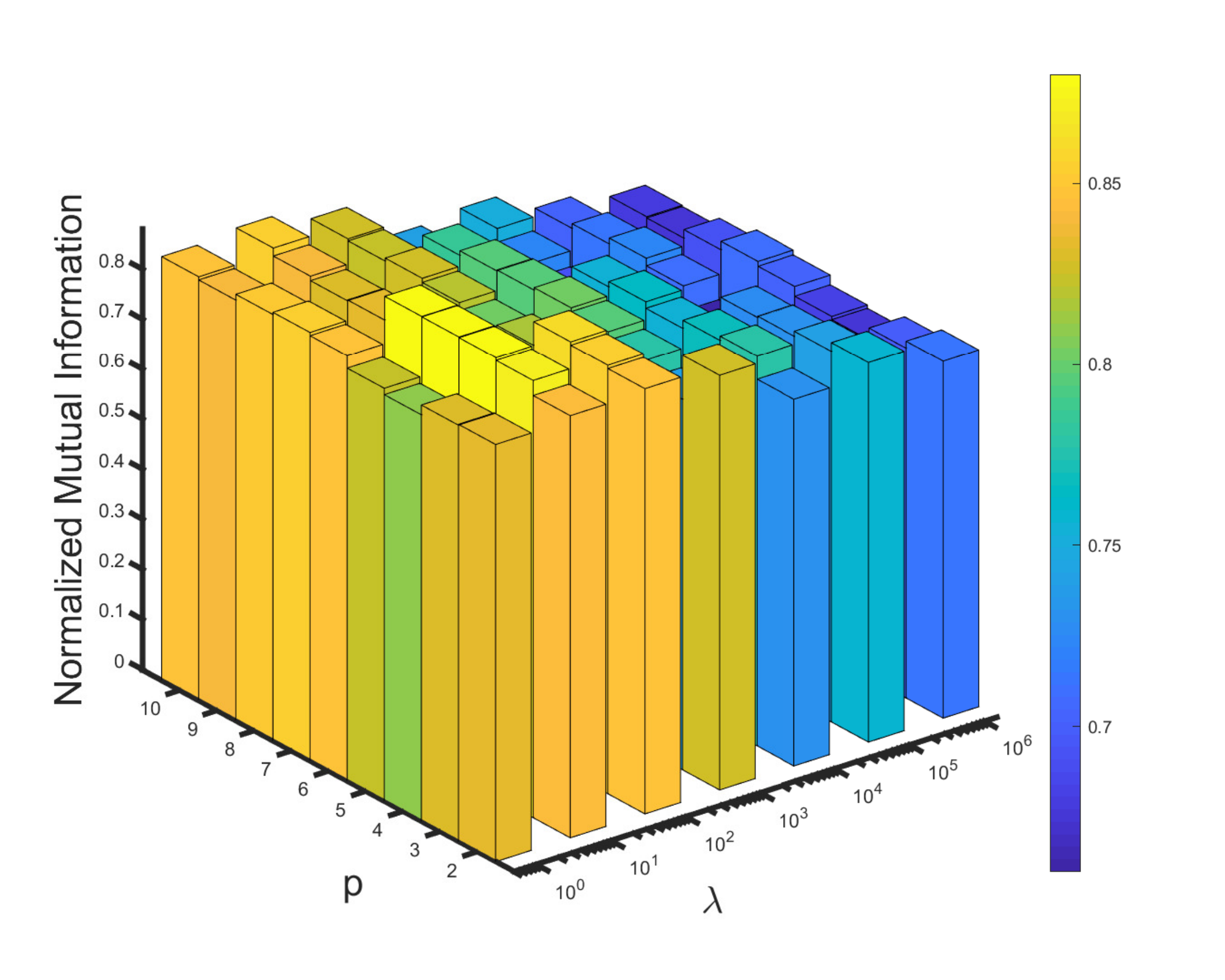}%
\label{Reuters_delta_Coh}
\end{minipage}
}\hspace{-5mm}
\subfloat[]{
\begin{minipage}[b]{0.2\textwidth}
\includegraphics[width=1.25in]{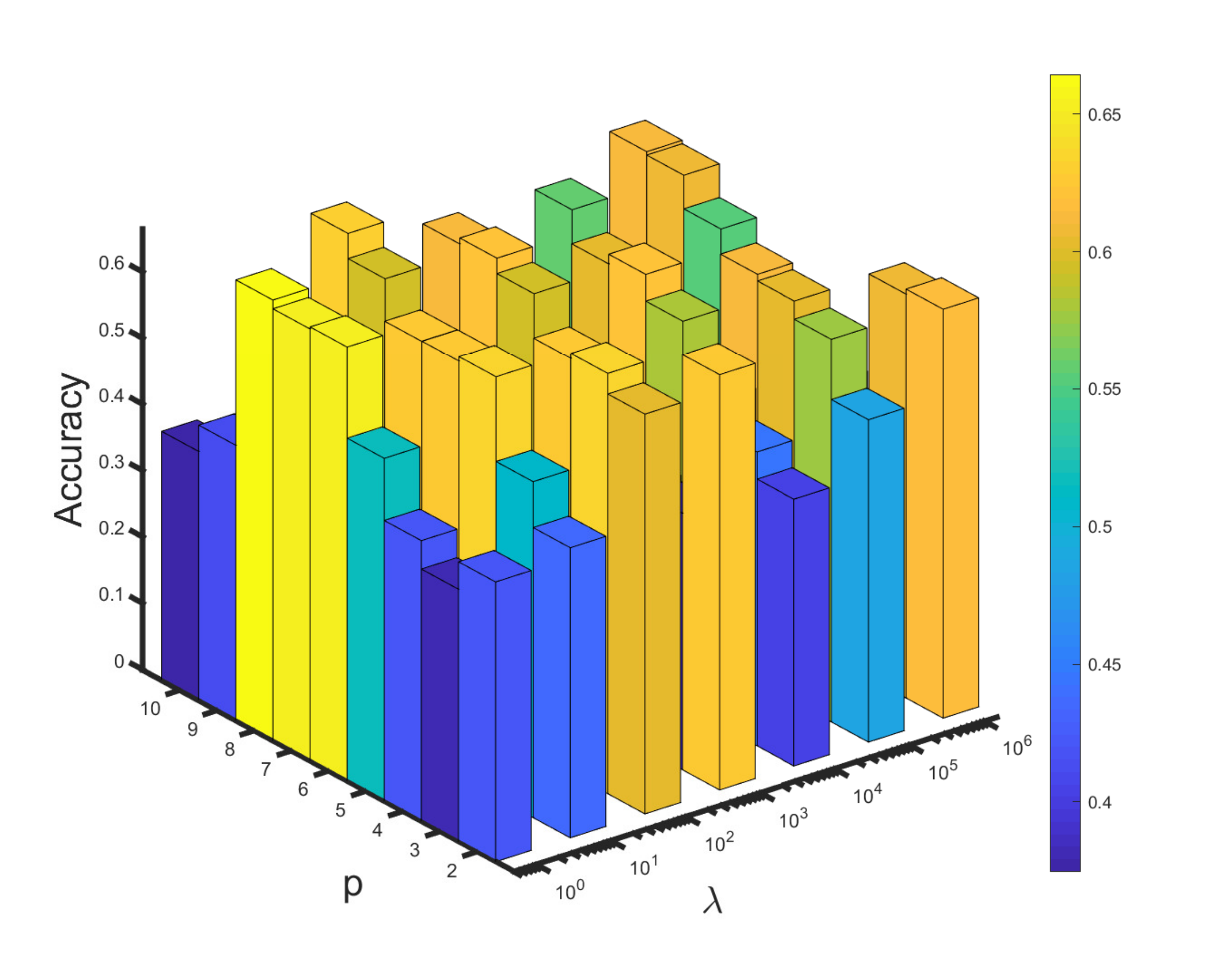}%
\label{Reuters_delta_Coh} \\
\includegraphics[width=1.25in]{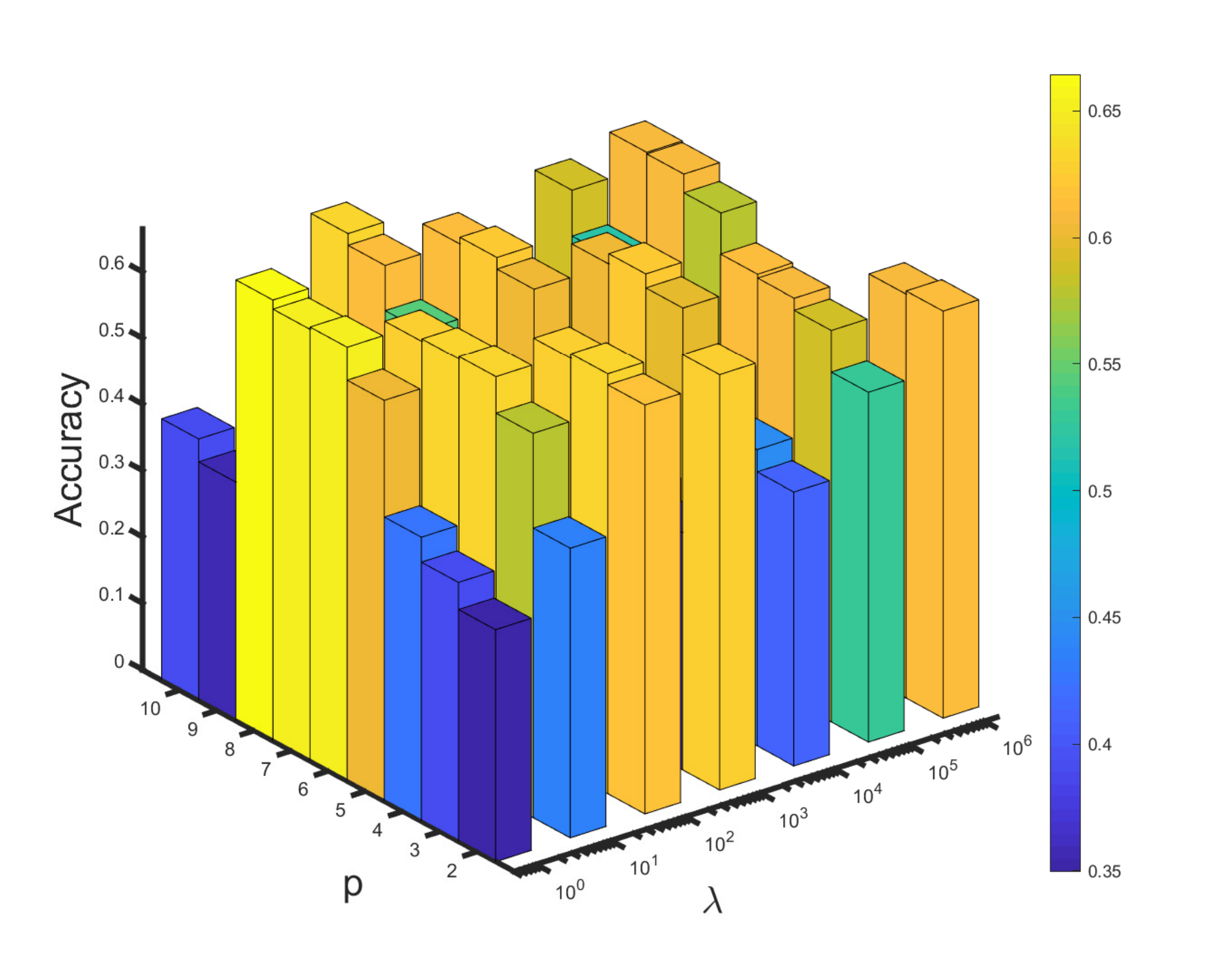}%
\label{Reuters_delta_Coh} \\
\includegraphics[width=1.25in]{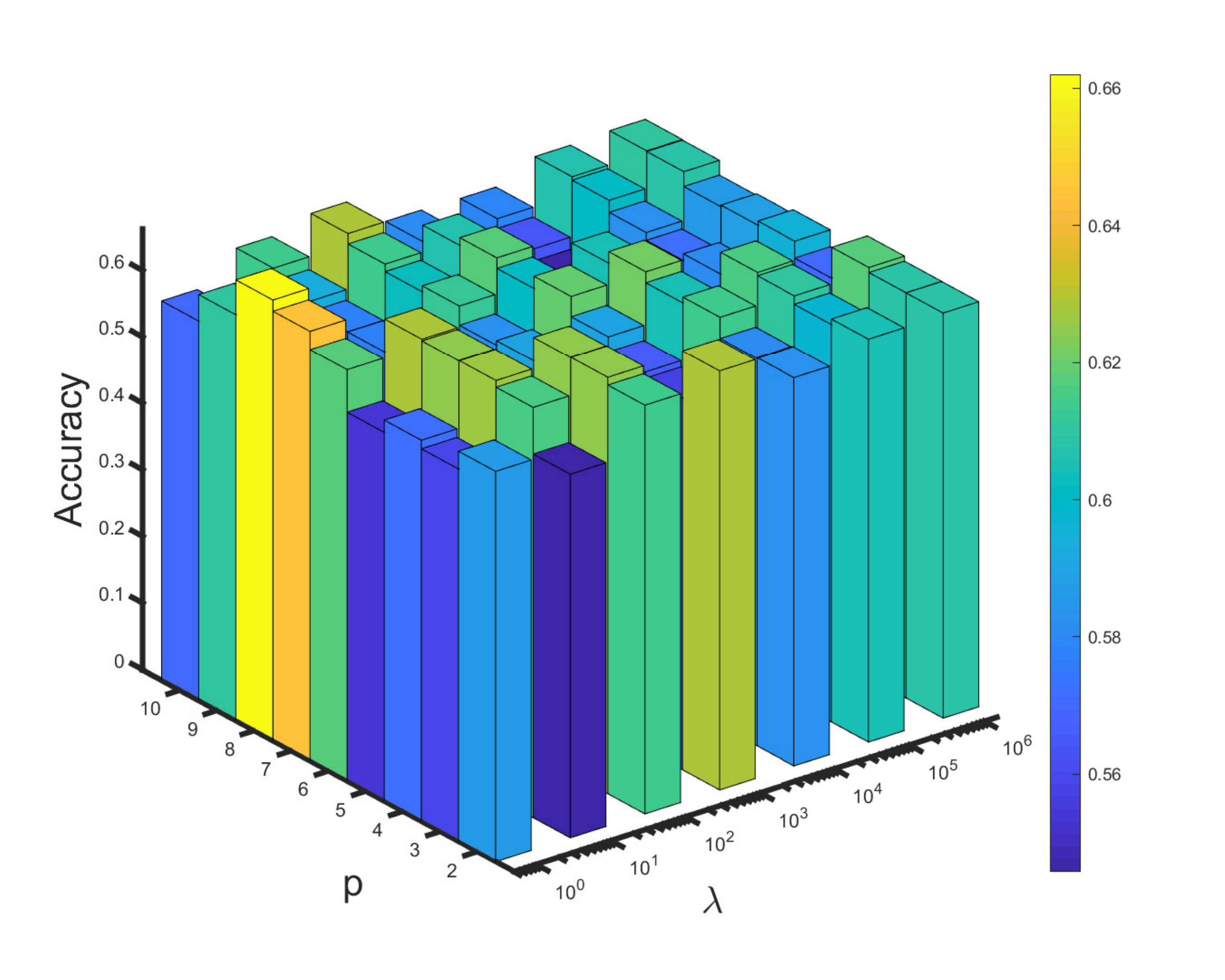}%
\label{Reuters_delta_Coh} \\
\includegraphics[width=1.25in]{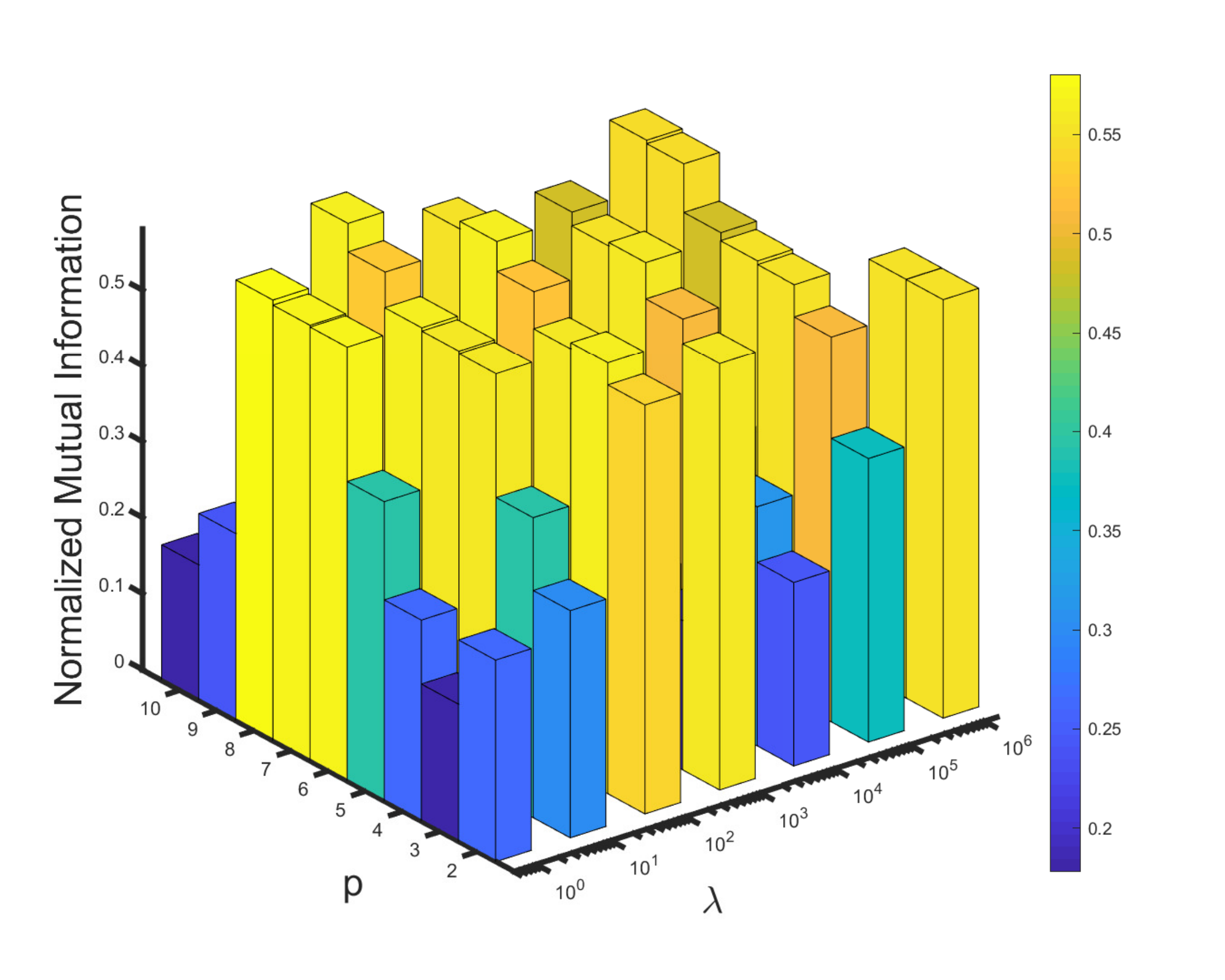}%
\label{Reuters_delta_Coh} \\
\includegraphics[width=1.25in]{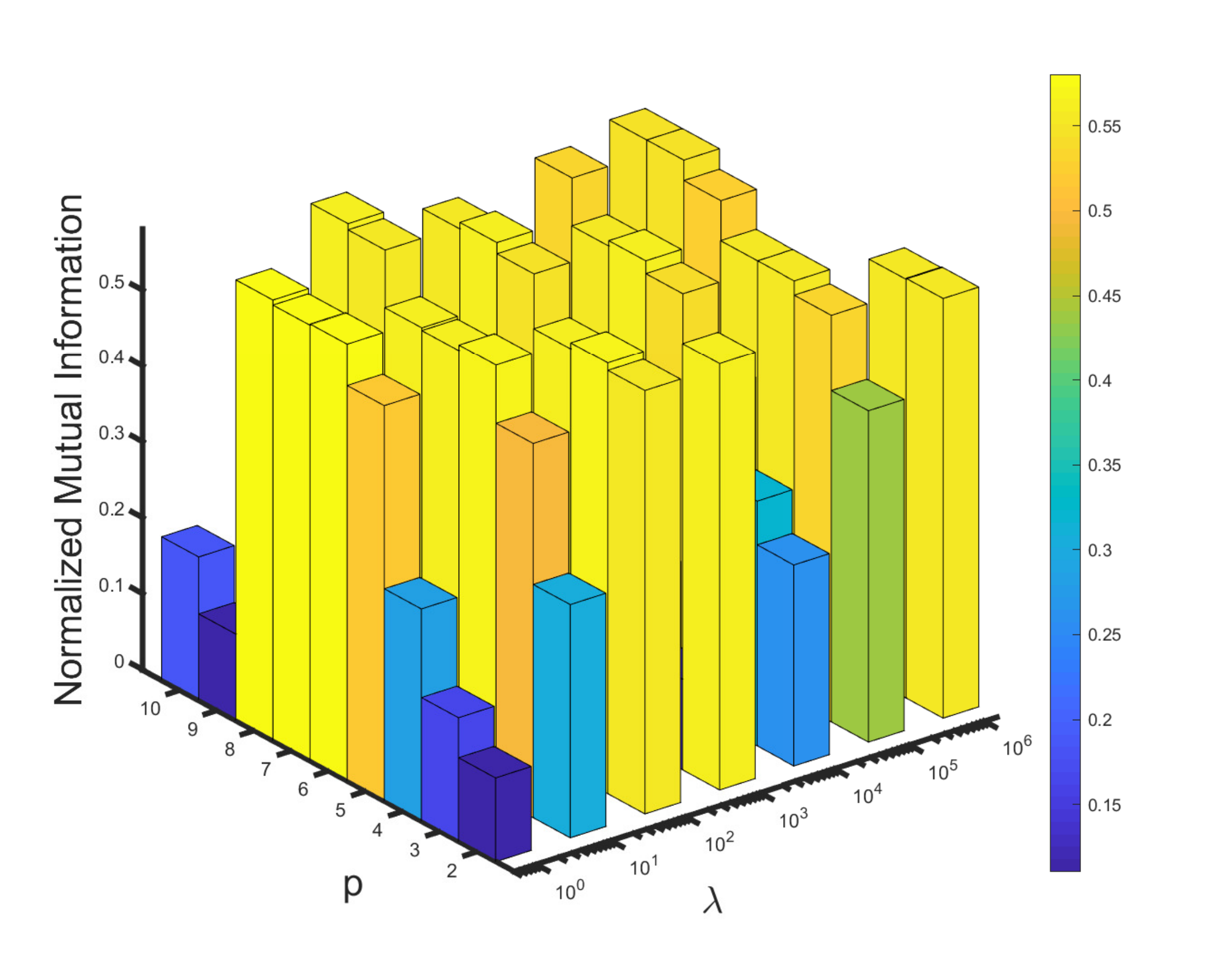}%
\label{Reuters_delta_Coh} \\
\includegraphics[width=1.25in]{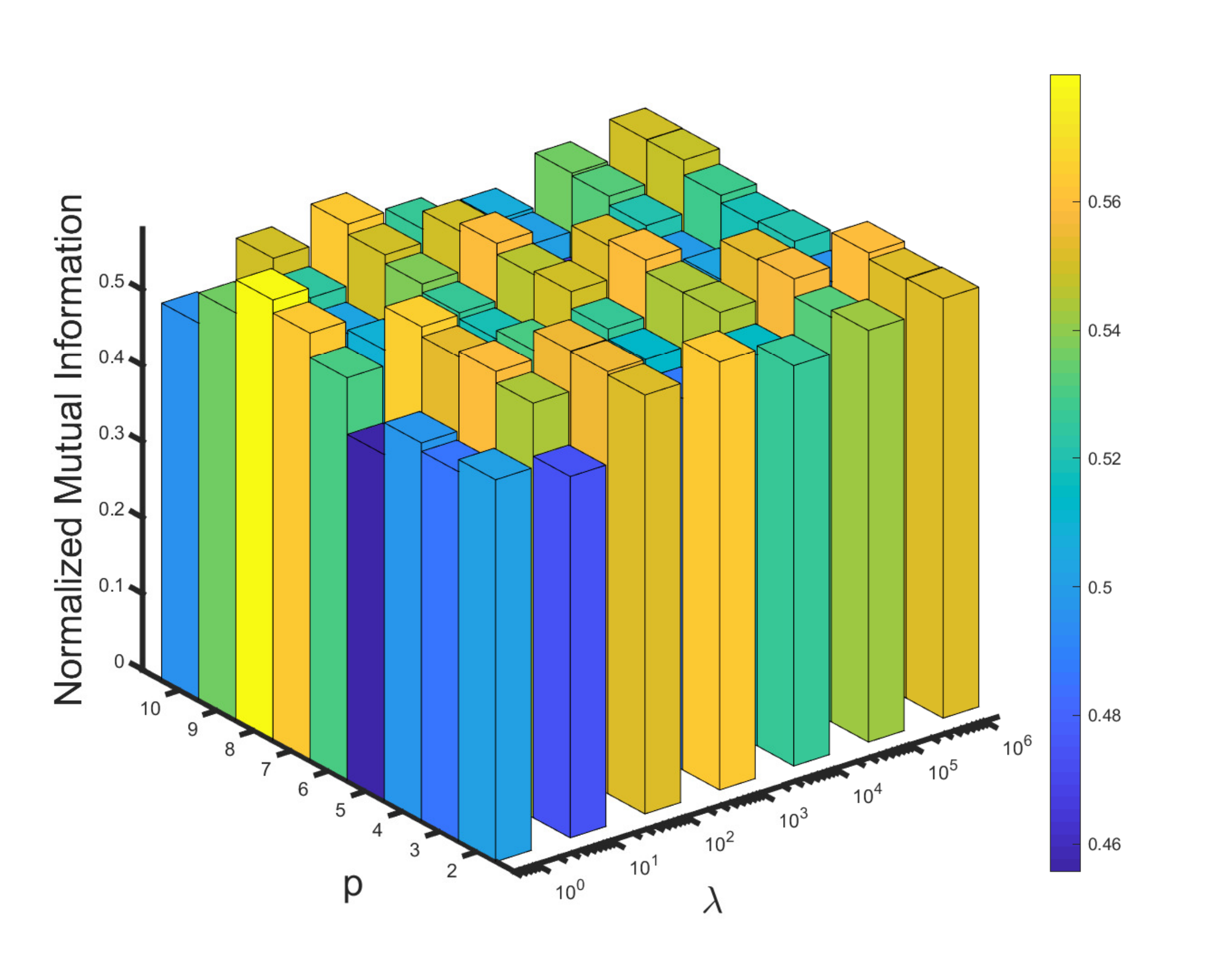}%
\label{Reuters_delta_Coh}
\end{minipage}
}\hspace{-5mm}

\vspace{2mm}
\caption{Sensitive of hyperparatmeters, $\lambda$ and $p$ in terms of ACC and NMI on five datasets, the first three figures show the ACC of RMNTF-Cauchy, RMNTF-CIM, RMNTF-Huber, respectively and they are contaminated by Laplace noise which the deviation is set to 120, and the last show the NMI of RMNTF-Cauchy, RMNTF-CIM, RMNTF-Huber, respectively. (a) Average accuracy and NMI on the subset of COIL100 which contains $10$ categories. (b) Average accuracy and NMI on the subset of FEI which contains $10$ categories. (c) Average accuracy and NMI on the subset of FERET which contains $10$ categories. (b) Average accuracy and NMI on the subset of ORL which contains $10$ categories. (b) Average accuracy and NMI on the subset of USPS which contains $10$ categories. }
\label{fig:hyperparameter}
\end{figure*}

\subsection{Evaluation Metrics}
To evaluate the clustering performance of these algorithms, we adopt two commonly used metrics: 1) accuracy(ACC) and 2) normalized mutual information (NMI). The ACC is defined by
\begin{equation}\label{ACC}
\begin{split}
\operatorname{ACC}({{\mathbf{l}}},{{\mathbf{\hat l}}}) = \frac{1}{n}\sum\limits_{i = 1}^n {\delta [{l_i},\operatorname{map}({{\hat l}_i})]},
\end{split}
\end{equation}
where $n$ is the total number of samples in a dataset. We have the cluster labels ${\mathbf{\hat l}} = \{ {\hat l_1},{\hat l_2}, \cdots ,{\hat l_n}\} $ and the ground-truth labels ${\mathbf{l}} = \{ {l_1},{l_2}, \cdots ,{l_n}\} $. $\delta (x,y)$ is set to $1$ if and only if $x=y$, and $0$ otherwise. $\operatorname{map}(\cdot)$ is a displacement mapping function that maps each cluster label $\hat{l_i}$ to the equivalent label from the dataset.
\par
Another metric NMI is defined by
\begin{equation}\label{NMI}
\begin{split}
\operatorname{NMI} ({\mathbf{l}},{\mathbf{\hat l}}) = \frac{{\operatorname{MI} ({\mathbf{l}},{\mathbf{\hat l}})}}{{\max[\operatorname{H} ({\mathbf{l}}),\operatorname{H} ({\mathbf{\hat l}})]}},
\end{split}
\end{equation}
where $\operatorname{H}(\mathbf{l})$ and $\operatorname{H}(\mathbf{\hat l})$ denote the entropy of $\mathbf{l}$ and $\mathbf{\hat l}$, respectively, and
\begin{equation}\label{NMI}
\begin{split}
\operatorname{MI}({\mathbf{l}},{\mathbf{\hat l}}) = \sum\limits_{{l_i} \in {\mathbf{l}}} {\sum\limits_{{l_i} \in {\mathbf{\hat l}}} {p({l_i},{{\hat l}_i})} } \log_2\left[ {\frac{{p({l_i},{{\hat l}_i})}}{{p({l_i})p({{\hat l}_i})}}} \right],
\end{split}
\end{equation}
${p({l_i})}$ and ${p({{\hat l}_i})}$ represent the marginal probability distribution functions of ${\mathbf{l}}$ and ${\mathbf{\hat l}}$, respectively, and ${p({l_i},{{\hat l}_i})}$ is the joint probability distribution function of ${\mathbf{l}}$ and ${\mathbf{\hat l}}$. $\operatorname{MI}({\mathbf{l}},{\mathbf{\hat l}})$ ranges from 0 to 1, with $\operatorname{MI}({\mathbf{l}},{\mathbf{\hat l}})=1$ if the two sets of clusters are identical, and $\operatorname{MI}({\mathbf{l}},{\mathbf{\hat l}})=0$ otherwise.

\subsection{Main Results}

\subsubsection{Basis Visualization and Convergence}
To compare the ability of extracting parts-based feature of tensor objects by NMF, GNMF, NTD, and RMNTF, respectively, we visualize the basic images extracted by each algorithms on 25 subjects randomly chosen from the AT\&T ORL dataset in Fig. \ref{fig:Reuters_delta1}. From the experimental results, we note that the proposed RMNTF extracts more localized parts of face images than other algorithms, since these images reconstructed from the base are more homogeneous. It means that RMNTD can provide a more sparse representation.
\par
In addition, we investigate the convergence speed of RMNTF. We show the convergence curves of the three implementations of RMNTF on five image datasets in Fig. \ref{fig:Reuters_delta2}. It is shown that the proposed three algorithms exhibit fast convergence rates, usually taking less than 100 iterations.

\subsubsection{Simulated Corruption and Clustering Results}
 In order to evaluate the robustness of RMNTF, we compare our algorithms with the state-of-the-art clustering algorithms on five image datasets contaminated by Laplace noise and salt \& pepper noise. The experimental results on COIL100 and FEI are presented in Fig. \ref{fig:Reuters_delta3} and Fig. \ref{fig:Reuters_delta4}. Due to the space limitation, the results on FERET, ORL and USPS are represented in the supplement file.
\par

Laplace noise and salt \& pepper noise sometimes exist in image corruption. However, the cost function of some traditional methods, such as NMF, usually adopt Euclidean distance. They cannot deal with this kind of data well since the noise distribution is not consistent with the noise assumption.
\par
The simulated Laplace noise obeys a Laplace distribution $La(0,\delta )$. We set the deviation $\delta$ from 40 to 280 and add the noise to each pixels of images randomly. The first two rows in Fig. \ref{fig:Reuters_delta3} and Fig. \ref{fig:Reuters_delta4} show the mean and standard deviations of average accuracy and NMI of RMNTF's three implementations and other nine representative algorithms. The experimental results confirms that RMNTF based method performs better than other methods when the deviation of Laplace noise is within 200. However, when the deviation is excessive and many outliers come into being, performance of all the methods reduce dramatically.
\par
In terms of salt \& pepper noise, we set the percentage of contaminated pixels from 5 to 60 percent for each image. Results are shown in the last two rows in Fig. \ref{fig:Reuters_delta3} and Fig. \ref{fig:Reuters_delta4}. With increase of corrupted pixels, only GNMF is competitive with RMNTF based methods. And they outperform than other methods. When more than 30 percent of pixels are corrupted, performance of all the algorithms is dramatically degraded and gradually reaches unanimity. Because it is difficult to separate outliers from inliers.

{We have conducted experiments on FERET, ORL, and USPS. The results are listed in Appendix A .}

\subsection{Effects of the hyperparameters}
This subsection investigates the effects of the hyperparameters of RMNTF based algorithms on the performance of clustering task. Experiments are conducted on five image datasets contaminated by Laplace noise which the deviation is set to 120. There are two hyperparameters $\lambda$ and $p$ need to be predefined. We report the average accuracy and NMI on 10 categories. For the three implementations of RMNTF, the parameters were setting as follows: fix $p = \{ 2,3,4,5,6,7,8,9,10\}$ and choose $\lambda  \in \{ {10^0},{10^1},{10^2},{10^3},{10^4},{10^5},{10^6}\} $.
\par
Fig. \ref{fig:hyperparameter} shows the effect of hyperparameters. The clustering performances are oscillate while $\lambda  < {10^4}$ and then tend to be stable while $\lambda$ increases in range $[{10^4},{10^6}]$. Therefore, $\lambda$ can be selected around $10^4$. It can be seen that the clustering performances are quite stable while the integer $p$ is in range $[5,8]$, which means that $p$ hardly influences RMNTF. In total, our methods are probably robust to $\lambda$ and they tend to achieve better performance when $\lambda$ is slightly smaller, but they may be influenced by hyperparameter $p$.

\section{Conclusion}
\label{Conclusion}

In this paper, we explored three robust cost function for manifold structure of nonnegative tucker factorization.
To deal with the minimization of non-convex cost functions, we derive an iterative half-quadratic minimization optimization.
Then, the optimization problem can be reduced to a weighted Euclidean distance of NTF.
The proposed methods are further utilizing manifold structure information to enhance the performance of accuracy and avoid the rotational ambiguity.
Due to the connection between robust loss functions and robust M-estimators, we adopt CIM, Huber and Cauchy functions to replace traditional Euclidean loss function.
The proposed methods combine manifold structures with robust loss function to improve the clustering accuracy under noisy data and outliers.
We investigated the effective of Laplacian noise and Salt \& Pepper noise on the performance of the models respectively.
Under a small degree of noise interference, the proposed algorithms improve greatly compared with other algorithms.
For example, in FEI database, when the deviation value of the Laplacian noise is 50, the accuracy of our methods are absolutely improved by 10\% compared with GNTD.
In general, experimental results show that the proposed methods outperform the comparison methods in terms of clustering accuracy and normalized mutual information under noisy data or outliers.





\bibliographystyle{elsarticle-num}
\bibliography{refs}

\appendix

\section{  }\label{appendices_A}

\begin{figure*}[t]
\vspace{-0.5cm} 
\setlength{\abovecaptionskip}{0cm} 
\setlength{\belowcaptionskip}{-0cm} 
\centering
\subfloat[]{
\begin{minipage}[b]{0.18\textwidth}
\includegraphics[width=1.5in]{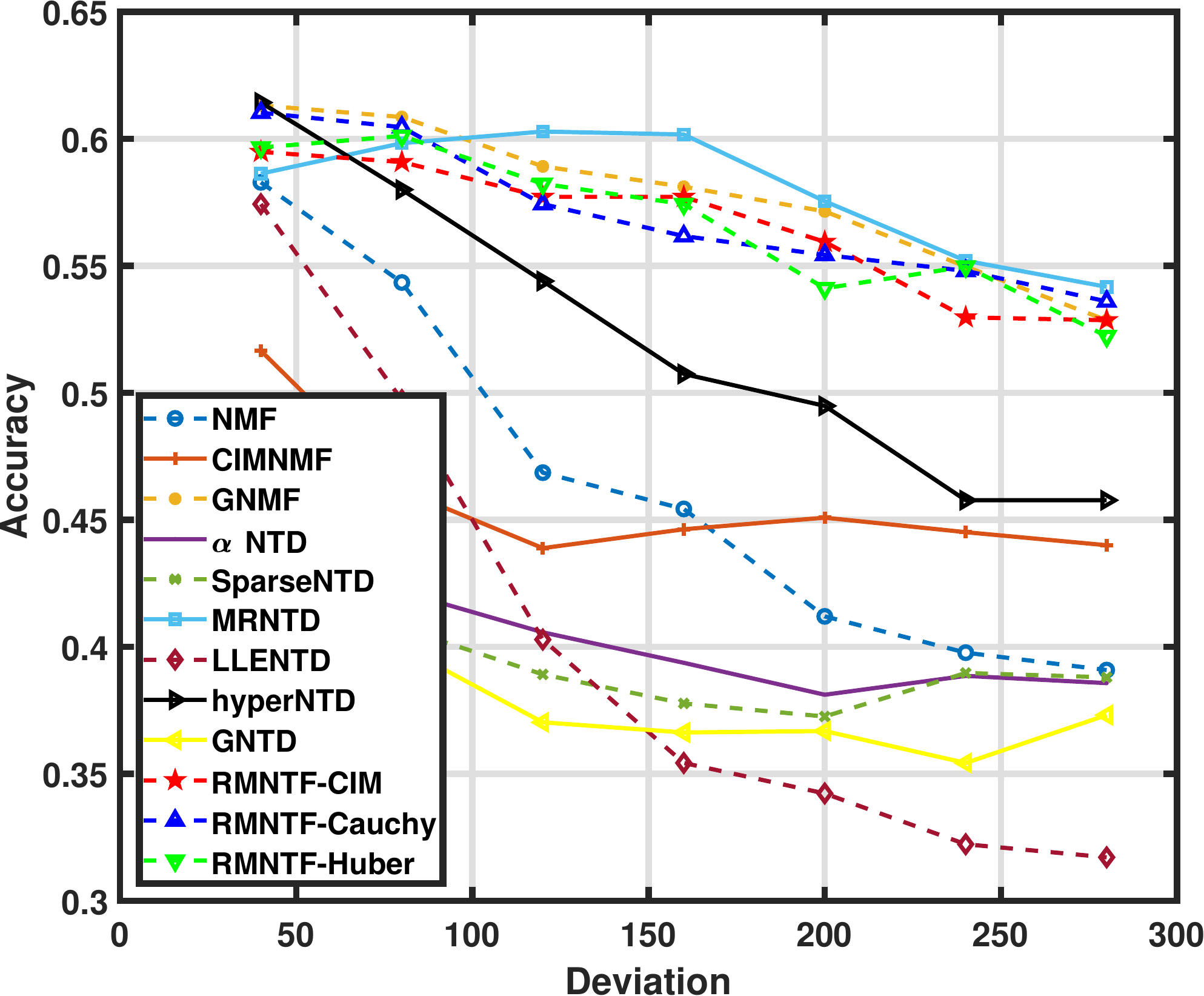}%
\label{Reuters_delta_Coh} \\
\includegraphics[width=1.5in]{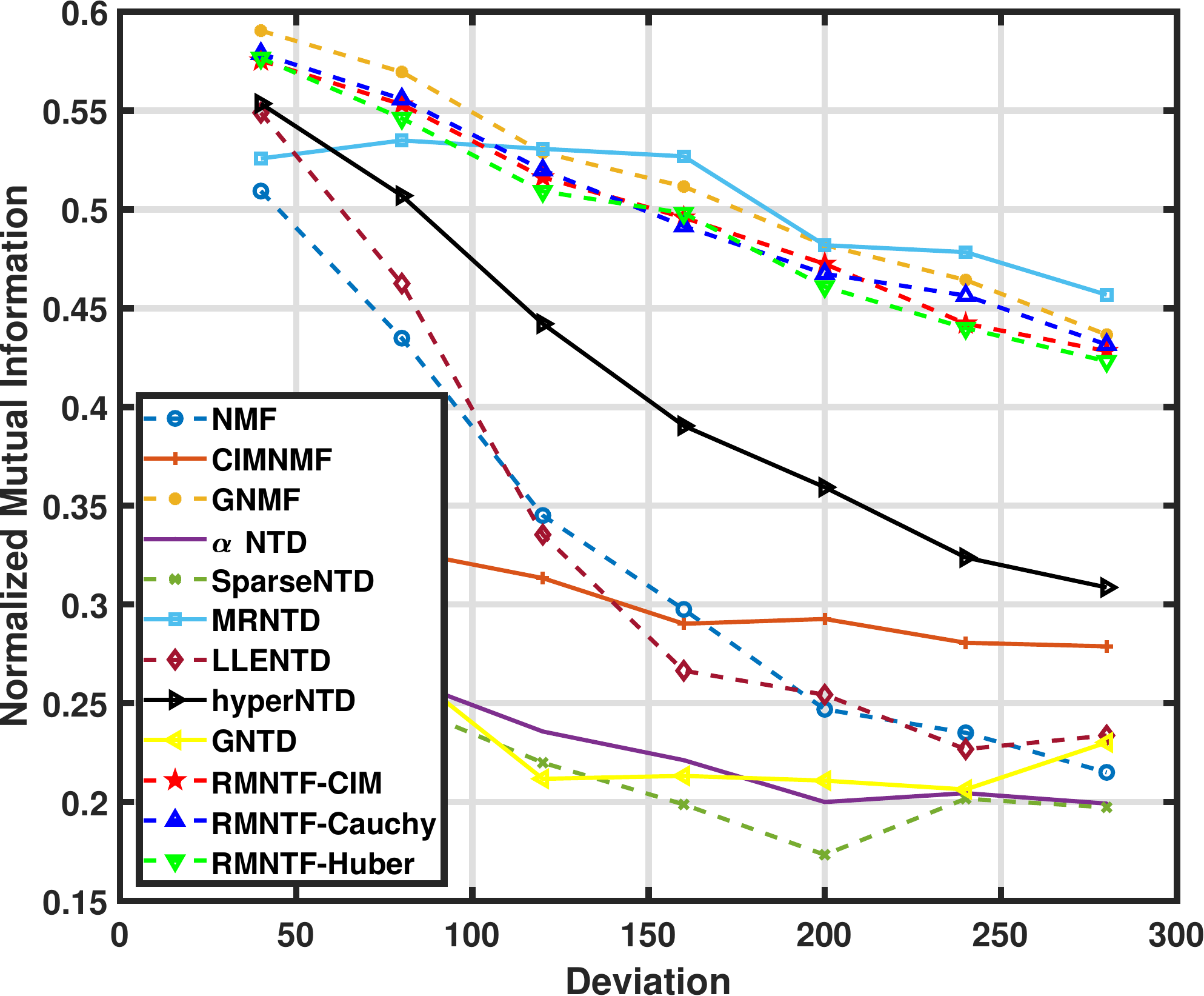}%
\label{Reuters_delta_Coh} \\
\includegraphics[width=1.5in]{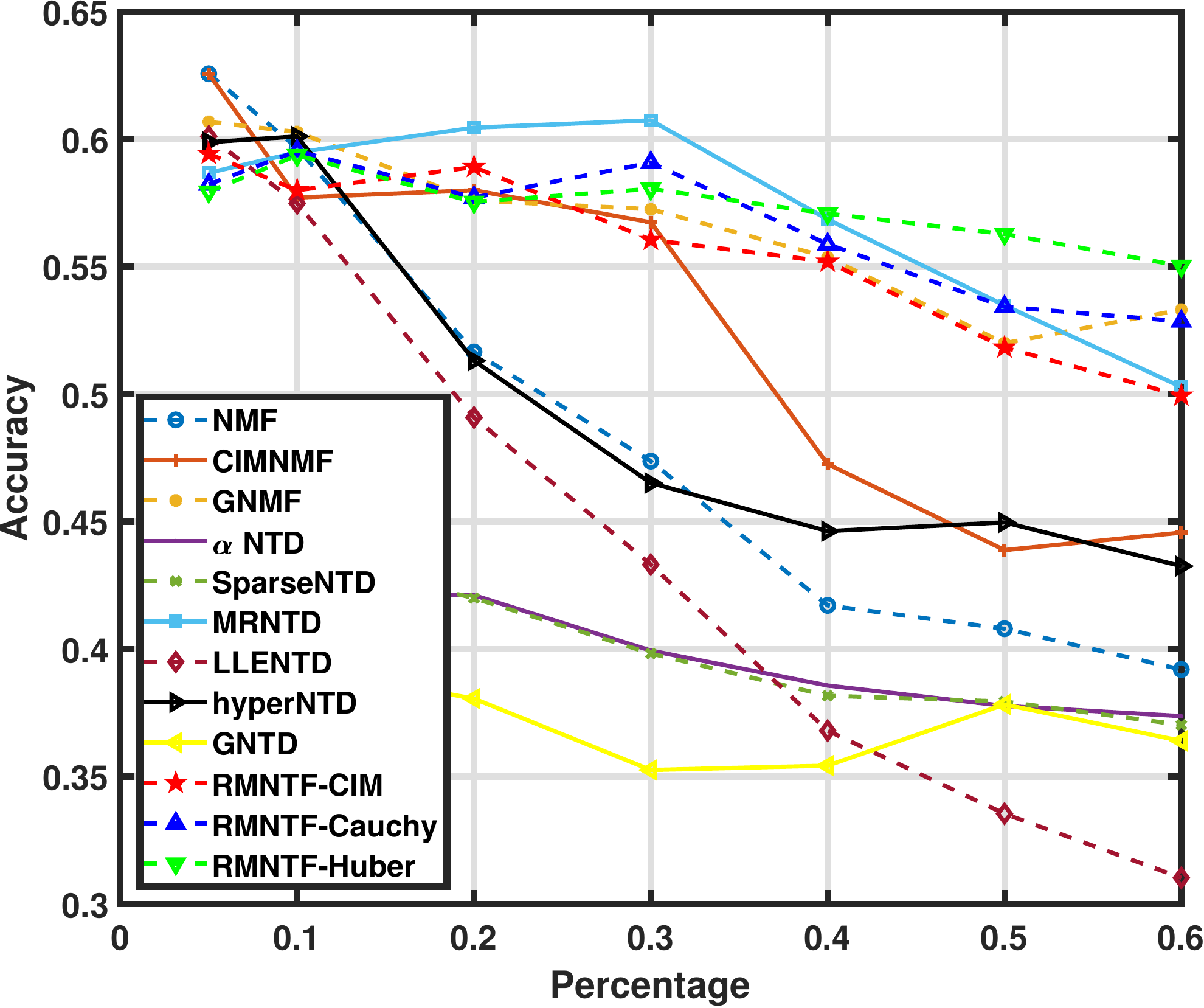}%
\label{Reuters_delta_Coh} \\
\includegraphics[width=1.5in]{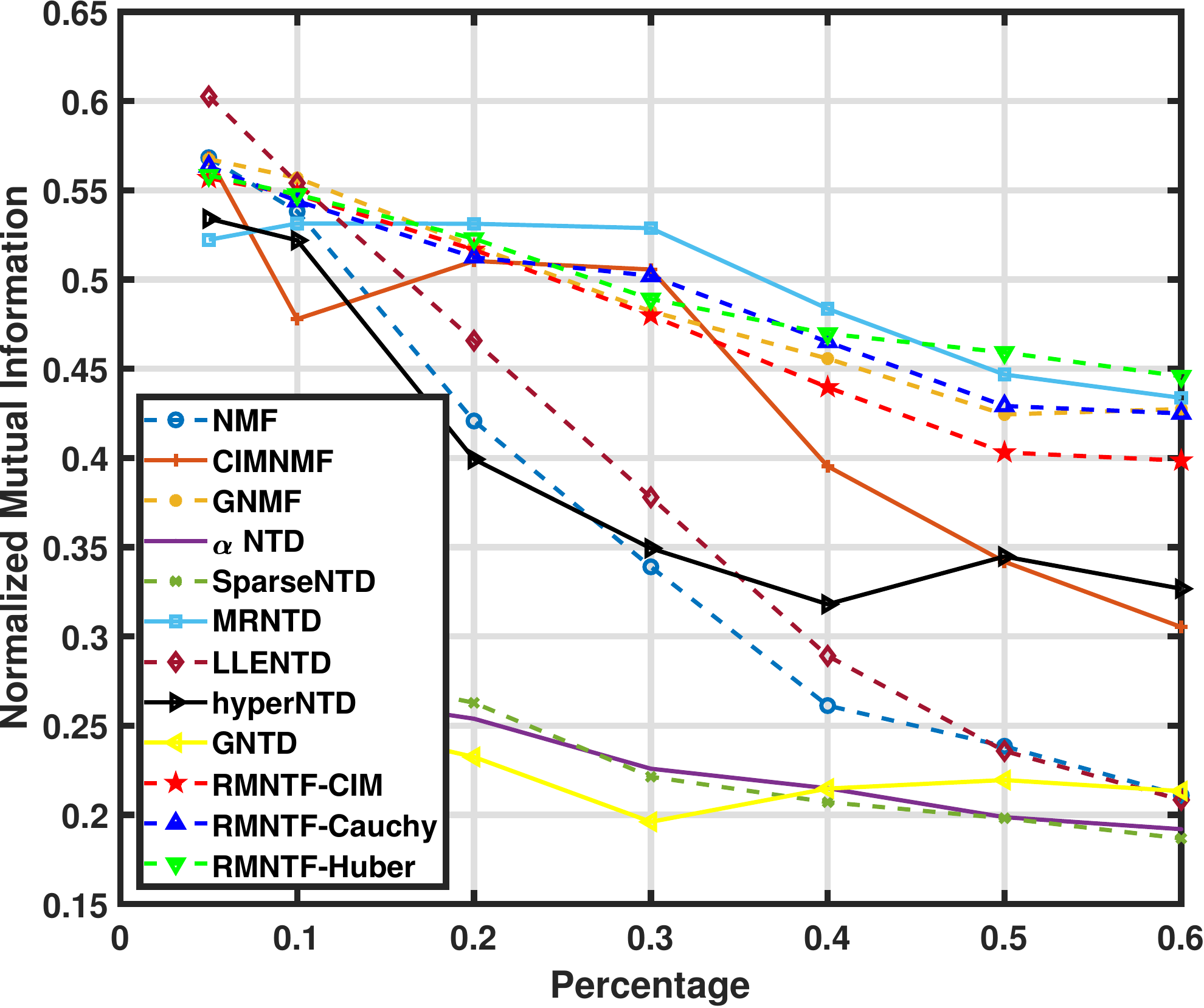}%
\end{minipage}
}
\hfil
\subfloat[]{
\begin{minipage}[b]{0.22\textwidth}
\includegraphics[width=1.5in]{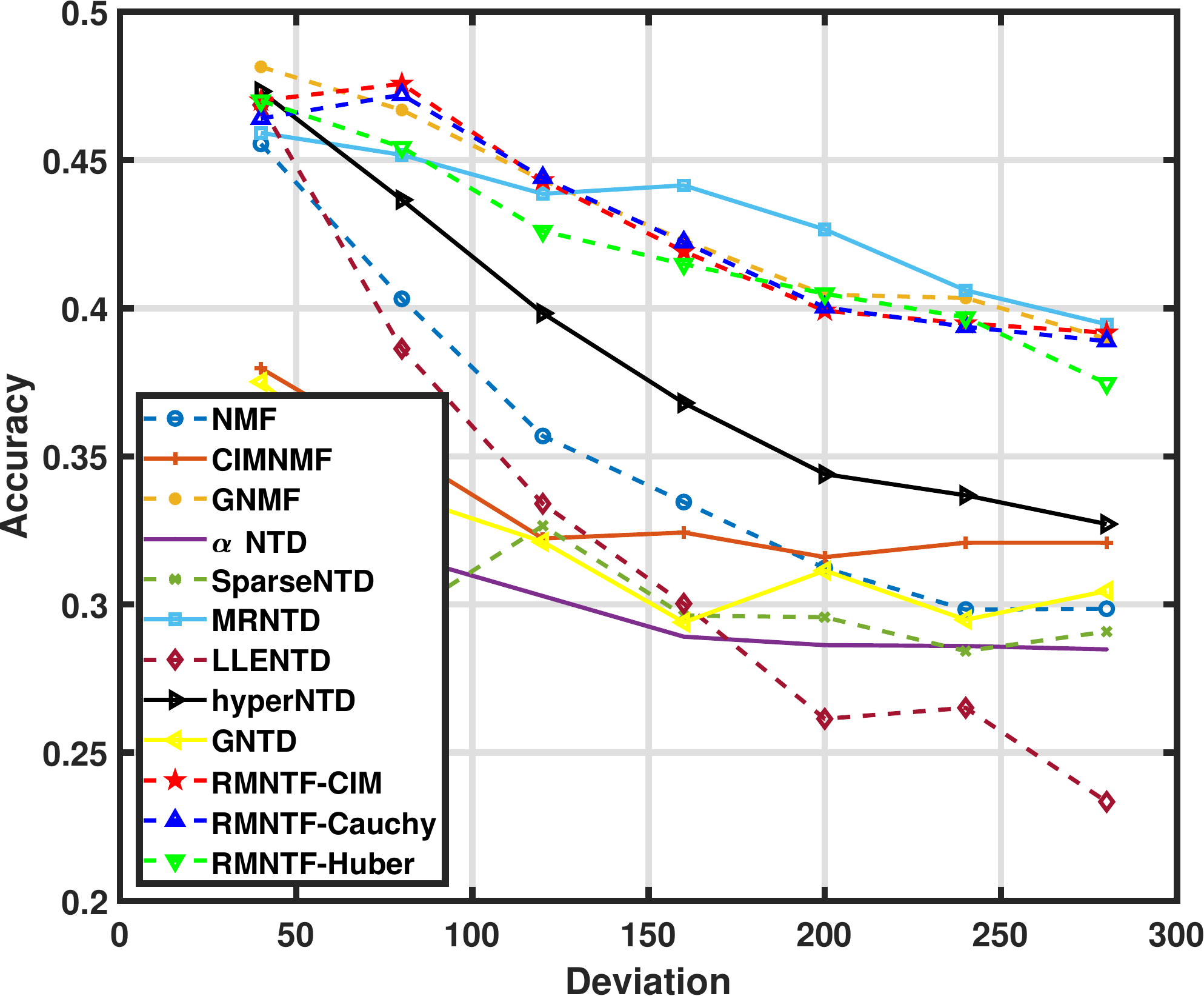}%
\label{Reuters_delta_Coh} \\
\includegraphics[width=1.5in]{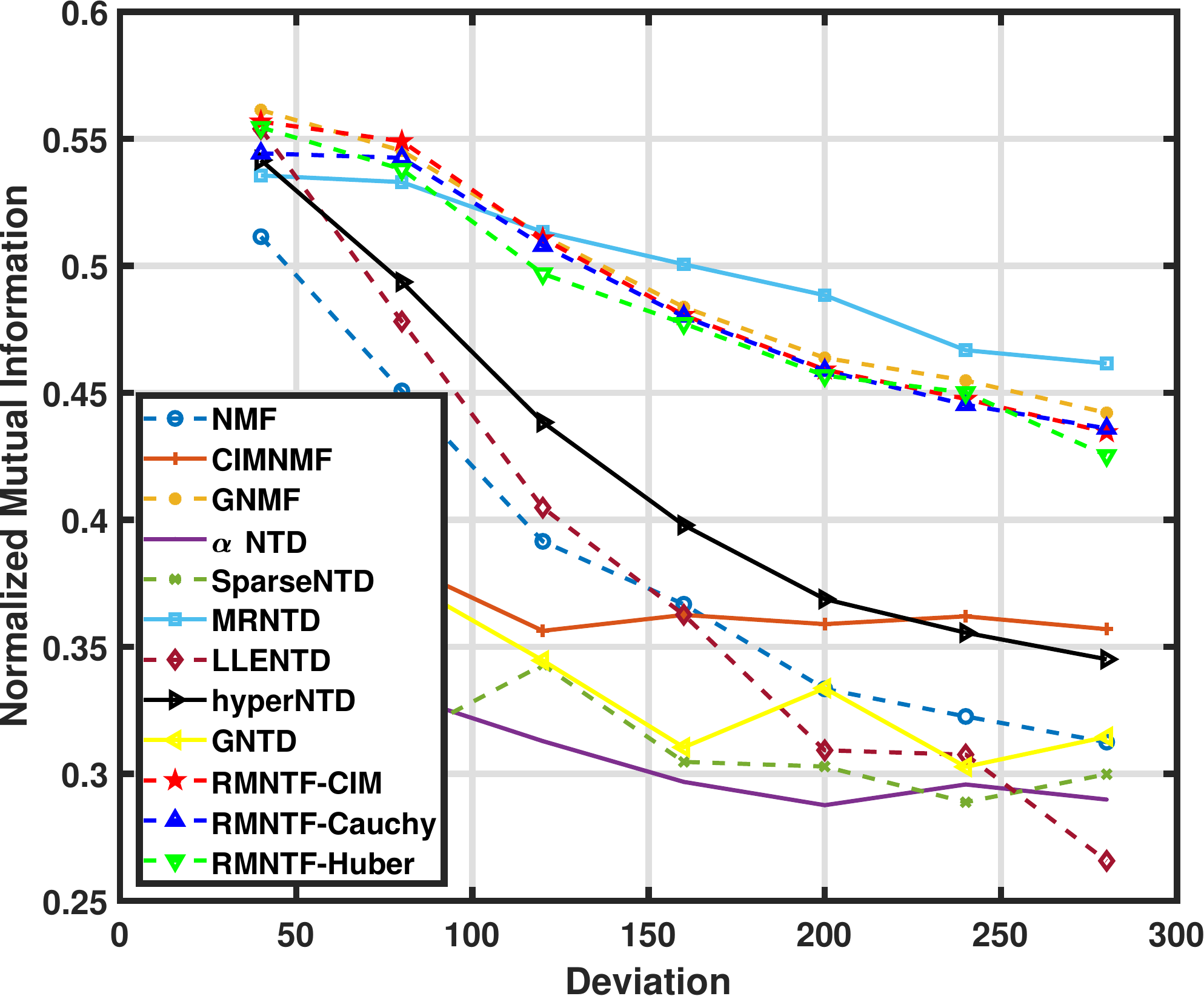}%
\label{Reuters_delta_Coh} \\
\includegraphics[width=1.5in]{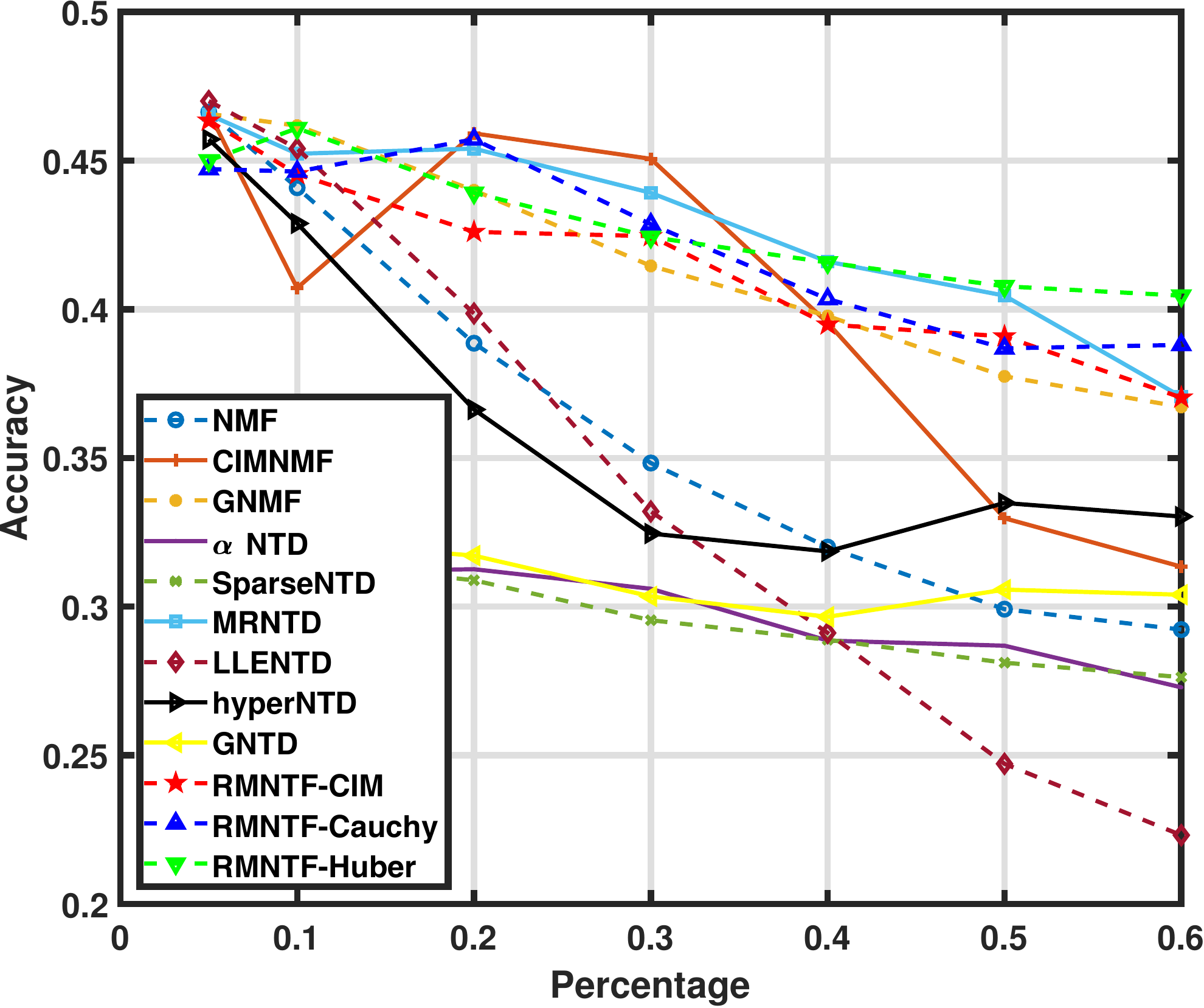}%
\label{Reuters_delta_Coh} \\
\includegraphics[width=1.5in]{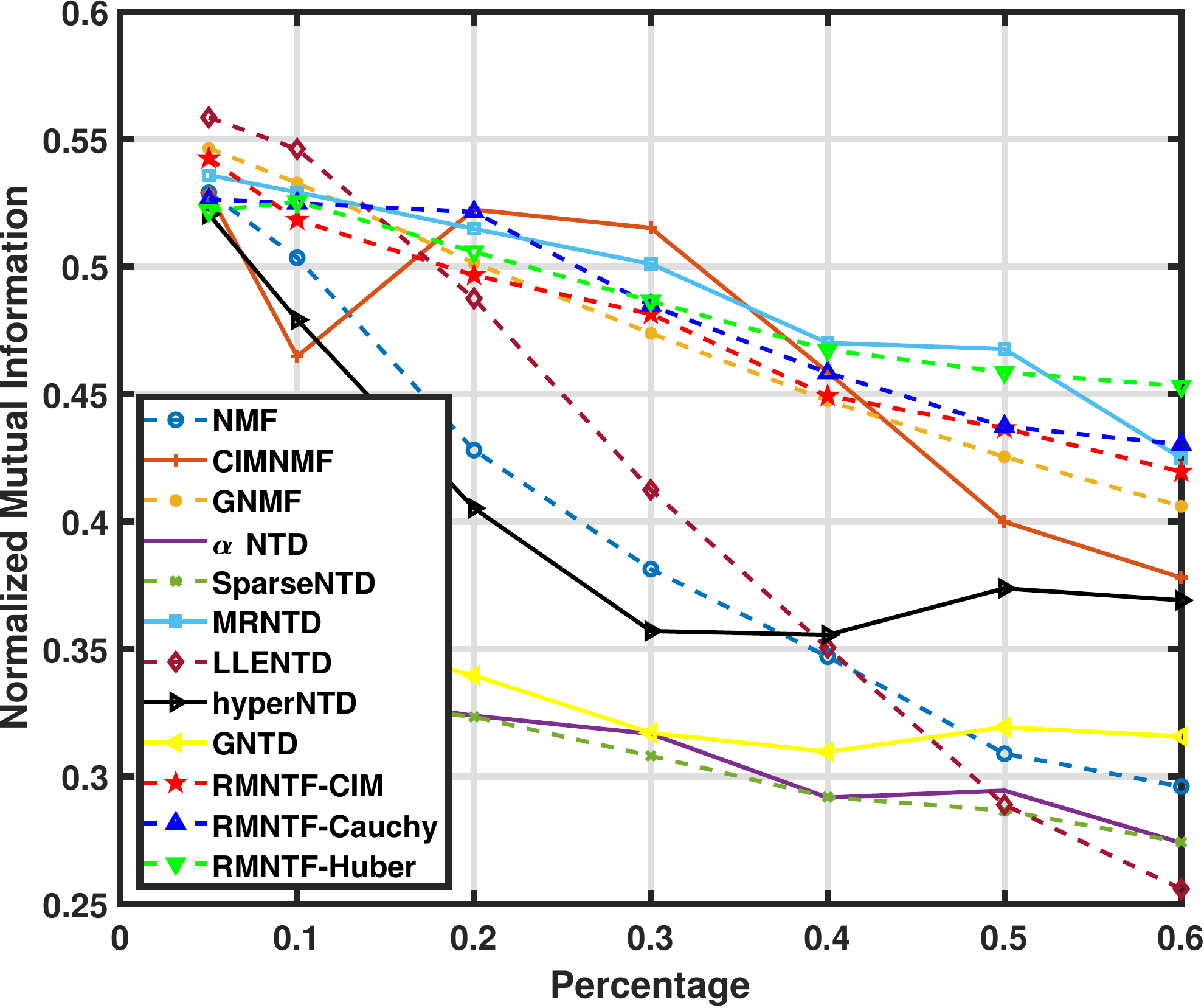}%
\end{minipage}
}
\subfloat[]{
\begin{minipage}[b]{0.22\textwidth}
\includegraphics[width=1.5in]{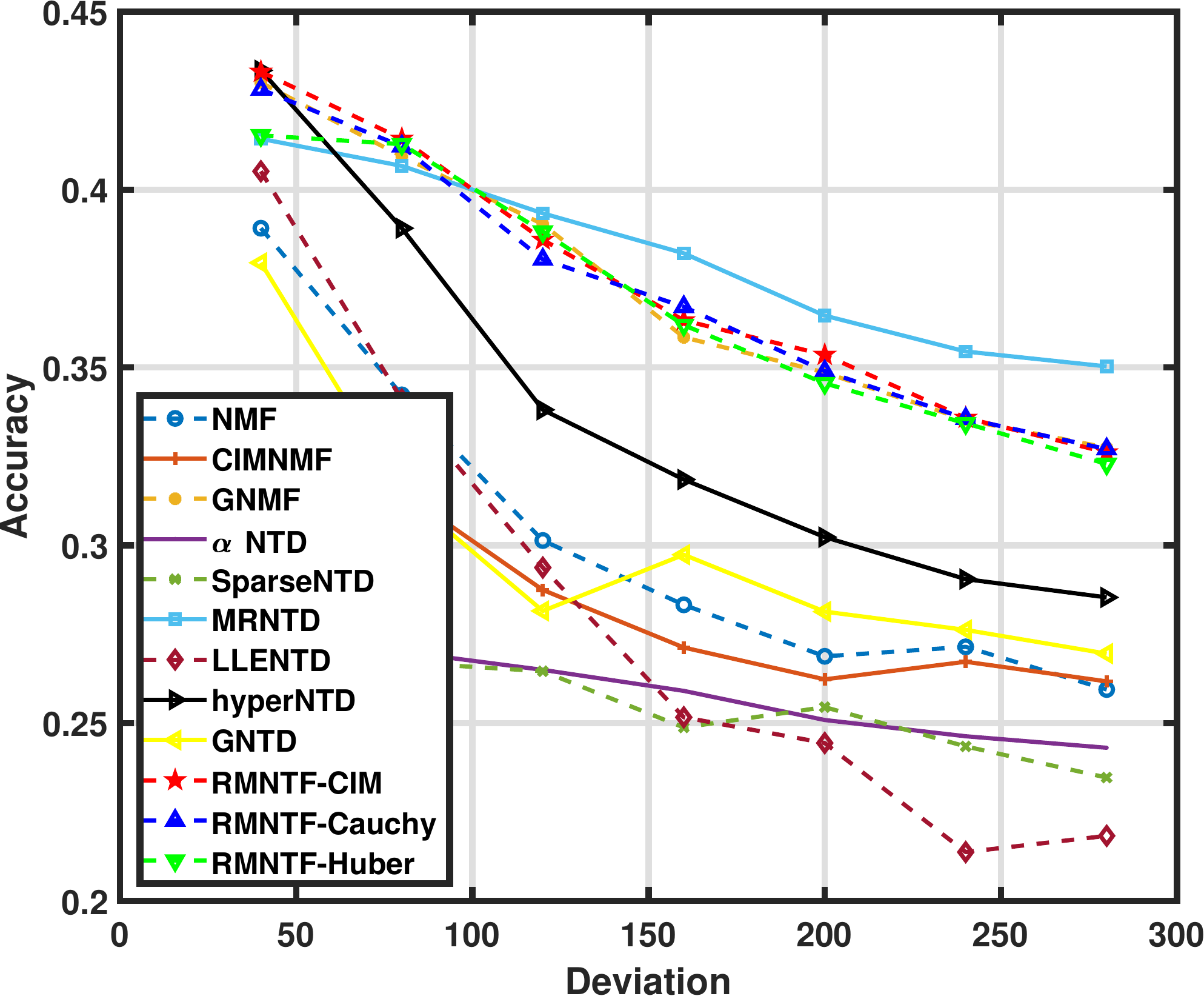}%
\label{Reuters_delta_Coh} \\
\includegraphics[width=1.5in]{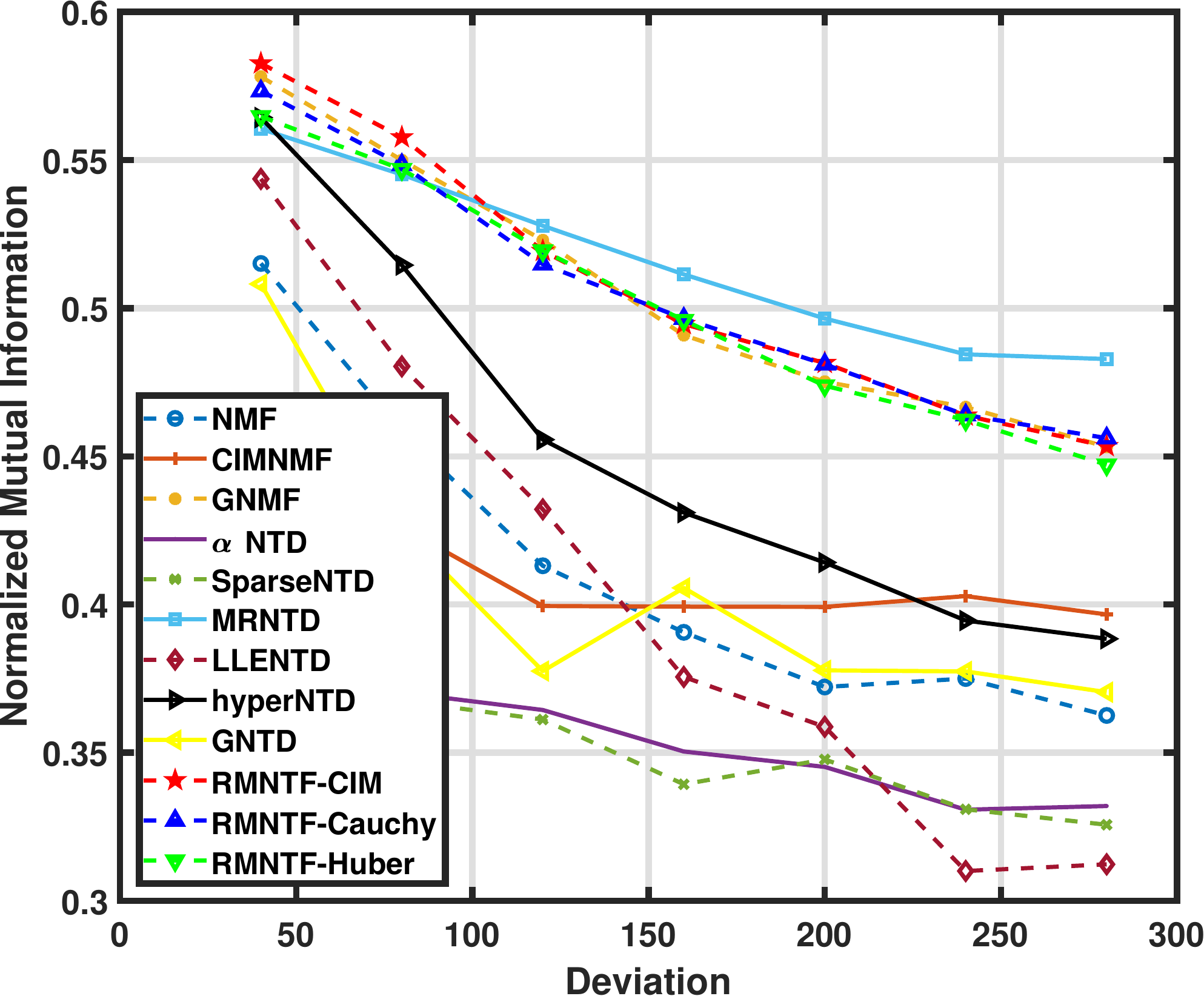}%
\label{Reuters_delta_Coh} \\
\includegraphics[width=1.5in]{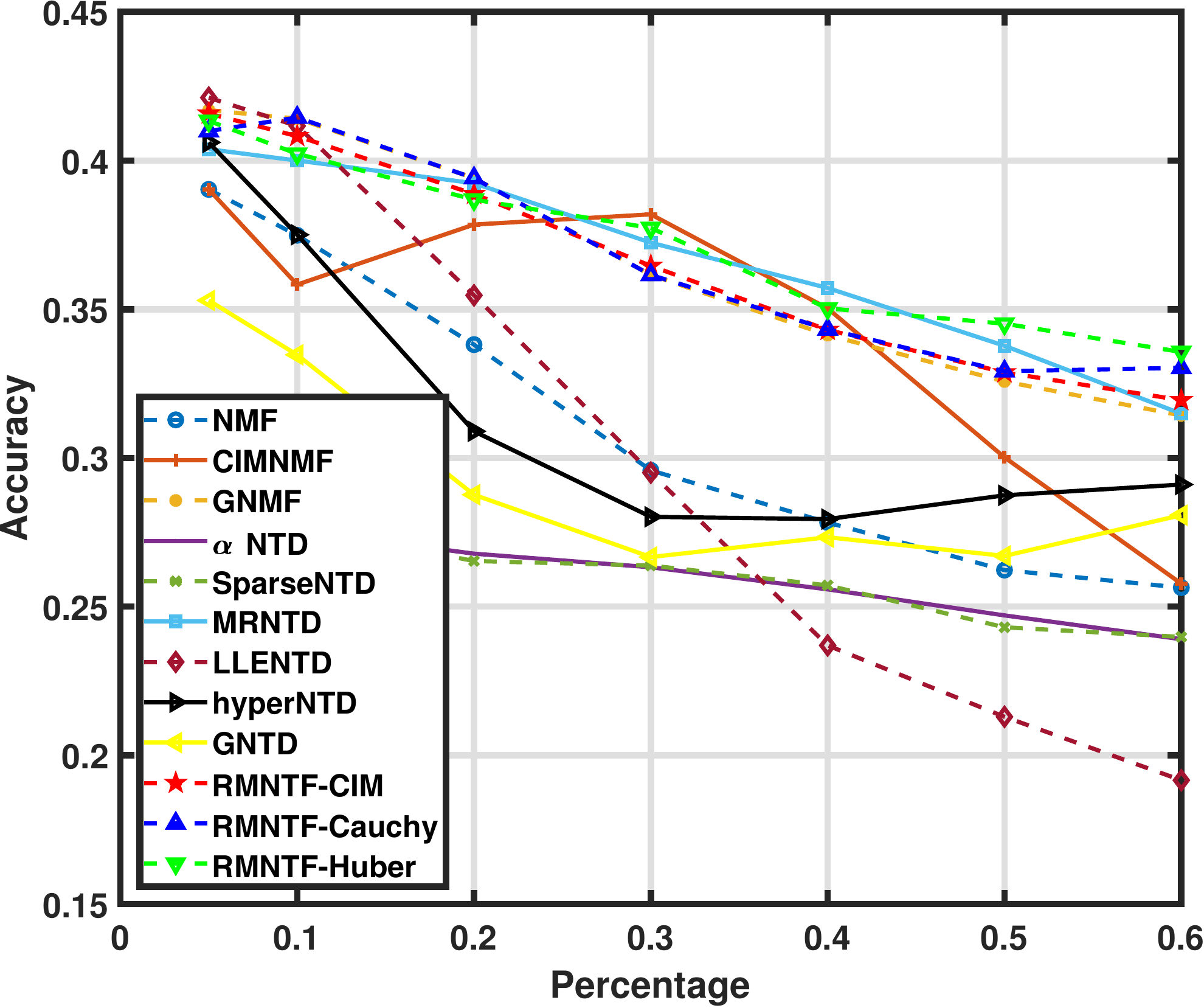}%
\label{Reuters_delta_Coh} \\
\includegraphics[width=1.5in]{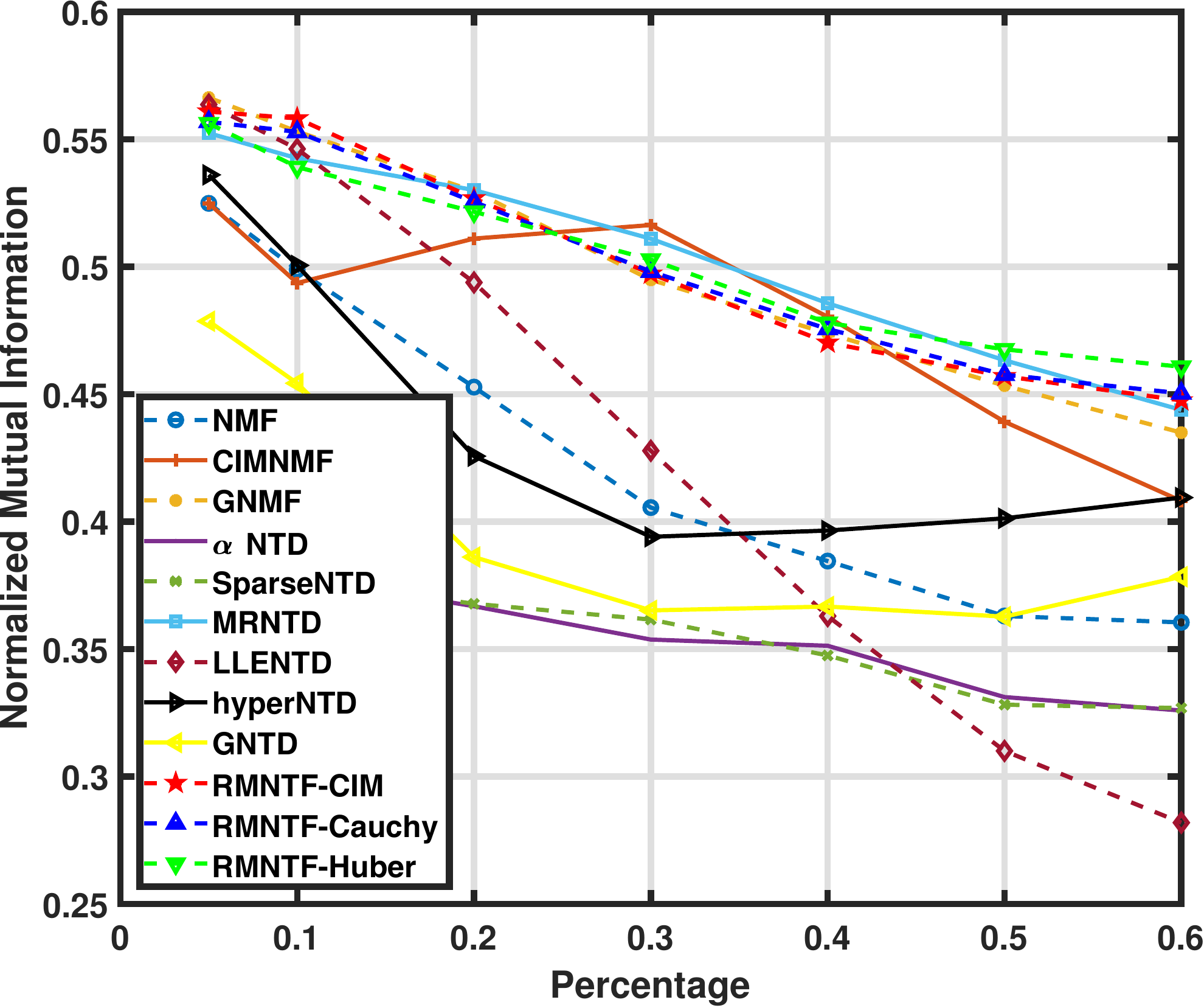}%
\end{minipage}
}
\subfloat[]{
\begin{minipage}[b]{0.22\textwidth}
\includegraphics[width=1.5in]{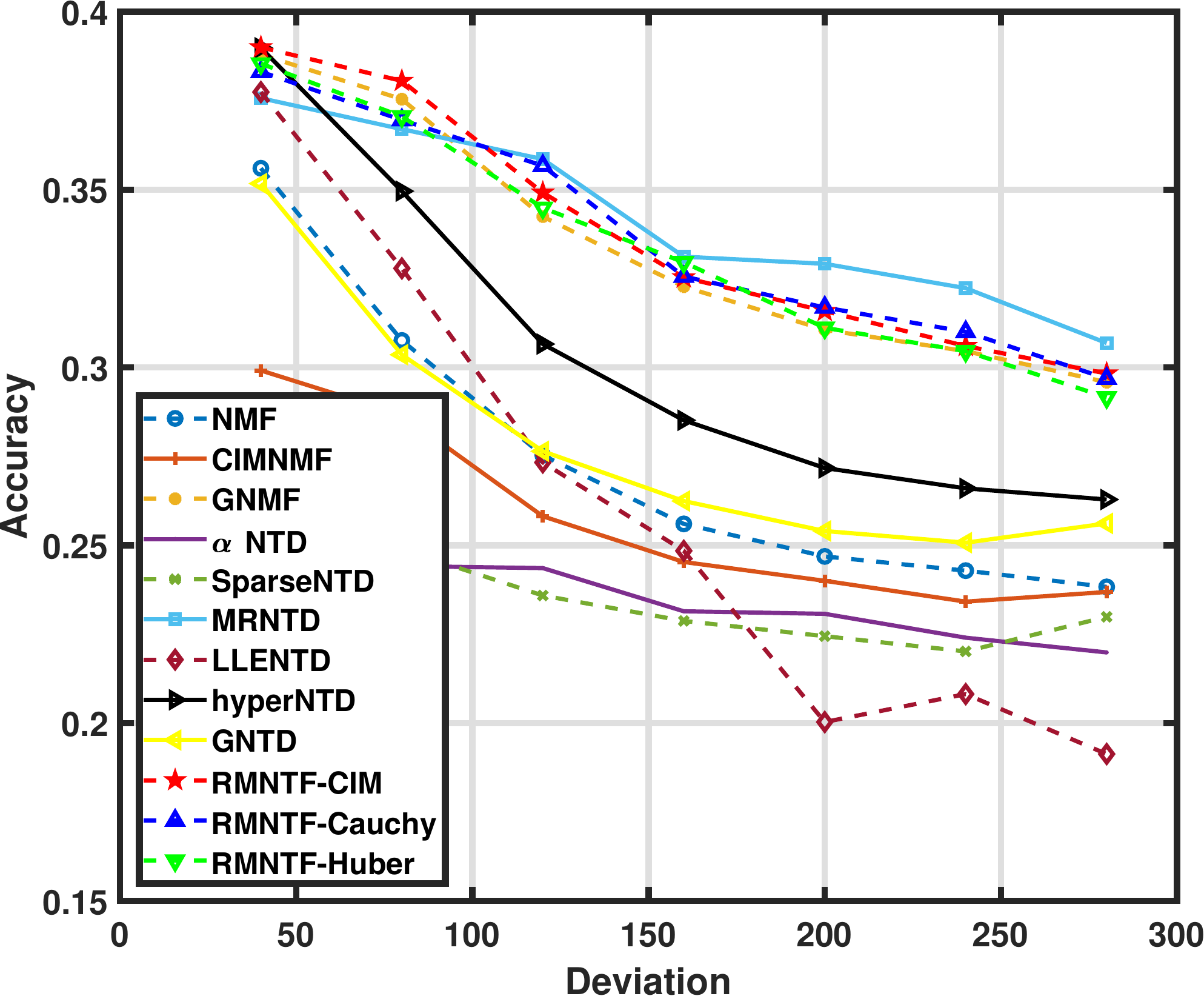}%
\label{Reuters_delta_Coh} \\
\includegraphics[width=1.5in]{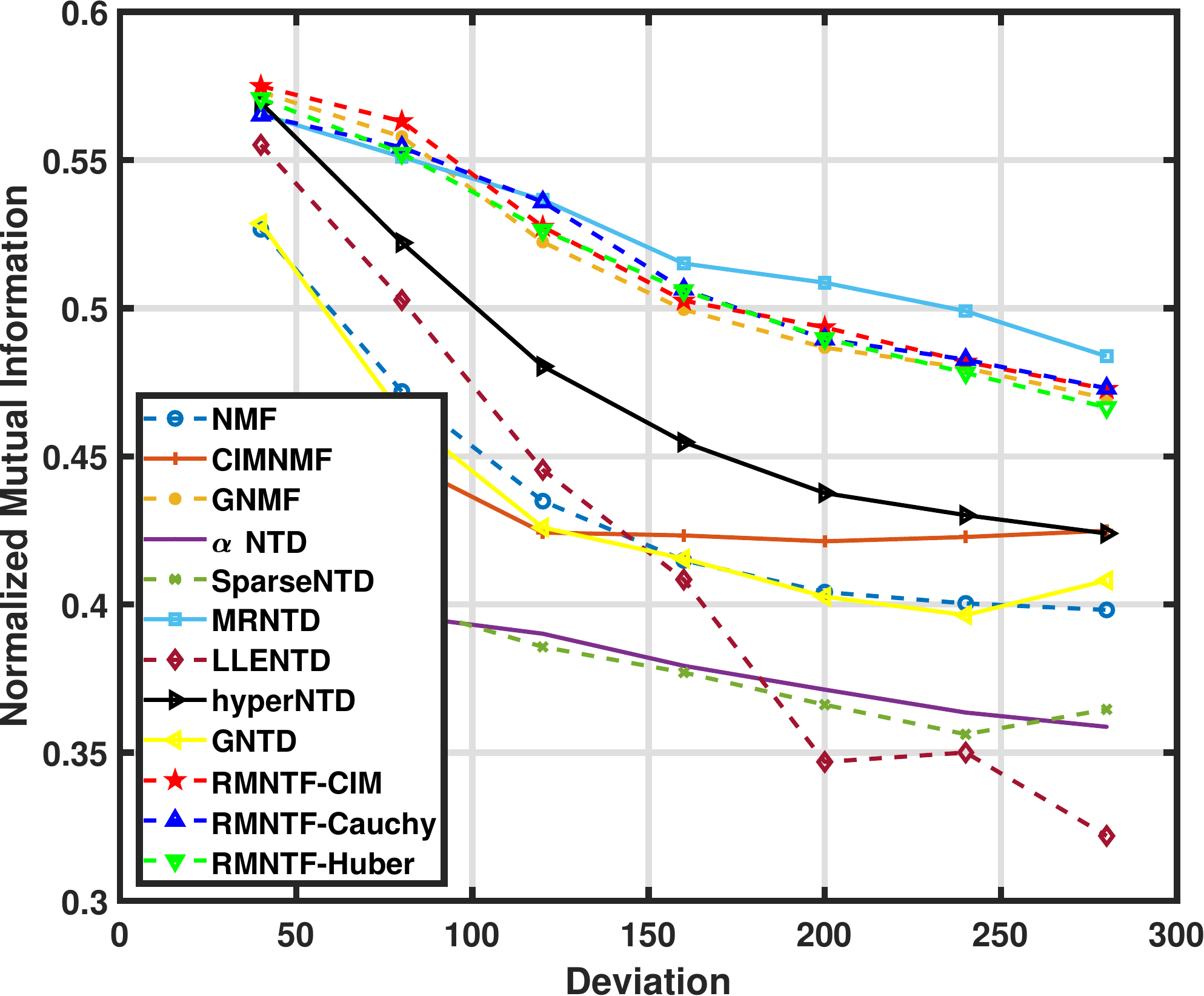}%
\label{Reuters_delta_Coh} \\
\includegraphics[width=1.5in]{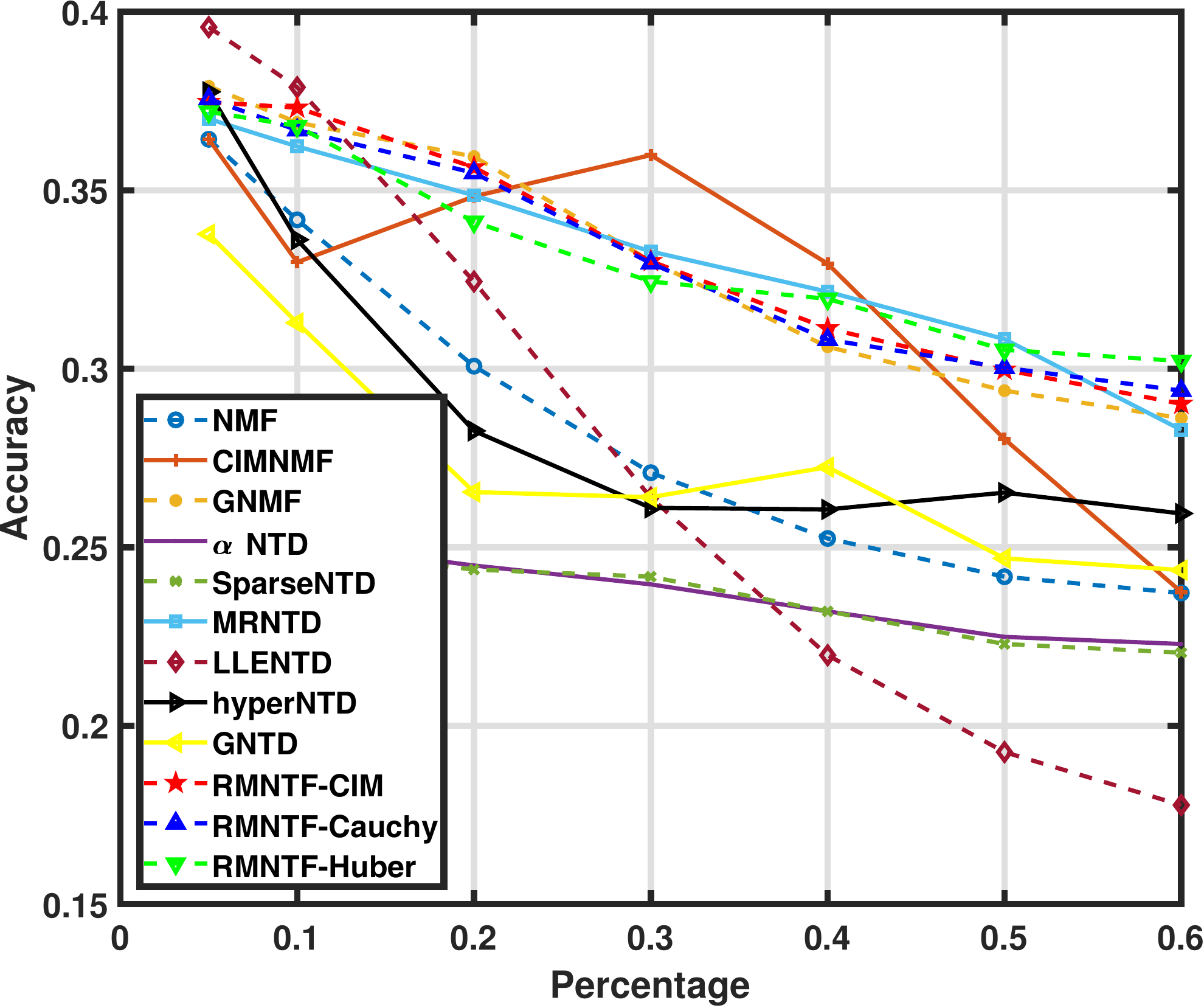}%
\label{Reuters_delta_Coh} \\
\includegraphics[width=1.5in]{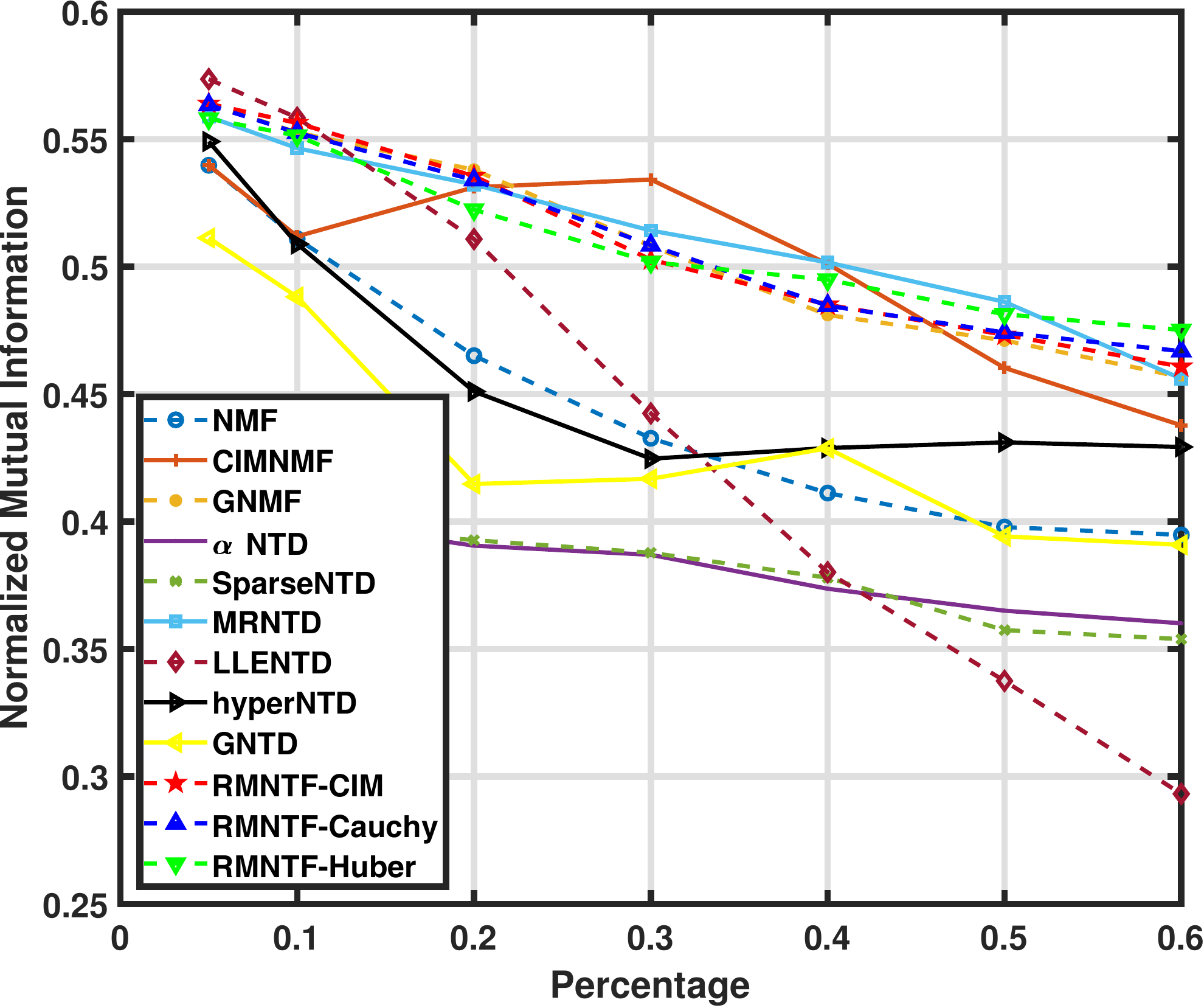}%
\end{minipage}
}

\vspace{2mm}
\caption{Evaluation of proposed methods on FERET database contaminated by Laplace noise and salt \& pepper noise, respectively. (a) Average accuracy and NMI on the subset of FERET which contains $5$ categories, the first two figures show the results contaminated by Laplace noise and the remains show the results contaminated by salt \& pepper noise. (b) Average evaluation on the subset of FERET which contains $10$ categories. (c) Average evaluation on the subset of FERET which contains $15$ categories. (d) Average evaluation on the subset of FERET which contains $20$ categories.}
\label{Reuters_delta}
\end{figure*}

\begin{figure*}[t]
\vspace{-0.5cm} 
\setlength{\abovecaptionskip}{0cm} 
\setlength{\belowcaptionskip}{-0cm} 
\centering
\subfloat[]{
\begin{minipage}[b]{0.18\textwidth}
\includegraphics[width=1.5in]{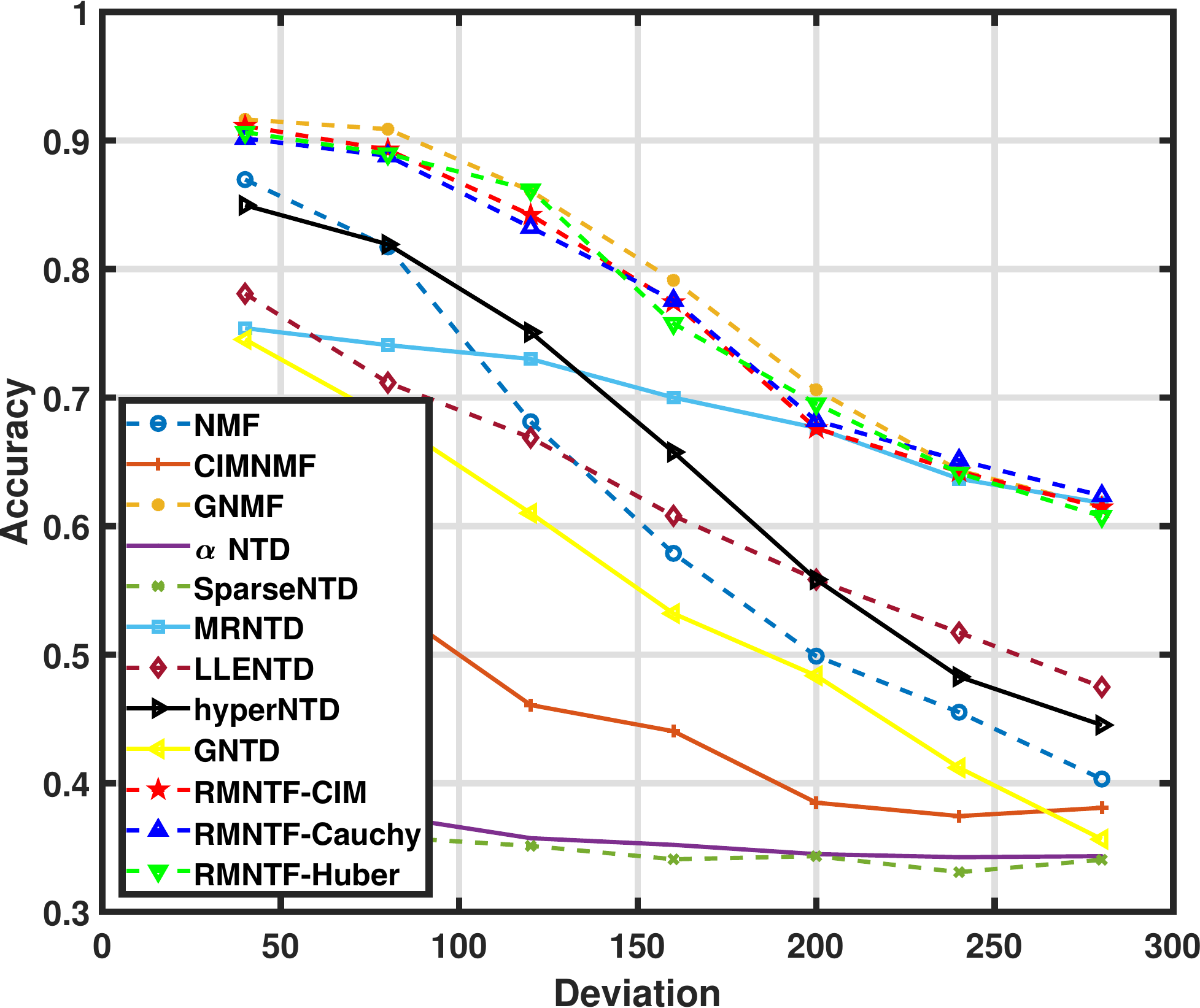}%
\label{Reuters_delta_Coh} \\
\includegraphics[width=1.5in]{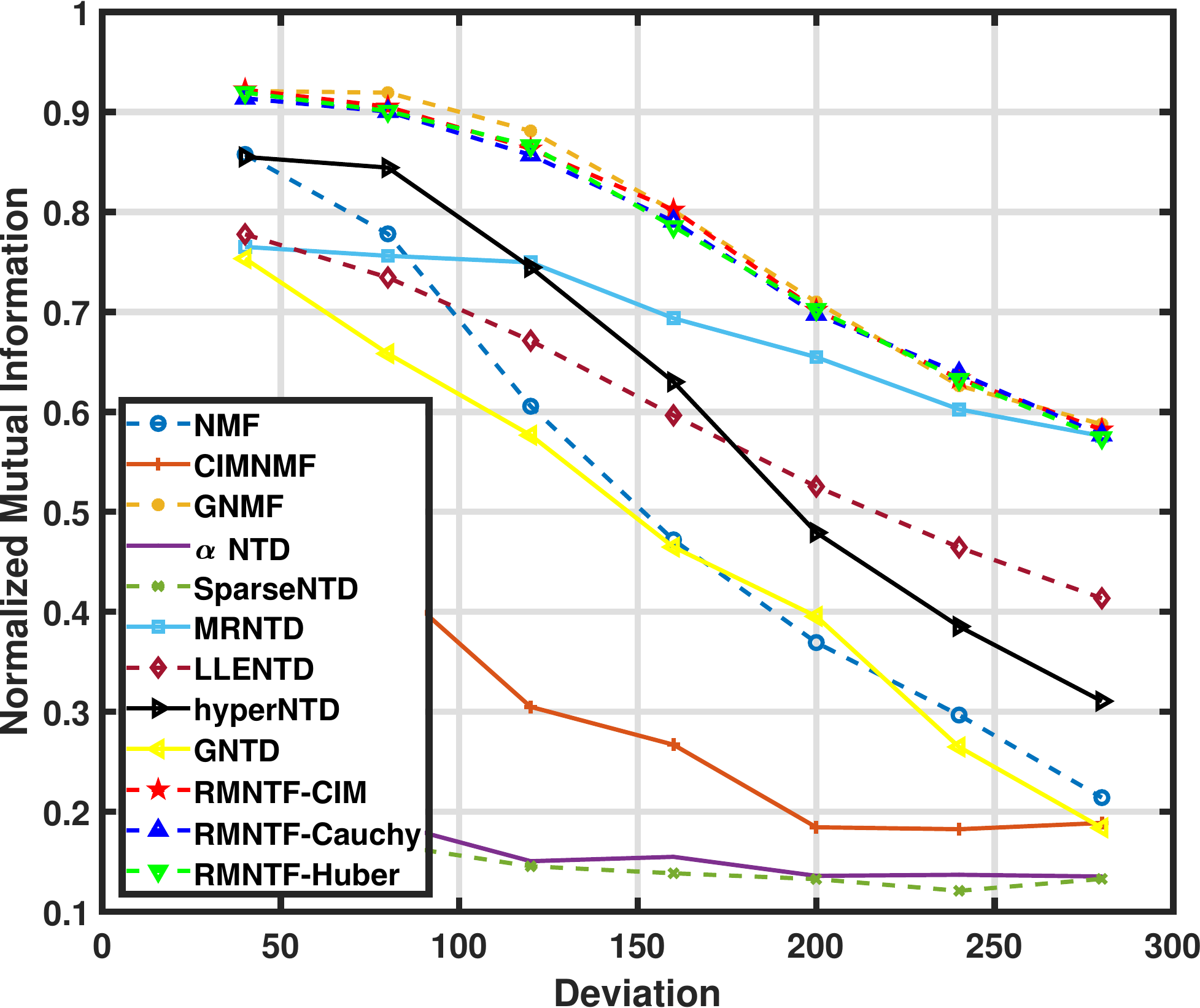}%
\label{Reuters_delta_Coh} \\
\includegraphics[width=1.5in]{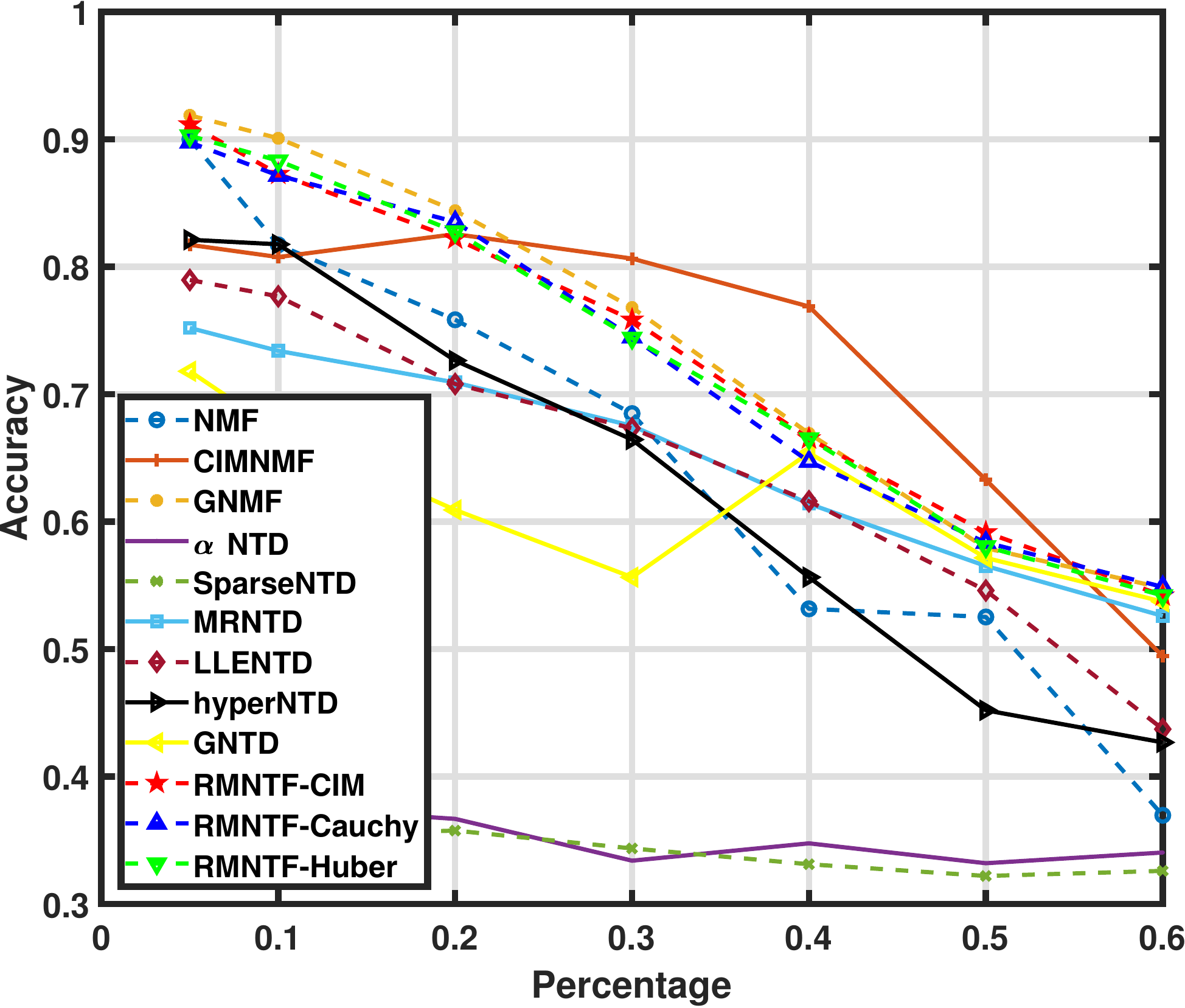}%
\label{Reuters_delta_Coh} \\
\includegraphics[width=1.5in]{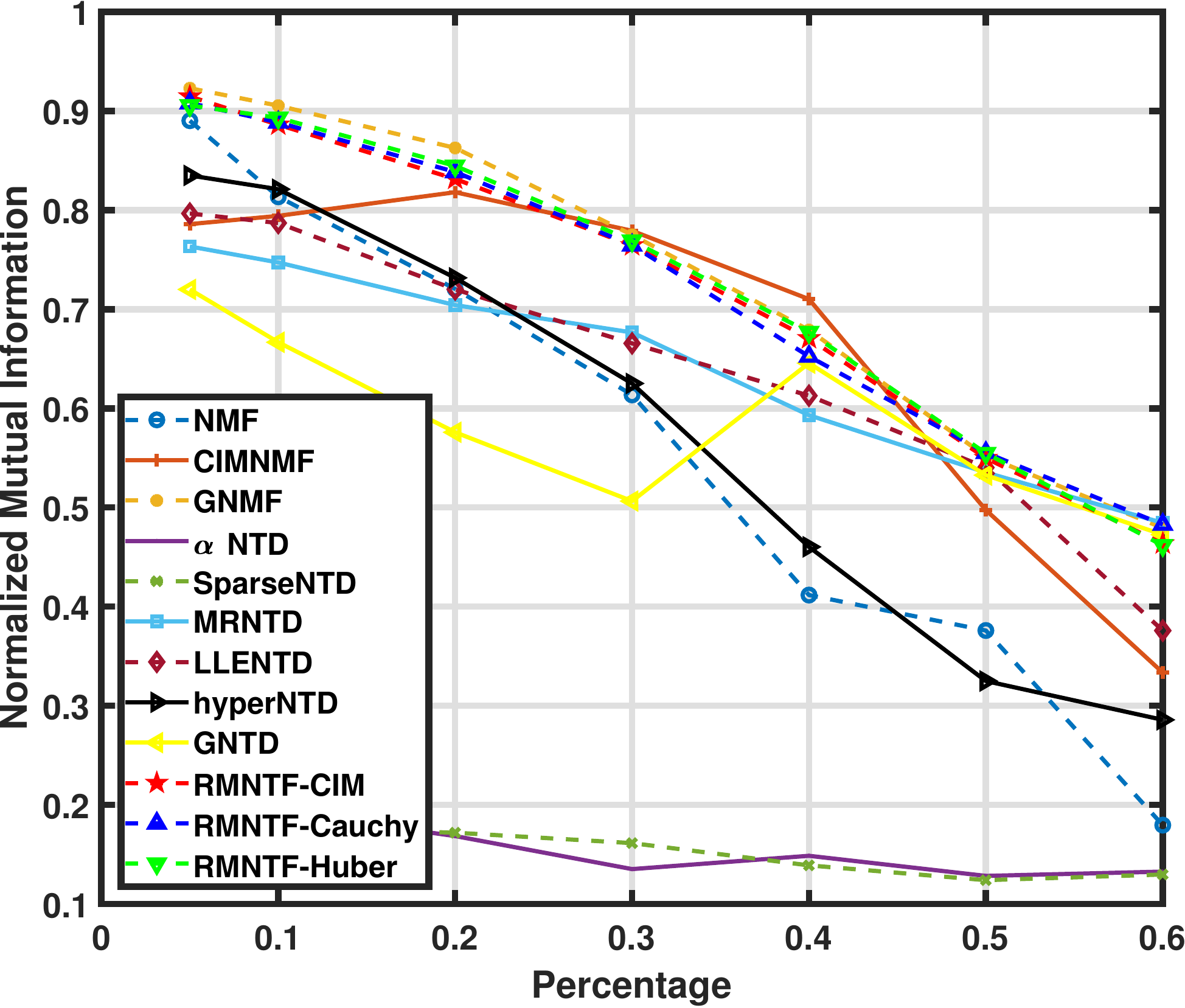}%
\end{minipage}
}
\hfil
\subfloat[]{
\begin{minipage}[b]{0.22\textwidth}
\includegraphics[width=1.5in]{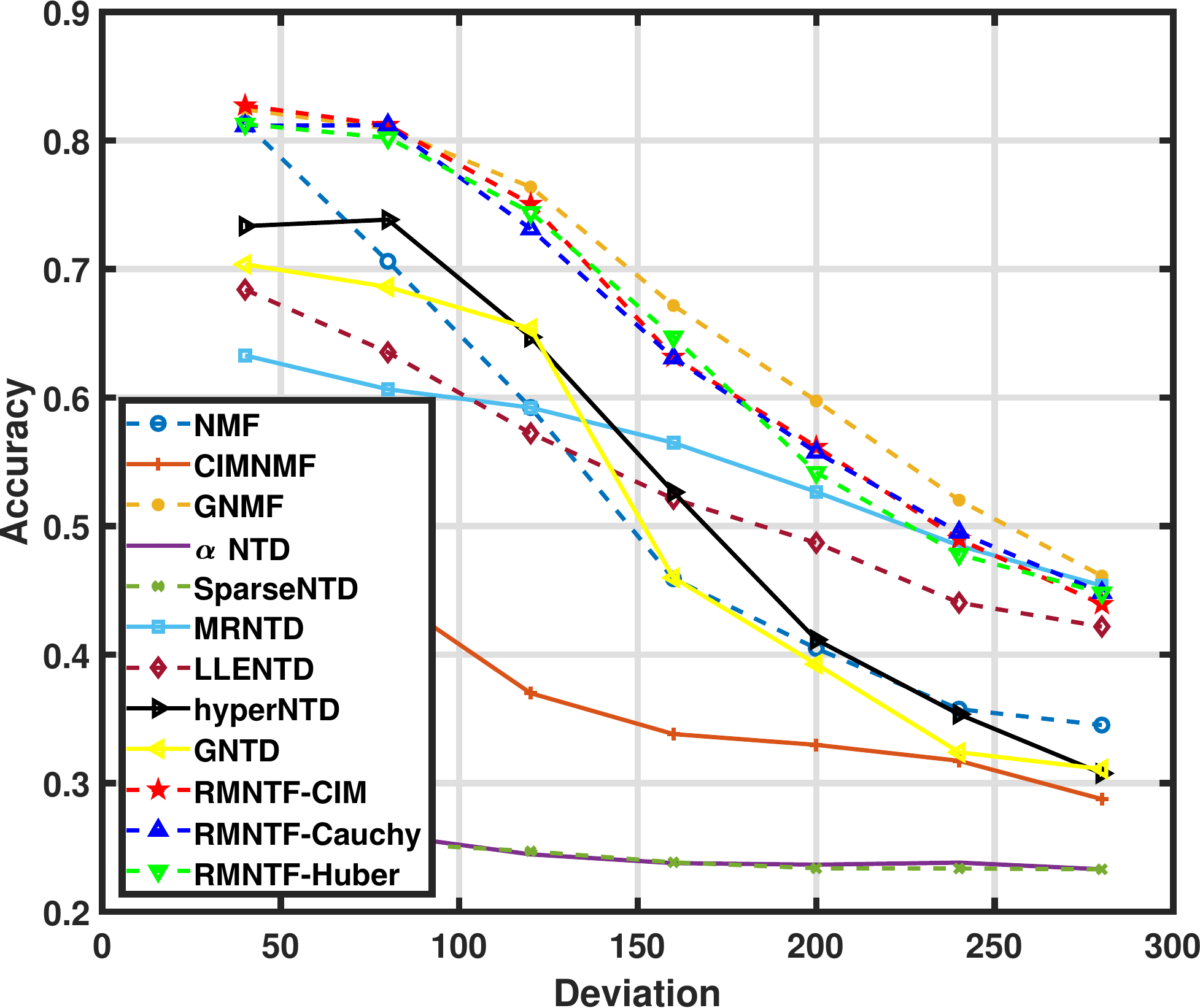}%
\label{Reuters_delta_Coh} \\
\includegraphics[width=1.5in]{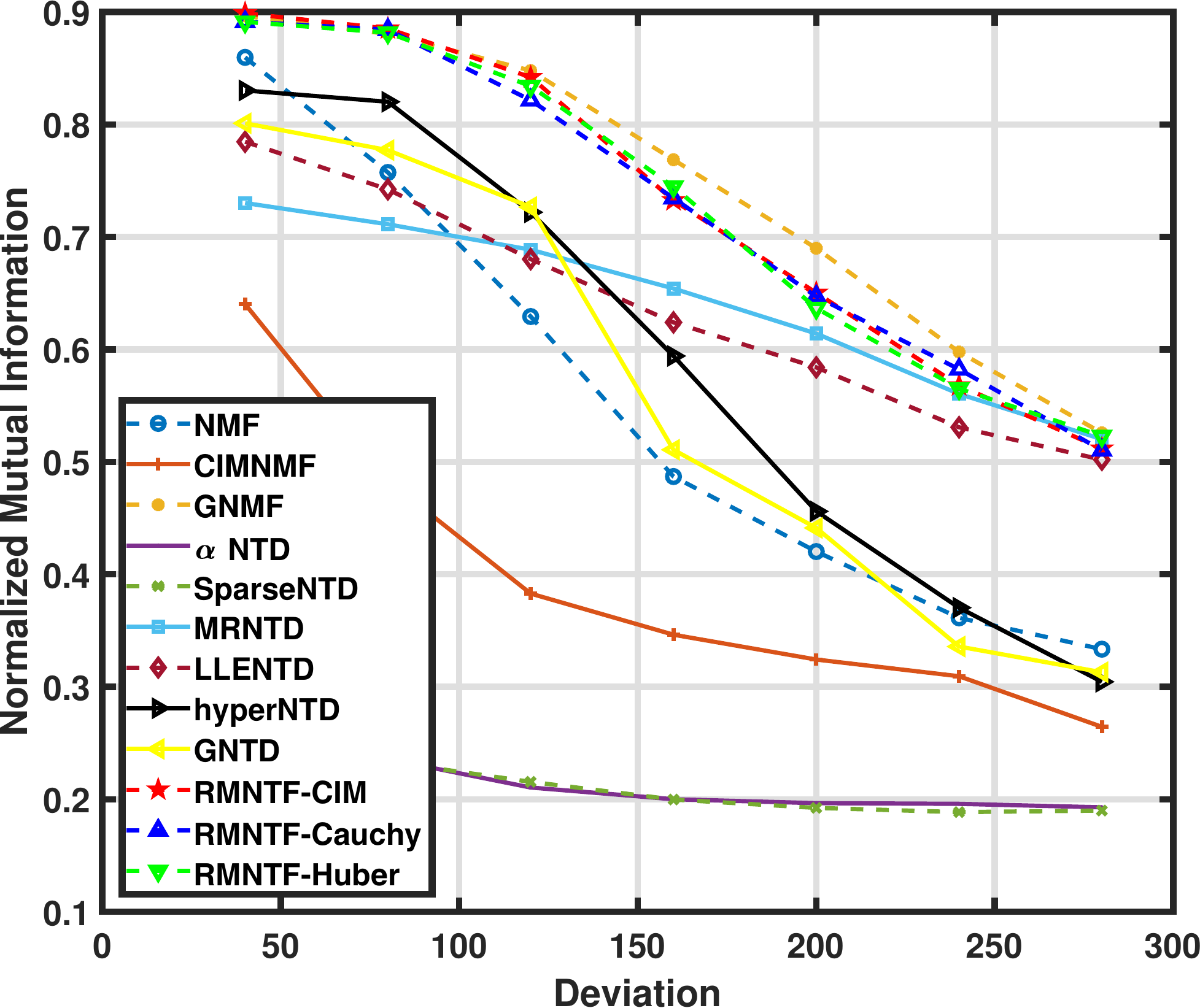}%
\label{Reuters_delta_Coh} \\
\includegraphics[width=1.5in]{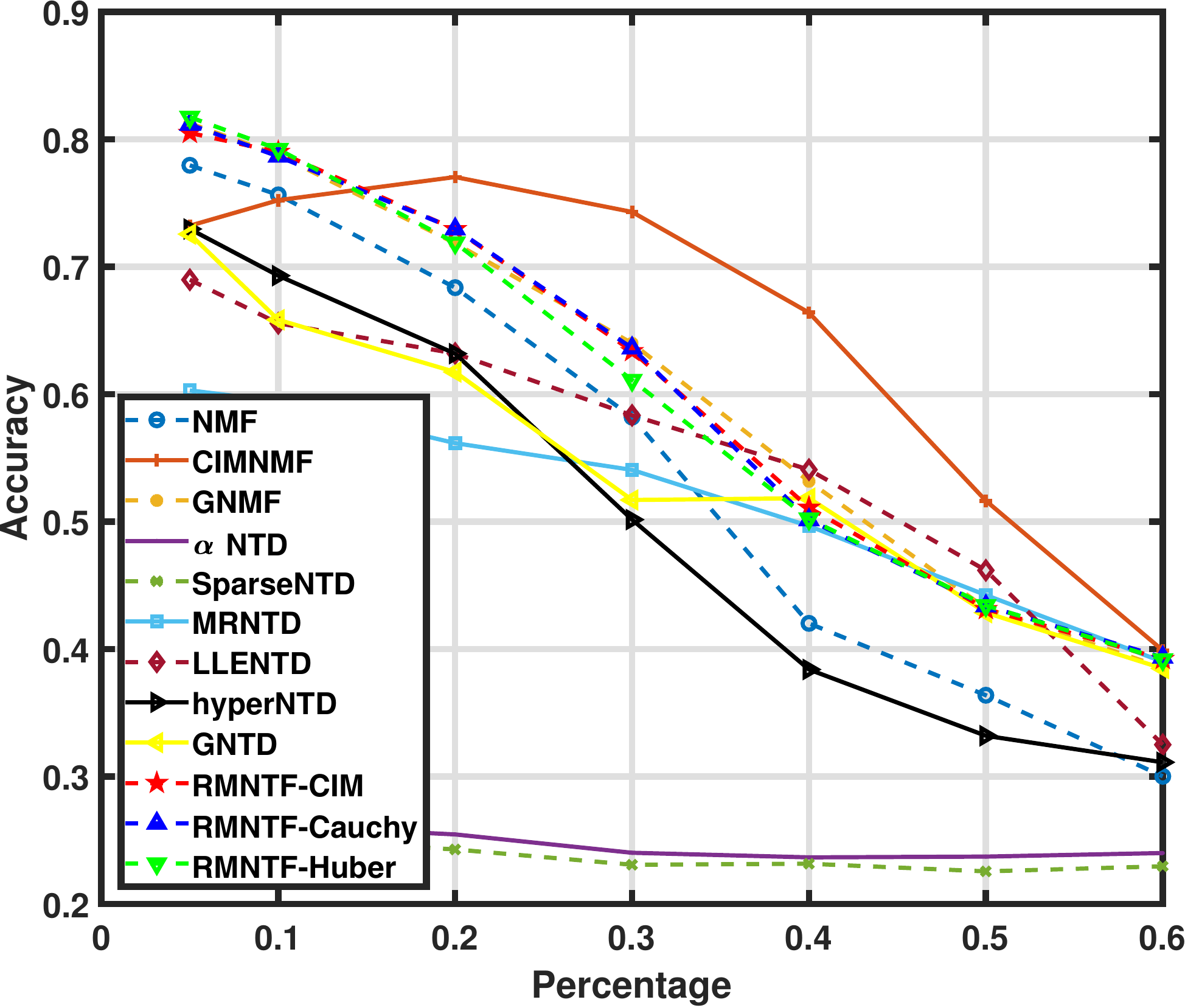}%
\label{Reuters_delta_Coh} \\
\includegraphics[width=1.5in]{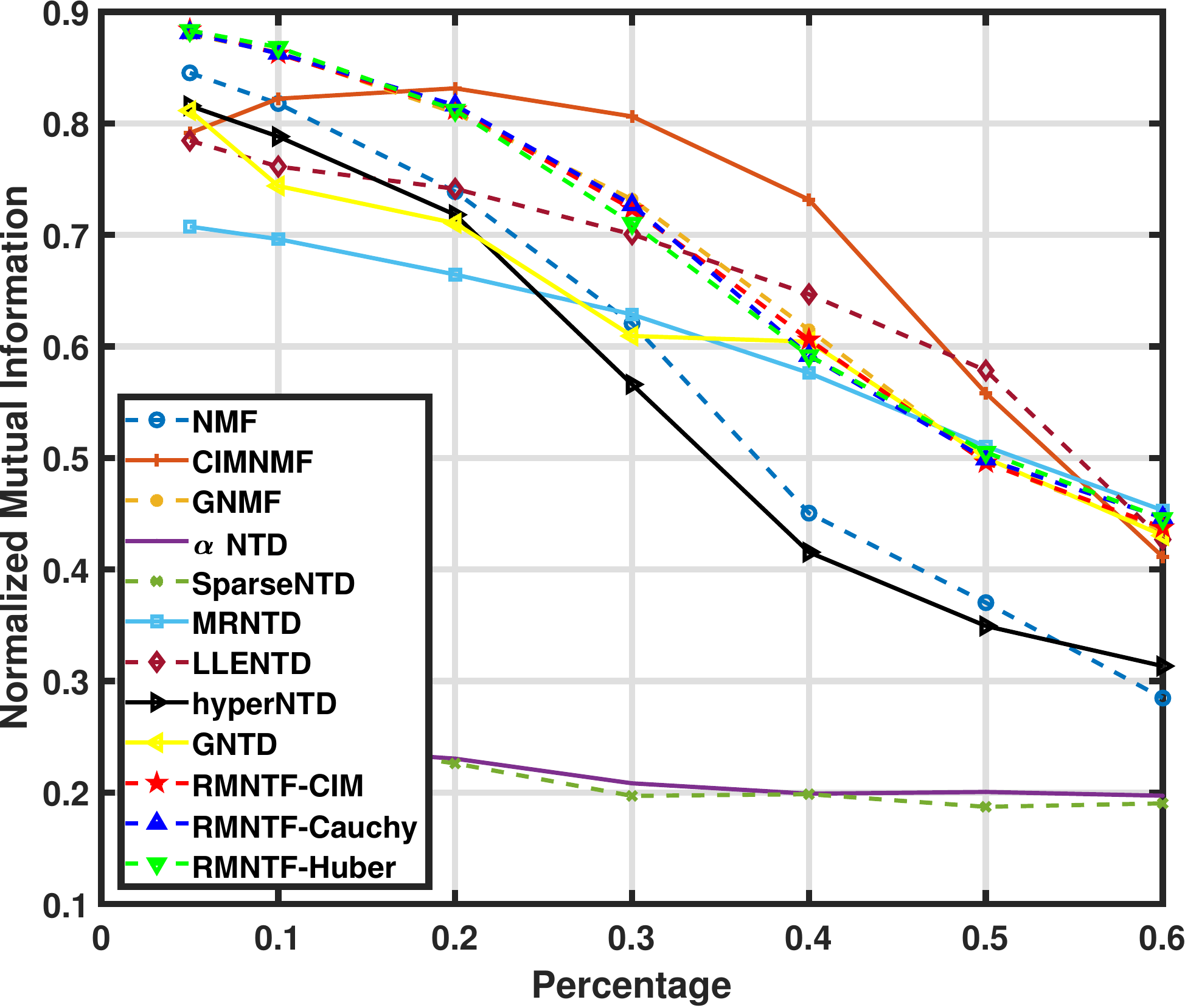}%
\end{minipage}
}
\subfloat[]{
\begin{minipage}[b]{0.22\textwidth}
\includegraphics[width=1.5in]{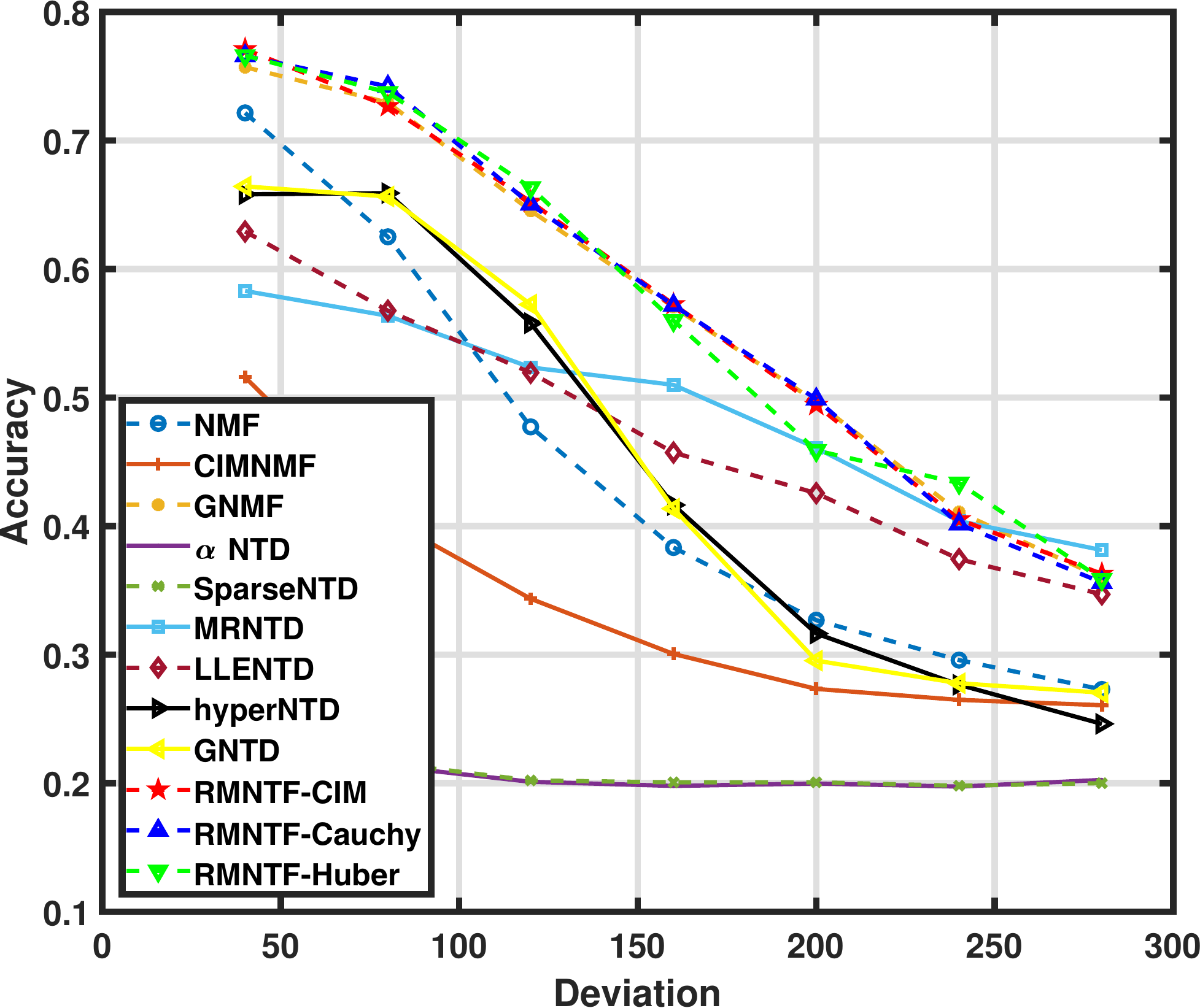}%
\label{Reuters_delta_Coh} \\
\includegraphics[width=1.5in]{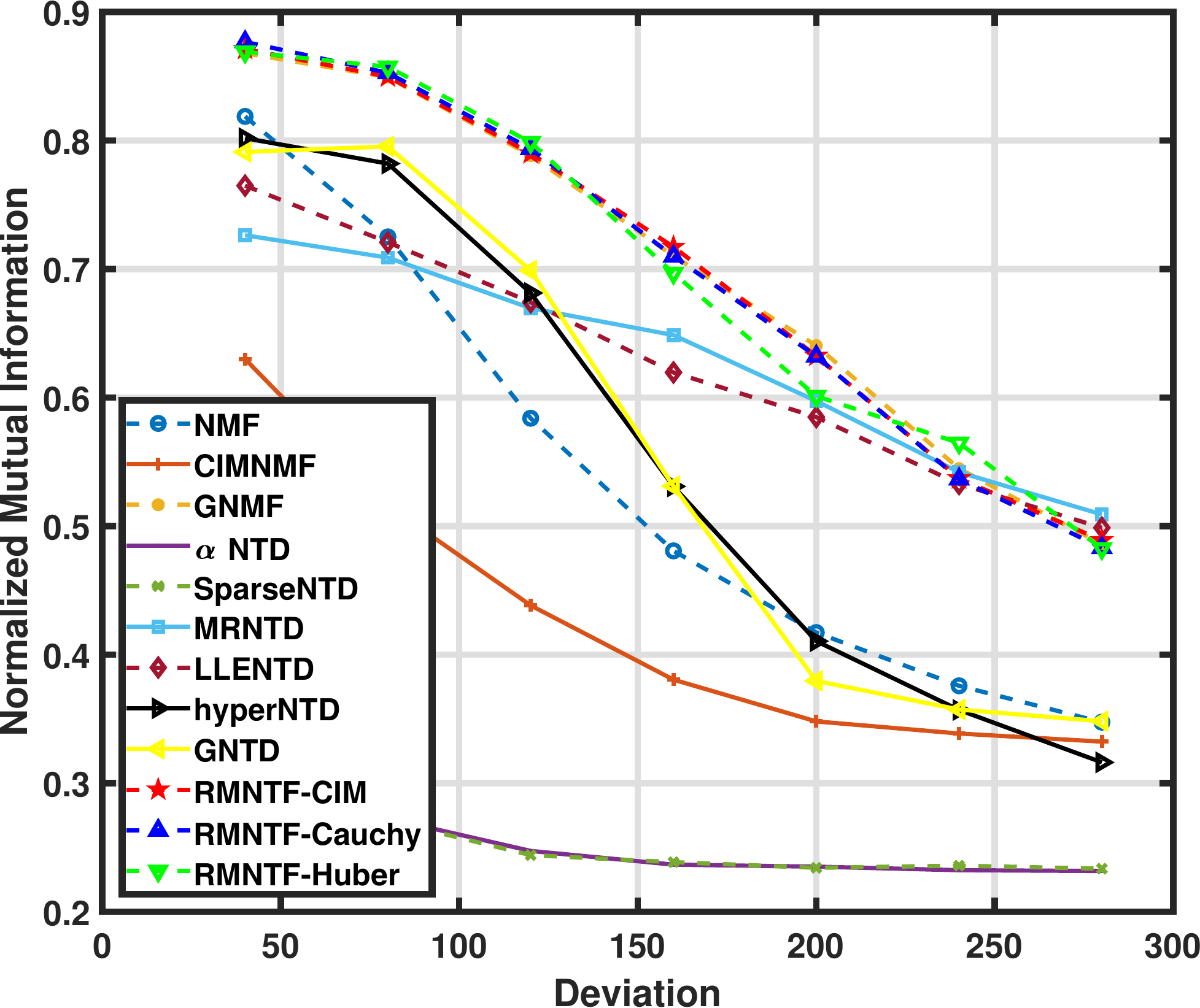}%
\label{Reuters_delta_Coh} \\
\includegraphics[width=1.5in]{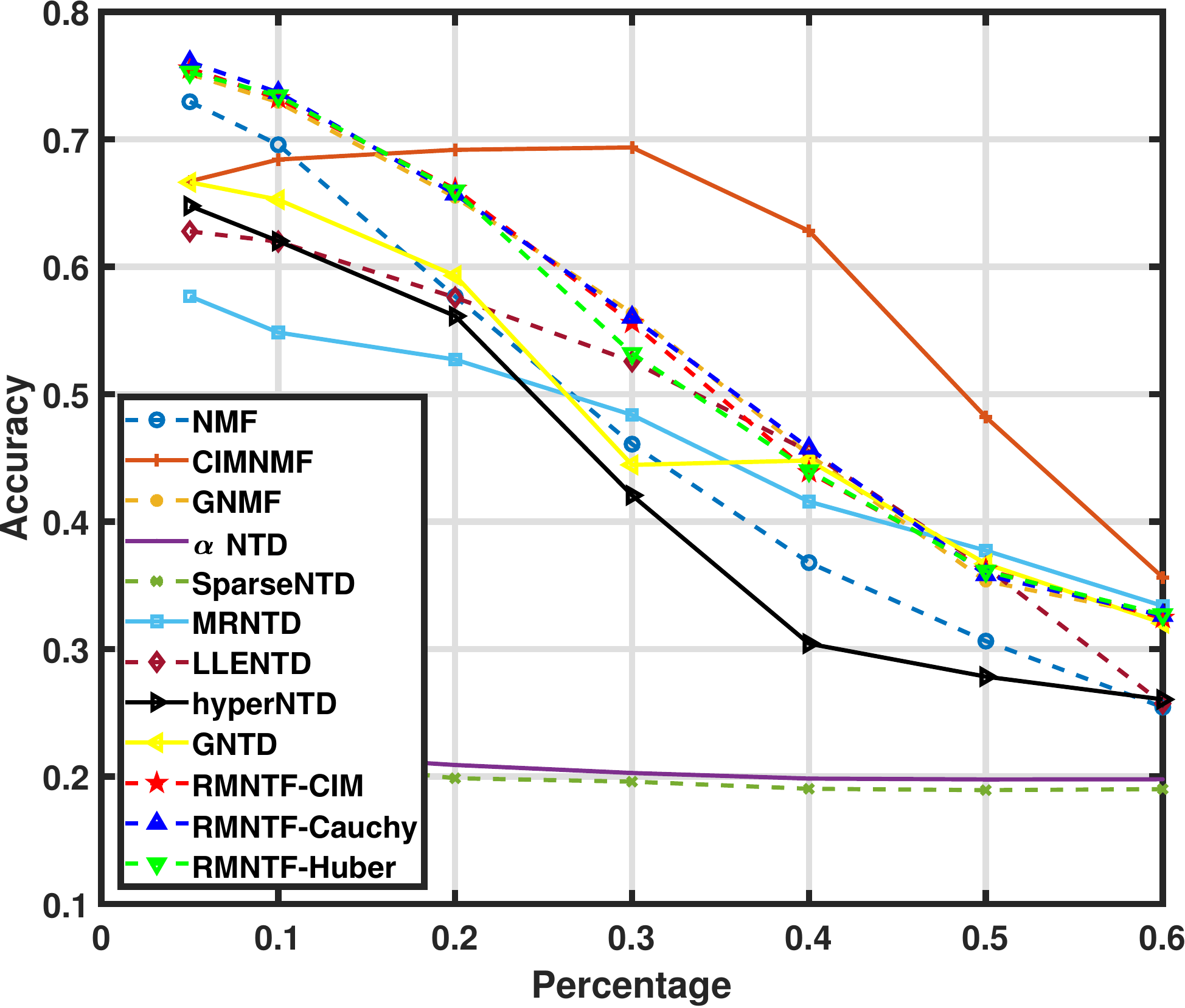}%
\label{Reuters_delta_Coh} \\
\includegraphics[width=1.5in]{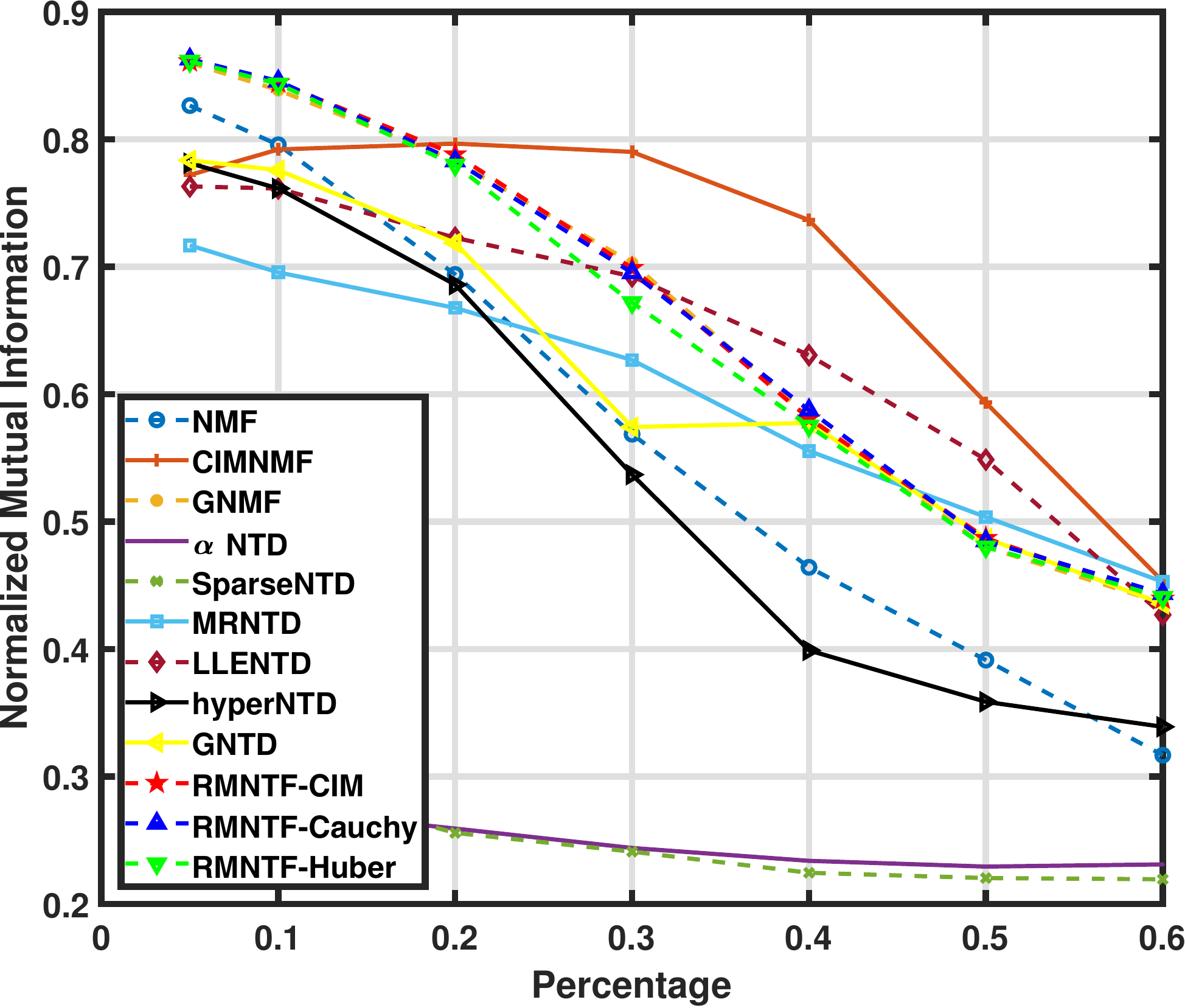}%
\end{minipage}
}
\subfloat[]{
\begin{minipage}[b]{0.22\textwidth}
\includegraphics[width=1.5in]{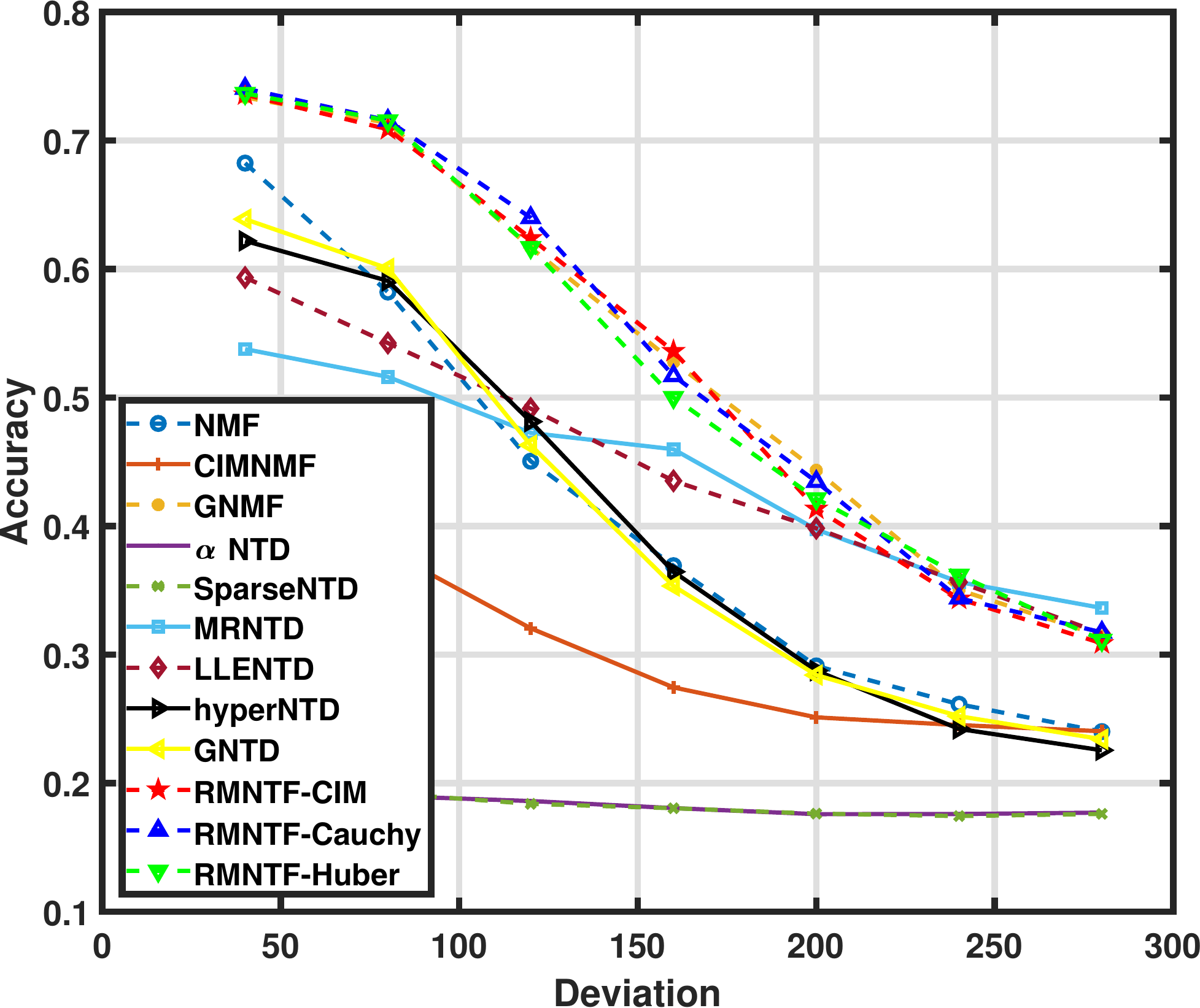}%
\label{Reuters_delta_Coh} \\
\includegraphics[width=1.5in]{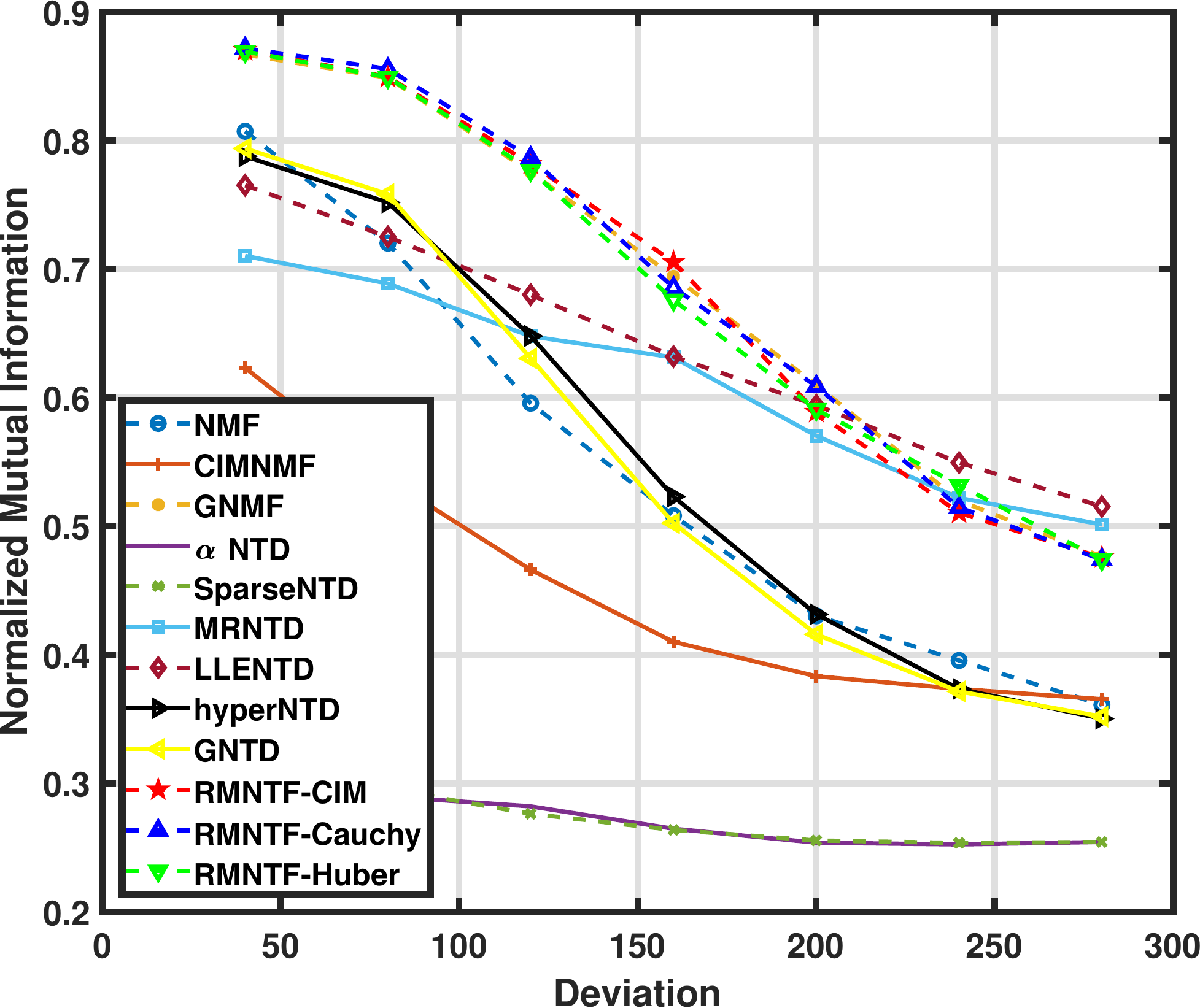}%
\label{Reuters_delta_Coh} \\
\includegraphics[width=1.5in]{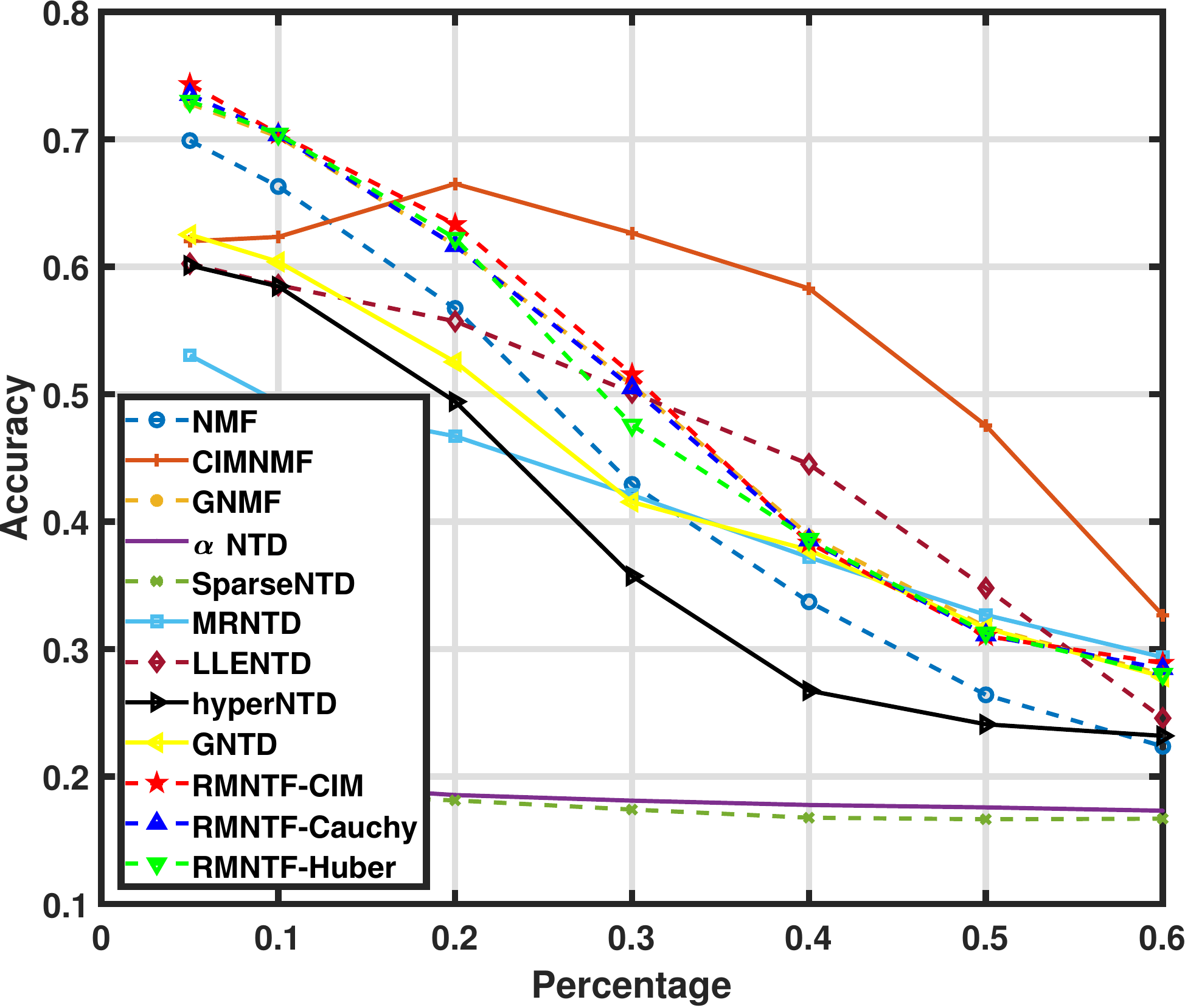}%
\label{Reuters_delta_Coh} \\
\includegraphics[width=1.5in]{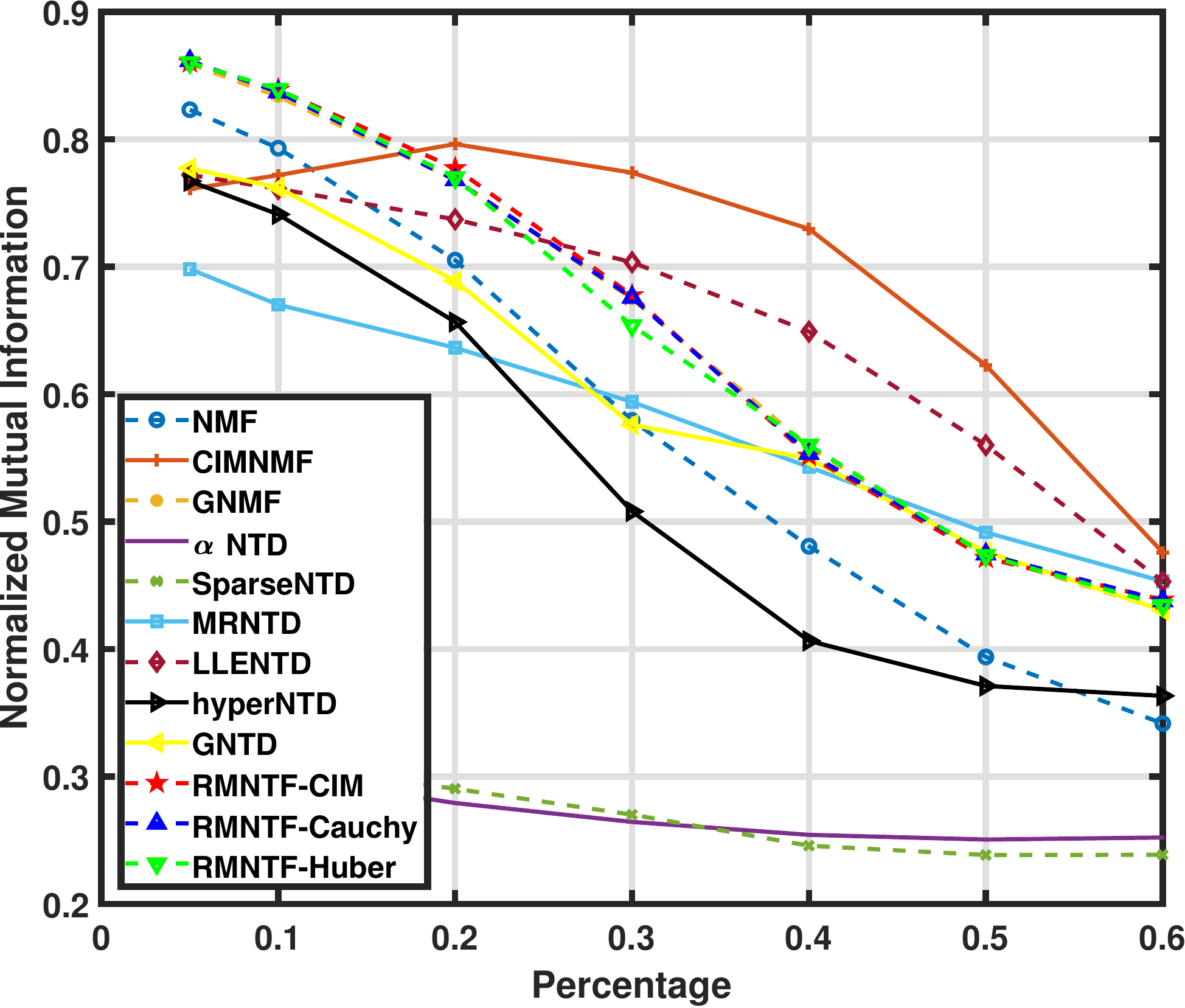}%
\end{minipage}
}

\vspace{2mm}
\caption{Evaluation of proposed methods on ORL database contaminated by Laplace noise and salt \& pepper noise, respectively. (a) Average accuracy and NMI on the subset of ORL which contains $5$ categories, the first two figures show the results contaminated by Laplace noise and the remains show the results contaminated by salt \& pepper noise. (b) Average evaluation on the subset of ORL which contains $10$ categories. (c) Average evaluation on the subset of ORL which contains $15$ categories. (d) Average evaluation on the subset of ORL which contains $20$ categories.}
\label{Reuters_delta}
\end{figure*}

\begin{figure*}[t]
\vspace{-0.5cm} 
\setlength{\abovecaptionskip}{0cm} 
\setlength{\belowcaptionskip}{-0cm} 
\centering
\subfloat[]{
\begin{minipage}[b]{0.18\textwidth}
\includegraphics[width=1.5in]{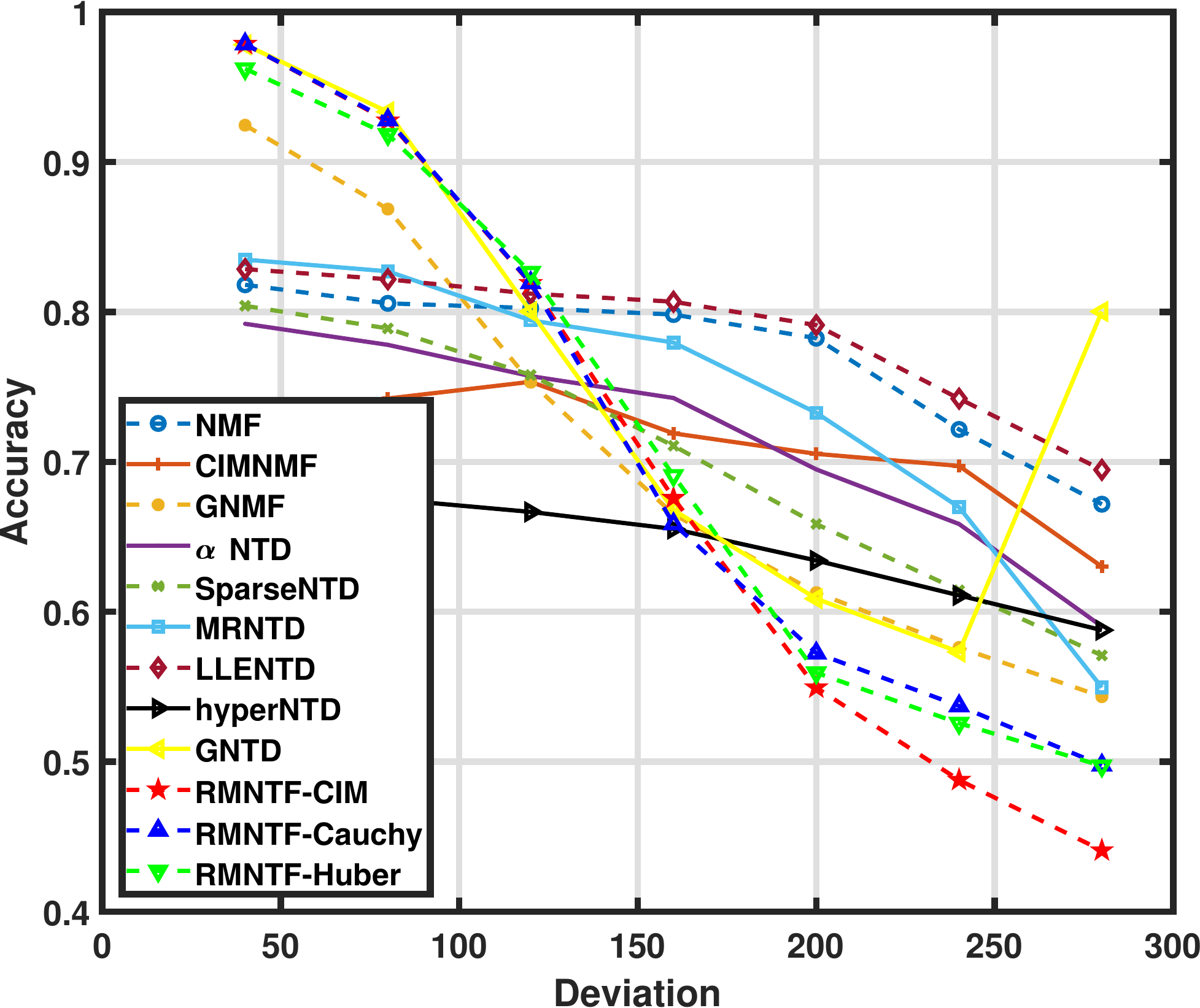}%
\label{Reuters_delta_Coh} \\
\includegraphics[width=1.5in]{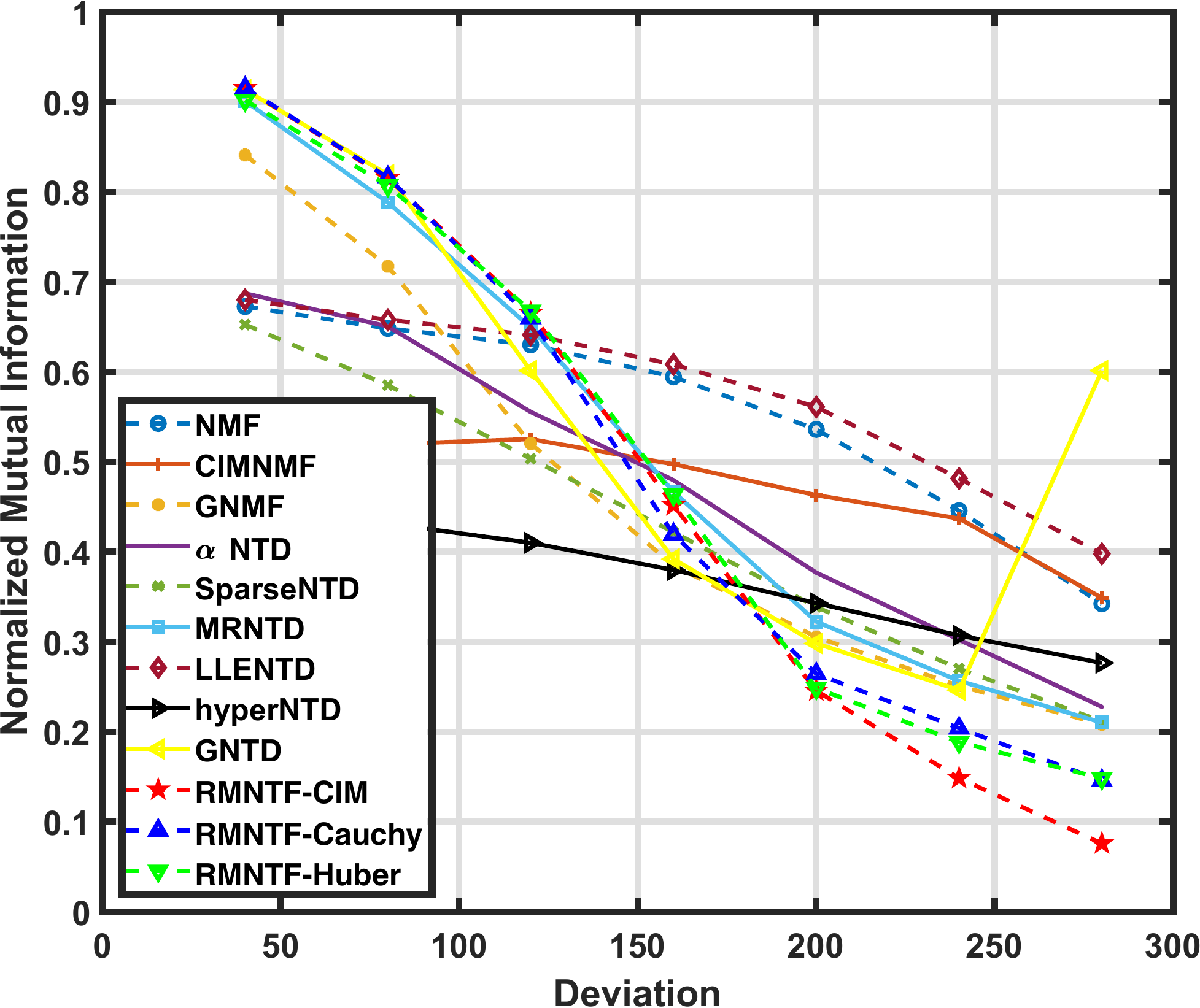}%
\label{Reuters_delta_Coh} \\
\includegraphics[width=1.5in]{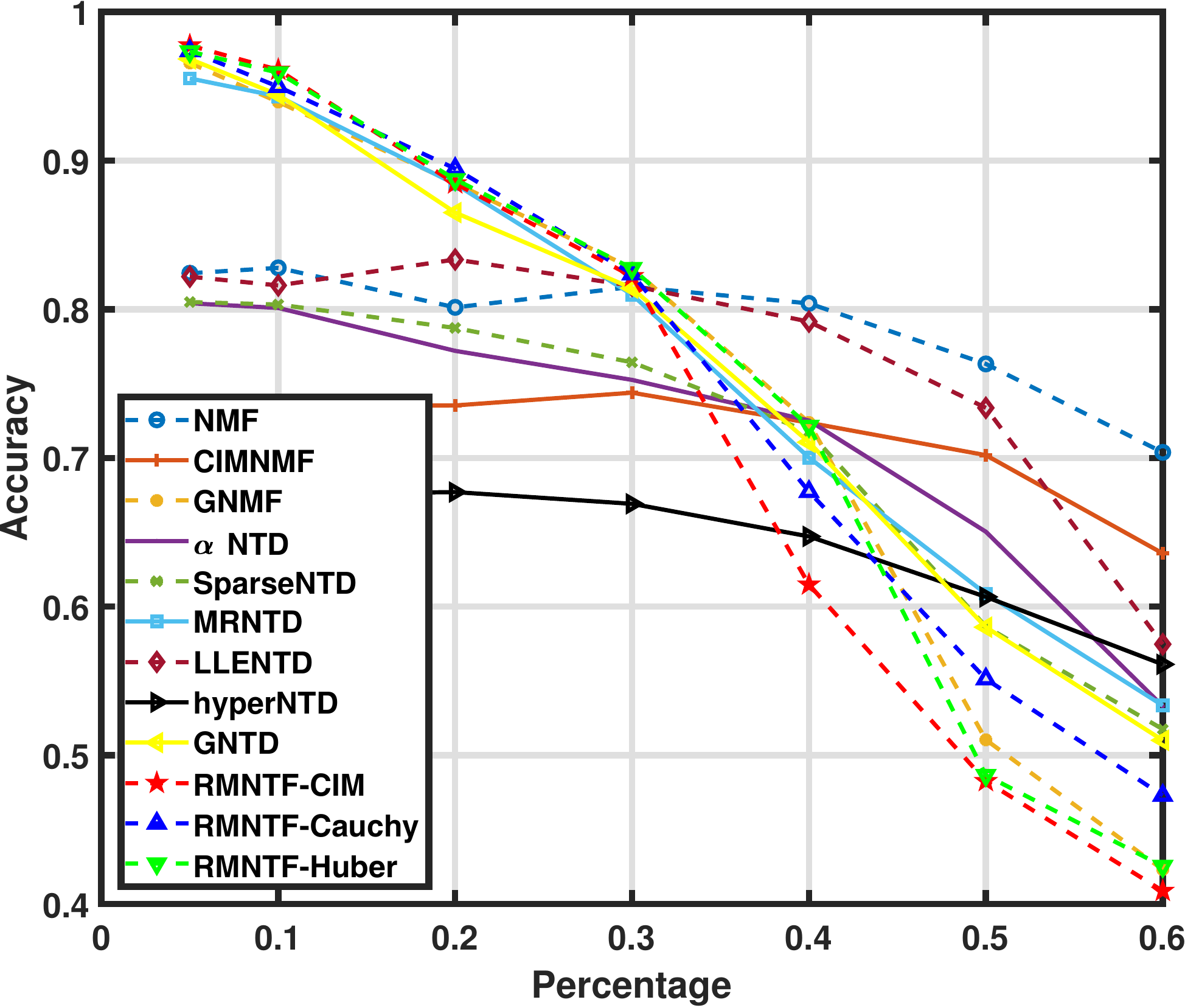}%
\label{Reuters_delta_Coh} \\
\includegraphics[width=1.5in]{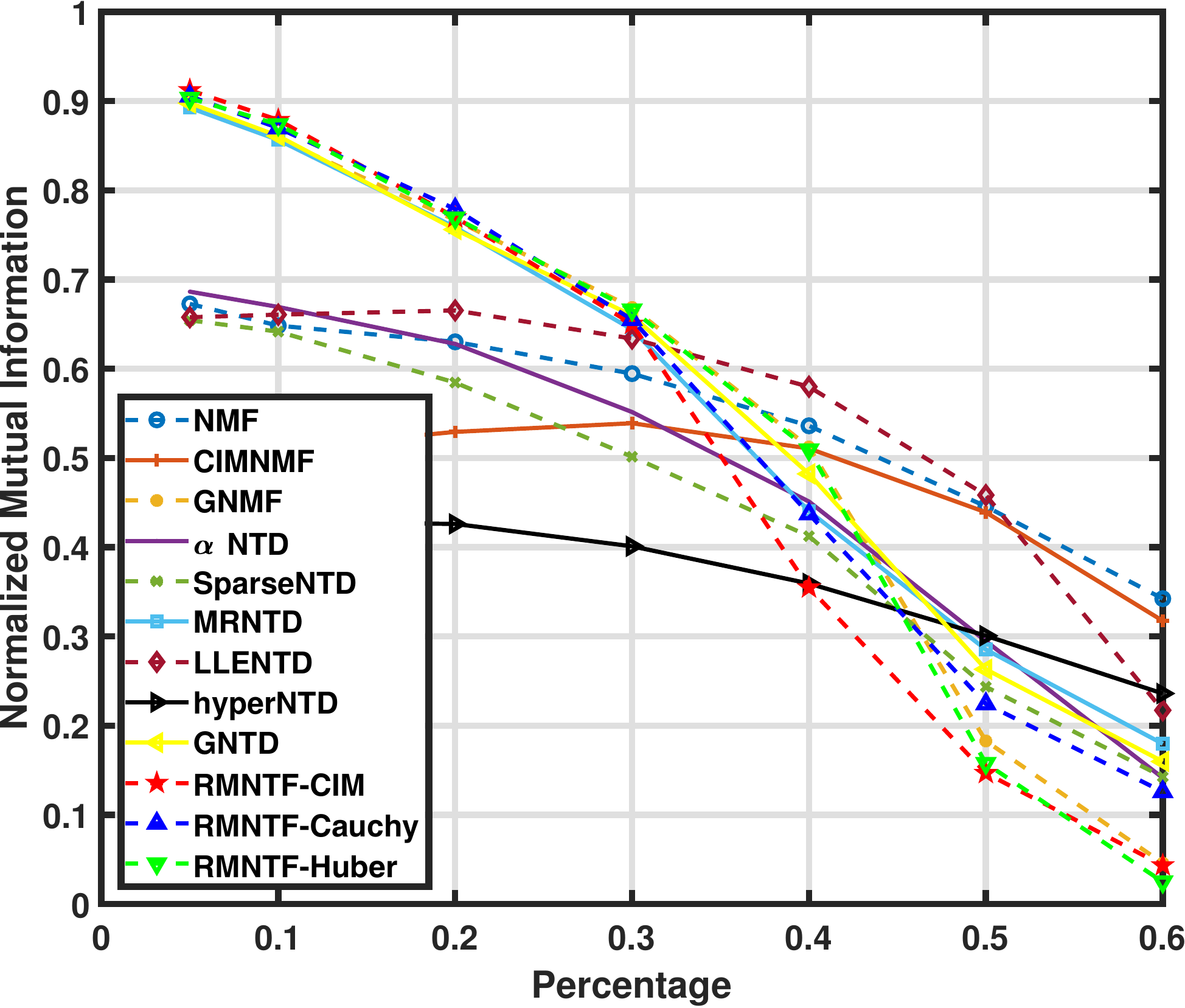}%
\end{minipage}
}
\hfil
\subfloat[]{
\begin{minipage}[b]{0.22\textwidth}
\includegraphics[width=1.5in]{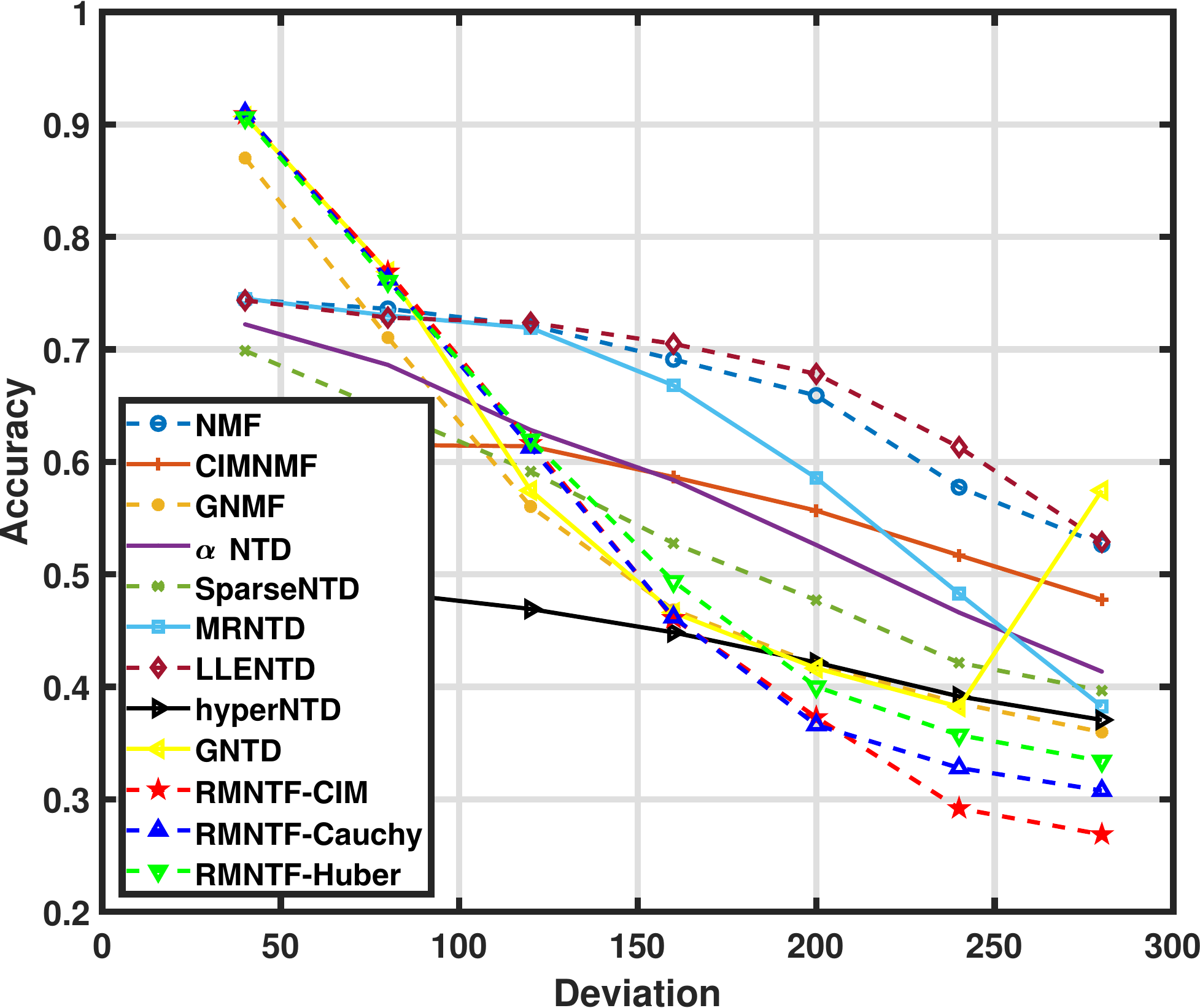}%
\label{Reuters_delta_Coh} \\
\includegraphics[width=1.5in]{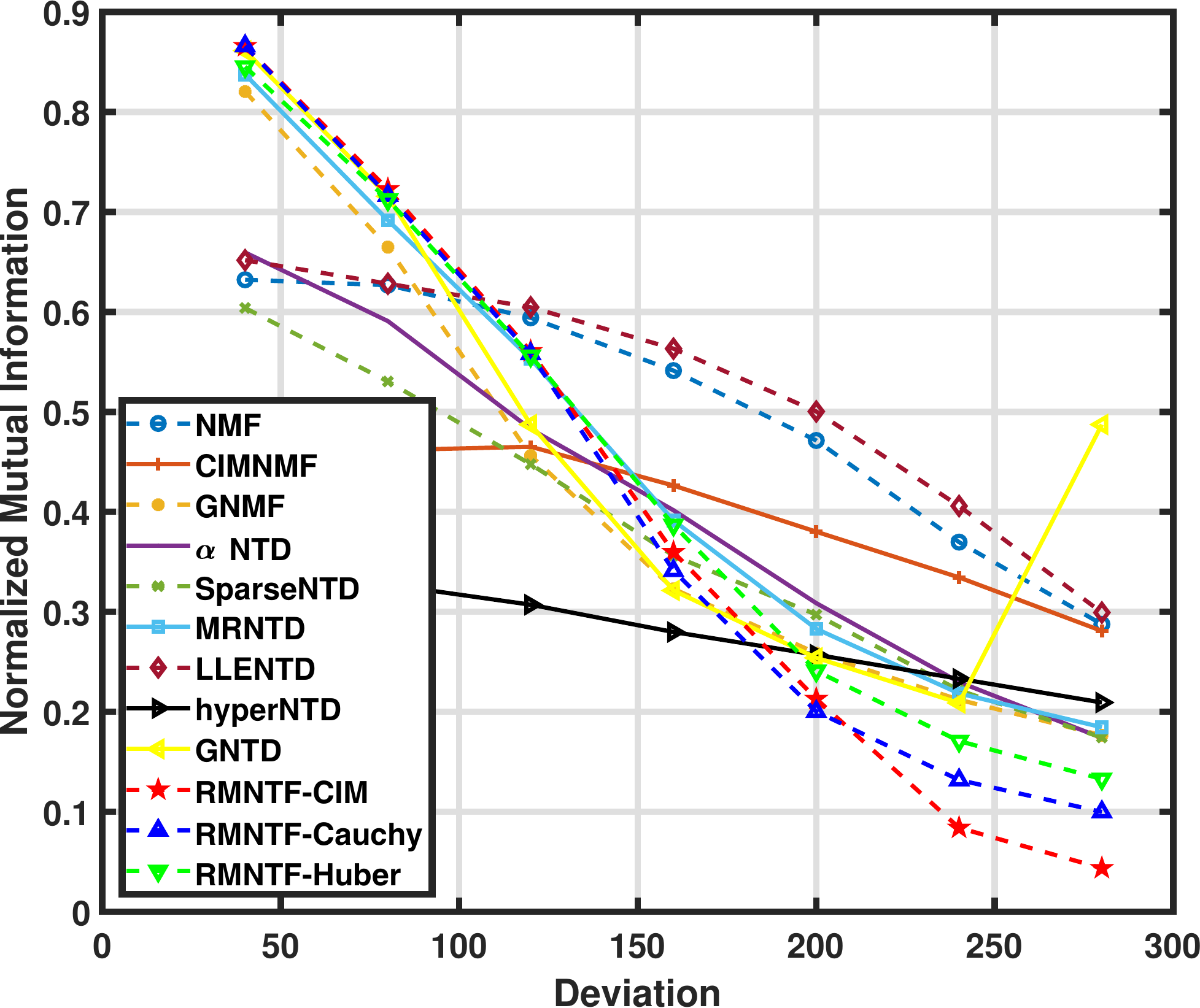}%
\label{Reuters_delta_Coh} \\
\includegraphics[width=1.5in]{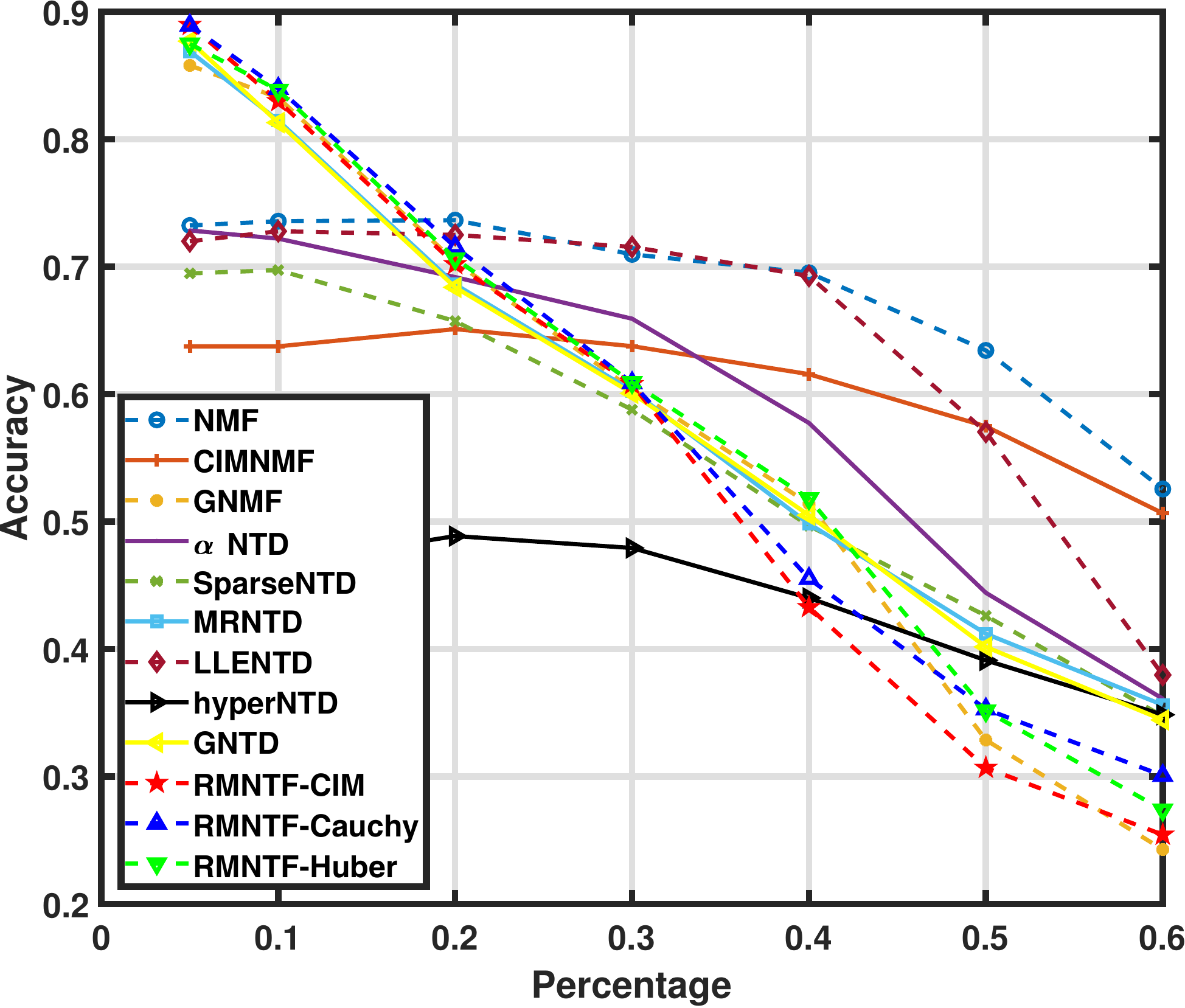}%
\label{Reuters_delta_Coh} \\
\includegraphics[width=1.5in]{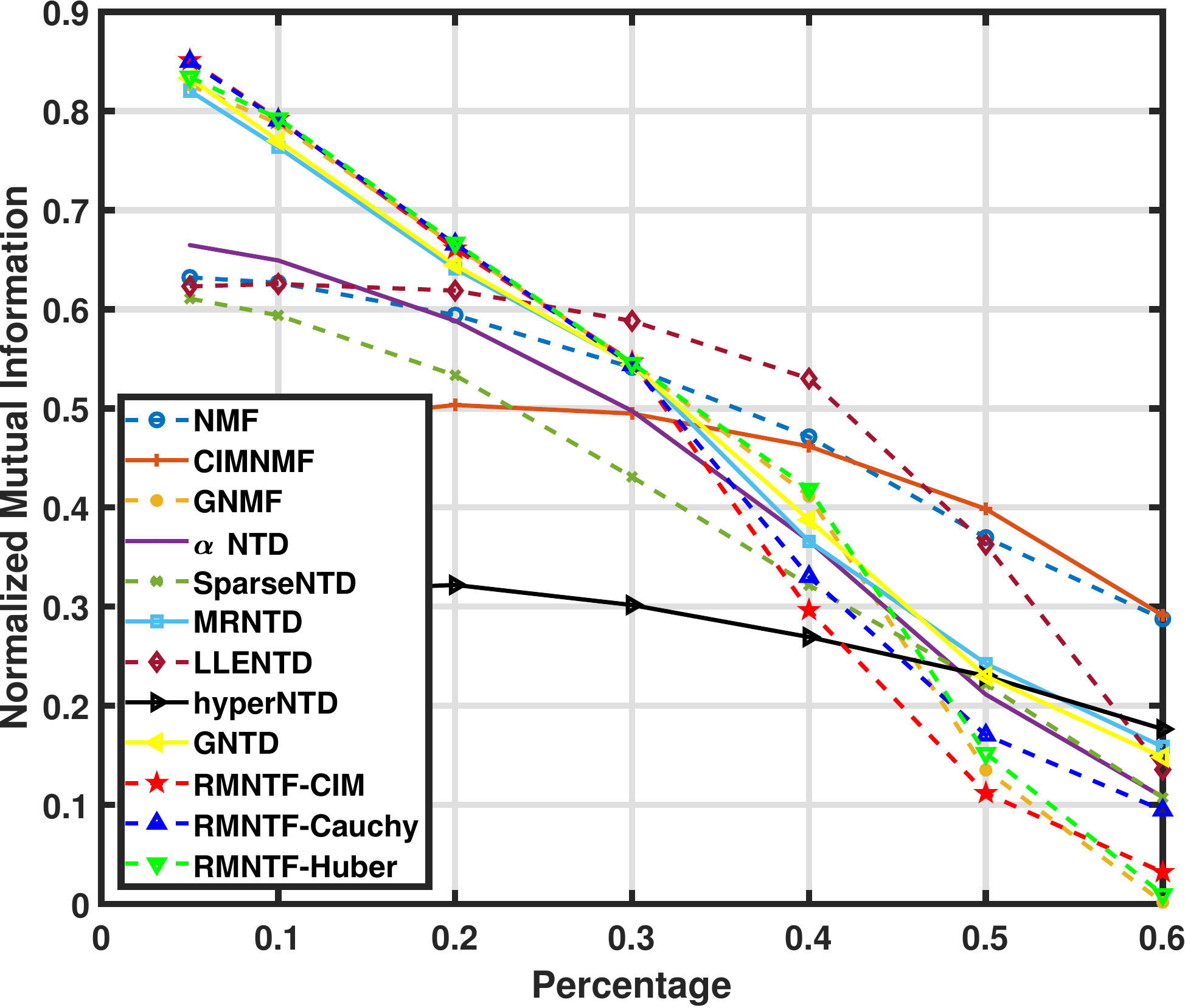}%
\end{minipage}
}
\subfloat[]{
\begin{minipage}[b]{0.22\textwidth}
\includegraphics[width=1.5in]{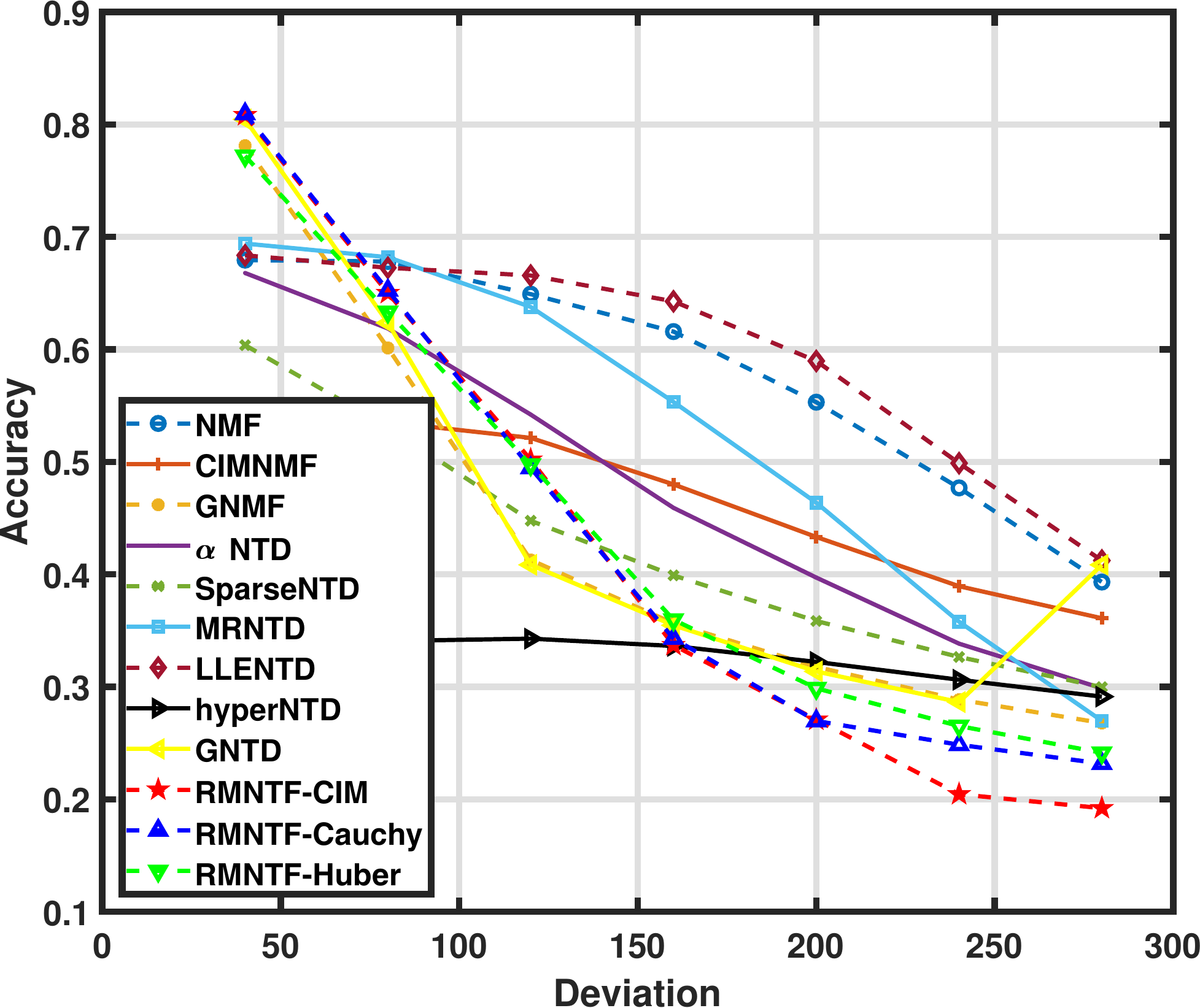}%
\label{Reuters_delta_Coh} \\
\includegraphics[width=1.5in]{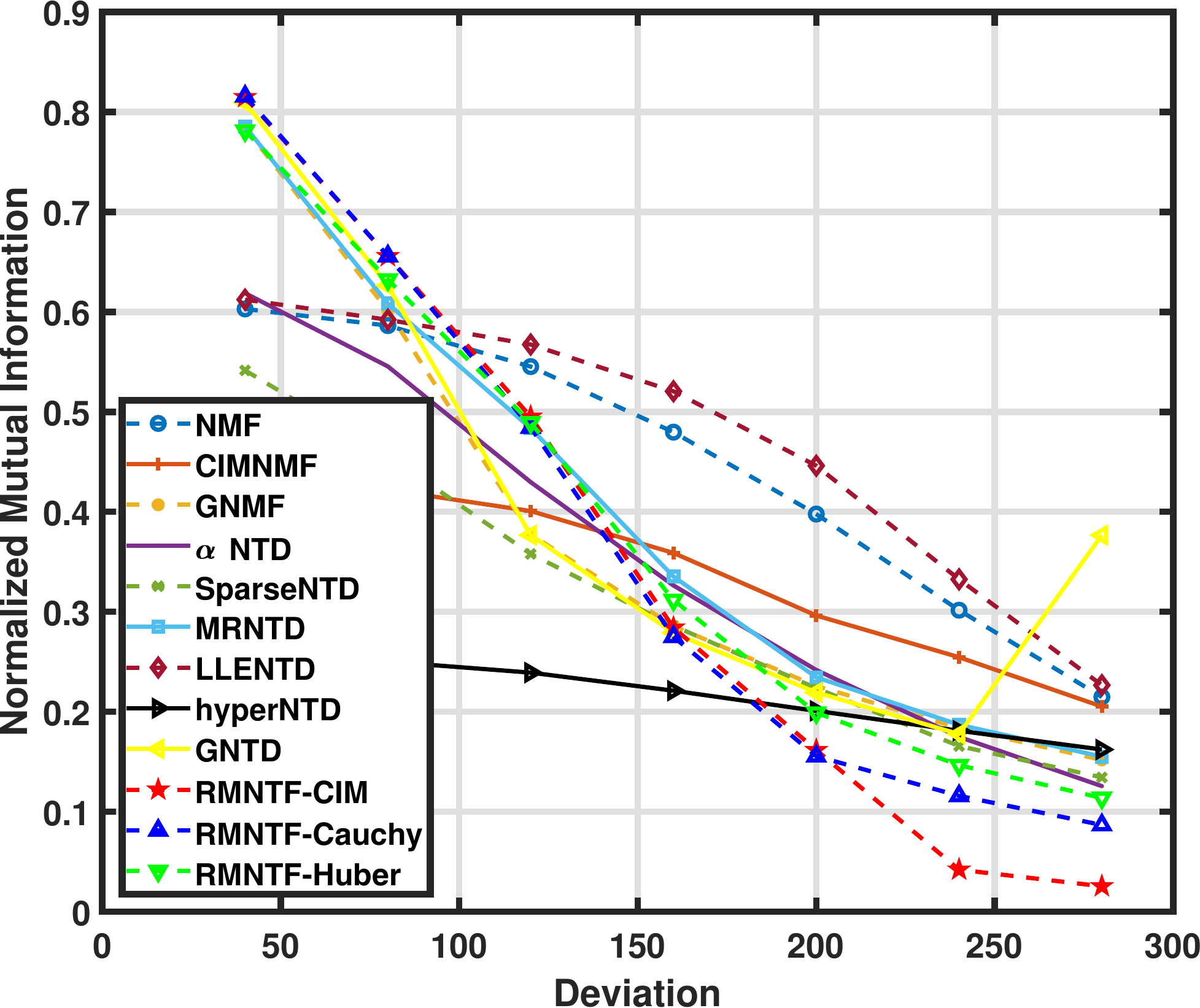}%
\label{Reuters_delta_Coh} \\
\includegraphics[width=1.5in]{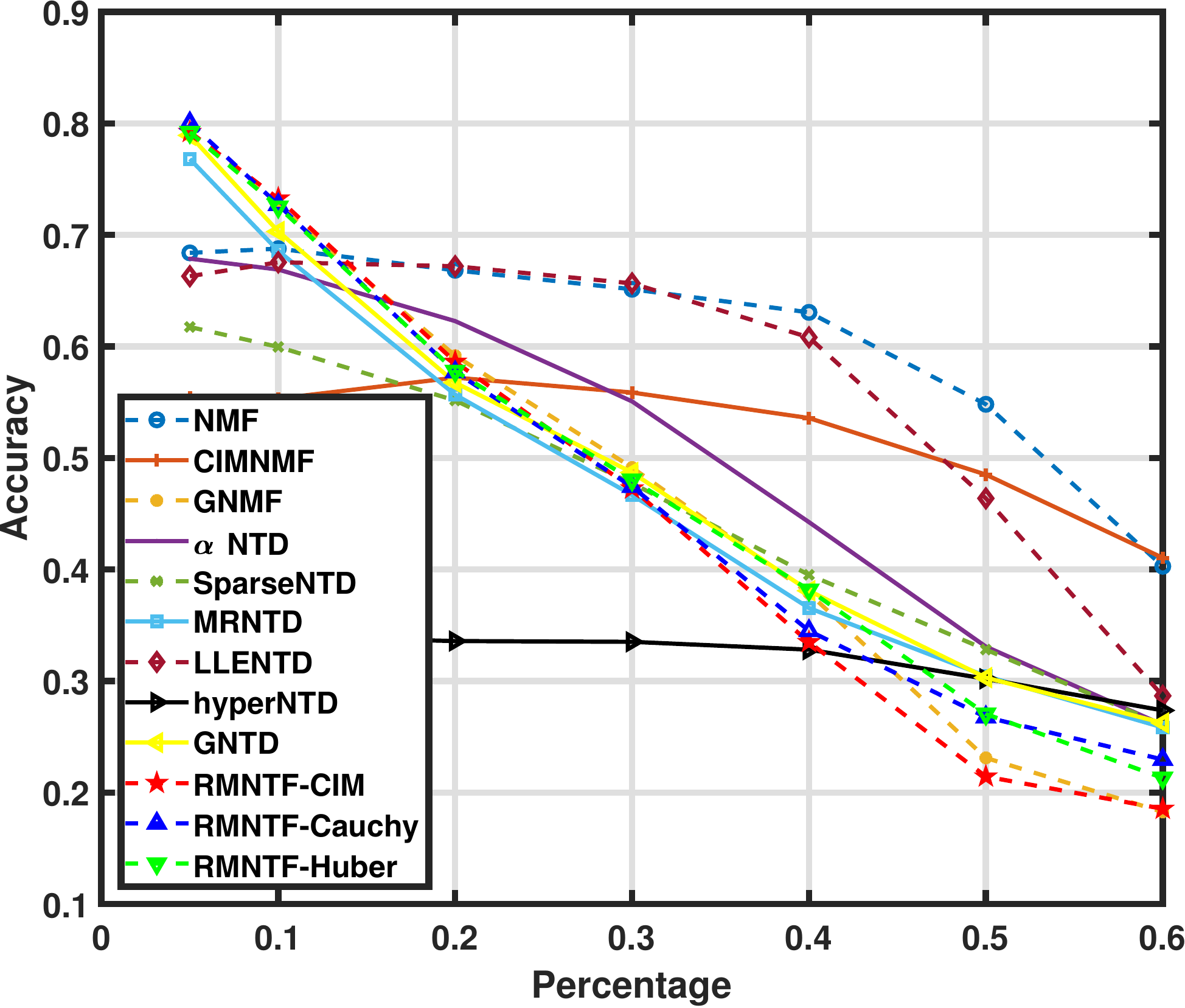}%
\label{Reuters_delta_Coh} \\
\includegraphics[width=1.5in]{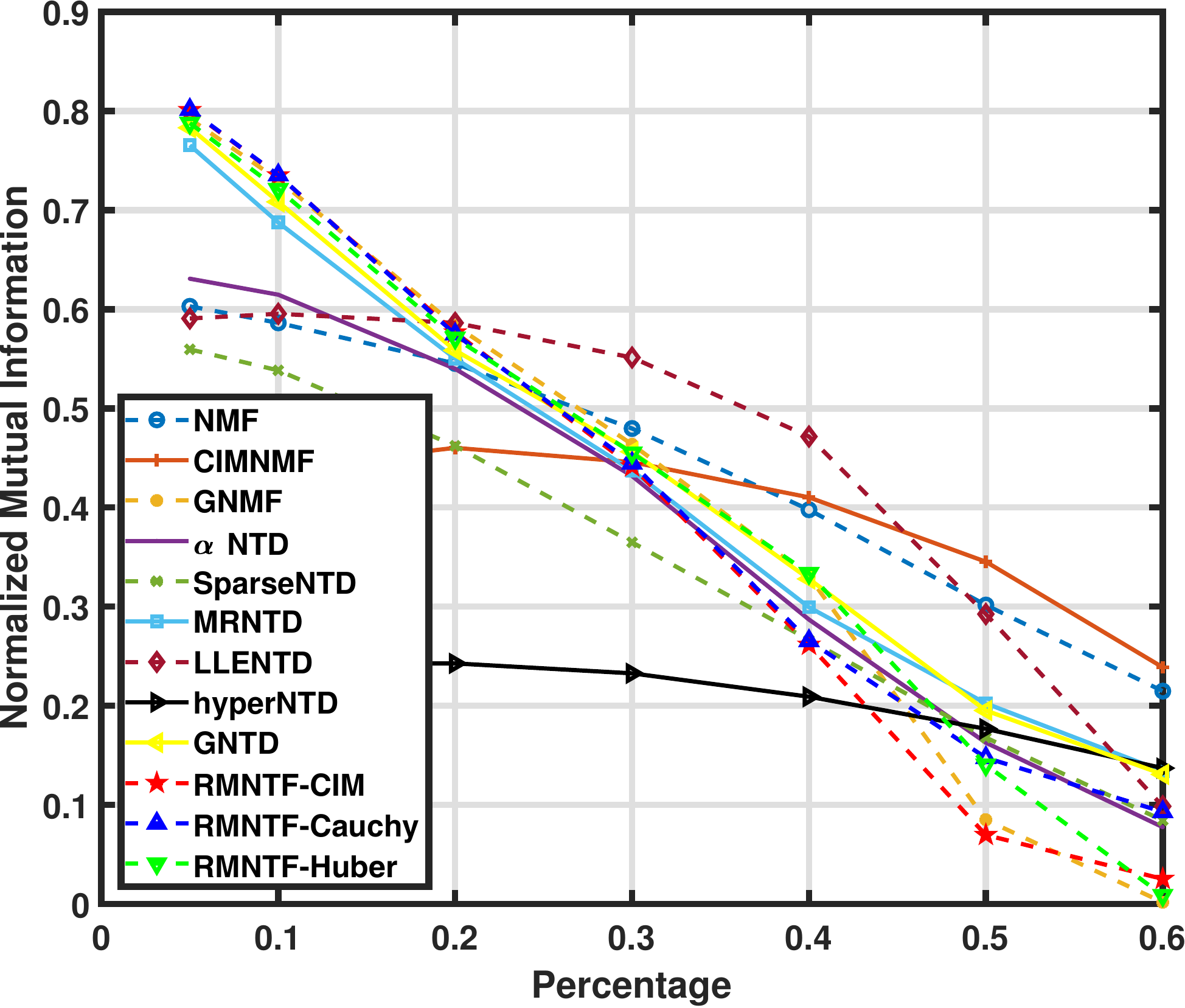}%
\end{minipage}
}
\subfloat[]{
\begin{minipage}[b]{0.22\textwidth}
\includegraphics[width=1.5in]{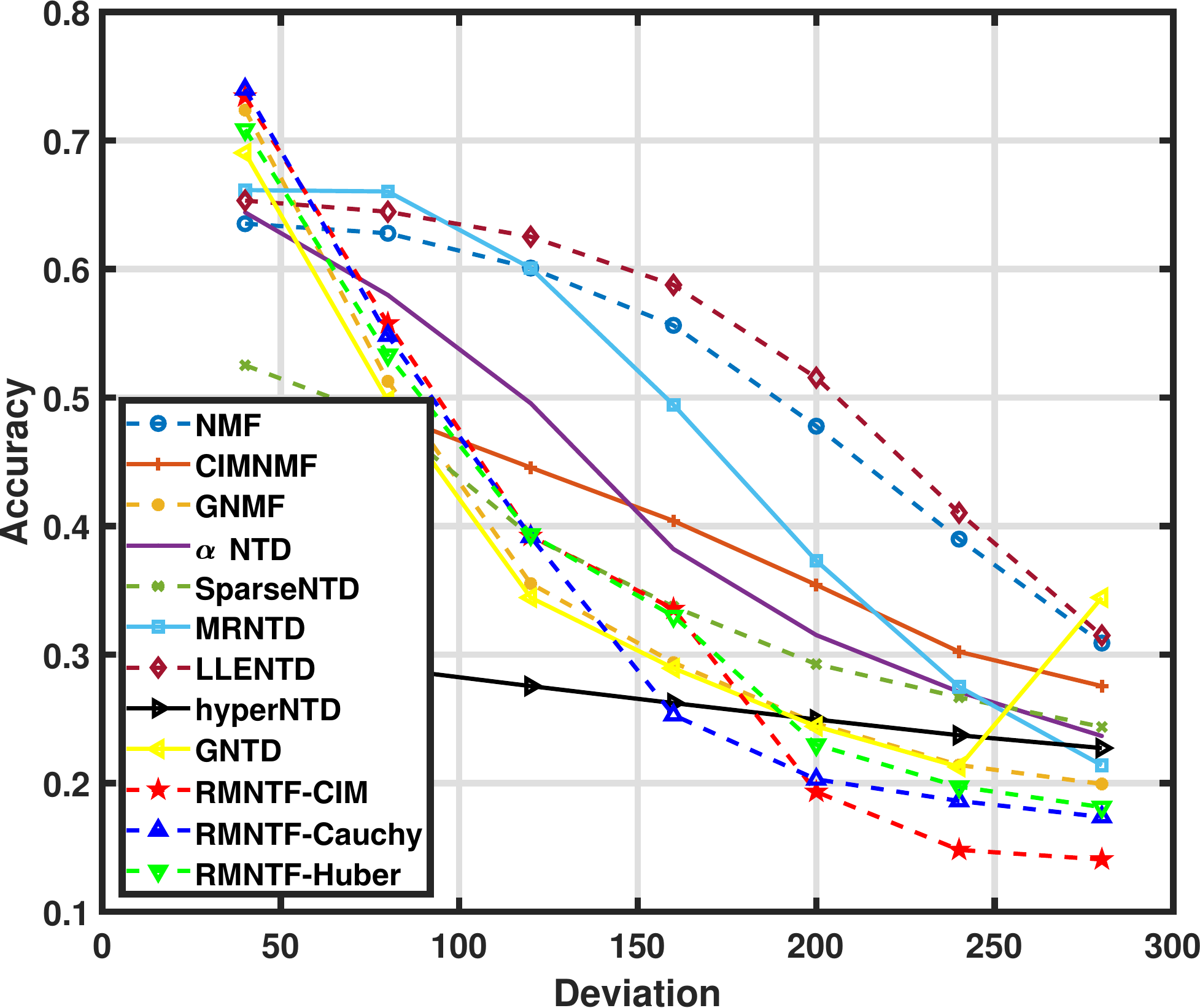}%
\label{Reuters_delta_Coh} \\
\includegraphics[width=1.5in]{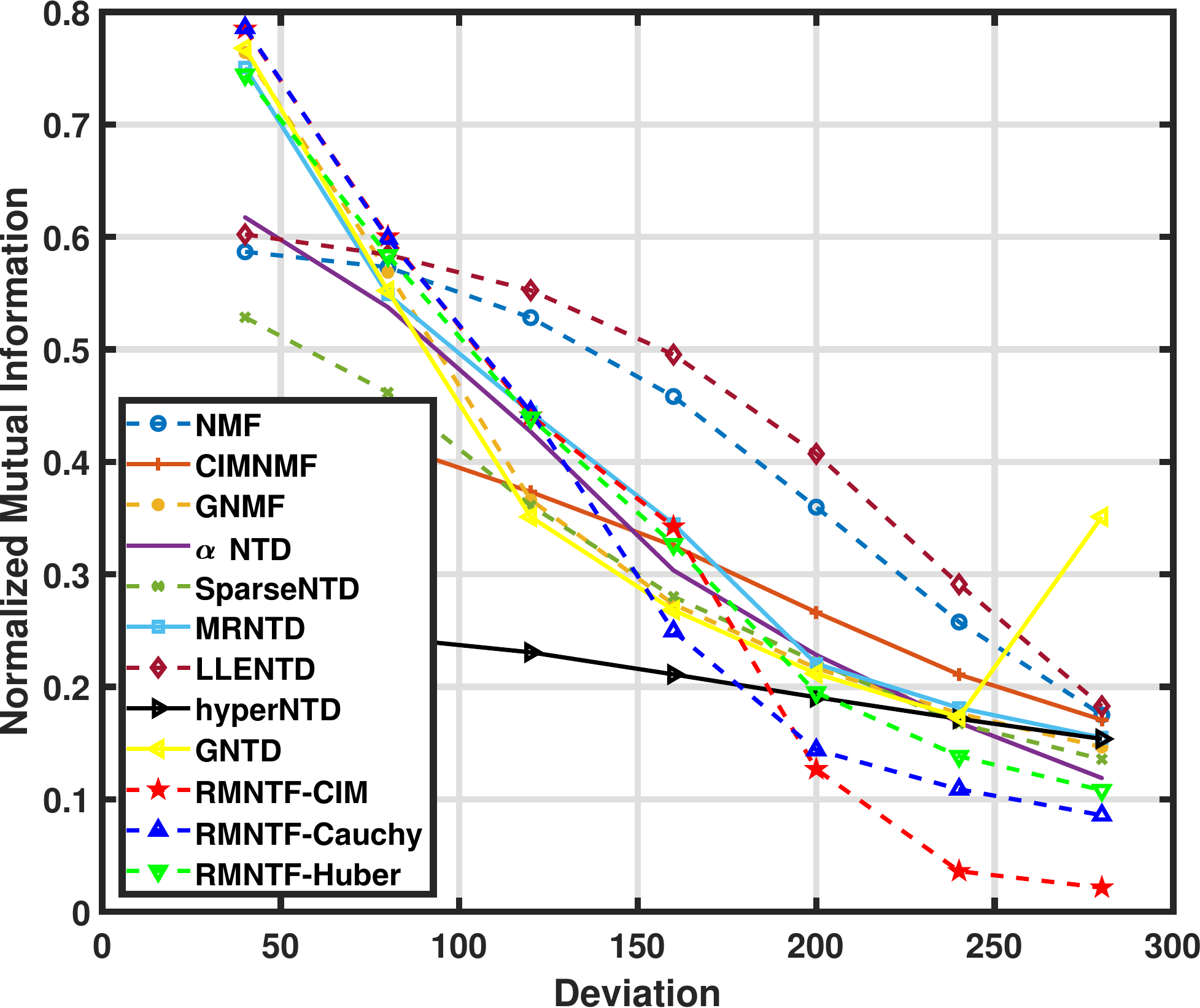}%
\label{Reuters_delta_Coh} \\
\includegraphics[width=1.5in]{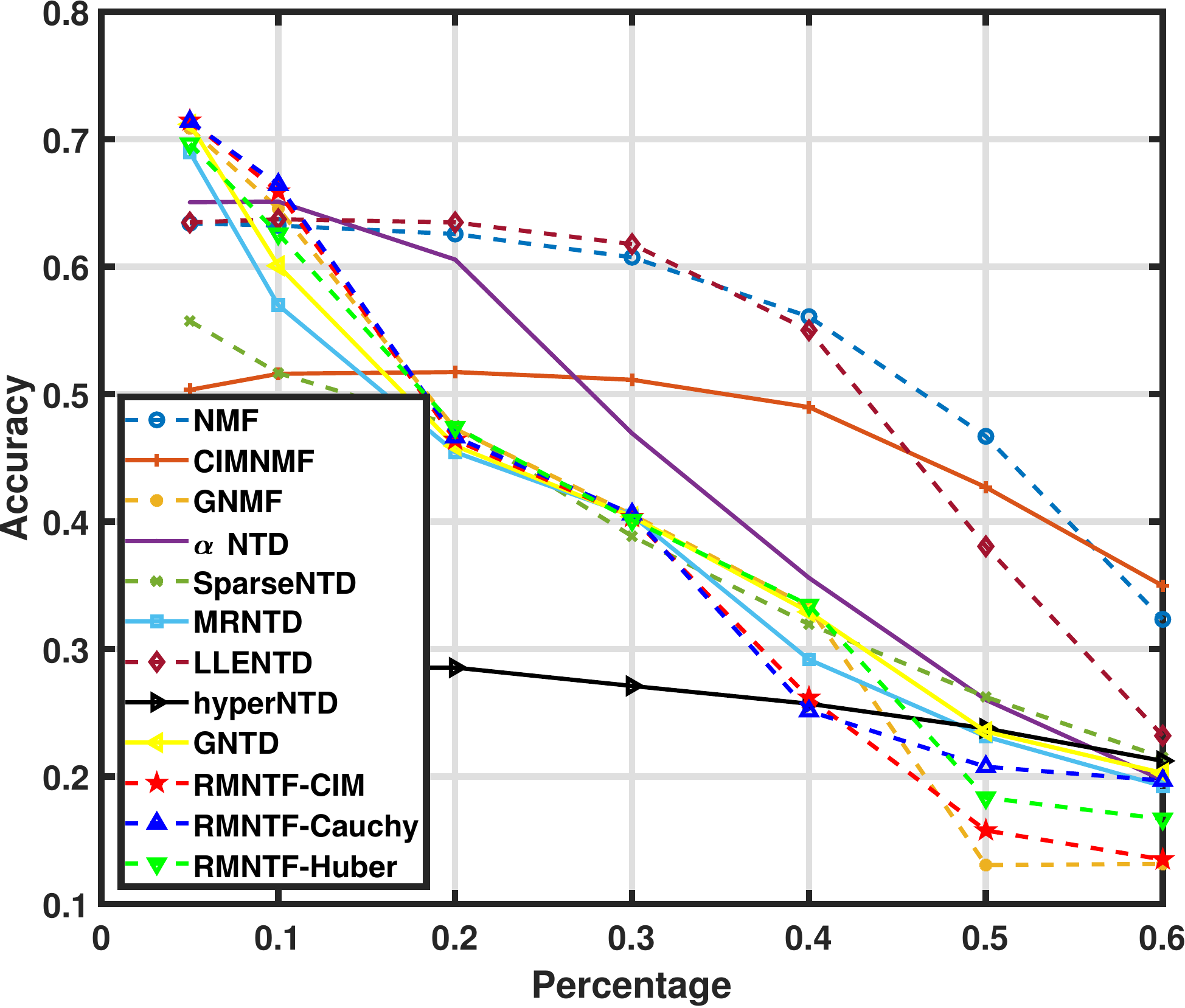}%
\label{Reuters_delta_Coh} \\
\includegraphics[width=1.5in]{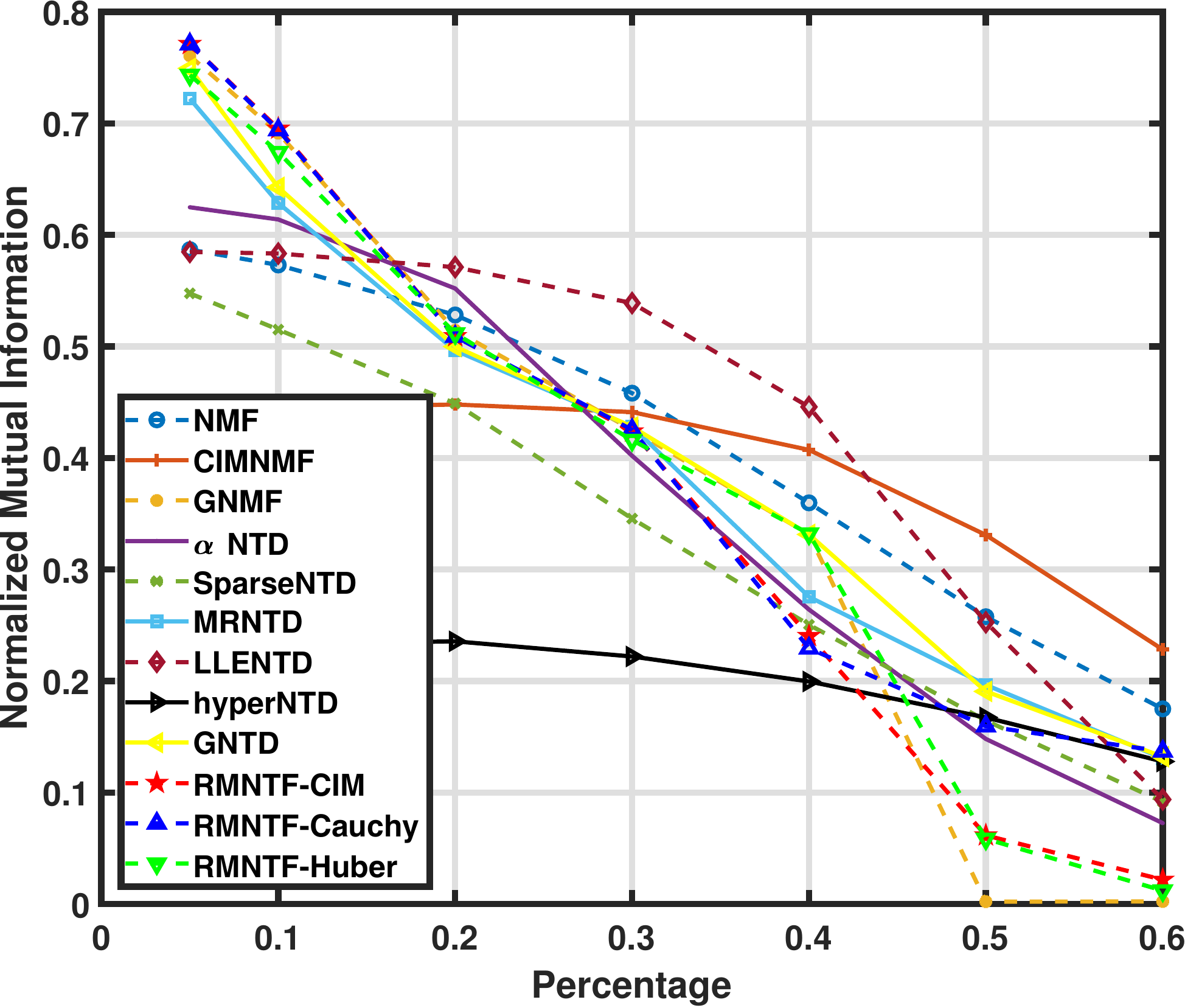}%
\end{minipage}
}

\vspace{2mm}
\caption{Evaluation of proposed methods on USPS database contaminated by Laplace noise and salt \& pepper noise, respectively. (a) Average accuracy and NMI on the subset of USPS which contains $3$ categories, the first two figures show the results contaminated by Laplace noise and the remains show the results contaminated by salt \& pepper noise. (b) Average evaluation on the subset of USPS which contains $5$ categories. (c) Average evaluation on the subset of USPS which contains $7$ categories. (d) Average evaluation on the subset of USPS which contains $10$ categories.}
\label{Reuters_delta}
\end{figure*}

\subsection{Discussion}

\subsubsection{Generalization error bound and convergence analysis}
\label{Convergence}
Before we prove Theorem 1, we first give the definition of an upper bound auxiliary function.

\begin{myDef}
$\mathcal{F}(x,x^\prime)$ is an upper bound auxiliary function for $f(x)$ if the following conditions are satisfied:
\begin{equation}\label{AuxiliaryF}
\begin{split}
 \mathcal{F}(x,x^\prime) \geq f(x), \mathcal{F}(x,x) = f(x).
\end{split}
\end{equation}
\end{myDef}

\begin{myCor}
If $\mathcal{F}(\cdot,\cdot)$ is an upper bound auxiliary function for $f(x)$, then $f(x)$ is non-increasing under the update rule
\begin{equation}\label{CorF}
\begin{split}
 x^{t+1} = \arg \min_x \mathcal{F}(x,x^t),
\end{split}
\end{equation}
\end{myCor}
\begin{proof}
\begin{equation}\label{ProCorF}
\begin{split}
 f(x^{t+1}) \leq \mathcal{F}(x^{t+1},x^{t}) \leq \mathcal{F}(x^t,x^t) = f(x^t).
\end{split}
\end{equation}
\end{proof}

\begin{myDef}\label{Taylor}
A function can be represented as an infinite sum of terms that are calculated from the values of the function's derivatives at a single point, which can be formulated as follows:
\begin{equation}\label{SumFunc}
\begin{split}
 f(x) = \sum_{i=0}^\infty \frac{f^{(i)}(a)}{i!}(x-a)^i,
\end{split}
\end{equation}
where $a$ is the point and $i$ is the order of partial derivatives.
\end{myDef}

Given the above definitions, the objective function of RMNTF-CIM with respect to the three univariate functions are obtained as (19), (24) and (28).
Then, we have the following four lemmas.

\begin{lem}\label{FactorMAnLem}
The auxiliary function for (19) is as follows:
\begin{equation}\label{AuxAn}
\begin{split}
 & \mathcal{F}([\mathbf{A}^{(n)}]_{ij},[\mathbf{A}^{(n)\star}]_{ij}) = f_A([\mathbf{A}^{(n)}]_{ij}) \\
 & + [\mathbf{A}^{(n)}]_{i\cdot} \mathbf{B}^{(n)T} \mathbf{T}_i [\mathbf{B}^{(n)}]_{\cdot j} \left( [\mathbf{A}^{(n)}]_{ij} - [\mathbf{A}^{(n)\star}]_{ij} \right) \\
 & + \frac{[\mathbf{A}^{(n)} \mathbf{B}^{(n)T} \mathbf{T}_i \mathbf{B}^{(n)}]_{ij} }{[\mathbf{A}^{(n)}]_{ij}} \left( [\mathbf{A}^{(n)}]_{ij} - [\mathbf{A}^{(n)\star}]_{ij}\right)^2,
\end{split}
\end{equation}
\end{lem}
where $\mathbf{B}^{(n)T} = \mathcal{S}_{(n)}\left( \otimes_{i\neq n}\mathbf{A}^{(i)T} \right) \in \mathbb{R}_{\geq 0}^{r_n \times I_1\dots I_{n-1}I_{n+1} \dots I_N}$.

\begin{proof}
For RMNTF-CIM, the objective function can be re written as
\begin{equation}\label{RMNTF_CIMObj}
\begin{split}
 &\mathcal{J}(\mathcal{X}, \mathcal{S}, \{\mathbf{A}^{(n)}\}_{n=1}^N) = 1- \\
 &\frac{1}{I_1\dots I_N}\sum_{i_1=1}^{I_N}\dots\sum_{i_n=1}^{I_N} \exp\left( -\frac{ \left( \mathcal{X}_{i_1\dots i_N} - \mathcal{S}\times_{n=1}^N \mathbf{A}^{(n)} \right)^2 }{2 \sigma^2} \right).
\end{split}
\end{equation}

It is obvious that $\mathcal{F}([\mathcal{W}_{(n)}]_{ij}, [\mathcal{W}_{(n)}]_{ij}) = [f_w(\mathcal{W}_{(n)})]_{ij}$, we only need to prove that $\mathcal{F}([\mathcal{W}_{(n)}]_{ij}) \geq [f_w(\mathcal{W}_{(n)})]_{ij}$.
The first-order partial derivative of (17) in element-wise is
\begin{equation}\label{AnPartial}
\begin{split}
  & \left[ \frac{\partial \mathcal{J}(\mathcal{X},\mathcal{S},\{\mathbf{A}^{(n)}\}_{n=1}^N)}{\partial \mathbf{A}^{(n)}} \right]_{i_1\dots i_N} = \frac{1}{\sigma^2}\frac{1}{\sqrt{2\pi}\sigma}\frac{1}{I_1\dots I_N} \cdot \\
  & \left[ \exp\left( -\frac{\left(\mathcal{X}-\mathcal{S}\times_{n=1}^N\mathbf{A}^{(n)}\right)^2}{2\sigma^2} \right) \right]_{i_1\dots i_N} \cdot \\
  &  \left[\mathcal{X}-\mathcal{S}\times_{n=1}^N\mathbf{A}^{(n)}\right]_{i_1\dots i_N} \cdot \left[ \frac{\partial \left(\mathcal{S}\times_{n=1}^N \mathbf{A}^{(n)}\right)}{\partial  \mathbf{A}^{(n)} } \right]_{i_1\dots i_N}.
\end{split}
\end{equation}
By the equivalent relationship $\mathcal{S} \times_{n=1}^N \mathbf{A}^{(n)} = \rm{fold}_n(\mathbf{A}^{(n)} \mathbf{B}^{(n)T})$, we have
\begin{equation}\label{AnPartial1}
\begin{split}
  \left[ \frac{\partial \left(\mathcal{S}\times_{n=1}^N \mathbf{A}^{(n)}\right)}{\partial  \mathbf{A}^{(n)} } \right]_{i_1\dots i_N} = \left[ \frac{\partial \mathbf{A}^{(n)} \mathbf{B}^{(n)T}}{\partial \mathbf{A}^{(n)}} \right]_{r_n m} = \left[ \mathbf{B}^{(n)} \right]_{r_n m},
\end{split}
\end{equation}
where $m = 1,\dots,I_1\dots I_{n-1} I_{n+1} \dots I_N$.

Hence, the first-order partial derivative of RMNTF-CIM is
\begin{equation}\label{AnPartia2}
\begin{split}
  &  \frac{\partial \mathcal{J}(\mathcal{X},\mathcal{S},\{\mathbf{A}^{(n)}\}_{n=1}^N)}{\partial \mathbf{A}^{(n)}} \propto \Bigg[ \exp\left( -\frac{\left(\mathcal{X}-\mathcal{S}\times_{n=1}^N\mathbf{A}^{(n)}\right)^2}{2\sigma^2} \right)\\
  & \circledast \left(\mathcal{S}\times_{n=1}^N\mathbf{A}^{(n)} - \mathcal{X}\right) \Bigg]_{(n)} \mathbf{B}^{(n)}  \\
  & = \left[ \mathcal{W} \circledast  \left(\mathcal{S}\times_{n=1}^N\mathbf{A}^{(n)} - \mathcal{X}\right)\right]_{(n)} \mathbf{B}^{(n)}.
\end{split}
\end{equation}
We assume that $\mathcal{W}$ is a new variable independent of $\mathbf{A}^{(n)}$.
Let $\mathbf{T}_i = \mathrm{Diag}([\mathcal{W}_{(n)}]_{i\cdot})$, then \eqref{AnPartia2} can be represented by
\begin{equation}\label{AnPartia3}
\begin{split}
 \frac{\partial \mathcal{J}(\mathcal{X},\mathcal{S},\{\mathbf{A}^{(n)}\}_{n=1}^N)}{\partial \mathbf{A}^{(n)}} \propto & \{[\mathbf{A}^{(n)}]_{i\cdot} \mathbf{B}^{(n)T} \mathbf{T}_i \}_{i=1}^{I_n} \mathbf{B}^{(n)} \\
 & - \{\mathcal{X}_{(n)} \mathbf{T}_i \}_{i=1}^{I_n} \mathbf{B}^{(n)},
\end{split}
\end{equation}
where $ \{[\mathbf{A}^{(n)}]_{i\cdot} \mathbf{B}^{(n)T} \mathbf{T}_i \}_{i=1}^{I_n} = \bigg[[\mathbf{A}^{(n)}]_{1\cdot} \mathbf{B}^{(n)T} \mathbf{T}_1; \dots; $ $[\mathbf{A}^{(n)}]_{I_n\cdot} \mathbf{B}^{(n)T} \mathbf{T}_{I_n}\bigg]$ $\in \mathbb{R}_{\geq 0}^{I_n\times I_1\dots I_p \dots I_N(p\neq n)}$.
The second-order derivative of RMNTF-CIM can be represented as:
\begin{equation}\label{AnPartia4}
\begin{split}
  \frac{\partial^2 \mathcal{J}(\mathcal{X},\mathcal{S},\{\mathbf{A}^{(n)}\}_{n=1}^N)}{\partial \mathbf{A}^{(n)} \partial \mathbf{A}^{(n)}} = \{ \mathbf{B}^{(n)T} \mathbf{T}_i \mathbf{B}^{(n)} \}_{i=1}^{I_n}.
\end{split}
\end{equation}

According to the Taylor expansion in Definition \ref{Taylor}, we can rewrite $\mathcal{J}(\mathcal{X},\mathcal{S},\{\mathbf{A}^{(n)}\}_{n=1}^N)$ to its Taylor expansion form with respect to $\mathbf{A}^{(n)}$:
\begin{equation}\label{AnPartia5}
\begin{split}
 & f([\mathbf{A}^{(n)}]_{ij}) = f_A([\mathbf{A}^{(n)}]_{ij}) \\
 & + [\mathbf{A}^{(n)}]_{i\cdot} \mathbf{B}^{(n)T} \mathbf{T}_i [\mathbf{B}^{(n)}]_{\cdot j} \left( [\mathbf{A}^{(n)}]_{ij} - [\mathbf{A}^{(n)\star}]_{ij} \right) \\
 & - [\mathcal{X}_{(n)}]_{i\cdot} \mathbf{T}_i [\mathbf{B}^{(n)}]_{\cdot j} \left( [\mathbf{A}^{(n)}]_{ij} - [\mathbf{A}^{(n)\star}]_{ij} \right)\\
 & + \frac{1}{2}[\mathbf{B}^{(n)}]^T_{\cdot j} \mathbf{T}_i [\mathbf{B}^{(n)}]_{\cdot j} \left( [\mathbf{A}^{(n)}]_{ij} - [\mathbf{A}^{(n)\star}]_{ij}\right)^2.
\end{split}
\end{equation}

The upper bound auxiliary function for (19) is defined as:
\begin{equation}\label{AnPartia6}
\begin{split}
 & \mathcal{F}([\mathbf{A}^{(n)}]_{ij},[\mathbf{A}^{(n)\star}]_{ij}) = f([\mathbf{A}^{(n)\star}]_{ij}) \\
 & + [\mathbf{A}^{(n)}]_{i\cdot} \mathbf{B}^{(n)T} \mathbf{T}_i [\mathbf{B}^{(n)}]_{\cdot j} \left( [\mathbf{A}^{(n)}]_{ij} - [\mathbf{A}^{(n)\star}]_{ij} \right) \\
 & - [\mathcal{X}_{(n)}]_{i\cdot} \mathbf{T}_i [\mathbf{B}^{(n)}]_{\cdot j} \left( [\mathbf{A}^{(n)}]_{ij} - [\mathbf{A}^{(n)\star}]_{ij} \right)\\
 & + \frac{[\mathbf{A}^{(n)} \mathbf{B}^{(n)T} \mathbf{T}_i \mathbf{B}^{(n)}]_{ij} }{[\mathbf{A}^{(n)}]_{ij}} \left( [\mathbf{A}^{(n)}]_{ij} - [\mathbf{A}^{(n)\star}]_{ij}\right)^2.
\end{split}
\end{equation}

Substituting \eqref{AnPartia5} into \eqref{AnPartia6},we find that $\mathcal{F}([\mathbf{A}^{(n)}]_{ij},[\mathbf{A}^{(n)\star}]_{ij})\geq f_A([\mathbf{A}^{(n)}]_{ij})$ is equivalent to
\begin{equation}\label{AnPartia7}
\begin{split}
 & \frac{1}{2}\frac{[\mathbf{A}^{(n)} \mathbf{B}^{(n)T} \mathbf{T}_i \mathbf{B}^{(n)}]_{ij} }{[\mathbf{A}^{(n)}]_{ij}} \left( [\mathbf{A}^{(n)}]_{ij} - [\mathbf{A}^{(n)\star}]_{ij}\right)^2 \\
 & \geq \frac{1}{2}[\mathbf{B}^{(n)}]^T_{\cdot j} \mathbf{T}_i [\mathbf{B}^{(n)}]_{\cdot j} \left( [\mathbf{A}^{(n)}]_{ij} - [\mathbf{A}^{(n)\star}]_{ij}\right)^2.
\end{split}
\end{equation}

Because we have
\begin{equation}\label{AnPartia8}
\begin{split}
 & \frac{[\mathbf{A}^{(n)} \mathbf{B}^{(n)T} \mathbf{T}_i \mathbf{B}^{(n)}]_{ij} }{[\mathbf{A}^{(n)}]_{ij}} = \frac{\sum_j \left( [\mathbf{A}^{(n)}]_{ij} \times [\mathbf{B}^{(n)T} \mathbf{T}_i \mathbf{B}^{(n)}]_{jj} \right)}{[\mathbf{A}^{(n)}]_{ij}} \\
 & \geq \frac{ [\mathbf{A}^{(n)}]_{ij} \times [\mathbf{B}^{(n)T} \mathbf{T}_i \mathbf{B}^{(n)}]_{jj} }{[\mathbf{A}^{(n)}]_{ij}} = [\mathbf{B}^{(n)T} \mathbf{T}_i \mathbf{B}^{(n)}]_{jj}.
\end{split}
\end{equation}

Now, we can demonstrate that \eqref{AnPartia7} holds, and \eqref{AnPartia6} is the upper bound auxiliary function for (19), the updates of $\mathbf{A}^{(n)}$ lead to a non-increasing of the objective function $f_A(\mathbf{A}^{(n)})$.
Because the elements of factor matrices $\mathbf{A}^{(n)}$ are nonnegative, and \eqref{AnPartia6} is a convex function, its minimum value can be achieved at
\begin{equation}\label{AnPartia9}
\begin{split}
  & [\mathbf{A}^{(n)\star}]_{ij}\\
  & = [\mathbf{A}^{(n)}]_{ij} - \frac{[\mathbf{A}^{(n)}]_{i\cdot} \mathbf{B}^{(n)T} \mathbf{T}_i [\mathbf{B}^{(n)}]_{\cdot j} - [\mathcal{X}_{(n)}]_{i\cdot} \mathbf{T}_i [\mathbf{B}^{(n)}]_{\cdot j}}{\frac{[\mathbf{A}^{(n)} \mathbf{B}^{(n)T} \mathbf{T}_i \mathbf{B}^{(n)}]_{ij} }{[\mathbf{A}^{(n)}]_{ij}}} \\
  & = [\mathbf{A}^{(n)}]_{ij} \times \frac{[\mathcal{X}_{(n)}]_{i\cdot} \mathbf{T}_i [\mathbf{B}^{(n)}]_{\cdot j}}{[\mathbf{A}^{(n)} \mathbf{B}^{(n)T} \mathbf{T}_i \mathbf{B}^{(n)}]_{ij}}.
\end{split}
\end{equation}
Lemma \ref{FactorMAnLem} is proved.
\end{proof}

\begin{lem}\label{FactorMANLem}
The auxiliary function for (24) is as follows:
\begin{equation}\label{AuxAN}
\begin{split}
\! & \mathcal{F}([\mathbf{A}^{(N)}]_{ij},[\mathbf{A}^{(N)\star}]_{ij})\! = \! f([\mathbf{A}^{(N)}]_{ij}) \! + \!  \left( [\mathbf{A}^{(N)}]_{ij} \! - \! [\mathbf{A}^{(N)\star}]_{ij} \right) \cdot \\
\! & \! \left( \! [\mathbf{A}^{(N)}]_{i\cdot} \mathbf{B}^{(N)T} \mathbf{T}_i [\mathbf{B}^{(N)}]_{\cdot j} \! - \! [\mathcal{X}_{(N)}]_{i\cdot} \mathbf{T}_i [\mathbf{B}^{(N)}]_{\cdot j} \! + \! \lambda\left[ \mathbf{L}\mathbf{A}^{(N)} \right]_{ij} \! \right) \\
 & \! + \! \frac{[\mathbf{A}^{(N)} \mathbf{B}^{(N)T} \mathbf{T}_i \mathbf{B}^{(N)}]_{ij} \! + \! \left[ \lambda \mathbf{D}\mathbf{A}^{(N)} \right]_{ij} }{[\mathbf{A}^{(N)}]_{ij}} \left( [\mathbf{A}^{(N)}]_{ij} \! - \! [\mathbf{A}^{(N)\star}]_{ij}\right)^2.
\end{split}
\end{equation}
\end{lem}

\begin{proof}
For RMNTF-CIM, the objective function with respect to $\mathbf{A}^{(N)}$ can be represented as
\begin{equation}\label{AuxAN1}
\begin{split}
 & \hat{\mathcal{J}}(\mathcal{X},\mathcal{S},\{\mathcal{A}^{(n)}\}_{n=1}^N) = \mathcal{J}(\mathcal{X},\mathcal{S},\{\mathbf{A}^{(n)}\}_{n=1}^{N}) \\
 & + \frac{\lambda}{2}\mathrm{Tr}\left( \mathbf{A}^{(N)T}\mathbf{L}\mathbf{A}^{(N)} \right) + \mathrm{Tr}\left( \boldsymbol{\Omega}_N \mathbf{A}^{(N)} \right).
\end{split}
\end{equation}

The first-order partial derivative of \eqref{AuxAN1} in element-wise is
\begin{equation}\label{AuxAN2}
\begin{split}
 & \frac{\partial \hat{\mathcal{J}}(\mathcal{X},\mathcal{S},\{\mathcal{A}^{(n)}\}_{n=1}^N)}{\partial \left[ \mathbf{A}^{(N)} \right]_{ij}} = \frac{1}{\sigma^2}\frac{1}{\sqrt{2\pi}\sigma}\frac{1}{I_1\dots I_N} \cdot \\
 & \left[ \exp\left( - \frac{\left( \mathcal{X} - \mathcal{S}\times_{n=1}^N \mathbf{A}^{(n)} \right)^2}{2\sigma^2} \right) \right]_{i_1\dots i_N} \cdot \\
 & \left[ \mathcal{X} - \mathcal{S}\times_{n=1}^N \mathbf{A}^{(n)} \right]_{i_1\dots i_N} \cdot \left[ \frac{\partial\left( \mathcal{S}\times_{n=1}^N\mathbf{A}^{(n)} \right)}{\partial \mathbf{A}^{(N)}} \right]_{i_1\dots i_N}  \\
 & + \frac{\lambda}{2} \left[ \left(\mathbf{L} + \mathbf{L}^T\right) \mathbf{A}^{(N)}\right]_{ij} + \left[\boldsymbol{\Omega}_N\right]_{ij}.
\end{split}
\end{equation}
Because of \eqref{AnPartial1}, we rewrite the first-order derivative of RMNTF-CIM:
\begin{equation}\label{AuxAN3}
\begin{split}
 & \frac{\partial \hat{\mathcal{J}}(\mathcal{X},\mathcal{S},\{\mathcal{A}^{(n)}\}_{n=1}^N)}{\partial \mathbf{A}^{(N)}} \propto \frac{\lambda}{2} \left(\mathbf{L} + \mathbf{L}^T\right) \mathbf{A}^{(N)} \\
 & + \left[ \mathcal{W} \circledast \left( \mathcal{S}\times_{n=1}^N \mathbf{A}^{(n)} - \mathcal{X} \right) \right]_{(N)}\mathbf{B}^{(N)}  + \boldsymbol{\Omega}_N \\
 = & \{[\mathbf{A}^{(N)}]_{i\cdot} \mathbf{B}^{(N)T} \mathbf{T}_i \}_{i=1}^{I_N} \mathbf{B}^{(N)} - \{\mathcal{X}_{(N)} \mathbf{T}_i \}_{i=1}^{I_N} \mathbf{B}^{(N)} \\
 & + \frac{\lambda}{2} \left(\mathbf{L} + \mathbf{L}^T\right) \mathbf{A}^{(N)} + \boldsymbol{\Omega}_N.
\end{split}
\end{equation}
The tensor variable $\mathcal{W}$ is independent of $\mathbf{A}^{(N)}$.
The second-order derivative of RMNTF-CIM with respect to $\mathbf{A}^{(N)}$ is represented as
\begin{equation}\label{AuxAN4}
\begin{split}
 & \frac{\partial^2 \hat{\mathcal{J}}(\mathcal{X},\mathcal{S},\{\mathcal{A}^{(n)}\}_{n=1}^N)}{\partial \mathbf{A}^{(N)} \partial \mathbf{A}^{(N)}} =  \{ \mathbf{B}^{(N)T} \mathbf{T}_i \mathbf{B}^{(N)} \}_{i=1}^{I_N}  + \lambda\mathbf{L}.
\end{split}
\end{equation}
According to the Taylor expansion, we rewrite $\hat{\mathcal{J}}$ with respect to $\mathbf{A}^{(N)}$:
\begin{equation}\label{AuxAN5}
\begin{split}
 & f \! (\left[ \mathbf{A}^{(N)} \right]_{ij}\! ) \! = \! f_A \! (\left[ \mathbf{A}^{(N)} \right]_{ij}) \! + \! \left( \! [\mathbf{A}^{(n)}]_{ij} \! - \! [\! \mathbf{A}^{(n)\star}\!]_{ij} \right) \! \cdot \\
 & \Bigg( [\mathbf{A}^{(N)}]_{i\cdot} \mathbf{B}^{(N)T} \mathbf{T}_i [\mathbf{B}^{(N)}]_{\cdot j} \! - \! [\mathcal{X}_{(N)}]_{i\cdot} \mathbf{T}_i [\mathbf{B}^{(N)}]_{\cdot j} \\
 & \! + \! \lambda\!\left[\! \mathbf{L}\mathbf{A}^{(n)}\! \right]_{ij}\! \Bigg) \! + \! \frac{1}{2} \left( [\mathbf{B}^{(N)}]^T_{\cdot j} \mathbf{T}_i [\mathbf{B}^{(N)}]_{\cdot j} \! + \! \lambda\!\left[\mathbf{L}\right]_{ii}\! \right) \\
 & \cdot \left( [\mathbf{A}^{(N)}]_{ij} - [\mathbf{A}^{(N)\star}]_{ij}\right)^2.
\end{split}
\end{equation}

The upper bound auxiliary function for (24) is denoted as:
\begin{equation}\label{AuxAN6}
\begin{split}
 & \mathcal{F}(\![\mathbf{A}\!^{(N)}]_{ij} \!, \![\mathbf{A}\!^{(N)\star}]_{ij}\!)\! = \! f(\![\mathbf{A}\!^{(N)\star}\!]_{ij}\!)\! +\! \left(\! [\mathbf{A}\!^{(N)}\!]_{ij}\! - \! [\mathbf{A}\!^{(N)\star}\!]_{ij}\! \right)\! \cdot\! \\
 & \! \Bigg(  \! [\mathbf{A}^{(N)}]_{i\cdot} \mathbf{B}^{(N)T} \mathbf{T}_i [\mathbf{B}^{(N)}]_{\cdot j}\! - \![\mathcal{X}_{(N)}]_{i\cdot} \mathbf{T}_i [\mathbf{B}^{(N)}]_{\cdot j} \\
 &\! + \!\lambda\!\left[ \mathbf{L}\mathbf{A}^{(N)} \right]_{ij}\! \Bigg) \! + \! \frac{[\mathbf{A}^{(N)} \mathbf{B}^{(N)T} \mathbf{T}_i \mathbf{B}^{(N)}]_{ij} \!+ \!\left[ \lambda \mathbf{D}\mathbf{A}^{(N)} \right]_{ij} }{[\mathbf{A}^{(N)}]_{ij}} \\
 & \cdot \left( [\mathbf{A}^{(N)}]_{ij} - [\mathbf{A}^{(N)\star}]_{ij}\right)^2.
\end{split}
\end{equation}

Substituting \eqref{AuxAN5} into \eqref{AuxAN6},we find that $\mathcal{F}([\mathbf{A}^{(N)}]_{ij},[\mathbf{A}^{(N)\star}]_{ij})\geq f_A([\mathbf{A}^{(N)}]_{ij})$ is equivalent to
\begin{equation}\label{AuxAN7}
\begin{split}
 & \frac{[\mathbf{A}^{(N)} \mathbf{B}^{(N)T} \mathbf{T}_i \mathbf{B}^{(N)}]_{ij} + \left[ \lambda \mathbf{D}\mathbf{A}^{(N)} \right]_{ij} }{[\mathbf{A}^{(N)}]_{ij}} \\
 & \geq \frac{1}{2} \left( [\mathbf{B}^{(N)}]^T_{\cdot j} \mathbf{T}_i [\mathbf{B}^{(N)}]_{\cdot j} + \lambda\left[\mathbf{L}\right]_{ii} \right).
\end{split}
\end{equation}

Because we have
\begin{equation}\label{AuxAN8}
\begin{split}
 & \frac{[\mathbf{A}^{(n)} \mathbf{B}^{(n)T} \mathbf{T}_i \mathbf{B}^{(n)}]_{ij} }{[\mathbf{A}^{(n)}]_{ij}} = \frac{\sum_j \left( [\mathbf{A}^{(n)}]_{ij} \times [\mathbf{B}^{(n)T} \mathbf{T}_i \mathbf{B}^{(n)}]_{jj} \right)}{[\mathbf{A}^{(n)}]_{ij}} \\
 & \geq \frac{ [\mathbf{A}^{(n)}]_{ij} \times [\mathbf{B}^{(n)T} \mathbf{T}_i \mathbf{B}^{(n)}]_{jj} }{[\mathbf{A}^{(n)}]_{ij}} = [\mathbf{B}^{(n)T} \mathbf{T}_i \mathbf{B}^{(n)}]_{jj},
\end{split}
\end{equation}
and
\begin{equation}\label{AuxAN9}
\begin{split}
 &  \left[ \lambda \mathbf{D}\mathbf{A}^{(N)} \right]_{ij} = \lambda \sum_{k}[\mathbf{D}]_{ik} [\mathbf{A}^{(N)}]_{kj} \geq \lambda [\mathbf{D}]_{ii} [\mathbf{A}^{(N)}]_{ij} \\
 & \geq \lambda \left[ \mathbf{D} - \mathbf{W} \right]_{ii}[\mathbf{A}^{(N)}]_{ij} = \lambda [\mathbf{L}]_{ii} [\mathbf{A}^{(N)}]_{ij}.
\end{split}
\end{equation}

Now, we demonstrate that \eqref{AuxAN7} holds, and \eqref{AuxAN6} is the upper bound auxiliary function for $f_A(\mathbf{A}^{(N)})$. Because the elements of $\mathbf{A}^{(N)}$ are nonnegative, and \eqref{AuxAN6} is convex, its minimum value can be achieved at
\begin{equation}\label{AuxAN10}
\begin{split}
\!  & [\mathbf{A}^{(N)\star}]_{ij} = [\mathbf{A}^{(N)}]_{ij} - \\
  & \frac{[\mathbf{A}^{(N)}]_{i\cdot} \mathbf{B}^{(N)T} \mathbf{T}_i [\mathbf{B}^{(N)}]_{\cdot j} \! - \! [\mathcal{X}_{(N)}]_{i\cdot} \mathbf{T}_i [\mathbf{B}^{(N)}]_{\cdot j} \! + \! \lambda\left[ \mathbf{L}\mathbf{A}^{(N)} \right]_{ij}}{\frac{[\mathbf{A}^{(N)} \mathbf{B}^{(N)T} \mathbf{T}_i \mathbf{B}^{(N)}]_{ij} }{[\mathbf{A}^{(N)}]_{ij}} + \left[ \lambda \mathbf{D}\mathbf{A}^{(N)} \right]_{ij}} \\
  & = [\mathbf{A}^{(N)}]_{ij} \times \frac{[\mathcal{X}_{(N)}]_{i\cdot} \mathbf{T}_i [\mathbf{B}^{(N)}]_{\cdot j} + \lambda \left[ \mathbf{W}\mathbf{A}^{(N)} \right]_{ij}}{[\mathbf{A}^{(N)} \mathbf{B}^{(N)T} \mathbf{T}_i \mathbf{B}^{(N)}]_{ij} + \left[ \lambda \mathbf{D}\mathbf{A}^{(N)} \right]_{ij}}.
\end{split}
\end{equation}
Lemma \ref{FactorMANLem} is proved.
\end{proof}

\begin{lem}\label{CoreTensorLem}
The auxiliary function for (28) is as follows:
\begin{equation}\label{AuxAN}
\begin{split}
 & \mathcal{F}([\mathrm{vec}(\mathcal{S})]_i,[\mathrm{vec}(\mathcal{S}^\star)]_i) = f(\left[ \mathrm{vec}(\mathcal{S}^\star) \right]_i) \\
 & + \left( [\mathbf{F}^T\mathbf{F}\mathrm{vec}(\mathcal{S})]_i - [\mathbf{F}^T\mathrm{vec}(\mathcal{X})]_i \right) \cdot [\mathrm{vec}(\mathcal{W})]_i \cdot \\
 & \left( [\mathrm{vec}(\mathcal{S})]_i - [\mathrm{vec}(\mathcal{S}^\star)]_i \right) \\
 & + \frac{1}{2} \frac{[\mathbf{F}^T\mathbf{F}\mathrm{vec}(\mathcal{S})]_i \cdot [\mathrm{vec}(\mathcal{W})]_i}{[\mathrm{vec}(\mathcal{S})]_i}   \cdot \left( [\mathrm{vec}(\mathcal{S})]_i - [\mathrm{vec}(\mathcal{S}^\star)]_i \right)^2.
\end{split}
\end{equation}
\end{lem}

\begin{proof}
First, we vectorize the objective function $\mathcal{J}(\mathcal{X},\mathcal{S},\{\mathbf{A}^{(n)}\}_{n=1}^N)$ with respect to the core tensor $\mathcal{S}$:
\begin{equation}\label{AuxCoretensor}
\begin{split}
 & \mathrm{vec}(\mathcal{J}(\mathcal{X},\mathcal{S},\{\mathbf{A}\}_{n=1}^N)) = 1- \\
 & \frac{1}{I_1\dots I_N} \sum_{i=1}^{I_1\dots I_N} \exp\left( - \frac{\left( [\mathrm{vec}(\mathcal{X})]_i - [\mathbf{F}\mathrm{vec}(\mathcal{S})]_i \right)^2}{2\sigma^2} \right).
\end{split}
\end{equation}

And then, we obtain the first-order partial derivative of $\mathrm{vec}(\mathcal{J}(\mathcal{X},\mathcal{S},\{\mathbf{A}\}_{n=1}^N))$:
\begin{equation}\label{AuxCoretensor1}
\begin{split}
 & \frac{\partial \mathrm{vec}(\mathcal{J}(\mathcal{X},\mathcal{S},\{\mathbf{A}\}_{n=1}^N))}{\partial [\mathrm{vec}(\mathcal{S})]_i} =  \\
 & \left( [\mathbf{F}^T\mathbf{F}\mathrm{vec}(\mathcal{S})]_i - [\mathbf{F}^T\mathrm{vec}(\mathcal{X})]_i \right) \cdot [\mathrm{vec}(\mathcal{W})]_i.
\end{split}
\end{equation}

The second-order partial derivative of $\mathrm{vec}(\mathcal{J}(\mathcal{X},\mathcal{S},\{\mathbf{A}\}_{n=1}^N))$:
\begin{equation}\label{AuxCoretensor2}
\begin{split}
 & \frac{\partial^2 \mathrm{vec}(\mathcal{J}(\mathcal{X},\mathcal{S},\{\mathbf{A}\}_{n=1}^N))}{\partial [\mathrm{vec}(\mathcal{S})]_i \partial [\mathrm{vec}(\mathcal{S})]_i} = [\mathbf{F}^T\mathbf{F}]_{ii} \cdot [\mathrm{vec}(\mathcal{W})]_i.
\end{split}
\end{equation}

According to the Taylor expansion, we rewrite $\mathcal{J}$ with respect to $\mathcal{S}$:
\begin{equation}\label{AuxCoretensor3}
\begin{split}
 & f([\mathrm{vec}(\mathcal{S})]_i) = f_\mathcal{S}(\left[ \mathrm{vec}(\mathcal{S}) \right]_i) + \left( [\mathrm{vec}(\mathcal{S})]_i - [\mathrm{vec}(\mathcal{S}^\star)]_i \right) \cdot \\
 & \left( [\mathbf{F}^T\mathbf{F}\mathrm{vec}(\mathcal{S})]_i - [\mathbf{F}^T\mathrm{vec}(\mathcal{X})]_i \right) \cdot [\mathrm{vec}(\mathcal{W})]_i \\
 & + \frac{1}{2}\left( [\mathrm{vec}(\mathcal{S})]_i - [\mathrm{vec}(\mathcal{S}^\star)]_i \right)^2 \cdot [\mathbf{F}^T\mathbf{F}]_{ii} \cdot [\mathrm{vec}(\mathcal{W})]_i.
\end{split}
\end{equation}

The upper bound auxiliary function for (28) can be represented as:
\begin{equation}\label{AuxCoretensor4}
\begin{split}
 & \mathcal{F}([\mathrm{vec}(\mathcal{S})]_i,[\mathrm{vec}(\mathcal{S}^\star)]_i) = f(\left[ \mathrm{vec}(\mathcal{S}^\star) \right]_i) \\
 & + \left( [\mathbf{F}^T\mathbf{F}\mathrm{vec}(\mathcal{S})]_i - [\mathbf{F}^T\mathrm{vec}(\mathcal{X})]_i \right) \cdot [\mathrm{vec}(\mathcal{W})]_i \cdot \\
 & \left( [\mathrm{vec}(\mathcal{S})]_i - [\mathrm{vec}(\mathcal{S}^\star)]_i \right) \\
 & + \frac{1}{2} \frac{[\mathbf{F}^T\mathbf{F}\mathrm{vec}(\mathcal{S})]_i \cdot [\mathrm{vec}(\mathcal{W})]_i}{[\mathrm{vec}(\mathcal{S})]_i}   \cdot \left( [\mathrm{vec}(\mathcal{S})]_i - [\mathrm{vec}(\mathcal{S}^\star)]_i \right)^2.
\end{split}
\end{equation}

Because we have:
\begin{equation}\label{AuxCoretensor5}
\begin{split}
 & \frac{[\mathbf{F}^T\mathbf{F}\mathrm{vec}(\mathcal{S})]_i \cdot [\mathrm{vec}(\mathcal{W})]_i}{[\mathrm{vec}(\mathcal{S})]_i} = \frac{\sum_{j} [\mathbf{F}^T\mathbf{F}]_{ij} [\mathrm{vec}\mathcal{S}]_{j}\cdot [\mathrm{vec}(\mathcal{W})]_i}{[\mathrm{vec}(\mathcal{S})]_i} \\
 & \geq [\mathbf{F}^T\mathbf{F}]_{ii} \cdot [\mathrm{vec}(\mathcal{W})]_i.
\end{split}
\end{equation}

Hence, $\mathcal{F}([\mathrm{vec}(\mathcal{S})]_i,[\mathrm{vec}(\mathcal{S}^\star)]_i) \geq f([\mathrm{vec}(\mathcal{S})]_i)$.

Now, we demonstrate that \eqref{AuxCoretensor4} holds, and \eqref{AuxCoretensor3} is the upper bound auxiliary function for $f_S(\mathcal{S})$. Because the elements of $\mathcal{S}$ are nonnegative, and \eqref{AuxCoretensor3} is convex, its minimum value can be achieved at
\begin{equation}\label{AuxCoretensor6}
\begin{split}
\!  & [\mathrm{vec}(\mathcal{S}^\star)]_i = [\mathrm{vec}(\mathcal{S})]_i - \\
  & \frac{\left( [\mathbf{F}^T\mathbf{F}\mathrm{vec}(\mathcal{S})]_i - [\mathbf{F}^T\mathrm{vec}(\mathcal{X})]_i \right) \cdot [\mathrm{vec}(\mathcal{W})]_i}{\frac{[\mathbf{F}^T\mathbf{F}\mathrm{vec}(\mathcal{S})]_i \cdot [\mathrm{vec}(\mathcal{W})]_i}{[\mathrm{vec}(\mathcal{S})]_i}} \\
  & = [\mathrm{vec}(\mathcal{S})]_i \times \frac{[\mathbf{F}^T\mathrm{vec}(\mathcal{X})]_i \cdot [\mathrm{vec}(\mathcal{W})]_i}{[\mathbf{F}^T\mathbf{F}\mathrm{vec}(\mathcal{S})]_i \cdot [\mathrm{vec}(\mathcal{W})]_i}.
\end{split}
\end{equation}
Lemma \ref{CoreTensorLem} is proved.
\end{proof}

\subsubsection{Robustness analysis of RMNTF}
\label{Robustness}

One important property of RMNTF is that its reconstruction is robust against outliers, as shown in experiments.
Here, we show that the robustness of RMNTF is mainly due to the weighted tensor and the regularization.
The previous work gLPCA \cite{jiang2013graph} and GLTD \cite{jiang2018image} have demonstrate the robustness of Laplacian regularization.
In this section, we further demonstrate that this property also occurs on RMNTF.

Suppose we have learned the optimal parameters $\mathcal{S}$, $\{\mathbf{A}^{(n)}\}_{n=1}^N$ and $\mathcal{W}$ from input training images $({\mathcal{X}}^1, {\mathcal{X}}^2, \dots, {\mathcal{X}}^{t})$. Now we have a text image $\mathcal{X}^{t\prime}$ and we aim to learn its low-rank representation $[\mathbf{A}^{(N)}]_{t\cdot}$ and reconstruction $\hat{\mathcal{X}}^{t\prime}$, while the parameters $\mathcal{S}_0$, $\{\mathbf{A}_0^{(n)}\}_{n=1}^{N}$ and $\mathcal{W}_0$ learned by training images are fixed.
This can be transformed as solving the problem:
\begin{equation}\label{RobustRMNTF}
\begin{split}
  \min_{[\mathbf{A}^{(N)}]_{t\cdot}} \mathcal{J}_0 & = \sum_{i_1}^{I_1}\dots\sum_{i_{n-1}=1}^{I_{N-1}} \mathcal{W}_0 \circledast \left( \mathcal{X}^{t\prime} - \hat{\mathcal{X}}^{t\prime} \right)^2 +\lambda \Phi([\mathbf{A}^{N}]_{t\cdot}) \\
  \hat{\mathcal{X}}^{t\prime} & = \mathcal{S}_0 \times_{n=1}^{N-1} \mathbf{A}_0^{(n)} \times [\mathbf{A}^{N}]_{t\cdot} \\
  \Phi([\mathbf{A}^{N}]_{t\cdot}) & = \sum_{j=1}^t\| [\mathbf{A}^{(N)}]_{t\cdot} - [\mathbf{A}_0^{(N)}]_{j\cdot} \|^2 v_j,
\end{split}
\end{equation}
where $[\mathbf{A}_0^{(N)}]_{j\cdot}$ is the $j$th row of $\mathbf{A}_0^{(N)}$.
Let $[\mathbf{A}^{(N)}]_{t\cdot} = ({a}_1^{\prime},\dots,{a}_n^\prime) = \mathbf{a}_t^\prime \in \mathbb{R}_{\geq0}^{1\times r_N}$, using the mode-n unfolding of tensor, problem \eqref{RobustRMNTF} can be reformulated as:
\begin{equation}\label{RobustRMNTF1}
\begin{split}
  \min_{[\mathbf{A}^{(N)}]_{t\cdot}} \mathcal{J}_0 = & \left\| \sqrt{{\mathcal{W}_0}_{(N)}} \circledast \left( {\mathcal{X}^{t\prime}}_{(N)} -  \mathbf{a}_t^{\prime}\mathbf{B}^{(N)T} \right) \right\|^2 \\
  & + \lambda\sum_{j=1}^t \left\| \mathbf{a}_t^{\prime} - [\mathbf{A}_0^{(N)}]_{j\cdot} \right\|^2 v_j,
\end{split}
\end{equation}
where $\mathbf{B}^{(N)T} = \mathcal{S}_{(N)}\left( \otimes_{i=1}^{N-1}\mathbf{A}^{(i)T} \right) \in \mathbb{R}_{\geq0}^{r_N\times I_1\dots I_{N-1}}$.
Let $\mathbf{B}^\prime = \sqrt{{\mathcal{W}_0}_{(N)}} \circledast \mathbf{B}^{(N)T}$ and $\tilde{\mathcal{X}}^t_{(N)} = \sqrt{{\mathcal{W}_0}_{(N)}} \circledast {\mathcal{X}^{t\prime}}_{(N)}$.
Then, we obtain first-order partial derivative as follows:
\begin{equation}\label{RobustRMNTF2}
\begin{split}
  \frac{\partial \mathcal{J}_0}{\partial \mathbf{a}_t^\prime} = & 2 \left( \mathbf{B}^\prime\mathbf{B}^{\prime T}\mathbf{a}_t^{\prime T} - \mathbf{B}^{\prime}{\mathcal{X}^{t\prime T}}_{(N)} \right) \\
  & + 2\lambda\left( \sum_{j=1}^tv_j \mathbf{a}^\prime_t - \sum_{j=1}^t[\mathbf{A}_0^{(N)}]_{j\cdot}v_j  \right) = 0.
\end{split}
\end{equation}
We obtain the optimal $\mathbf{a}^\prime_t$ as
\begin{equation}\label{RobustRMNTF3}
\begin{split}
  \! \mathbf{a}^\prime_t \! = \! \left[ \! \mathbf{B}^\prime\mathbf{B}^{\prime T} \! + \! \left( \! \sum_{j=1}^tv_j \! \right) \! \mathbf{I} \right]^{-1} \! \left( \! {\mathcal{X}^{t\prime}}_{(N)}\mathbf{B}^{\prime T} \! + \! \lambda \! \sum_{j=1}^t[\mathbf{A}_0^{(N)}]_{j\cdot}v_j \! \right) \!,
\end{split}
\end{equation}
where $\mathbf{I} \in \mathbb{R}^{r_N\times r_N}$ is an identity matrix.

We note that the problem \eqref{RobustRMNTF} has a closed form solution.
In standard NTF model, $\lambda = 0$, and the second term of equation \eqref{RobustRMNTF} should be removed. Hence, the solution of $\mathbf{a}^\prime_t$ may be singular value. In another word, if some corruptions happened to some elements of test image $\mathcal{X}^{t\prime}$, the regularization $\Phi(\cdot)$ and the weighted tensor $\mathcal{W}$ will restore these corruptions properly.

\subsubsection{Invariance of RMNTF}
\label{Uniqueness}

RMNTF has the following invariance property where the reconstruction $\hat{\mathcal{X}} = \mathcal{S}\times_{n=1}^N \mathbf{A}^{(n)}$ is invariant under the transformation by permutation matrices and nonnegative diagonal matrices.
Let $\mathbf{P}^{(n)}$ and $\mathbf{Q}^{(n)}$ be a permutation matrix and a nonnegative diagonal matrix, respectively. Factor matrices $\{\mathbf{A}^{(n)}\}_{n=1}^N$ and core tensor $\mathcal{S}$ are transformed as:
\begin{eqnarray}
\left\{ \begin{array}{ll}
\hat{\mathbf{A}}^{(1)} = \mathbf{A}^{(1)}\mathbf{P}^{(1)}\mathbf{Q}^{(1)} \\
\hat{\mathbf{A}}^{(2)} = \mathbf{A}^{(2)}\mathbf{P}^{(2)}\mathbf{Q}^{(2)} \\
\dots \\
\hat{\mathbf{A}}^{(N)} = \mathbf{A}^{(N)}\mathbf{P}^{(N)}\mathbf{Q}^{(N)} \\
\hat{\mathcal{S}} = (\mathbf{P}^{(1)}\mathbf{Q}^{(1)})^T \times_1 \dots \times_{N-1} (\mathbf{P}^{(N)}\mathbf{Q}^{(N)})^T \times_N \mathcal{S}.
\end{array}\right.\label{RMNTF_Invariance}
\end{eqnarray}

\begin{myDef}{(Uniqueness of NTF)}
The NTF model $\mathcal{X} = \mathcal{S}\times_{n=1}^N \mathbf{A}^{(n)}$ is essentially unique, if $\mathbf{A}^{(n)} = \hat{\mathbf{A}}^{(n)}\mathbf{P}^{(n)}\mathbf{D}^{(n)}, \forall n,$ holds for any other NTF model $\hat{\mathcal{X}} = \hat{\mathcal{S}}\times_{n=1}^N \hat{\mathbf{A}}^{(n)}$, where $\mathbf{P}^{(n)}$ and $\mathbf{D}^{(n)}$ are permutation matrix and nonnegative matrix, respectively.
\end{myDef}

To demonstrate the uniqueness of RMNTF model, we need to study the difference of uniqueness between the mode-n unfolding of RMNTF model and RMNTF model.
\begin{myLemma}\label{UniqueRMNTFLem1}
If the RMNTF model $\mathcal{X} = \mathcal{S}\times_{n=1}^{N} \mathbf{A}^{(n)}$ is uniqueness, $\mathcal{X}_{(p)} = \mathbf{A}^{(p)}\mathbf{B}^{(p)T}$ is the unique nonnegative matrix decomposition of matrix $\mathcal{X}_{(n)}, \forall n$.
\end{myLemma}

\begin{proof}
Suppose there exists a non-trivial matrix $\mathbf{C}$ such that $\hat{\mathcal{X}}_{n} = (\mathbf{A}^{(n)}\mathbf{C})(\mathbf{C}^{-1}\mathbf{B}^{(n)T})$ is another solution of mode-n unfolding of RMNTF. Let $\hat{\mathbf{A}}^{(n)} = \mathbf{A}^{(n)}\mathbf{C}$, and $\hat{\mathbf{B}}^{(n)T} = \mathbf{C}^{(-1)} \mathbf{B}^{(n)T}$, then, $\mathrm{fold}_n(\hat{\mathbf{A}}^{(n)}\hat{\mathbf{B}}^{(n)T}) = \mathcal{X}_{(n)}$. However, when $p \neq n$, it satisfied that $\hat{\mathcal{X}}_{(p)} \neq \mathcal{X}_{(p)}$. This contradicts the assumption that the RMNTF is essentially unique.
\end{proof}

\begin{myLemma}\label{UniqueRMNTFLem2}
The mode-N unfolding of RMNTF model $\mathcal{X}_{(N)} = \mathbf{A}^{(N)}\mathbf{B}^{(N)T}$ has an essentially unique solution of nonnegative matrix decomposition, then the RMNTF model of $\mathcal{X}$ is essentially unique.
\end{myLemma}

\begin{proof}
Suppose that $\mathcal{X}_{(N)}$ has an essentially unique solution of NMF and $\mathcal{X}_{(N)} = \mathbf{A}^{(N)} \mathcal{S}_{(N)}{\mathbf{Z}}^{(N)T}$, where ${\mathbf{Z}}^{(N)} = \otimes_{n\neq N}\mathbf{A}^{(n)}$, both $\mathcal{S}_{(N)}$ and ${\mathbf{Z}}^{(N)}$ can be uniquely estimated.
we introduce the permutation matrix $\mathbf{P} = \otimes_{n\neq N}\mathbf{P}^{(N)}$ and nonnegative diagonal matrix $\mathbf{Q} = \otimes_{n\neq N}\mathbf{Q}^{(n)}$, suppose that
\begin{equation}\label{UniqueRMNTF}
\begin{split}
  \hat{\mathbf{Z}}^{(N)} & = \mathbf{Z}^{(N)}\mathbf{P}\mathbf{Q} = \left( \otimes_{n\neq N} \mathbf{A}^{(n)} \right)\mathbf{P}\mathbf{Q} \\
  & = \otimes_{n\neq N}\mathbf{A}^{(n)}\mathbf{P}^{(n)}\mathbf{Q}^{(n)} = \otimes_{n \neq N} \hat{\mathbf{A}}^{(n)},
\end{split}
\end{equation}
where $\hat{\mathbf{A}}^{(n)} = \mathbf{A}^{(n)}\mathbf{P}^{(n)}\mathbf{Q}^{(n)}$.
Here, we only need to demonstrate that $\hat{\mathbf{A}}^{(n)}$ can be uniquely estimated from $\hat{\mathbf{Z}}^{(N)}$.

Motivated by \cite{van1993approximation}, we appropriately rearrange the elements of $\hat{\mathbf{Z}}^{(N)}$ and reshape it to a tensor $\hat{\mathcal{Z}}$ such that
\begin{equation}\label{UniqueRMNTF1}
\begin{split}
  \hat{\mathcal{Z}} = \mathrm{vec}(\hat{\mathbf{A}}^{(1)}) \circ \dots \circ \mathrm{vec}(\hat{\mathbf{A}}^{(N-1)}),
\end{split}
\end{equation}
where the tensor $\hat{\mathcal{Z}}$ is a rank-one tensor \cite{huang2013non} and $\hat{\mathbf{A}}^{(n)}$ can be uniquely estimated from $\hat{\mathcal{Z}}$.
The lemma \ref{UniqueRMNTFLem2} has been proved.

\end{proof}







\end{document}